\documentclass[Afour,sageh,times]{sagej}

\usepackage{multicol}
\usepackage{amsmath,bm}
\usepackage{amsmath,amssymb,amsfonts,amsthm}
\usepackage{xfrac} 
\usepackage{mathtools} 
\usepackage[font=footnotesize]{caption}
\usepackage{url}
\usepackage{array}

\usepackage{multirow}
\usepackage{cancel}
\usepackage{subcaption} 
\usepackage{multicol}
\usepackage{graphicx}
\usepackage{epstopdf}
\usepackage{bm}
\usepackage{float}
\usepackage{tikz}
\usetikzlibrary{positioning}
\usetikzlibrary{shapes,snakes}
\usepackage{algorithm}
\usepackage{algpseudocode}
\usepackage{lipsum} 
\usepackage{empheq} 
\usepackage{xcolor} 
\usepackage{fancyhdr} 
\usepackage{tabulary} 
\usepackage{color} 
\usepackage{booktabs} 
\usepackage{adjustbox}
\usepackage{todonotes}
\usepackage{tablefootnote}
\DeclareGraphicsExtensions{.eps,.pdf,.png,.jpg,.JPG,.gif}
\newcommand\scalemath[2]{\scalebox{#1}{\mbox{\ensuremath{\displaystyle #2}}}}

\newtheorem{thm}{Theorem}

\newtheorem{lem}[thm]{Lemma}
\newtheorem{proposition}{Proposition}

\usepackage{moreverb,url}
\usepackage[colorlinks,bookmarksopen,bookmarksnumbered,citecolor=red,urlcolor=red,draft]{hyperref}
\def\volumeyear{2021}

\begin{document}

\runninghead{Yulin Yang et al.}

\title{
Online Self-Calibration for \\
Visual-Inertial Navigation Systems: \\
Models, Analysis and Degeneracy
}

\author{Yulin Yang\affilnum{1}, Patrick Geneva\affilnum{1}, Xingxing Zuo\affilnum{2} and Guoquan Huang\affilnum{1}}

\affiliation{\affilnum{1} Robot Perception and Navigation Group,
University of Delaware, Newark, DE 19716, USA.
\affilnum{2} Department of Informatics, Technical University of Munich, Garching b. Munich, 85748, Germany.
}

\corrauth{Yulin Yang, Department of Mechanical Engineering, University of
Delaware, 126 Spencer Lab, Newark, DE 19716, USA.
}
\email{yuyang@udel.edu}

\begin{abstract}
In this paper, we study in-depth the problem of online self-calibration for robust and accurate visual-inertial state estimation.
In particular, we first perform a complete observability analysis for visual-inertial navigation systems (VINS) with full calibration of sensing  parameters, 
including IMU and camera intrinsics and IMU-camera spatial-temporal extrinsic calibration, along with readout time of rolling shutter (RS) cameras (if used). 
We investigate different inertial model variants containing IMU intrinsic parameters that encompass most commonly used models for low-cost inertial sensors.
With these models, the state transition matrix and visual measurement Jacobians are analytically derived 
and the observability analysis of linearized VINS with full sensor calibration is performed.
The analysis results prove that, as intuitively assumed in the literature, VINS with full sensor calibration  has four unobservable directions, corresponding to the system's global yaw and translation, while all sensor calibration parameters are observable given fully-excited 6-axis motion. 
Moreover, we, for the first time, identify primitive degenerate motions for IMU and camera intrinsic calibration, which, when combined, may produce complex degenerate motions.
This result holds true for the different inertial model variants investigated in this work and has significant impacts on  practical applications of online self-calibration to many robotic platforms.
Extensive Monte-Carlo simulations and real-world experiments are performed to validate both the observability analysis and identified degenerate motions,
showing that online self-calibration improves system accuracy and robustness to calibration inaccuracies.
We compare the proposed {\em online} self-calibration on commonly-used IMUs against the state-of-art {\em offline} calibration toolbox Kalibr, 
and show that the proposed system achieves better consistency and repeatability. 
As sensor calibration plays an important role in making real VINS work well in practice, 
based on our analysis and experimental evaluations, we also provide practical guidelines for how to perform online IMU-camera sensor self-calibration.
\end{abstract} 

\maketitle
\allowdisplaybreaks

\section{Introduction}

Due to the decreasing cost of integrated inertial/visual sensor rigs, 
visual-inertial navigation system (VINS), which fuses high-rate inertial readings from an IMU and images of the surrounding environment from a camera, has gained great popularity in 6 degree-of-freedom (DoF) motion tracking for mobile devices and autonomous robots -- such as micro aerial vehicles (MAV)~\citep{Shen2014ICRA}, self-driving cars~\citep{Heng2019ICRA}, unmanned ground vehicles (UGV)~\citep{Zhang2019arXiv,Lee2020IROS} and smart phones~\citep{Li2013IJRR,Guo2014RSS} -- during the past decades \citep{Huang2019ICRA}. 
Many efficient and robust VINS algorithms based on filtering \citep{Mourikis2007ICRA,Li2013IJRR,Wu2015RSS,Geneva2020ICRA,Eckenhoff2021TRO} or batch least squares solver \citep{Leutenegger2014IJRR,Forster2017tro,Qin2018TRO,Usenko2020RAL,Campos2021TRO} techniques have been developed in recent years to address this pose estimation problem.  

There are many factors which attribute to VINS performance, such as visual feature tracking/triangulation, velocity/biases initialization and sensor calibration. 
Among them, robust and accurate sensor calibration -- including the rigid transformation between sensors (spatial calibration), time offset between IMU-camera (temporal calibration), image line readout time for rolling shutter (RS) cameras, and IMU/camera intrinsics -- is crucial, especially when plug-and-play visual-inertial sensor rigs with widely available off-the-shelf low-cost IMUs and rolling shutter cameras are deployed.
In addition, sensor calibration itself can vary due to extended usage, sensor failure replacement and environmental effects such as varying temperature, humidity, vibrations, non-rigid mounting, and among others. 
For example, IMU biases and intrinsics suffer from the temperature and humidity changes \citep{Li2014ICRA}, 
and rigid transformation between IMU and camera can vary if the sensor is replaced or subjected to vibration. 
As such, online sensor self-calibration in VINS has attracted significant attentions and research efforts \citep{Li2013IJRR,Guo2014RSS,Yang2019RAL,Eckenhoff2021TRO,Yang2020RSS} in recent years, due to its potential to handle poor prior calibration or calibration changes, which can degrade the state estimate accuracy in the case that these calibrations are treated to be true. 

System observability analysis for VINS with online IMU-camera \citep{Mirzaei2008TRO,Kelly2011IJRR,Guo2013ICRA} or IMU/camera intrinsic \citep{Tsao2019Sensor,Yang2020RSS} calibration has also been carried out to show that these calibration parameters can be identified given  fully excited motions. 
However, complete analysis for VINS with full calibration parameters -- including IMU/camera intrinsics, IMU-camera rigid transformation, temporal time offset, and camera RS readout time -- is still absent from the existing literature.

Blindly performing online calibration is risky, as in most cases domain knowledge on specific motions and prior distribution choices are needed to ensure calibration can converge consistently \citep{Schneider2019SENSORS}.
In the meantime, existing research efforts \citep{Li2014IJRRa,Yang2019RAL,Yang2020RSS} have also identified several basic motion profiles, termed degenerate motions, that cause online sensor self-calibration failures. 
In this work, we investigate degenerate motions which impact the deployment of VINS on mobile robots, which typically have constrained motions, when jointly estimating IMU/camera intrinsics, IMU-camera spatial-temporal calibration, and RS readout time.
For example, aerial and ground vehicles can only perform a few motion profiles due to their under-actuation, and can easily ``fall'' into degenerate conditions for calibration during typical deployment.
As compared to our investigation into degenerate motions when performing full-parameter self-calibration, and their impacts especially for under-actuated autonomous robots, most approaches on VINS sensor self-calibration are limited to either handheld or trajectory segments involving rich motion information \citep{Li2014ICRA,Schneider2019Sensor}.

In this paper, we build an accurate and robust monocular VINS estimator with full self-calibration. 
We also investigate in-depth the observability analysis for visual-inertial self-calibration and perform degenerate motion analysis for {\em all} calibration parameters.
In particular, the main contributions of this work include:
\begin{itemize}
	\item An efficient filter-based visual-inertial estimator capable of performing self-calibration for all spatial-temporal extrinsic and intrinsic parameters.
	\item We perform a complete observability and degeneracy analysis for the proposed visual-inertial models and, for the first time, identify the degenerate motions that cause IMU and camera intrinsic parameters to become unobservable.
	\item Extensive simulations and real-world experiments are performed to verify the parameter convergence of the estimator with online self-calibration under fully-excited 6DoF motion and a series of identified degenerate motions of practical significance.
    \item We show that under general motion, self-calibration is necessary to achieve accurate pose estimation in a robust manner for consumer grade sensors, which continue to become more prevalent, with only minimal computational impact. 
    Additionally, we show that degenerate motions can and do have a significant negative impact on the performance of the estimator, leading to a series of guideline recommendations.
\end{itemize}

The rest of the paper is organized as follows:
After reviewing the related work in Section \ref{sec:related} and estimation preliminaries in Section \ref{sec:ekf}, 
we present the sensing models including inertial and camera models in Section \ref{sec:model}. 
In Section \ref{sec:statetrans}, we derive the lienarized system dynamics and measurement model of the VINS with full sensor calibration.
Based on that, we perform the observability analysis and  degenerate motion identification in  Sections \ref{sec:obsanalysis} and \ref{sec:degen_identif},
while the proposed estimator is presented in Section~\ref{sec:estimator}.
In Sections \ref{sec:exp_sim}, \ref{sec:exp_tum_rs}, \ref{sec:exp_virig} and \ref{sec:exp_degenerate}, we extensively validate our analysis and estimator through both simulations and real-world experiments.
Finally, we offer discussions and final remarks in Sections \ref{sec:discussion} and \ref{sec:conclusion}.

\section{Related Works} \label{sec:related}

Extensive works have studied online or offline IMU and camera calibration for VINS. 
However, the joint self-calibration of all the calibration parameters for visual and inertial sensors (including IMU/camera intrinsics, IMU-camera spatial-temporal calibration and RS readout time) was not investigated sufficiently,
and the observability analysis and degenerate motion identification for the complete VINS with calibration are still missing from the literature. 
In terms of the complete parameter calibration for VINS, the related works can be divided into the following four categories: 

\subsection{Camera Calibration}

Visual-only offline calibration of  camera intrinsic parameters is well studied~\citep{Hartley2004}. 
For example, 
\cite{Faugeras1992ECCV} demonstrated camera self-calibration without a pattern,
based on which \cite{Qian2004CVIU} and \cite{Civera2009ICRA} proposed to use sum of Gaussian filters. 
Recently, \cite{Agudo2020ICPR} extended the above works to use RGB video with objects of non-rigid shape for camera self-calibration.

Many works have also focused on RS camera calibration, especially for the image line readout time. 
For instance, 
after formulating the geometric models of RS cameras and investigating how RS affects the image generation, 
\cite{Meingast2005ArXiv}  leveraged flashing LED lights to calibrate the RS readout time.
Similarly, \cite{Oth2013CVPR} used continuous time trajectory representation to model RS effects and performed calibration with an April tag pattern~\citep{Olson2011ICRA}. 
However, these methods assume known camera intrinsics (including distortion parameters) when calibrating the RS readout time.

\cite{Nguyen2017CVIU} investigated the self-calibration of multiple omnidirectional RS cameras. 
They first initialized the rigid transformation and time offset between cameras based on the structure-from-motion (SFM) trajectories generated from each monocular camera with global shutter (GS) assumption. 
Then, all the camera related parameters (i.e., the rigid transformation, time offset, camera intrinsics, and RS readout time for each camera) are refined by a bundle adjustment (BA) with all the camera measurements. 
\cite{Kukelova2020ECCV} presented the first minimal solution to absolute pose estimation of a RS camera with unknown focal length and
unknown radial distortion, from seven point correspondences. 
With the proposed solvers they can achieve accurate solutions for camera poses, RS readout time, focal length and radial distortion. 
Note that both \cite{Kukelova2020ECCV} and \cite{Nguyen2017CVIU} are offline calibration algorithms and rely only on visual sensors.  
\cite{Karpenko2011CSTR} proposed to combine a gyroscope for RS camera image correction with natural scenes. 
However, it requires the pre-calibration of rotation between gyroscope and RS camera and the motion of the camera is restricted to rotation only. 
\cite{Tsao2019Sensor} investigated online camera intrinsic calibration (only focal length and principal points) within a VINS framework.  
In contrast to the above works only focusing on camera calibration, 
our proposed method fuses the measurements of an IMU and a monocular camera, and provides online calibration of all the camera parameters including camera intrinsics and RS readout time,
Both radial-tangential (\textit{radtan}) and equivalent-distant (\textit{equidist}) distortion~\citep{Furgale2013IROS} are supported. 
\subsection{IMU Intrinsic Calibration}

Generally, the gyroscope and acceleration biases are needed for accurate inertial modeling. 
They are both modeled as random walks and estimated as part of IMU states.  
It is a common practice to estimate biases online in VINS such as \cite{Jones2011IJRR}, \cite{Kelly2011IJRR}, \cite{Leutenegger2014IJRR} and \cite{Mourikis2007ICRA}. 

Besides these biases, the IMU intrinsic parameters -- including the scale correction and axis misalignment for gyroscope and accelerometer, the rotation from gyroscope or accelerometer frame to IMU frame, and the gravity sensitivity -- 
also need to be calibrated offline or online, especially for low-cost inertial sensors.  
\cite{Bristeau2011IFAC} calibrated the IMU's scale correction and axis misalignment for aerial vehicles  by solving a least squares problem with known special sensor motions.  
\cite{Xiao2019Sensors} improved the IMU pre-integration \citep{Forster2017tro} to incorporate the IMU intrinsic parameters in a keyframe based VINS algorithm for online self-calibration. 
\cite{Jung2020TIM} studied IMU intrinsic calibration within multi-state constrained Kalman filter (MSCKF) by using a stereo camera and an IMU sensor,
where they also examined the inertial calibration results under planar and random motions. 

Building upon our prior  work \citep{Yang2020RSS}, 
in which we have investigated online IMU intrinsic calibration with the minimal sensor configuration of a single IMU and a monocular camera and compared the performance of four different IMU intrinsic model variants in VINS,
in this work, 
we study 18 different IMU intrinsic model variants which can encompass or be equivalent to most published IMU models for inertial navigation
and perform online self-calibration. 
Comprehensive degenerate motion analysis, which can cause online self-calibration to fail, is also provided.

\subsection{Joint IMU-Camera Self Calibration}

Since VINS fuses IMU measurements and camera images, the joint calibration of IMU-camera parameters is preferred to improve system accuracy and robustness. 
Extensive works have studied joint sensor calibration in VINS.
For instance, 
\cite{Mirzaei2008TRO} proposed to use an extended Kalman filter (EKF) for the spatial calibration (i.e. the rigid transformation between the camera and IMU) of VINS and performed an observability analysis.
They showed that the rigid transformation is not fully observable under one-axis rotation. 
\cite{Zacharian2010IPIN} proposed to use the recursive Sigma-Point Kalman filter to estimate IMU intrinsics and IMU-camera spatial parameters with measurements from an IMU and a monocular camera. 
However, \cite{Mirzaei2008TRO} and \cite{Zacharian2010IPIN} did not calibrate the camera intrinsics or IMU-camera time offset and both relied on calibration chessboards. 

\cite{Furgale2013IROS} developed the well-known calibration toolbox: Kalibr, a continuous-time spline-based batch estimator, for IMU-camera extrinsics, time offset and camera intrinsics calibration.
\cite{Rehder2016ICRA} extended Kalibr to incorporate and estimate IMU intrinsics (including scaling parameters, axis misalignments, and gravity sensitivity). 
\cite{Huai2021arXiv} further extended the above work to calibrate readout time for RS cameras.  
\cite{Nikolic2016Sensors} formulated a maximum likelihood estimation problem based on discrete IMU poses to calibrate the IMU-camera spatial-temporal and IMU intrinsic parameters. 
The above mentioned works are all offline methods and need calibration targets. In addition, they do not support full-parameter joint optimization of camera intrinsics with other calibration parameters.

\cite{Schneider2019Sensor} reduced IMU-camera calibration optimization complexity by selecting the most informative trajectory segments for calibration. 
The selection is based on the information matrix of the measurements from the trajectory segments. 
Although this work does not need a calibration board, the temporal calibration between IMU and camera is not included.

Many recent VINS algorithms perform online IMU-camera joint calibration. 
\cite{Qin2018TRO} and \cite{Qin2018IROS}
is able to perform online IMU-camera extrinsic and time offset calibration with natural scene to improve the system robustness and accuracy. 
\cite{Guo2014RSS} proposed to use linear pose interpolation to model RS effects and calibrate readout time. 
\cite{Eckenhoff2019ICRAa} proposed a multi-camera aided VINS with online IMU-camera spatial-temporal and camera intrinsic calibration. 
\cite{Eckenhoff2021TRO} further proposed a generalized polynomial based pose interpolation for readout time calibration of RS cameras.  
However, the IMU intrinsics were not considered in the above systems. 
The closest work to ours is by \cite{Li2014ICRA} which included IMU-camera extrinsics, time offset, rolling-shutter readout time, camera and IMU intrinsics into the state vector within MSCKF \citep{Mourikis2007ICRA} based visual-inertial odometry and successfully calibrated all these parameters. 
They showed that these parameters can converge in simulation with fully excited motions and verified the system using a real-world experiment.
However, system observability and degenerate motion analysis are still missing, which are the focus of our work along with more extensive multi-run statistical validations. 
In addition, we also compared different IMU model variants which have appeared in literature to investigate their estimation performances within VINS.

\subsection{Observability and Degeneracy}

Observability analysis plays an important role in state estimation \citep{Huang2010IJRR,Martinelli2012TRO,Hesch2013TRO}, especially when the system incorporates calibration parameters \citep{Martinelli2011TRO,Li2014IJRRa,Yang2019RAL,Yang2020RSS}.
We wish to identify whether these calibration parameters can be calibrated with visual and inertial measurements, and also identify degenerate motions, which might cause calibration to fail. 
In addition, observability properties can be leveraged for consistent estimator design \citep{Huang2012thesis,Hesch2013TRO,Wu2017IROS}. 
\cite{Mirzaei2008TRO}  performed the observability analysis for VINS with IMU-camera spatial calibration and showed that the spatial calibration is observable given fully excited motion of the IMU. 
They also found that that one-axis rotation is degenerate for the calibration of IMU-camera translation.
\cite{Kelly2011IJRR} studied the IMU-camera self-calibration and performed nonlinear observability analysis using Lie derivative to show that the rigid transformation between IMU-camera is observable given random motions. 
\cite{Guo2013ICRA} simplified the proof and analytically showed that the spatial calibration between the IMU and RGBD camera is observable. 
\cite{Li2014IJRRa} analyzed the identifiability for IMU-camera temporal calibration given the measurements of a single IMU and a monocular camera
and identified a degenerate motion that can cause the IMU-camera time offset to become unobservable. 
\cite{Jung2020TIM} studied the observability of stereo VINS with IMU intrinsics also based on Lie derivative to build the observability matrix
and showed that the IMU intrinsics (including scale correction and axis misalignment for gyroscope and accelerometer, respectively) is observable given fully exited motions. 
\cite{Tsao2019Sensor} built the observability matrix for VINS using linearized system model and showed that the camera intrinsics (only including focal length and principal points in their work) is observable. 
However, none of the above mentioned works ever performed and verified the observability analysis with full-parameter calibration for VINS.

In our previous work \citep{Yang2019RAL}, we built the observability matrix for VINS using the linearized system with IMU-camera spatial-temporal calibration
and showed that given fully excited motions all these calibration parameters are observable. 
We have also, for first time, identified four degenerate motions that can cause these calibration to become unobservable. 
In our recent work \citep{Yang2020RSS}, we performed observability analysis for monocular VINS with IMU intrinsic calibration (including scale and axis-misalignment for gyroscope and accelerometer, the rotation from gyroscope or accelerometer to IMU frame),
and identified the degenerate motions for the IMU intrinsics.
Building upon these prior works, in this work, we perform full-parameter calibration -- including IMU intrinsics with gravity sensitivity, camera intrinsics and the IMU-camera spatial-temporal calibration with RS readout time -- for VINS with a single IMU and a monocular RS camera. 
Comprehensive observability analysis and degenerate motion identification are performed for these calibration parameters. Both simulations and real world experiments are also leveraged to verify our analysis.

\section{Estimation Preliminaries} \label{sec:ekf}

The state $\mathbf{x}$ is propagated forward from timestep $k-1$ to  $k$ using incoming system control inputs $\mathbf u_{k-1}$ based on the following generic nonlinear function:
\begin{align}
    \mathbf{x}_{k} = f(\mathbf{x}_{k-1}, \mathbf{u}_{k-1}, \mathbf{w}_{k-1})
    \label{eq:imu_dynamics}
\end{align}
where $\mathbf{w}_{k-1}\sim \mathcal{N}(\mathbf{0},\mathbf{Q}_{k-1})$ denotes the white Gaussian noise of the control input.
The state estimate at $k$ can be predicted from the state estimate at $k-1$ with the nonlinear system:
\begin{align}
    \hat{\mathbf{x}}_{k|k-1} = f(\hat{\mathbf{x}}_{k-1|k-1}, \mathbf{u}_{k-1}, \mathbf{0})
\end{align}
where the subscript $i|j$ denotes the estimate at time $i$ given the measurements up to time $j$;
$\hat{\mathbf{x}}$ denotes the estimated value for state $\mathbf{x}$ and $\tilde{\mathbf{x}} = \mathbf{x} \boxminus \hat{\mathbf{x}}$ represents the error states, where ``$\boxminus$'' can be defined within a manifold \citep{Barfoot2017Book}. 
The inverse operation of ``$\boxminus$'' can be defined as $\mathbf{x}=\hat{\mathbf{x}}\boxplus\tilde{\mathbf{x}}$, accordingly. 
The state covariance can be defined on the error states as $\tilde{\mathbf{x}}\sim \mathcal{N}(\mathbf{0},\mathbf{P})$. 

After linearizing the nonlinear function [see Eq.~\eqref{eq:imu_dynamics}] at current state estimate $\hat{\mathbf{x}}_{k-1|k-1}$, the propagated state covariance $\mathbf{P}_{k|k-1}$ for state estimate $\hat{\mathbf{x}}_{k|k-1}$ can be computed  as:
\begin{align}
    \tilde{\mathbf{x}}_{k|k-1} & \simeq \boldsymbol{\Phi}_{k-1|k-1} 
    \tilde{\mathbf{x}}_{k-1|k-1} + 
    \mathbf{G}_{k-1}\mathbf{w}_{k-1}
    \\
    \mathbf{P}_{k|k-1} & =
    \bm\Phi_{k-1} \mathbf{P}_{k-1|k-1}\bm\Phi_{k-1}^\top
    + 
    \mathbf{G}_{k-1}
    \mathbf{Q}_{k-1}
    \mathbf{G}^{\top}_{k-1}
    \label{eq:propcov}
\end{align}
where $\bm\Phi_{k-1}$ and $\mathbf{G}_{k-1}$ are the state transition matrix and noise Jacobians, respectively.
A nonlinear measurement function can be described as:
\begin{align}
    \mathbf{z}_{k} &= \mathbf{h}({\mathbf{x}}_{k}) + \mathbf{n}_{k}
    \label{eq:nonlinear-meas}
\end{align}
where $\mathbf{n}_{k} \sim\mathcal{N}(\mathbf{0},\mathbf{R}_{k})$ is white Gaussian noise.
For the EKF update, we need to first linearize the above equation at the current state estimate $\hat{\mathbf{x}}_{k|k-1}$ as:
\begin{align}
\mathbf{z}_{k}&= \mathbf{h}(\hat{\mathbf{x}}_{k|k-1}\boxplus\tilde{\mathbf{x}}_{k|k-1}) + \mathbf{n}_{k}\\
 &\simeq \mathbf{h}(\hat{\mathbf{x}}_{k|k-1})  + \mathbf{H}_k  \tilde{\mathbf{x}}_{k|k-1}  + \mathbf{n}_{k} \\
\Rightarrow ~\tilde{\mathbf{z}}_{k}
& \triangleq \mathbf{z}_k - \mathbf{h}(\hat{\mathbf{x}}_{k|k-1})
\simeq \mathbf{H}_k \tilde{\mathbf{x}}_{k|k-1} + \mathbf{n}_{k}
\label{eq:linearizedsystem}
\end{align}
where $\mathbf{H}_k$ is the measurement Jacobian and $\tilde{\mathbf{z}}_k$ is the measurement residual. 
With these, we can now perform an EKF update to refine state estimates and covariance at time step $k$ \citep{Maybeck1979}:
\begin{align}
\hat{\mathbf{x}}_{k|k} &= {\hat{\mathbf{x}}_{k|k-1}} \boxplus \mathbf{K}_{k}({\mathbf{z}}_{k} - \mathbf{h}(\hat{\mathbf{x}}_{k|k-1})) \label{eq:stateupdate}\\
\mathbf{P}_{k|k} &= \mathbf{P}_{k|k-1} - \mathbf{K}_{k} \mathbf{H}_k\mathbf{P}_{k|k-1} \\
\mathbf{K}_{k} &= \mathbf{P}_{k|k-1}\mathbf{H}_k^\top (\mathbf{H}_k\mathbf{P}_{k|k-1}\mathbf{H}_k^\top + \mathbf{R}_{k})^{-1}
\end{align}

\section{Sensing Models} \label{sec:model}

\subsection{IMU Intrinsic Model}
\label{sec:imu intrinsic model}

We define an IMU as containing two separate frames of reference (see Fig. \ref{fig:IMU_frame}):
gyroscope frame $\{w\}$, accelerometer frame $\{a\}$. 
The base ``inertial'' frame $\{I\}$ should be determined to coincide with either $\{w\}$ or $\{a\}$. 
Different from the model in \cite{Schneider2019Sensor}, we define the raw angular velocity reading ${}^w\boldsymbol{\omega}_m$ from the gyroscope and linear acceleration readings ${}^a\mathbf{a}_m$ from the accelerometer as: 
\begin{align}
{}^w\boldsymbol{\omega}_{m} & = 
    \mathbf{T}_{w} {}^w_I\mathbf{R} {}^I\boldsymbol{\omega} + 
    \mathbf{T}_{g} {}^I\mathbf{a} + \mathbf{b}_g + \mathbf{n}_g 
    \\
    {}^a\mathbf{a}_{m} & = \mathbf{T}_a {}^a_I\mathbf{R} {}^I\mathbf{a}
    + \mathbf{b}_a + \mathbf{n}_a
\end{align}
where \textcolor{black}{$\mathbf{T}_{w}$ and $\mathbf{T}_a$} are invertible $3\times3$ matrices which represent the scale imperfection and axis misalignment for $\{w\}$ and $\{a\}$, respectively.
${}^w_I\mathbf{R}$ and ${}^a_I\mathbf{R}$ denote the rotation from the gyroscope frame and acceleration frame to base ``inertial'' frame $\{I\}$, respectively. 
Note that, if we choose $\{I\}$ coincides with $\{w\}$, then ${}^w_I\mathbf{R} = \mathbf{I}_3$. Otherwise, ${}^a_I\mathbf{R} = \mathbf{I}_3$.
$\mathbf{b}_{g}$ and $\mathbf{b}_a$ are the gyroscope and accelerometer biases, which are modeled as random walks;
$\mathbf{n}_g$ and $\mathbf{n}_a$ are the zero-mean Gaussian noises contaminating the measurements. 
$\mathbf{T}_g$ denotes the gravity sensitivity, which represents the effects of gravity to the gyroscope readings.
Similar to the works by \cite{Li2014ICRA} and \cite{Schneider2019Sensor}, we do not take into account the translation between the gyroscope and accelerometer, since it is negligible for most IMUs.
We can write the true (or corrected) angular velocity ${}^I\boldsymbol{\omega}$ and linear acceleration ${}^I\mathbf{a}$ as:
\begin{align}
    \label{eq:IMU reading model}
	{}^I\boldsymbol{\omega} 
	& = {}^I_{w}\mathbf{R} \mathbf{D}_w 
    \left(
    {}^w\boldsymbol{\omega}_m - \mathbf{T}_g {}^I\mathbf{a} - \mathbf{b}_g - \mathbf{n}_g
    \right)
	\\
	\label{eq:IMU reading model1}
	{}^I\mathbf{a} 
	& = {}^I_a\mathbf{R} \mathbf{D}_a 
    \left(
    {}^a\mathbf{a}_m - \mathbf{b}_a - \mathbf{n}_a    \right)
\end{align}
where \textcolor{black}{$\mathbf{D}_w = \mathbf{T}^{-1}_w$ and $\mathbf{D}_a = \mathbf{T}^{-1}_a$}. 
In practice we calibrate $\mathbf{D}_a$, $\mathbf{D}_w$, ${}^I_a\mathbf{R}$ (or ${}^I_w\mathbf{R}$) and $\mathbf{T}_g$ to prevent the need to have the unnecessary matrix inversions in the above measurement equations.
We only calibrate either ${}^I_w\mathbf{R}$ or ${}^I_a\mathbf{R}$ in Eq.~\eqref{eq:IMU reading model} and Eq.~\eqref{eq:IMU reading model1} since the base ``inertial'' frame  coincides with one of sensor frames.
If both ${}^I_w\mathbf{R}$ and ${}^I_a\mathbf{R}$ were calibrated, it would make the rotation between the IMU and camera unobservable due to over parameterization which will be validated in Section \ref{sec:overparameterization}.

\begin{figure}
\centering
\includegraphics[width=0.75\linewidth]{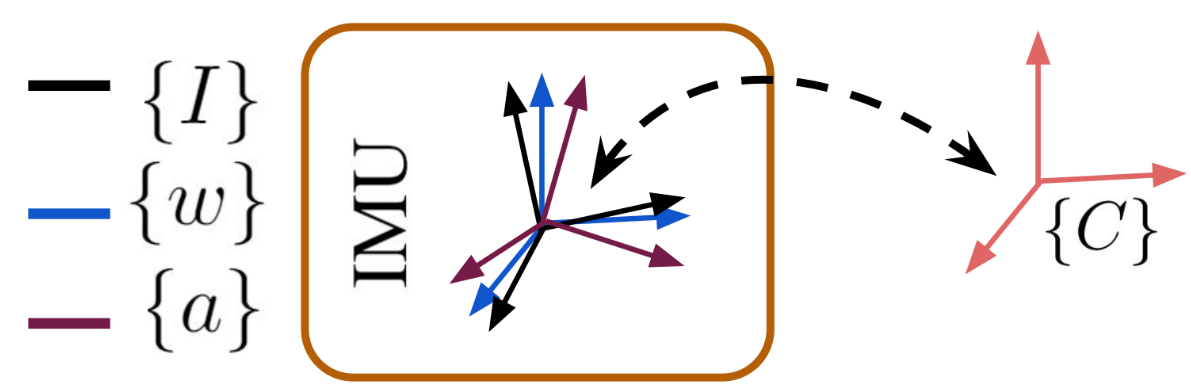}
\caption{
An IMU sensor composed of accelerometer and gyroscope.
The base ``inertial'' frame can be determined to coincide with either accelerometer frame $\{a\}$ or gyroscope frame $\{w\}$. 
There is a rigid 6D transformation between camera frame $\{C\}$ and inertial frame $\{I\}$. 
}
\label{fig:IMU_frame}
\end{figure}

\begin{table}[]
\caption{
IMU model variants and their estimated parameters.
}
\begin{tabular}{ccccccc}
\toprule
\textbf{Model}     & \textbf{Dim.} & $\mathbf{D}_w$     & $\mathbf{D}_a$     & ${}^I_w\mathbf{R}$ & ${}^I_a\mathbf{R}$ & $\mathbf{T}_g$    \\ \midrule
\textit{imu0}  & 0    & -                  & -                  & -                  & -                  & -                 \\ \midrule
\textit{imu1}  & 15   & $\mathbf{D}_{w6}$  & $\mathbf{D}_{a6}$  & ${}^I_w\mathbf{R}$ & -                  & -                 \\
\textit{imu2}  & 15   & $\mathbf{D}_{w6}$  & $\mathbf{D}_{a6}$  & -                  & ${}^I_a\mathbf{R}$ & -                 \\
\textit{imu3}  & 15   & $\mathbf{D}_{w9}$  & $\mathbf{D}_{a6}$  & -                  & -                  & -                 \\
\textit{imu4}  & 15   & $\mathbf{D}_{w6}$  & $\mathbf{D}_{a9}$  & -                  & -                  & -                 \\ \midrule
\textit{imu5}  & 18   & $\mathbf{D}_{w6}$  & $\mathbf{D}_{a6}$  & ${}^I_w\mathbf{R}$ & ${}^I_a\mathbf{R}$ & -                 \\
\textit{imu6}         & 24   & $\mathbf{D}'_{w6}$ & $\mathbf{D}'_{a6}$ & ${}^I_w\mathbf{R}$ &      -            & $\mathbf{T}_{g9}$ \\ \midrule
\textit{imu11} & 21   & $\mathbf{D}_{w6}$  & $\mathbf{D}_{a6}$  & ${}^I_w\mathbf{R}$ & -                  & $\mathbf{T}_{g6}$ \\
\textit{imu12} & 21   & $\mathbf{D}_{w6}$  & $\mathbf{D}_{a6}$  & -                  & ${}^I_a\mathbf{R}$ & $\mathbf{T}_{g6}$ \\
\textit{imu13} & 21   & $\mathbf{D}_{w9}$  & $\mathbf{D}_{a6}$  & -                  & -                  & $\mathbf{T}_{g6}$ \\
\textit{imu14} & 21   & $\mathbf{D}_{w6}$  & $\mathbf{D}_{a9}$  & -                  & -                  & $\mathbf{T}_{g6}$ \\ \midrule
\textit{imu21} & 24   & $\mathbf{D}_{w6}$  & $\mathbf{D}_{a6}$  & ${}^I_w\mathbf{R}$ & -                  & $\mathbf{T}_{g9}$ \\
\textit{imu22} & 24   & $\mathbf{D}_{w6}$  & $\mathbf{D}_{a6}$  & -                  & ${}^I_a\mathbf{R}$ & $\mathbf{T}_{g9}$ \\
\textit{imu23} & 24   & $\mathbf{D}_{w9}$  & $\mathbf{D}_{a6}$  & -                  & -                  & $\mathbf{T}_{g9}$ \\
\textit{imu24} & 24   & $\mathbf{D}_{w6}$  & $\mathbf{D}_{a9}$  & -                  & -                  & $\mathbf{T}_{g9}$ \\ \midrule
\textit{imu31} & 9    & -                  & $\mathbf{D}_{a9}$  & -                  & -                  & -                 \\
\textit{imu32} & 9    & $\mathbf{D}_{w9}$  & -                  & -                  & -                  & -                 \\
\textit{imu33} & 6    & -                  & -                  & -                  & -                  & $\mathbf{T}_{g6}$ \\
\textit{imu34} & 9    & -                  & -                  & -                  & -                  & $\mathbf{T}_{g9}$ \\
\bottomrule
\end{tabular}
\centering
\label{tab:imu model}
\end{table}

\subsubsection{IMU intrinsic model variants.}\label{sec:imu intrinsic model four}

Given the above general model [see Eq.~\eqref{eq:IMU reading model} and Eq.~\eqref{eq:IMU reading model1}], 
different choices of these intrinsic parameters can be made \citep{Li2014ICRA,Nikolic2016Sensors,Rehder2017Sensor,Schneider2019Sensor,Xiao2019Sensors,Jung2020TIM}.
In the following, we present a range of commonly-used IMU intrinsic model variants, 
and  will later compare against each other within an online filter-based VINS.
Specifically, each model is defined as follows:
\begin{itemize}
	\item \textit{imu1}: includes the rotation ${}^I_w\mathbf{R}$, $6$ parameters for $\mathbf{D}_{w}$ (and thus denoted by $\mathbf{D}_{w6}$) and 6 parameters for $\mathbf{D}_{a}$ (and thus denoted by $\mathbf{D}_{a6}$),
	as they assume the upper-triangular structure:
	\begin{align}
		\mathbf{D}_{*6}  = 
		\scalemath{0.9}{
		\begin{bmatrix}
		d_{*1} & d_{*2} & d_{*4} \\
		0 & d_{*3} & d_{*5} \\
		0 & 0 & d_{*6}
		\end{bmatrix}
		}
	    \label{eq:d6_general}
	\end{align}
	\item \textit{imu2}:
	includes the rotation ${}^I_a\mathbf{R}$ instead, $\mathbf{D}_{a6}$ and $\mathbf{D}_{w6}$,
	which is the model used by \cite{Schneider2019Sensor}.
	\item \textit{imu3}: combines \textit{imu1}'s $\mathbf{D}_{w6}$ and ${}^I_w\mathbf{R}$ into a general $3\times3$ matrix containing 9 parameters in total.
	Thus, in this variant we estimate the upper-triangle $\mathbf{D}_{a6}$ and a full matrix $\mathbf{D}_{w9}$ as: 
	\begin{align}
	\mathbf{D}_{*9}  = 
	\scalemath{0.9}{
	\begin{bmatrix}
	d_{*1} & d_{*4} & d_{*7} \\
	d_{*2} & d_{*5} & d_{*8} \\
	d_{*3} & d_{*6} & d_{*9}
	\end{bmatrix}
	\label{eq:d9_general}
	}
	\end{align}
	\item \textit{imu4}:
	is an extension of \textit{imu2} with a combination of the $\mathbf{D}_{a6}$ and ${}^I_a\mathbf{R}$.
	Similarly, in this variant we estimate the upper-triangle $\mathbf{D}_{w6}$ and a full matrix $\mathbf{D}_{a9}$.
	\item \textit{imu1A} ($A=1,\cdots,4$): combines \textit{imuA} with a 6-parameter gravity sensitivity $\mathbf{T}_{g6}$ as:
	\begin{align}
		\mathbf{T}_{g6}  = 
		\scalemath{0.9}{
		\begin{bmatrix}
		t_{g1} & t_{g2} & t_{g4} \\
		0 & t_{g3} & t_{g5} \\
		0 & 0 & t_{g6}
		\end{bmatrix}
		}
	    \label{eq:t6_general}
	\end{align}
	\item \textit{imu2A} ($A=1,\cdots,4$): combines \textit{imuA} a the 9-parameter gravity sensitivity $\mathbf{T}_{g9}$ as:
	\begin{align}
		\mathbf{T}_{g9}  = 
		\scalemath{0.9}{
		\begin{bmatrix}
		t_{g1} & t_{g4} & t_{g7} \\
		t_{g2} & t_{g5} & t_{g8} \\
		t_{g3} & t_{g6} & t_{g9}
		\end{bmatrix}
		}
	    \label{eq:t9_general}
	\end{align}
	\item \textit{imu5}: contains  $\mathbf{D}_{w6}$, $\mathbf{D}_{a6}$,  ${}^I_w\mathbf{R}$ and ${}^I_a\mathbf{R}$. 
	This is a redundant over-parameterized model which will be used to verify that ${}^I_w\mathbf{R}$ and ${}^I_a\mathbf{R}$ should not be calibrated simultaneously. 
	\item \textit{imu6}: contains $\mathbf{D}'_{w6}$, $\mathbf{D}'_{a6}$, ${}^I_w\mathbf{R}$ and $\mathbf{T}_{g9}$. 
	This is equivalent to the \textit{scale-misalignment} IMU intrinsic model \citep{Rehder2016ICRA} used in the calibration toolbox \citep{Furgale2013IROS}. 
	$\mathbf{D}'_{*6}$ assumes the lower triangular structure:
	\begin{align}
	    \mathbf{D}'_{*6} & =
	    \begin{bmatrix}
	    d_{*1} & 0 & 0 \\
	    d_{*2} & d_{*4} & 0 \\
	    d_{*3} & d_{*5} & d_{*6}
	    \end{bmatrix}
	\end{align}
	\item \textit{imu3A} ($A=1,\cdots,4$): models a subset of the parameters of the general model while assuming the others known; 
	that is, only calibrates  $\mathbf{D}_{a9}$ in {\em imu31}, $\mathbf{D}_{w9}$ in {\em imu32}, $\mathbf{T}_{g6}$ in {\em imu33}, and $\mathbf{T}_{g9}$ in {\em imu34}.
\end{itemize}
These different models are summarized in Table \ref{tab:imu model}.
Note that for presentation clarity, \textit{imu22} is used in the ensuing  system derivations and  analysis.

\subsection{Camera Model}

\begin{figure}
\centering
\includegraphics[width=0.75\columnwidth]{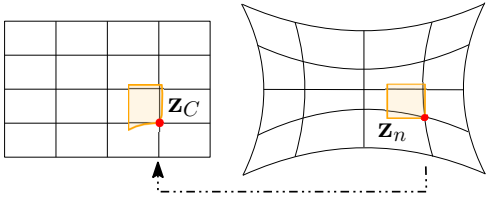}
\caption{Distorting from normalized to a raw image pixel.}
\label{fig:undistort}
\end{figure}

Consider a 3D point feature, ${}^G\mathbf{p}_f$, that is captured by a camera with visual measurement function written as: 
\begin{align}
    \label{eq:camera_hc}
    \mathbf{z}_C & = 
    \begin{bmatrix}
    u \\ v
    \end{bmatrix} 
    + \mathbf{n}_C
\end{align}
where $\mathbf{n}_C$ denotes the measurement noise; 
$u$ and $v$ are the distorted image pixel coordinates:
\begin{align}
    \label{eq:camera_hd}
    \begin{bmatrix}
    u \\ v
    \end{bmatrix} 
    &= 
    \mathbf{h}_d
    \left(
    \mathbf{z}_n
    , \mathbf{x}_{Cin}
    \right)
\end{align}
where $\mathbf{z}_n = [ u_n ~ v_n ]^{\top}$ represents the normalized image pixel and $\mathbf{h}_d(\cdot)$ maps the normalized image pixel onto the image plane based on the camera intrinsic parameters $\mathbf{x}_{Cin}$ and camera model.
Specifically, a pinhole model with radial-tangential (\textit{radtan}) or equivalent-distant (\textit{equidist}) distortion can be used, and the \textit{radtan} model is used in the following derivations and analysis [see \cite{Furgale2013IROS,OPENCV_library}].
Fig.~\ref{fig:undistort} visualizes this image distortion operation.
$\mathbf{x}_{Cin}$ and $\mathbf{h}_d(\cdot)$ are given by:
\begin{align}
    \label{eq:camera intrinsics}
    \mathbf{x}_{Cin}  &= 
    \begin{bmatrix}
    f_u & f_v & c_u & c_v & k_1 & k_2 & p_1 & p_2
    \end{bmatrix}^{\top}
    \\
    \begin{bmatrix}
    u \\ v
    \end{bmatrix} &  =
    \begin{bmatrix}
    f_u & 0 \\
    0 & f_v
    \end{bmatrix}
    \begin{bmatrix}
    u_d \\ v_d
    \end{bmatrix}
    +
    \begin{bmatrix}
    c_u \\ c_v
    \end{bmatrix}
    \\
    \begin{bmatrix}
    u_d \\ v_d
    \end{bmatrix} 
    & = 
    \begin{bmatrix}
    du_n+2p_1u_nv_n + p_2(r^2+2u^2_n)\\
    dv_n+p_1(r^2+2v^2_n) + 2p_2u_nv_n
    \end{bmatrix}
\end{align}
where $r^2 = u^2_n + v^2_n$; $d = 1 + k_1 r^2 + k_2 r^4$; $f_u$ and $f_v$ are the camera focal length; 
$c_u$ and $c_v$ denotes the image principal point; 
$k_1$ and $k_2$ represent the radial distortion coefficients while $p_1$ and $p_2$ are tangential distortion coefficients. 

Normalized image pixel $u$ and $v$ can be acquired by projecting 3D feature ${}^{C}\mathbf{p}_f = [{}^{C}x_f ~ {}^{C}y_f ~ {}^{C}z_f ]^{\top}$ in camera frame into 2D plane as:
\begin{align}
    \mathbf{z}_n
    \label{eq:camera_hp}
    & =
    \mathbf{h}_p
    \left({}^{C}\mathbf{p}_{f}
    \right) 
    \triangleq
    \frac{1}{{}^{C}z_f}
    \begin{bmatrix}
    {}^{C}x_f \\ {}^{C}y_f
    \end{bmatrix}
    \\
    \label{eq:camera_ht}
    {}^{C}\mathbf{p}_{f} 
    & =
    \mathbf{h}_t(
    {}^I_G\mathbf{R}, {}^G\mathbf{p}_I, {}^C_I\mathbf{R}, {}^C\mathbf{p}_I,{}^G\mathbf{p}_f
    )
    \\
    &
    \triangleq 
    {}^C_I\mathbf{R} {}^{I}_G\mathbf{R}
    \left(
    {}^G\mathbf{p}_f - {}^G\mathbf{p}_{I}
    \right)
    + {}^C\mathbf{p}_I
    \notag
\end{align}
where $\{{}^C_I\mathbf{R}, {}^C\mathbf{p}_I\}$ represents the rigid transformation between the IMU and camera frames.

\subsubsection{Temporal calibration.}

Two common variants of camera sensing modes are global shutter (GS) and rolling shutter (RS).
GS cameras expose all pixels at a single time instance, while, typically lower-cost, RS cameras expose each row sequentially.
As shown by \cite{Guo2014RSS}, it may lead to large estimation errors if this RS effect is not taken into account when using RS cameras. 
Additionally, the camera and IMU measurement timestamps can be incorrect due to processing or communication delays, or different clock references.
To address this, we model both the time offset and camera readout time to ensure all measurements are processed in a common clock frame of reference and at the correct corresponding poses.
Specially,  $t_d$ denotes the time offset between IMU and camera timeline and $t_r$ denotes the RS readout time for the whole image. 
If $t$ denotes the time when the pixel is captured, the RS measurement function for a pixel captured in the $m$-th row (out of total $M$ rows) is given by: 
\begin{align}
    \label{eq:camera_ht(t)}
    {}^{C}\mathbf{p}_{f} 
    & =
    \mathbf{h}_t(
    {}^{I(t)}_G\mathbf{R}, {}^G\mathbf{p}_{I(t)},{}^C_I\mathbf{R}, {}^C\mathbf{p}_I, {}^G\mathbf{p}_f
    )
    \\
    & \triangleq
    {}^C_I\mathbf{R} {}^{I(t)}_G\mathbf{R}
    \left(
    {}^G\mathbf{p}_f - {}^G\mathbf{p}_{I(t)}
    \right)
    + {}^C\mathbf{p}_I
    \notag
    \\
    \label{eq:td}
    t_I &= t_C + t_d
    \\
    t & = t_I + \frac{m}{M} t_r  \label{eq:rolling_shutter_time}
\end{align}
where $t_I$ is the IMU state time corresponding to the captured image time $t_C$ when the first row of the image is collected.
If the readout time $t_r = 0$, then the camera is actually a GS camera and all rows are a function of the same pose.
As usual,
$\{{}^G_{I(t)}\mathbf{R}, {}^G\mathbf{p}_{I(t)}\}$ is the IMU global pose corresponding to the camera measurement time $t$.

\section{System Models} \label{sec:statetrans}

The state vector $\mathbf{x}$ of the visual-inertial system under consideration includes the inertial navigation state $\mathbf{x}_{I}$, IMU intrinsic parameter $\mathbf{x}_{in}$, IMU-camera spatial-temporal extrinsic calibration $\mathbf{x}_{IC}$, camera intrinsic calibration $\mathbf{x}_{Cin}$ and feature positions $\mathbf{x}_f$, which is given by:
\begin{align}
    \label{eq:x}
    \mathbf{x} & = 
    \begin{bmatrix}
    \mathbf{x}^{\top}_{I} & 
    \mathbf{x}^{\top}_{IC} & 
    \mathbf{x}^{\top}_{Cin} &
    \mathbf{x}^{\top}_{f}
    \end{bmatrix}^{\top}
    \\ 
    \mathbf{x}_I & = 
    \begin{bmatrix}
    \mathbf{x}^{\top}_n & | & \mathbf{x}^{\top}_b & | & \mathbf{x}^{\top}_{in} 
    \end{bmatrix}^{\top}
    \\
    &
    =
    \begin{bmatrix}
    {}^I_G\bar{q}^{\top} & 
    {}^G\mathbf{p}^{\top}_{I} & 
    {}^G\mathbf{v}^{\top}_{I} & 
    | &
    \mathbf{b}^{\top}_{g} & 
    \mathbf{b}^{\top}_{a} & 
    | & \mathbf{x}^{\top}_{in} 
    \end{bmatrix}^{\top}
    \notag
    \\
    \mathbf{x}_{in} & = 
	\begin{bmatrix}
	\mathbf{x}^{\top}_{Dw} & \mathbf{x}^{\top}_{Da} & {}^I_a\bar{q}^{\top} & \mathbf{x}^{\top}_{Tg}
	\end{bmatrix}^{\top}
	\\
	\mathbf{x}_{IC} & = 
	\begin{bmatrix}
	{}^C_I\bar{q}^{\top} & {}^C\mathbf{p}^{\top}_I & t_d & t_r
	\end{bmatrix}^{\top}
\end{align}
where ${}^I_G\bar{q}$ denotes quaternion with JPL convention~\citep{Trawny2005_Q_TR} and corresponds to the rotation matrix ${}^I_G\mathbf{R}$, which represents the rotation from $\{G\}$ to $\{I\}$.
${}^G\mathbf{p}_{I}$ and ${}^G\mathbf{v}_I$ denote the IMU position and velocity in $\{G\}$.
$\mathbf{x}_n$ denotes the IMU navigation states containing the ${}^I_G\bar{q}$, ${}^G\mathbf{p}_I$ and ${}^G\mathbf{v}_I$. 
$\mathbf{x}_b$ denotes the IMU bias states containing $\mathbf{b}_g$ and $\mathbf{b}_a$. 
${}^C_I\bar{q}$ and ${}^C\mathbf{p}_I$ denotes the rigid transformation between $\{C\}$ and $\{I\}$. 
$t_d$ and $t_r$ represent the IMU-camera time offset and camera readout time. 
IMU intrinsics, $\mathbf{x}_{in}$, contains $\mathbf{x}_{Dw}$, $\mathbf{x}_{Da}$, $\mathbf{x}_{Tg}$ and ${}^I_a\bar{q}$, where $\mathbf{x}_{Dw}$, $\mathbf{x}_{Da}$ and $\mathbf{x}_{Tg}$ are non-zero elements stored column-wise in $\mathbf{D}_w$, $\mathbf{D}_a$ and $\mathbf{T}_g$. 
Specifically, they are defined as:  
\begin{align}
    \mathbf{x}_{D*} & =
    \begin{bmatrix}
    d_{*1} ~~ d_{*2} ~~  d_{*3} ~~  d_{*4} ~~  d_{*5} ~~  d_{*6}
    \end{bmatrix}^{\top}
    \\
    \mathbf{x}_{Tg} & = 
    \begin{bmatrix}
    t_{g1} ~~ t_{g2} ~~ t_{g3} ~~ t_{g4} ~~ t_{g5} ~~ t_{g6} ~~ t_{g7} ~~ t_{g8} ~~ t_{g9}
    \end{bmatrix}^{\top}
\end{align}
It is important to note we use the quaternion left multiplicative error defined by $\bar q \approx [ \frac{1}{2}\delta \bm \theta^{\top} ~ 1 ]^{\top} \otimes \hat{\bar q}$, where $\otimes$ denotes quaternion multiplication and error state is equivalent to the $SO(3)$ error (i.e. ${}^I_G\mathbf R \approx ( \mathbf I_3 - \lfloor \delta \bm \theta \rfloor ) {}^I_G\hat{\mathbf R}$) \citep{Trawny2005_Q_TR}.
The dynamics of the inertial navigation state $\mathbf{x}_I$ is given by~\citep{Chatfield1997}: 
\begin{align}
    \label{eq:imu dynamics}
	{}^I_G\dot{\bar{q}} & = 
	\frac{1}{2}\boldsymbol{\Omega}({}^I\boldsymbol{\omega}){}^I_G\bar{q}  
	~,~~ 
	{}^G\dot{\mathbf{p}}_{I}  = {}^G\mathbf{v}_{I} 
	\\
	{}^G\dot{\mathbf{v}}_I & = {}^I_G\mathbf{R}^\top{}^I\mathbf{a} - {}^G\mathbf{g}
	~,~~  
	\dot{\mathbf{b}}_g  = \mathbf{n}_{wg}
	~,~~ 
	\dot{\mathbf{b}}_a  = \mathbf{n}_{wa}
	\notag
\end{align}
where  $\boldsymbol{\Omega}(\bm{\omega})= \begin{bmatrix} -\lfloor{\bm\omega}  \rfloor & \bm\omega \\ -\bm\omega^T & 0 \end{bmatrix}$,
$\mathbf{n}_{wg}$ and $\mathbf{n}_{wa}$ are zero-mean white Gaussian noises driving $\mathbf{b}_g$ and $\mathbf{b}_a$, respectively,
and the known global gravity assumes ${}^{G}\mathbf{g} = \left[0 ~ 0 ~ 9.81\right]^{\top}$,
while the rest of the states have zero dynamics.

\subsection{Analytic Inertial Integration}

In the following, we present the analytic IMU integration and corresponding error state transition matrix,
which was originally presented in our prior work \citep{Yang2020ICRA} and later extended  to include IMU intrinsics \citep{Yang2020RSS}.
Specifically, we compute the integration of IMU dynamics based on Eq.~\eqref{eq:imu dynamics} from time step $t_k$ to $t_{k+1}$:
\begin{align} 
    {}^{I_{k+1}}_G{\mathbf{R}} & = 
    \Delta \mathbf{R}^{\top}_k {}^{I_k}_G{\mathbf{R}} \label{eq:aci-1}
    \\
    {}^G{\mathbf{p}}_{I_{k+1}} & =  
    {}^{G}{\mathbf{p}}_{I_k} + {}^G{\mathbf{v}}_{I_k}\delta t_k
    + 
    {}^{I_k}_G{\mathbf{R}}^\top
    \Delta \mathbf{p}_k
    - \frac{1}{2}{}^G\mathbf{g}\delta t^2_k
    \label{eq:aci-2}
    \\
    {}^G\hat{\mathbf{v}}_{I_{k+1}} & 
    = {}^{G}\hat{\mathbf{v}}_{I_k} 
    + {}^{I_k}_G{\mathbf{R}}^\top
    \Delta \mathbf{v}_k - {}^G\mathbf{g}\delta t_k
    \label{eq:aci-3}
    \\
    \mathbf{b}_{g_{k+1}} & = \mathbf{b}_{g_{k}} + 
    \scalemath{0.85}{
    \int^{t_{k+1}}_{t_k} \mathbf{n}_{wg} d \tau 
    }
    \label{eq:aci-4}
    \\
    \mathbf{b}_{a_{k+1}} & = \mathbf{b}_{a_{k}} + 
    \scalemath{0.85}{
    \int^{t_{k+1}}_{t_k} \mathbf{n}_{wa} d \tau 
    } \label{eq:aci-5}
\end{align}
where $\delta t_k = t_{k+1} - t_{k}$, and  the three IMU integration quantities are given by:
\begin{align}
    \Delta \mathbf{R}_k & \triangleq
    {}^{I_k}_{I_{k+1}}\mathbf{R}
    =  \exp
    \scalemath{0.9}{
    \left(\int^{t_{k+1}}_{t_{k}} {}^{I_{\tau}}\boldsymbol{\omega} d \tau\right) 
    }\label{eq:integration_components_1}
    \\
    \Delta \mathbf{p}_{k} & \triangleq 
    \scalemath{0.9}{\int^{t_{k+1}}_{t_{k}} \int^{s}_{t_{k}} {}^{I_k}_{I_\tau}\mathbf{R} {}^{I_{\tau}} \mathbf{a}  d \tau d s}
    \label{eq:integration_components_2}
    \\
    \Delta \mathbf{v}_{k} & \triangleq  
    \scalemath{0.9}{
    \int^{t_{k+1}}_{t_{k}}
    {}^{I_k}_{I_\tau}\mathbf{R} {}^{I_{\tau}} \mathbf{a}  d \tau
    }\label{eq:integration_components_3}
\end{align}
where $\textrm{exp}(\cdot)$ is the ${SO}(3)$ matrix exponential \citep{chirikjian2011stochastic}.
The current best estimate of the true angular velocity, ${}^{I_k}{\boldsymbol{\omega}}$,  and linear acceleration, ${}^{I_k}{\mathbf{a}}$, within this time interval $[t_k,t_{k+1}]$ is the expectation of Eq. \eqref{eq:IMU reading model} and Eq. \eqref{eq:IMU reading model1}.
Assuming constant ${}^{I_k}\hat{\boldsymbol{\omega}}$ and ${}^{I_k}\hat{\mathbf{a}}$ within the time interval, we  approximate $\Delta \hat{\mathbf{R}}_k$, $\Delta \hat{\mathbf{p}}_k$ and $\Delta \hat{\mathbf{v}}_k$ as: 
\begin{align}
    \Delta \hat{\mathbf{R}}_k & \simeq \exp \left(
     {}^{I_k}\hat{\boldsymbol{\omega}}\delta t_k
     \right)
     \\
     \Delta \hat{\mathbf{p}}_k & \simeq
     \left(
     \int^{t_{k+1}}_{t_k} \int^{s}_{t_k} 
     {}^{I_k}_{I_{\tau}}\hat{\mathbf{R}}
     d \tau ds
    \right)
     {}^{I_k}\hat{\mathbf{a}}
     \triangleq \boldsymbol{\Xi}_2 {}^{I_k} \hat{\mathbf{a}}
     \\
    \boldsymbol{\Xi}_2  &\triangleq 
    \int^{t_{k+1}}_{t_k} \int^{s}_{t_k}
    \exp \left(
     {}^{I_k}\hat{\boldsymbol{\omega}}\delta\tau
     \right)
    d\tau ds
    \\ 
     \Delta \hat{\mathbf{v}}_k & \simeq
     \left(
     \int^{t_{k+1}}_{t_k}
     {}^{I_k}_{I_{\tau}}\hat{\mathbf{R}}d\tau
    \right)
     {}^{I_k}\hat{\mathbf{a}}
     \triangleq  \boldsymbol{\Xi}_1 {}^{I_k}\hat{\mathbf{a}}
     \\
     \boldsymbol{\Xi}_1 & \triangleq 
    \int^{t_{k+1}}_{t_k} 
    \exp \left(
     {}^{I_k}\hat{\boldsymbol{\omega}}\delta\tau
     \right)
    d{\tau}
\end{align}
where $\delta\tau = t_{\tau} - t_{k}$,  $\boldsymbol{\Xi}_1$ and $\boldsymbol{\Xi}_2$ are defined as integration components which can be evaluated either analytically \citep{Yang2020RSS} or numerically using the Runge–Kutta fourth-order (RK4) method. 
 ${}^{I_k}\hat{\boldsymbol{\omega}}$ and ${}^{I_k}\hat{\mathbf{a}}$ are computed as (note that we drop the timestamp $k$ for simplicity): 
\begin{align}
    {}^I\hat{\boldsymbol{\omega}} & = {}^I_w\hat{\mathbf{R}} \hat{\mathbf{D}}_{w} {}^w\hat{\boldsymbol{\omega}} \\
    {}^w\hat{\boldsymbol{\omega}} & =   {}^w\boldsymbol{\omega}_m - \hat{\mathbf{T}}_g {}^I\hat{\mathbf{a}}-\hat{\mathbf{b}}_g
    =:
    \begin{bmatrix}
    {}^w\hat{w}_1 & {}^w\hat{w}_2 & {}^w\hat{w}_3
    \end{bmatrix}^{\top} \\
    {}^I\hat{\mathbf{a}} &= {}^I_a\hat{\mathbf{R}} \hat{\mathbf{D}}_{a} {}^a\hat{\mathbf{a}} \\
    {}^a\hat{\mathbf{a}} & = {}^a\mathbf{a}_m - \hat{\mathbf{b}}_a
    =:
    \begin{bmatrix}
    {}^a\hat{a}_1 & {}^a\hat{a}_2  & {}^a\hat{a}_3
    \end{bmatrix}^{\top}
\end{align}
where  ${}^I_w\hat{\mathbf{R}}=\mathbf I_3$  for \textit{imu22}. 
As ${}^{I_k}\hat{\boldsymbol{\omega}}$ and ${}^{I_k}\hat{\mathbf{a}}$ are assumed to be constant, 
 the state estimate at $t_{k+1}$ is propagated as follows [see Eq. \eqref{eq:aci-1}-\eqref{eq:aci-5}]: 
\begin{align}
    \label{eq:prop_R}
    {}^{I_{k+1}}_G\hat{\mathbf{R}} & 
     \simeq  
     \Delta \mathbf{R}^{\top}_k
     {}^{I_k}_G\hat{\mathbf{R}}
    \\
    \label{eq:prop_p}
    {}^G\hat{\mathbf{p}}_{I_{k+1}} & \simeq  
    {}^{G}\hat{\mathbf{p}}_{I_k} + {}^G\hat{\mathbf{v}}_{I_k}\delta t_k
    + 
    {}^{I_k}_G\hat{\mathbf{R}}^\top
    \Delta \hat{\mathbf{p}}_k
    - \frac{1}{2}{}^G\mathbf{g}\delta t^2_k
    \\
    \label{eq:prop_v}
    {}^G\hat{\mathbf{v}}_{I_{k+1}} & \simeq  {}^{G}\hat{\mathbf{v}}_{I_k} 
    + {}^{I_k}_G\hat{\mathbf{R}}^\top
    \Delta \hat{\mathbf{v}}_k - {}^G\mathbf{g}\delta t_k
    \\
    \label{eq:prop_bg}
    \hat{\mathbf{b}}_{g_{k+1}} & = \hat{\mathbf{b}}_{g_k}
    \\
    \label{eq:prop_ba}
    \hat{\mathbf{b}}_{a_{k+1}} & = \hat{\mathbf{b}}_{a_k}
\end{align}

\subsection{Linearized System Model}

We first linearize the three IMU pre-integration components [see Eq. \eqref{eq:integration_components_1}-\eqref{eq:integration_components_3}]:
\begin{align}
    \Delta \mathbf{R}_k & =
    \Delta \hat{\mathbf{R}}_k \Delta \tilde{\mathbf{R}}_k
    \triangleq
    \Delta \hat{\mathbf{R}}_k
    \exp \left(
    \mathbf{J}_r (\Delta \hat{\boldsymbol{\theta}}_k)
    {}^{I_k}\tilde{\boldsymbol{\omega}}\delta t_k
    \right)
    \\
    \Delta \mathbf{p}_k & =
    \Delta \hat{\mathbf{p}}_k + \Delta \tilde{\mathbf{p}}_k
    \triangleq 
    \Delta \hat{\mathbf{p}}_k -\boldsymbol{\Xi}_4 {}^{I_k}\tilde{\boldsymbol{\omega}} + 
    \boldsymbol{\Xi}_2 {}^{I_k}\tilde{\mathbf{a}} 
    \\
    \Delta \mathbf{v}_k & =
    \Delta \hat{\mathbf{v}}_k + \Delta \tilde{\mathbf{v}}_k
    \triangleq
    \Delta \hat{\mathbf{v}}_k -\boldsymbol{\Xi}_3 {}^{I_k}\tilde{\boldsymbol{\omega}} + 
    \boldsymbol{\Xi}_1 {}^{I_k}\tilde{\mathbf{a}} 
\end{align}
where $\mathbf{J}_r(\Delta \hat{\boldsymbol{\theta}}_k) \triangleq \mathbf{J}_r \left( {}^{I_k}\hat{\boldsymbol{\omega}}\delta t_k \right)$ denotes the right Jacobian of ${SO}(3)$ \citep{chirikjian2011stochastic}.
The derivation and the definitions of
${}^{I_k}\tilde{\boldsymbol{\omega}}$ and ${}^{I_k}\tilde{\mathbf{a}}$ can be found in Appendix \ref{adp:imu jacobians}.
The integrated components $\boldsymbol{\Xi}_3$ and $\boldsymbol{\Xi}_4$ are defined as:
\begin{align}
    \boldsymbol{\Xi}_3 & \triangleq 
    \int^{t_{k+1}}_{t_k} {}^{I_k}_{I_{\tau}}\mathbf{R}
    \lfloor {}^{I_{\tau}}\mathbf{a} \rfloor
    \mathbf{J}_r \left( {}^{I_k}{\boldsymbol{\omega}} \delta \tau \right) \delta \tau
    d{\tau}
    \\
    \boldsymbol{\Xi}_4 & \triangleq 
    \int^{t_{k+1}}_{t_k} \int^{s}_{t_k} {}^{I_k}_{I_{\tau}}\mathbf{R}
    \lfloor {}^{I_{\tau}}\mathbf{a} \rfloor
    \mathbf{J}_r \left( {}^{I_k}{\boldsymbol{\omega}} \delta \tau \right) \delta \tau
    d{\tau}ds
\end{align}
With the IMU preintegration, the linearized inertial navigation system of error state is given by:
\begin{align}
    \delta \boldsymbol{\theta}_{k+1} & \simeq 
    \Delta \hat{\mathbf{R}}^{\top}_k \delta \boldsymbol{\theta}_k 
    + \mathbf{J}_r \left( \Delta \hat{\boldsymbol{\theta}}_k \right) \delta t_k {}^{I_k}\tilde{\boldsymbol{\omega}}
    \notag
    \\
    {}^G\tilde{\mathbf{p}}_{I_{k+1}} & \simeq 
    {}^G\tilde{\mathbf{p}}_{I_{k}}+{}^G\tilde{\mathbf{v}}_k\delta t_k 
    - {}^{I_{k}}_G\hat{\mathbf{R}}^\top
    \lfloor \Delta \hat{\mathbf{p}}_k \rfloor 
    \delta \boldsymbol{\theta}_k 
    + {}^{I_{k}}_G\hat{\mathbf{R}}^\top \Delta \tilde{\mathbf{p}}_k
    \notag
    \\
    {}^G\tilde{\mathbf{v}}_{I_{k+1}} & \simeq 
    {}^G\tilde{\mathbf{v}}_{I_{k}} 
    - {}^{I_{k}}_G\hat{\mathbf{R}}^\top
    \lfloor \Delta \hat{\mathbf{v}}_k \rfloor 
    \delta \boldsymbol{\theta}_k 
    + {}^{I_{k}}_G\hat{\mathbf{R}}^\top\Delta \tilde{\mathbf{v}}_k
    \notag
\end{align}
As such, the full linearized error-state system for \textit{imu22} is:
\begin{align}
    \label{eq:state transition}
    \tilde{\mathbf{x}}_{I_{k+1}} & \simeq \boldsymbol{\Phi}_{I{(k+1,k)}}\tilde{\mathbf{x}}_{I_k} + \mathbf{G}_{Ik}\mathbf{n}_{dk}
    \\
    \boldsymbol{\Phi}_{I(k+1,k)} & = 
    \begin{bmatrix}
    \boldsymbol{\Phi}_{nn}  & 
    \boldsymbol{\Phi}_{wa} \mathbf{H}_b & 
    \boldsymbol{\Phi}_{wa} \mathbf{H}_{in} \\
    \mathbf{0}_{6\times 9} & 
    \mathbf{I}_6 & 
    \mathbf{0}_{6\times 24} \\
    \mathbf{0}_{24\times 9} & \mathbf{0}_{24\times 6} & \mathbf{I}_{24}
    \end{bmatrix}
    \\
    \mathbf{G}_{Ik} & = 
    \begin{bmatrix}
    \boldsymbol{\Phi}_{wa} \mathbf{H}_n & \mathbf{0}_{9\times 6} \\
    \mathbf{0}_{6} & \mathbf{I}_6 \delta t_k \\
    \mathbf{0}_{24\times 6} & \mathbf{0}_{24\times 6}
    \end{bmatrix}
\end{align}
where $\boldsymbol{\Phi}_{I(k+1,k)}$ and $\mathbf{G}_{Ik}$ are the state transition matrix and noise Jacobians for the inertial state $\mathbf{x}_I$ dynamics,
$\mathbf{H}_{b}$, $\mathbf{H}_{in}$ and $\mathbf{H}_n$ are Jacobians related to bias, IMU intrinsics and noises, which can be found in Appendix \ref{adp:imu jacobians},
$\mathbf{n}_{dk} = [\mathbf{n}^{\top}_{dg} ~ \mathbf{n}^{\top}_{da} ~ \mathbf{n}^{\top}_{dwg} ~ \mathbf{n}^{\top}_{dwa}]^\top$ is the discrete-time IMU noises,
while $\boldsymbol{\Phi}_{nn}$ and $\boldsymbol{\Phi}_{wa}$ can be computed as:
\begin{align}
    \boldsymbol{\Phi}_{nn} & =
    \begin{bmatrix}
    \Delta \hat{\mathbf{R}}^{\top}_k & \mathbf{0}_3 & \mathbf{0}_3 \\
    -{}^{I_k}_G\hat{\mathbf{R}}^\top\lfloor \Delta \hat{\mathbf{p}}_k \rfloor & \mathbf{I}_3 & \mathbf{I}_3 \delta t_k \\
    -{}^{I_k}_G\hat{\mathbf{R}}^\top\lfloor \Delta \hat{\mathbf{v}}_k \rfloor & \mathbf{0}_3 & \mathbf{I}_3 
    \end{bmatrix}
    \\
    \boldsymbol{\Phi}_{wa} & = 
    \begin{bmatrix}
    \mathbf{J}_r(\delta \boldsymbol{\theta}_k) \delta t_k & \mathbf{0}_3 \\
    -{}^{I_k}_G\hat{\mathbf{R}}^\top\boldsymbol{\Xi}_4 & {}^{I_k}_G\hat{\mathbf{R}}^\top\boldsymbol{\Xi}_2 \\
    -{}^{I_k}_G\hat{\mathbf{R}}^\top\boldsymbol{\Xi}_3 & {}^{I_k}_G\hat{\mathbf{R}}^\top\boldsymbol{\Xi}_1 
    \end{bmatrix}
\end{align}
Without loss of generality, we consider a single 3D feature ${}^G\mathbf{p}_f$ in the state vector $\mathbf{x}_f$.
Since there is zero dynamics for $\mathbf{x}_{IC}$, $\mathbf{x}_{Cin}$ and $\mathbf{x}_f$, 
we can write the state transition matrix for the whole state vector $\mathbf{x}$ as [see Eq.~\eqref{eq:x}]: 
\begin{align}
    \boldsymbol{\Phi}_{k+1,k} & = 
    \begin{bmatrix}
    \boldsymbol{\Phi}_{I(k+1,k)} & \mathbf{0} & \mathbf{0} & \mathbf{0} \\
    \mathbf{0} & \boldsymbol{\Phi}_{IC} & \mathbf{0} & \mathbf{0}        \\
    \mathbf{0} & \mathbf{0}  & \boldsymbol{\Phi}_{Cin} & \mathbf{0} \\
    \mathbf{0} & \mathbf{0}  & \mathbf{0}  & \boldsymbol{\Phi}_{f}
    \end{bmatrix}
\end{align}
where $\boldsymbol{\Phi}_{IC}=\mathbf{I}_{8}$, $\boldsymbol{\Phi}_{Cin}=\mathbf{I}_{8}$, and $\boldsymbol{\Phi}_f = \mathbf{I}_3$.

\subsection{Linearized Measurement Model}\label{sec:update}

We first build the overall camera measurements function $\mathbf{h}_C(\cdot)$ by incorporating the distortion function $\mathbf{h}_d(\cdot)$ [see Eq. \eqref{eq:camera_hd}], the projection function $\mathbf{h}_p(\cdot)$ [see Eq. \eqref{eq:camera_hp}] and the transformation function $\mathbf{h}_t(\cdot)$ [see Eq. \eqref{eq:camera_ht(t)}]:
\begin{align}
\label{eq:visual_full_meas}
\mathbf{z}_{C}
&= \mathbf{h}_C(\mathbf x) + \mathbf n_{C} \\
&=\mathbf{h}_d(\mathbf{z}_{n}, \mathbf{x}_{Cin}) + \mathbf{n}_{C} \\
&= \mathbf{h}_d(\mathbf{h}_p({}^{C_k}\mathbf{p}_f), \mathbf{x}_{Cin}) + \mathbf{n}_{C} \\
&= \mathbf{h}_d(\mathbf{h}_p(\mathbf{h}_t({}^{C(t)}_{G}\mathbf{R},{}^{G}\mathbf{p}_{C(t)},{}^{G}\mathbf{p}_f)), \mathbf{x}_{Cin}) + \mathbf{n}_{C} \label{eq:featmeasurement}
\end{align}
We need to linearize the overall visual model for the update, which is given by:
\begin{align}
    \label{eq:visual_full_meas_linearized}
    \tilde{\mathbf{z}}_C & \simeq
    \mathbf{H}_C \tilde{\mathbf x} + \mathbf{n}_C
\end{align}
where $\tilde{\mathbf{z}}_C \triangleq \mathbf{z}_C - \mathbf{h}_C(\hat{\mathbf{x}})$ and $\mathbf{H}_C \triangleq \frac{\partial \tilde{\mathbf{z}}_C}{\partial \tilde{\mathbf{x}}}$. 
Using the chainrule we get the following Jacobian matrix:
\begin{align}
    \label{eq:Hc}
    \mathbf{H}_C & 
    =
    \begin{bmatrix}
    \frac{\partial \tilde{\mathbf{z}}_C}{\partial \tilde{\mathbf{x}}_I} & 
    \frac{\partial \tilde{\mathbf{z}}_C}{\partial \tilde{\mathbf{x}}_{IC}} & 
    \frac{\partial \tilde{\mathbf{z}}_C}{\partial \tilde{\mathbf{x}}_{Cin}} & 
    \frac{\partial \tilde{\mathbf{z}}_C}{\partial \tilde{\mathbf{x}}_f}
    \end{bmatrix}
    \\
    & = 
    \begin{bmatrix}
    \mathbf{H}_{\mathbf{p}_f}
    \frac{\partial {}^{C}\tilde{\mathbf{p}}_f}{\partial \tilde{\mathbf{x}}_I}& 
    \mathbf{H}_{\mathbf{p}_f}
    \frac{\partial {}^{C}\tilde{\mathbf{p}}_f}{\partial \tilde{\mathbf{x}}_{IC}} & 
    \frac{\partial \tilde{\mathbf{z}}_C}{\partial \tilde{\mathbf{x}}_{Cin}} & 
    \mathbf{H}_{\mathbf{p}_f}
    \frac{\partial {}^{C}\tilde{\mathbf{p}}_f}{\partial \tilde{\mathbf{x}}_f}
    \end{bmatrix}
    \notag
\end{align}
where $\mathbf{H}_{\mathbf{p}_f}=
\frac{\partial \tilde{\mathbf{z}}_C}{\partial \tilde{\mathbf{z}}_n} 
\frac{\partial \tilde{\mathbf{z}}_n}{\partial {}^{C}\tilde{\mathbf{p}}_f}$.
All the pertinent matrices $\frac{\partial {}^{C}\tilde{\mathbf{p}}_f}{\partial \tilde{\mathbf{x}}_I}$, $\frac{\partial {}^{C}\tilde{\mathbf{p}}_f}{\partial \tilde{\mathbf{x}}_{IC}}$, $\frac{\partial {}^{C}\tilde{\mathbf{p}}_f}{\partial \tilde{\mathbf{x}}_f}$ and $\mathbf{H}_{\mathbf{p}_f}$ can be computed as shown in Appendix \ref{sec:apx_cam_jacob}.

\section{Observability Analysis} \label{sec:obsanalysis}

Observability analysis plays an important role in determining whether or not the states are estimable for given measurements
and can also be leveraged to identify degenerate motions that can negatively affect estimation performance \citep{Huang2012thesis,Martinelli2014FTR}.
While the observability analysis of VINS has been well studied \citep{Hesch2014TRO}, 
the observability properties and degenerate motions of VINS with full self-calibration (in particular, IMU and camera intrinsic calibration)  have not been sufficiently investigated.
To this end,
following \cite{Hesch2014TRO}, we construct the observability matrix as follows:
\begin{align}
    \mathcal{O} & = 
    \begin{bmatrix}
    \mathcal{O}_{1} \\
    \mathcal{O}_{2} \\
    \vdots \\
    \mathcal{O}_{k}
    \end{bmatrix}
    =
    \begin{bmatrix}
    \mathbf{H}_{C1}\boldsymbol{\Phi}_{1,1} \\
    \mathbf{H}_{C2}\boldsymbol{\Phi}_{2,1} \\
    \vdots \\
    \mathbf{H}_{Ck}\boldsymbol{\Phi}_{k,1}
    \end{bmatrix}
\end{align}
We write the $k$-th row of $\mathcal{O}$ as:
\begin{align}
    \label{eq:M}
    \mathcal{O}_{k} & = 
    \begin{bmatrix}
    \mathbf{M}_n & \mathbf{M}_{b} & \mathbf{M}_{in} & \mathbf{M}_{IC} & \mathbf{M}_{Cin} & \mathbf{M}_f
    \end{bmatrix}
\end{align}
where $\mathbf{M}_n$, $\mathbf{M}_b$, $\mathbf{M}_{in}$, $\mathbf{M}_{IC}$, $\mathbf{M}_{Cin}$ and $\mathbf{M}_f$ represent the matrix block relating to the state [see Eq. \eqref{eq:x}] with detailed derivations in Appendix \ref{apx:obs_matrix_M}.
We now look to find the unobservable subspace $\mathbf{N}$ such that $\mathcal{O}\mathbf{N}=\mathbf{0}$.
The following can be found: 
\begin{lem}
\label{lem:obs}
Given fully excited motions, monocular VINS system with online calibration of IMU intrinsics $\mathbf{x}_{in}$, camera intrinsics $\mathbf{x}_{Cin}$ and IMU-camera spatial-temporal parameters $\mathbf{x}_{IC}$ (including RS readout time) has 4 unobservable directions, which relate to the global yaw and global translation. 
\begin{align}
    \label{eq:N}
    \mathbf{N} = 
    \begin{bmatrix}
    {}^{I_1}_G\hat{\mathbf{R}} {}^G\mathbf{g} & \mathbf{0}_3 \\
    -\lfloor {}^G\hat{\mathbf{p}}_{I_1} \rfloor {}^G\mathbf{g} & \mathbf{I}_3 \\
    -\lfloor {}^G\hat{\mathbf{v}}_{I_1} \rfloor {}^G\mathbf{g} & \mathbf{0}_3 \\
    \mathbf{0}_{46\times 1} & \mathbf{0}_{46\times 3}  \\
    -\lfloor {}^G\hat{\mathbf{p}}_{f} \rfloor {}^G\mathbf{g} & \mathbf{I}_3 
    \end{bmatrix}
\end{align}
\end{lem}
\begin{proof}
See Appendix \ref{sec:proof_of_lemma1}. 
\end{proof}
We  notice that the terms $\mathbf{M}_{in}$ and $\mathbf{M}_{IC}$ (shown in Appendix \ref{apx:obs_matrix_M}) contain ${}^{w}\hat{\boldsymbol{\omega}}$, ${}^a\hat{\mathbf{a}}$, ${}^I\hat{\boldsymbol{\omega}}$ and ${}^G\hat{\mathbf{v}}_I$,  corresponding to the sensor platform motion. 
This implies that, $\mathbf{M}_{in}$ and $\mathbf{M}_{IC}$, corresponding to IMU intrinsics $\mathbf{x}_{in}$ and IMU-camera spatial-temporal parameters $\mathbf{x}_{IC}$ (including RS effects), are motion-dependent and time-varying.
Specifically, we have  the following properties for the IMU intrinsics and camera spatial-temporal parameters:
\begin{proposition}
For monocular VINS, the IMU intrinsic calibration and IMU-camera spatial-temporal calibration (including RS readout time) are sensitive to sensor motions. Given fully excited motions, $\mathbf{x}_{in}$ and $\mathbf{x}_{IC}$ are observable. 
\end{proposition}
Given this observation and numerical simulations of VINS based on a monocular camera and IMU, shown in Fig.~\ref{fig:sim_full}, we can confirm that all these calibration parameters are observable and can be estimated given fully-excited motions.
While we omit the derivations and simulation results here, the other IMU intrinsic model variants besides the \textit{imu22} presented are also fully observable in the case of fully-excited motions.

Similarly, the camera intrinsics, $\mathbf{M}_{Cin}$, are mainly affected by the environmental structure (the $u$ and $v$ measurements of the 3D point features) with no motion terms are involved.
Hence, we have: 
\begin{proposition}
For monocular VINS, the camera intrinsic calibration $\mathbf{x}_{Cin}$ is affected by the structure of the environment features and less sensitive to sensor motions. 
\end{proposition}
The camera intrinsic parameters are observable for most motion cases, even for under-actuated motions (i.e. planar motion), which can also be verified by our simulation results shown in Fig. \ref{fig:sim_full}-\ref{fig:sim_planar}.
While we omit the results here, the \textit{equidist} camera distortion model also satisfies the above proposition.

\section{Degenerate Motion Identification} \label{sec:degen_identif}

While the observability properties found in the preceding section hold with \textit{general} motions, this may not always be the case in reality and thus the identification of \textit{degenerate} motion profiles that cause extra unobservable directions in the state space, becomes important.
Based on the observability analysis, we can further identify 
the degenerate motions corresponding to the state of the system, such as the inertial state and features, along with the calibration parameters introduced during online self-calibration.
As the degenerate motion analysis of VINS has been studied in the prior work \citep{Martinelli2012TRO,Hesch2013TRO,Wu2017ICRA,Yang2019TRO}, 
we here focus only on motions that cause the calibration parameters to become unobservable.

\subsection{Inertial IMU Intrinsic Parameters}

As mentioned in Section~\ref{sec:obsanalysis},  
$\mathbf{M}_{in}$ is heavily motion affected, and thus, the IMU intrinsics are extremely susceptible to be unobservable under certain motions.
Because the bias terms and IMU intrinsics are tightly-coupled, by carefully inspecting the observability matrix $\mathcal{O}$, 
we find a selection of basic motion types which can cause the IMU intrinsics to become unobservable for \textit{imu22}.
Note that similar results are applicable to other IMU model variants.

\subsubsection{Degenerate motions for $\mathbf{D}_w$.}

As the gyroscope related IMU intrinsics $\mathbf{D}_w$ are coupled with gyroscope bias $\mathbf{b}_g$ and the angular velocity readings ${}^w\boldsymbol{\omega}$ from the IMU,
we have the following results:
\begin{lem} \label{lem:dw}
If any component of ${}^w\boldsymbol{\omega}$ (including ${}^w{\omega}_1$, ${}^w{\omega}_2$, ${}^w{\omega}_3$) is constant,  
then $\mathbf{D}_w$ will become unobservable. 
\end{lem}
\begin{proof}
If ${}^ww_1$ is constant, $d_{w1}$ will be unobservable with unobservable directions as:
\begin{align}
    \scalemath{0.86}{\mathbf{N}_{w1}} & = 
    \scalemath{0.86}{
    \begin{bmatrix}
    \mathbf{0}_{1\times9} & 
    (\hat{\mathbf{D}}^{-1}_w  {}^I_w\hat{\mathbf{R}}^{\top}  \mathbf{e}_1)^{\top}{}^ww_1& 
    \mathbf{0}_{1\times3} & 
    1 & \mathbf{0}_{1\times 42} 
    \end{bmatrix}^{\top}
    }
\end{align}
If ${}^ww_2$ is constant, $d_{w2}$ and $d_{w3}$ will be unobservable with unobservable directions as:
\begin{align}
    \scalemath{0.86}{\mathbf{N}_{w2}} & = 
    \scalemath{0.86}{
    \begin{bmatrix}
    \mathbf{0}_{1\times9} & 
    (\hat{\mathbf{D}}^{-1}_w {}^I_w\hat{\mathbf{R}}^{\top}   \mathbf{e}_1)^{\top}{}^ww_2 & 
    \mathbf{0}_{1\times4} & 
    1 & \mathbf{0}_{1\times 41} 
    \\
    \mathbf{0}_{1\times9} & 
    (\hat{\mathbf{D}}^{-1}_w {}^I_w\hat{\mathbf{R}}^{\top}   \mathbf{e}_2)^{\top}{}^ww_2 & 
    \mathbf{0}_{1\times5} & 
    1 & \mathbf{0}_{1\times 40} 
    \end{bmatrix}^{\top}
    }
\end{align}
If ${}^ww_3$ is constant, $d_{w4}$, $d_{w5}$ and $d_{w6}$ are unobservable with unobservable directions as:
\begin{align}
    \scalemath{0.86}{\mathbf{N}_{w3}} & = 
    \scalemath{0.86}{
    \begin{bmatrix}
    \mathbf{0}_{1\times9} & 
    (\hat{\mathbf{D}}^{-1}_w {}^I_w\hat{\mathbf{R}}^{\top}  \mathbf{e}_1)^{\top}{}^ww_3 & 
    \mathbf{0}_{1\times6} & 
    1 & \mathbf{0}_{1\times 39} 
    \\
    \mathbf{0}_{1\times9} & 
    (\hat{\mathbf{D}}^{-1}_w {}^I_w\hat{\mathbf{R}}^{\top}  \mathbf{e}_2)^{\top}{}^ww_3 & 
    \mathbf{0}_{1\times7} & 
    1 & \mathbf{0}_{1\times 38} 
    \\
    \mathbf{0}_{1\times9} & 
    (\hat{\mathbf{D}}^{-1}_w {}^I_w\hat{\mathbf{R}}^{\top}  \mathbf{e}_3)^{\top}{}^ww_3 & 
    \mathbf{0}_{1\times8} & 
    1 & \mathbf{0}_{1\times 37} 
    \end{bmatrix}^{\top}
    }
\end{align}
See Appendix \ref{sec:proof} for the verification of these unobservable directions. 
\end{proof}
\subsubsection{Degenerate motions for $\mathbf{D}_a$.}
Similarly, as ${}^a\mathbf{a}$ can affect the observability property for the accelerometer related IMU intrinsics $\mathbf{D}_a$, we have: 
\begin{lem} \label{lem:da}
If any component of ${}^a\mathbf{a}$ (including ${}^aa_1$, ${}^aa_2$ and ${}^aa_3$) is constant, then $\mathbf{D}_a$ will become unobservable. 
\end{lem}
\begin{proof}
If ${}^aa_1$ is constant, $d_{a1}$, pitch and yaw of ${}^I_a\mathbf{R}$ are unobservable with unobservable directions as:
\begin{align}
    &\scalemath{0.86}{\mathbf{N}_{a1}}
    = 
   \scalemath{0.86}{
    \begin{bmatrix}
    \mathbf{0}_{12\times 1}  & \mathbf{0}_{12\times1} & \mathbf{0}_{12\times 1} 
    \\
    \hat{\mathbf{D}}^{-1}_a \mathbf{e}_1 {}^aa_1 &
    \hat{\mathbf{D}}^{-1}_a \mathbf{e}_2 \hat{d}_{a1}{}^aa_1 & 
    \hat{\mathbf{D}}^{-1}_a \mathbf{e}_3 \hat{d}_{a1}\hat{d}_{a3}{}^aa_1 
    \\
    \mathbf{0}_{6\times 1} & \mathbf{0}_{6\times1}& \mathbf{0}_{6\times 1} \\
    1 &  0  & 0 \\
    0 & \hat{d}_{a3} &  0 \\
    0 & -\hat{d}_{a2} & 0 \\
    0 & \hat{d}_{a5} & \hat{d}_{a6}\hat{d}_{a3}  \\
    0 & -\hat{d}_{a4} & -\hat{d}_{a2}\hat{d}_{a6}  \\
    0 & 0 & \hat{d}_{a2}\hat{d}_{a5} - \hat{d}_{a4}\hat{d}_{a3} \\
    \mathbf{0}_{3\times1} & -{}^I_a\hat{\mathbf{R}}\mathbf{e}_3 & {}^I_a\hat{\mathbf{R}}(\mathbf{e}_1\hat{d}_{a2}+\mathbf{e}_2\hat{d}_{a3})  \\
    \mathbf{0}_{28\times 1} & \mathbf{0}_{28\times1} & \mathbf{0}_{28\times 1}
    \end{bmatrix}
    }
\end{align}
If ${}^aa_2$ is constant, $d_{a2}$, $d_{a3}$ and roll of ${}^I_a\mathbf{R}$ are unobservable with unobservable directions as:
\begin{align}
    & \scalemath{0.86}{\mathbf{N}_{a2}}
     = 
     \scalemath{0.86}{
    \begin{bmatrix}
    \mathbf{0}_{12\times 1}  & \mathbf{0}_{12\times1} & \mathbf{0}_{12\times 1} 
    \\
    \hat{\mathbf{D}}^{-1}_a \mathbf{e}_1 {}^aa_2 &
    \hat{\mathbf{D}}^{-1}_a \mathbf{e}_2 {}^aa_2 & 
    \hat{\mathbf{D}}^{-1}_a \mathbf{e}_3 \hat{d}_{a3}{}^aa_2 
    \\
    \mathbf{0}_{6\times 1} & \mathbf{0}_{6\times1}& \mathbf{0}_{6\times 1} \\
    0 &  0  & 0 \\
    1 & 0 &  0 \\
    0 & 1 & 0 \\
    0 & 0 & 0  \\
    0 & 0 & \hat{d}_{a6}  \\
    0 & 0 & -\hat{d}_{a5} \\
    \mathbf{0}_{3\times1} & \mathbf{0}_{3\times1} & -{}^I_a\hat{\mathbf{R}}\mathbf{e_1}  \\
    \mathbf{0}_{28\times 1} & \mathbf{0}_{28\times1} & \mathbf{0}_{28\times 1}
    \end{bmatrix}
    }
\end{align}
If ${}^aa_3$ is constant, $d_{a4}$, $d_{a5}$ and $d_{a6}$ are unobservable with unobservable directions as:
\begin{align}
    \scalemath{0.86}{\mathbf{N}_{a3}} & = 
    \scalemath{0.86}{
    \begin{bmatrix}
    \mathbf{0}_{1\times12} & 
    (\hat{\mathbf{D}}^{-1}_a\mathbf{e}_1)^{\top}{}^aa_3 & 
    \mathbf{0}_{1\times9} & 
    1 & \mathbf{0}_{1\times 33} 
    \\
    \mathbf{0}_{1\times12} & 
    (\hat{\mathbf{D}}^{-1}_a\mathbf{e}_2)^{\top}{}^aa_3 & 
    \mathbf{0}_{1\times10} & 
    1 & \mathbf{0}_{1\times 32} 
    \\
    \mathbf{0}_{1\times12} & 
    (\hat{\mathbf{D}}^{-1}_a\mathbf{e}_3)^{\top}{}^aa_3 & 
    \mathbf{0}_{1\times11} & 
    1 & \mathbf{0}_{1\times 31} 
    \end{bmatrix}^{\top}
    }
\end{align}
See Appendix \ref{sec:proof} for the verification of these unobservable directions. 
\end{proof}
\subsubsection{Degenerate motions for $\mathbf{T}_g$.}
As ${}^I\mathbf{a}$ (the acceleration in IMU frame) can affect the observability property for the gravity sensitivity $\mathbf{T}_g$, by close inspection of special configurations for ${}^I\mathbf{a}$, we have: 
\begin{lem} \label{lem:tg}
If any component of ${}^I\mathbf{a}$ (including ${}^Ia_1$, ${}^Ia_2$ and ${}^Ia_3$) is constant, then $\mathbf{T}_g$ will become unobservable. 
\end{lem}
\begin{proof}
If ${}^Ia_1$ is constant, $t_{g1}$, $t_{g2}$ and $t_{g3}$ are unobservable with unobservable directions as:
\begin{align}
    &\scalemath{0.86}{\mathbf{N}_{g1}}
    = 
   \scalemath{0.86}{
    \begin{bmatrix}
    \mathbf{0}_{3\times9} & \mathbf{I}_3{}^I{a}_1 & \mathbf{0}_{3\times 18} & -\mathbf{I}_{3} & \mathbf{0}_{3\times 25}
    \end{bmatrix}^{\top}
    }
    \notag
\end{align}
If ${}^Ia_2$ is constant, $t_{g4}$, $t_{g5}$ and $t_{g6}$ are unobservable with unobservable directions as:
\begin{align}
    &\scalemath{0.86}{\mathbf{N}_{g2}}
    = 
   \scalemath{0.86}{
    \begin{bmatrix}
    \mathbf{0}_{3\times9} & \mathbf{I}_3{}^I{a}_2 & \mathbf{0}_{3\times 21} & -\mathbf{I}_{3} & \mathbf{0}_{3\times 22}
    \end{bmatrix}^{\top}
    }
    \notag
\end{align}
If ${}^Ia_3$ is constant, $t_{g7}$, $t_{g8}$ and $t_{g9}$ are unobservable with unobservable directions as:
\begin{align}
    &\scalemath{0.86}{\mathbf{N}_{g3}}
    = 
   \scalemath{0.86}{
    \begin{bmatrix}
    \mathbf{0}_{3\times9} & \mathbf{I}_3{}^I{a}_3 & \mathbf{0}_{3\times 24} & -\mathbf{I}_{3} & \mathbf{0}_{3\times 19}
    \end{bmatrix}^{\top}
    }
    \notag
\end{align}
See Appendix \ref{sec:proof} for the verification of these unobservable directions. 
\end{proof}

\subsubsection{Discussion on IMU intrinsic degeneracy.}

\begin{table}
\caption{
Summary of  basic degenerate motions for online IMU intrinsics calibration (\textit{imu22}).
}
\begin{tabular}{@{}ccl@{}}
\toprule
\textbf{Motion Types}   & \textbf{Dim.} & \multicolumn{1}{c}{\textbf{Unobservable Parameters}} \\ \midrule
constant ${}^w\omega_1$ & {1}                    & $d_{w1}$                                            \\
constant ${}^w\omega_2$ & {2}                    & $d_{w2}, d_{w3}$                                   \\
constant ${}^w\omega_3$ & {3}                    & $d_{w4}, d_{w5}, d_{w6}$                          \\ \midrule
constant ${}^aa_1$      & {3}                    & $d_{a1}$, pitch and yaw of ${}^I_a\mathbf{R}$       \\
constant ${}^aa_2$      & {3}                    & $d_{a2}, d_{a3}$, roll of ${}^I_a\mathbf{R}$       \\
constant ${}^aa_3$      & {3}                    & $d_{a4}, d_{a5}, d_{a6}$                \\ \midrule
constant ${}^Ia_1$      & {3}                    &  $t_{g1}, t_{g2}, t_{g3}$        \\
constant ${}^Ia_2$      & {3}                    & $t_{g4}, t_{g5}, t_{g6}$        \\
constant ${}^Ia_3$      & {3}                    & $t_{g7}, t_{g8}, t_{g9}$                
\\ \bottomrule
\end{tabular}
\centering
\label{table:degenerate summary}
\end{table}

It is evident from the above analysis that the IMU intrinsic calibration is sensitive to sensor motion and thus all 6 axes need to be excited to ensure all of them can be calibrated.
These findings are summarized in Table \ref{table:degenerate summary}.
It should be noted that 
any combination of these primitive motions is still degenerate and causes all related parameters to become unobservable 
(e.g., planar motion with constant acceleration).
It is also important to mention that it is common that ${}^I_a\mathbf{R} \simeq \mathbf{I}_3$ and $\mathbf{D}_a \simeq \mathbf{I}_3$ for most IMUs, and thus, ${}^a\hat{\mathbf{a}} \simeq {}^I\mathbf{a}$.
As such, the degenerate motions for $\mathbf{D}_a$ will also lead to the unobservability of $\mathbf{T}_g$, and vice-versa. 
Again, this degenerate motion analysis can  be extended to other model variants, which is omitted here for brevity.

\subsection{IMU-Camera Spatial-Temporal Parameters}
\begin{table}
\renewcommand{\arraystretch}{1.2}
	\caption{Summary of basic degenerate motions for online IMU-camera spatial-temporal calibration.}
	\begin{adjustbox}{width=\columnwidth,center}
	\begin{tabular}{ccc}
		\hline		
		\textbf{Motion Types} & \textbf{Unobservable Parameters} & \textbf{Observable} \\\hline
		pure translation & $^C\mathbf{p}_I$ & $^C_I\mathbf{R}$, $t_d$, $t_r$ \\ \hline
		one-axis rotation & $^C\mathbf{p}_I$ along rotation axis & $^C_I\mathbf{R}$, $t_d$, $t_r$ \\ \hline
		constant $^I\boldsymbol{\omega}$  & $t_d$ and & \multirow{2}{*}{$^C_I\mathbf{R}$, $t_r$}  \\
		constant $^I\mathbf{v}$ & $^C\mathbf{p}_I$ along rotation axis & \\ \hline	
		constant $^I\boldsymbol{\omega}$ & $t_d$ and & \multirow{2}{*}{$^C_I\mathbf{R}$, $t_r$} \\
		constant $^G\mathbf{a}$ & $^C\mathbf{p}_I$ along rotation axis & \\ \hline
	\end{tabular}
	\end{adjustbox}
	\centering
	\label{tab:cam degenerate}
\end{table}

Leveraging our previous work \citep{Yang2019RAL} where we have studied four commonly-seen degenerate motions of VINS with only IMU-camera spatial-temporal calibration,  we here show these degenerate motions hold true for VINS with full-parameter calibration: 
\begin{lem} \label{lem:imu-cam}
The IMU-camera spatial-temporal calibration will become unobservable, if the sensor platform undergoes the following degenerate motions: 
\begin{itemize}
    \item Pure translation
    \item One-axis rotation
    \item Constant local angular and linear velocity
    \item Constant local angular velocity and global linear acceleration
\end{itemize}
\end{lem}
\begin{proof}
If the system undergoes pure translation (no rotation), the translation part $^C\mathbf{p}_I$ of the spatial calibration will be unobservable, 
residing along the following unobservable directions:
\begin{align} \label{eq:Npt}
\mathbf{N}_{pt} & =
\begin{bmatrix}
\mathbf{0}_{3\times 45}  &
{\mathbf{I}_3} &
\mathbf{0}_{3\times 10} & 
-({^G_{I_1}\hat{\mathbf{R}}}{^I_C\hat{\mathbf{R}}})^{\top}
\end{bmatrix}^{\top}
\end{align}

If the system undergoes random (general) translation but with only one-axis rotation, the translation calibration $^C\mathbf{p}_I$ along the rotation axis will be unobservable, with the following unobservable direction:
\begin{align} \label{eq:Noa}
\mathbf{N}_{oa} & = 
\begin{bmatrix}
\mathbf{0}_{1\times 45}  &\!\!
({^C_I\hat{\mathbf{R}}}{^I\hat{\mathbf{k}}})^{\top} &\!\!
0_{1\times 10} &\!\!
-({^G_{I_1}\hat{\mathbf{R}}}{^I\hat{\mathbf{k}}})^{\top}\!\!
\end{bmatrix}^{\top}
\end{align}
where ${}^I\mathbf{k}$ is the constant rotation axis in the IMU frame $\{I\}$.

If the VINS undergoes constant local angular velocity $^I\boldsymbol{\omega}$ and linear velocity  $^I\mathbf{v}$, 
the time offset $t_d$ will be unobservable with the following unobservable direction:
\begin{align} \label{eq:Nt1}
\mathbf{N}_{t1} \!=\!
\scalemath{0.95}{
\begin{bmatrix}
\mathbf{0}_{1\times 42} &\!\!
({^C_I\hat{\mathbf{R}}}{^I\hat{\boldsymbol{\omega}}})^{\top}&\!\!
-({^C_I\hat{\mathbf{R}}}{^I\hat{\mathbf{v}}})^{\top} &\!\!
-1 &\!\! 
\mathbf{0}_{1\times 12} \!\! 
\end{bmatrix}^{\top}
}
\end{align}

If the VINS undergoes constant local angular velocity $^I\boldsymbol{\omega}$ and global acceleration  $^G\mathbf{a}$, 
the time offset $t_d$ will be unobservable with the following unobservable direction:
\begin{align} \label{eq:Nt2}
&\mathbf{N}_{t2}  = 
\\
&
\scalemath{0.9}{
\begin{bmatrix}
\mathbf{0}_{1\times 6} &\!\! 
{^G\hat{\mathbf{a}}} &\!\! 
\mathbf{0}_{1\times 30} &\!\! 
({^C_I\hat{\mathbf{R}}}{^I\hat{\boldsymbol{\omega}}})^{\top} &
\mathbf{0}_{1\times 3} &
-1 & \!\! 
\mathbf{0}_{1\times 9}
& \!\! 
-(^G\hat{\mathbf{v}}_{I_1})^{\top}
\end{bmatrix}^{\top}
}
\notag
\end{align}

\end{proof}

Table \ref{tab:cam degenerate} summarizes these degenerate motions for completeness. 
It is important to note that unlike  $t_d$ (whose Jacobian is mainly affected by the sensor motion), 
the Jacobian for RS readout time, $t_r$, is also affected by the feature observations due to the term $\frac{m}{M}$ [see Eq. \eqref{eq:job tr}], and is observable, as hundreds of features can be observed from different image rows during exploration.

\subsection{Camera Intrinsic Parameters}

As mentioned before, the camera intrinsics are mainly affected by the observed features. 
By investigating special feature configurations, we find the following degenerate case for camera calibration when using a \textit{radtan} distortion model: 
\begin{lem} \label{lem:cam-intrinsics}
The camera intrinsics will become unobservable if the following conditions are satisfied: 
\begin{itemize}
    \item The features keep the same depth relative to the camera (e.g., ${}^Cz_f$ is constant in value).
    \item The camera moves with one-axis rotation and the  rotation axis is defined as ${}^C\mathbf{k}=\mathbf{e}_3$. 
\end{itemize}
\end{lem}
\begin{proof}
The camera focal length $f_u$, $f_v$, the camera distortion model $k_1$, $k_2$, $p_1$ and $p_2$ will become unobservable with unobservable direction:
\begin{align}
    \mathbf{N}_{Cin} & = 
    \begin{bmatrix}
    \mathbf{0}_{1\times 47} ~~
    f_u ~~ f_v ~~ \mathbf{0}_{1\times 2} ~~ 2k_1 ~~ 4k_2 ~~ p_1 ~~ p_2  ~~ {}^G\mathbf{k}^{\top}
    \end{bmatrix}^{\top}
    \notag
\end{align}
with ${}^G\mathbf{k}= {}^G_{I_0}\hat{\mathbf{R}} {}^I_C\hat{\mathbf{R}} {}^C\mathbf{k} {}^C{z}_f$ . 

 See Appendix \ref{sec:proof} for the verification of these unobservable directions. 
\end{proof}

As an example, if a ground vehicle is performing planar motion with a upward facing camera only observing features from the ceilings, 
the above two conditions will hold and thus the camera intrinsics with \textit{radtan} distortion model will be unobservable. 
Nevertheless, since it is common to observe hundreds of features, it might be rarely the case that every feature maintains the same relative depth, ${}^C{z}_f$, to the camera,
and thus, this degeneracy may not happen in practice if features are tracked uniformly throughout images. 
It is interesting to note that this degenerate case does not work for camera models with \textit{equidist} distortion.

\section{State Estimator} \label{sec:estimator}

Leveraging our MSCKF-based VINS~\citep{Geneva2020ICRA}, the proposed estimator extends the state vector $\mathbf{x}_k$ at time step $k$ to include the current IMU state $\mathbf{x}_{I_k}$, a sliding window of cloned IMU poses $\mathbf{x}_c$, the calibration parameters ($\mathbf{x}_{IC}$ and $\mathbf{x}_{Cin}$) and feature state $\mathbf{x}_f$. 
\begin{align}
    \label{eq:xk}
    \mathbf{x}_k & = 
    \begin{bmatrix}
    \mathbf{x}^{\top}_{I_k} & 
    \mathbf{x}^{\top}_{c}  &
    \mathbf{x}^{\top}_{IC} & 
    \mathbf{x}^{\top}_{Cin} &
    \mathbf{x}^{\top}_{f}
    \end{bmatrix}^{\top}
    \\
    \mathbf{x}_c & = 
    \begin{bmatrix}
    {}^{I_{ck-1}}_G\bar{q}^{\top} ~~ {}^G\mathbf{p}^{\top}_{I_{ck-1}} ~~ \ldots ~~ 
    {}^{I_{ck-n}}_G\bar{q}^{\top} ~~ {}^G\mathbf{p}^{\top}_{I_{ck-n}} 
    \end{bmatrix}^{\top}
\end{align}
where $\mathbf{x}_{I}$, $\mathbf{x}_{IC}$, $\mathbf{x}_{Cin}$ and $\mathbf{x}_f$ are the same as Eq. \eqref{eq:x}, 
$\mathbf{x}_c$ denotes the sliding window containing $n$ cloned IMU poses with index from $ck-n$ to $ck-1$. 
Note that the IMU intrinsics $\mathbf{x}_{in}$ are contained in the current IMU state $\mathbf{x}_{I_k}$. 

As $\mathbf{x}_c$, $\mathbf{x}_{IC}$, $\mathbf{x}_{Cin}$ and $\mathbf{x}_f$ have zero dynamics, 
we only propagate the estimate and covariance of the next IMU state based on Eq. \eqref{eq:prop_R}-\eqref{eq:prop_ba} and Eq. \eqref{eq:state transition}, which all incorporate the IMU intrinsics $\mathbf{x}_{in}$. 

As in \citep{Li2013IJRR}, we handle the IMU-camera time offset $t_d$ when we clone the ``true'' IMU pose corresponding to image measurements. 
For example, if we clone the current IMU pose $\{{}^{I_k}_G\bar{q}, {}^G\mathbf{p}_{I_k}\}$ into the sliding window as $\{{}^{I_{ck}}_G\bar{q}, {}^G\mathbf{p}_{I_{ck}}\}$ using: 
\begin{align}
    {}^G_{I_{ck}}\mathbf{R} & \simeq {}^G_{I_{k}}\mathbf{R}\exp({}^{I_k}\hat{\boldsymbol{\omega}} \tilde{t}_d)
    \\
    {}^G\mathbf{p}_{I_{ck}} & \simeq {}^G\mathbf{p}_{I_k} + {}^G\mathbf{v}_{I_k} \tilde{t}_d
\end{align}
with the linearized clone Jacobians as:
\begin{align}
    \begin{bmatrix}
    \delta \boldsymbol{\theta}_{I_{ck}} \\
    {}^G\tilde{\mathbf{p}}_{I_{ck}}
    \end{bmatrix} & \simeq
    \begin{bmatrix}
    \mathbf{I}_3 & \mathbf{0}_3  & {}^{I_k}\hat{\boldsymbol{\omega}} \\
    \mathbf{0}_3 & \mathbf{I}_3 & {}^G\hat{\mathbf{v}}_{I_k}
    \end{bmatrix}
    \begin{bmatrix}
    \delta \boldsymbol{\theta}_{I_k} \\
    {}^G\tilde{\mathbf{p}}_{I_{k}} \\
    \tilde{t}_d
    \end{bmatrix}
\end{align}
Both $\mathbf{x}_{in}$ and $t_d$ will be updated through correlations when visual feature measurements are present. 

Features are processed in two different ways: short features update the state through the MSCKF nullspace operation \citep{Mourikis2007ICRA, Yang2017IROS}, and long-tracked features are initialized into the state vector and refined over time for improved accuracy \citep{Li2013RSS}.
We utilize first-estimates Jacobians (FEJ) \citep{Huang2010IJRR,Li2013IJRR} to preserve the system unobservable subspace and improve the estimator consistency. 
We directly model the camera intrinsic and IMU-camera spatial calibration through the visual measurement functions [see Eq. \eqref{eq:visual_full_meas}] and update them in the filter with Jacobians in Eq. \eqref{eq:Hc}.

For the RS cameras, the feature measurements from different image rows are captured at different timestamps. 
This indicates that we cannot directly find a cloned pose in the sliding window for $\{{}^G_{I(t)}\mathbf{R}, {}^G\mathbf{p}_{I(t)}\}$ shown in Eq. \eqref{eq:camera_ht(t)}.
Therefore, for the readout time calibration, we model the feature measurement affected by RS effects through pose interpolation \citep{Guo2014RSS,Eckenhoff2021TRO}. 
For example, if the feature measurement is in the $m$-th row with total $M$ rows in an image, we can find two bounding clones $ci-1$ and $ci$ based on the measurement time $t$.  
Hence, the corresponding time $t$ is between two clones  within the sliding window, that is: $t_{ck-n}\leq t_{ci-1} \leq t \leq t_{ci} \leq t_{ck}$. 
We can then find the \textit{virtual} IMU pose $\{{}^G_{I(t)}\mathbf{R}, {}^G\mathbf{p}_{I(t)}\}$ between clones $ci-1$ and $ci$ with: 
\begin{align}
    \label{eq:lam}
    \lambda & = (t_I + \frac{m}{M}t_r - t_{ci-1})/ (t_{ci} - t_{ci-1}) \\
    \label{eq:inter_rot}
    {}^G_{I(t)}\mathbf{R} & = {}^G_{I_{ci-1}} \mathbf{R} 
    \exp \left(
    \lambda \log
    \left(
    {}^G_{I_{ci-1}} \mathbf{R}^{\top}{}^G_{I_{ci}} \mathbf{R}
    \right)
    \right)
    \\
    \label{eq:inter_p}
    {}^G\mathbf{p}_{I(t)} & = (1-\lambda){}^G\mathbf{p}_{I_{ci-1}} + \lambda {}^G\mathbf{p}_{I_{ci}}
\end{align}

To summarize, feature measurements which occur at different rows of the image can be related to the state vector defined in Eq. \eqref{eq:xk} through the above linear pose interpolation.
This measurement function can then be linearized for use in the EKF update (see Appendix \ref{apd:inter jacob}).
Note that a higher-order polynomial pose interpolation as used by \cite{Eckenhoff2021TRO} can be utilized if necessary.
In the above derivations, although we assume the image timestamp refers to the timestamp of the first image row, the interpolation scalar in Eq. \eqref{eq:lam} can be easily modified to account for situations when the image timestamp of the RS camera refers to the middle or last image row.

\begin{figure*}
\centering
\begin{subfigure}{.48\textwidth}
\includegraphics[trim=0mm 0 0mm 0,clip,width=\linewidth]{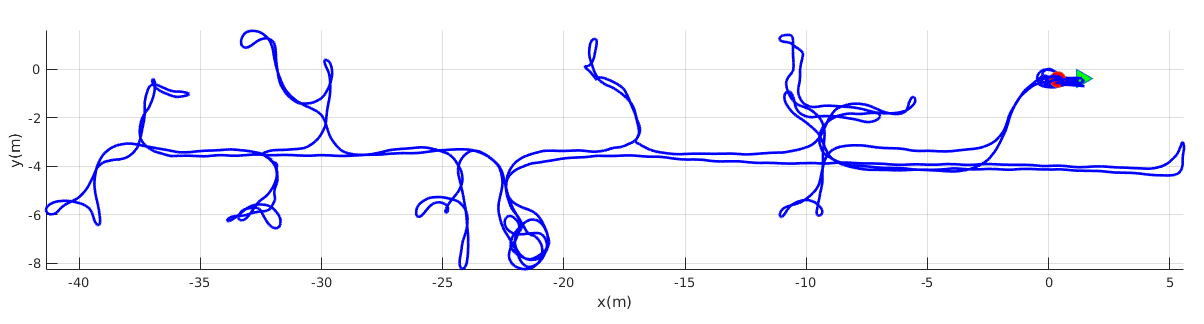}
\end{subfigure}
\begin{subfigure}{.17\textwidth}
\includegraphics[trim=0 0 0 0,clip,width=\linewidth]{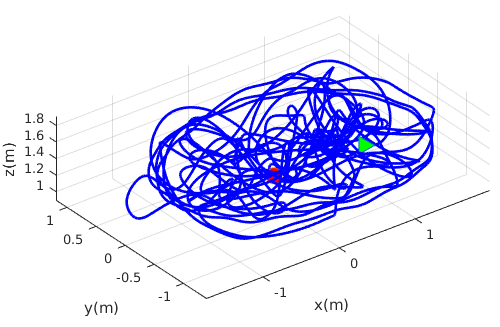}
\end{subfigure}
\begin{subfigure}{.16\textwidth}
\includegraphics[trim=0 0 0 0,clip,width=\linewidth]{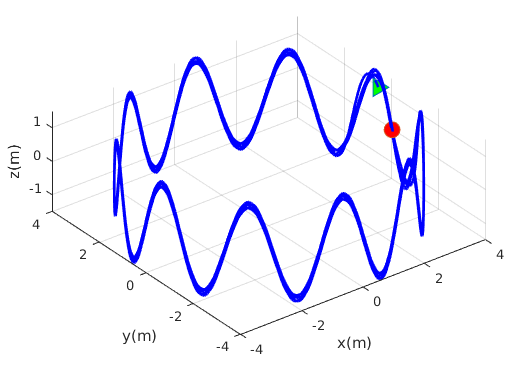}
\end{subfigure}
\begin{subfigure}{.16\textwidth}
\includegraphics[trim=0 0 0 0,clip,width=\linewidth]{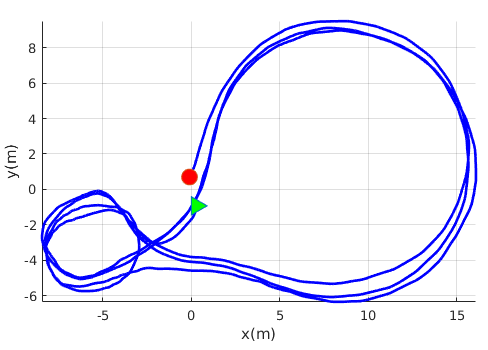}
\end{subfigure}
\caption{
Simulated trajectories for Monte-Carlo simulations.
Left: \textit{tum\_corridor} with fully excited 3D motion;
Middle left: \textit{tum\_room} with 1 axis rotation and 3D translation;
Middle right: \textit{sine\_3d} with constant acceleration along $x$ direction;
Right: \textit{udel\_gore} planar motion with constant $z$ and only yaw rotation.  
The green triangle and red circle denote the beginning and ending of these trajectories, respectively. 
}
\label{fig:sim_traj}
\end{figure*}
\begin{table}
\centering
\caption{
Simulation parameters and prior standard deviations that perturbations of measurements and initial states were drawn from.
}
\label{tab:sim_params}
\begin{adjustbox}{width=\columnwidth,center}
\begin{tabular}{ccccc} \toprule
\textbf{Parameter} & \textbf{Value} & \textbf{Parameter} & \textbf{Value} \\ \midrule
IMU Scale & 0.003 & IMU Skew & 0.003 \\
Rot. atoI (rad) & 0.003 & Rot. wtoI (rad) & 0.003 \\
Gyro. White Noise & 1.6968e-04 & Gyro. Rand. Walk & 1.9393e-05 \\
Accel. White Noise & 2.0000e-3 & Accel. Rand. Walk & 3.0000e-3 \\
Focal Len. (px/m) & 0.50 & Cam. Center (px) & 0.60 \\
d1 and d2 & 0.008 & d3 and d4  & 0.002 \\
Rot. CtoI (rad) & 0.004 & Pos. IinC (m) & 0.010 \\
Readout Time (ms) & 0.5 & Timeoff (s) & 0.005 \\
Cam Freq. (hz) & 20 & IMU Freq. (hz) & 400 \\
Avg. Feats & 100 & Num. SLAM & 50 \\
Num. Clones & 20 & Feat. Rep. & GLOBAL \\\bottomrule
\end{tabular}
\end{adjustbox}
\end{table}

\begin{figure*}
\centering
\begin{subfigure}{.245\textwidth}
\includegraphics[trim=0 9mm 0 0,clip,width=\linewidth]{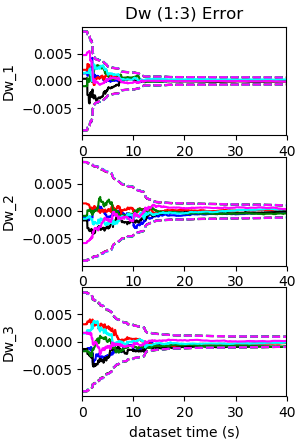}
\end{subfigure}
\begin{subfigure}{.245\textwidth}
\includegraphics[trim=0 9mm 0 0,clip,width=\linewidth]{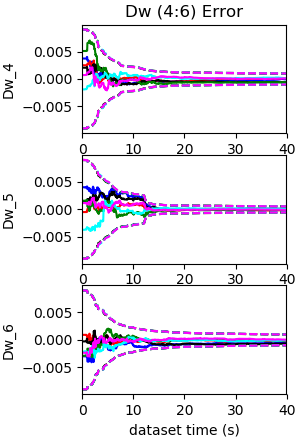}
\end{subfigure}
\begin{subfigure}{.245\textwidth}
\includegraphics[trim=0 9mm 0 0,clip,width=\linewidth]{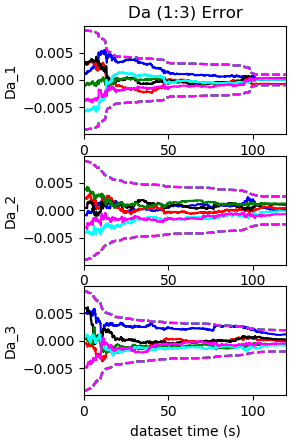}
\end{subfigure}
\begin{subfigure}{.245\textwidth}
\includegraphics[trim=0 9mm 0 0,clip,width=\linewidth]{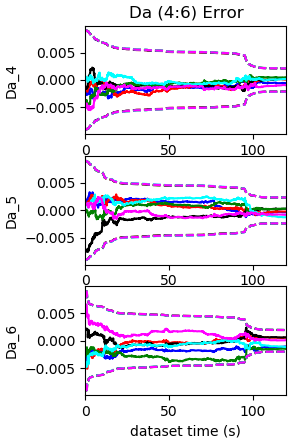}
\end{subfigure}
\begin{subfigure}{.245\textwidth}
\includegraphics[trim=0 9mm 0 0,clip,width=\linewidth]{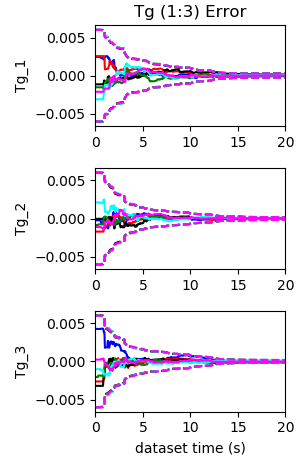}
\end{subfigure}
\begin{subfigure}{.245\textwidth}
\includegraphics[trim=0 9mm 0 0,clip,width=\linewidth]{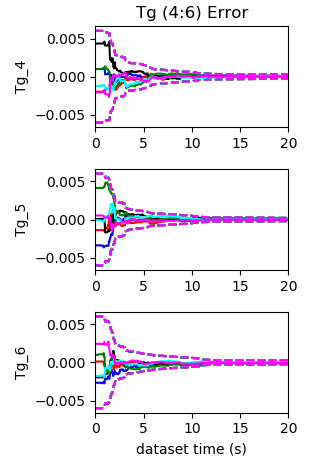}
\end{subfigure}
\begin{subfigure}{.245\textwidth}
\includegraphics[trim=0 9mm 0 0,clip,width=\linewidth]{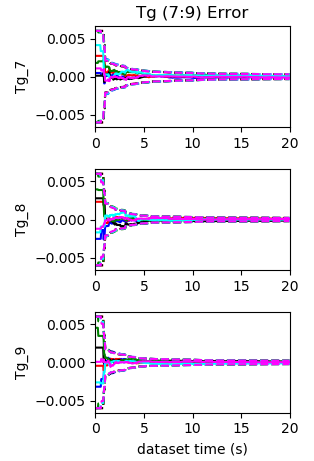}
\end{subfigure}
\begin{subfigure}{.245\textwidth}
\includegraphics[trim=0 9mm 0 0,clip,width=\linewidth]{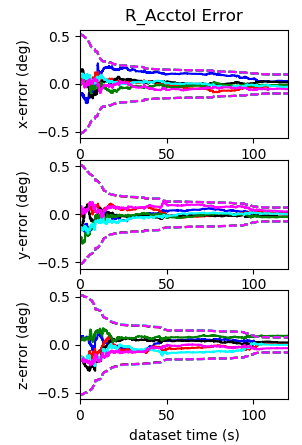}
\end{subfigure}
\begin{subfigure}{.245\textwidth}
\includegraphics[trim=0 0 0 0,clip,width=\linewidth]{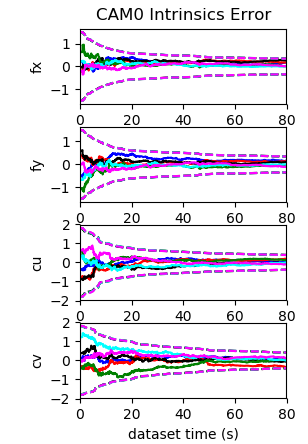}
\end{subfigure}
\begin{subfigure}{.245\textwidth}
\includegraphics[trim=0 0 0 0,clip,width=\linewidth]{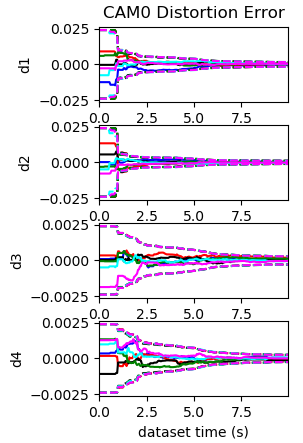}
\end{subfigure}
\begin{subfigure}{.245\textwidth}
\includegraphics[trim=0 0 0 0,clip,width=\linewidth]{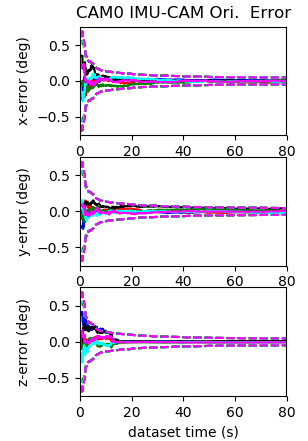}
\end{subfigure}
\begin{subfigure}{.245\textwidth}
\includegraphics[trim=0 0 0 0,clip,width=\linewidth]{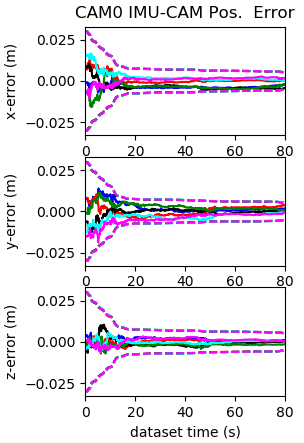}
\end{subfigure}
\caption{
Calibration results for proposed system evaluated on \textit{tum\_corridor} with fully excited motion (using \textit{imu22} and \textit{radtan}). 
$3\sigma$ bounds (dotted lines) and estimation errors (solid lines) for six different runs (different colors) with different realization of the measurement noise and initial perturbations.
All the calibration parameters converge nicely.
}
\label{fig:sim_full}
\end{figure*}

\begin{figure*}
\centering
\includegraphics[trim=0 0 0 0,clip,width=\textwidth]{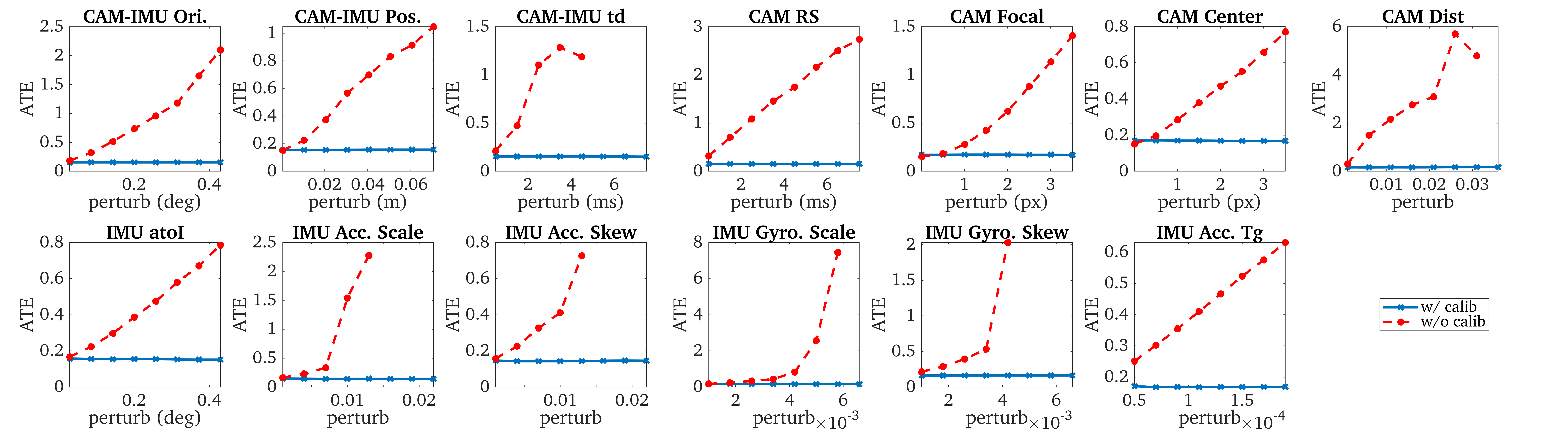}
\caption{
Absolute trajectory errors (ATE) using \textit{imu22} and \textit{radtan} on the \textit{tum\_corridor} with full 3D motion given different levels of perturbation.
Each parameter calibrated was only perturbed while other parameters were initialized to their true values and not estimated.
ATE above eight meters were not reported in the figures and can be considered diverged.
}
\label{fig:sim_perturb_errors}
\end{figure*}

\begin{table*}
\centering
\caption{
Average absolute trajectory error (ATE) and normalized estimation errror squared (NEES) over 20 runs of the proposed system evaluated on \textit{tum\_corridor} with true or perturbed calibration parameters, with and without online calibration.
\textit{radtan} camera distortion model and different IMU intrinsic models are used. 
The notation ``true'' means the groundtruth calibration, while ``perturbed'' means the perturbed calibration states.
Failures are denoted with ``-''. 
}
\label{tab:sim_calibcompare}
\begin{adjustbox}{width=\textwidth,center}
\begin{tabular}{ccccc|ccccc} \toprule
\textbf{IMU Model} & \textbf{ATE (deg)} & \textbf{ATE (m)} & \textbf{Ori. NEES} & \textbf{Pos. NEES} & \textbf{IMU Model} & \textbf{ATE (deg)} & \textbf{ATE (m)} & \textbf{Ori. NEES} & \textbf{Pos. NEES} \\ \midrule
true w/ calib \textit{imu1}   & 0.462     & 0.164   & 1.910     & 1.423     & perturbed w/ calib \textit{imu1}   & 0.454     & 0.163   & 2.173     & 1.473     \\
true w/ calib \textit{imu2}   & 0.460     & 0.164   & 2.103     & 1.422     & perturbed w/ calib \textit{imu2}   & 0.446     & 0.162   & 2.150     & 1.465     \\
true w/ calib \textit{imu3}   & 0.461     & 0.163   & 1.883     & 1.422     & perturbed w/ calib \textit{imu3}   & 0.459     & 0.163   & 2.125     & 1.454     \\
true w/ calib \textit{imu4}   & 0.458     & 0.163   & 2.102     & 1.424     & perturbed w/ calib \textit{imu4}   & 0.450     & 0.162   & 2.150     & 1.456     \\ \midrule
true w/ calib \textit{imu11}  & 0.544     & 0.177   & 1.947     & 1.472     & perturbed w/ calib \textit{imu11}  & 0.550     & 0.178   & 2.243     & 1.498     \\
true w/ calib \textit{imu12}  & 0.540     & 0.177   & 2.123     & 1.476     & perturbed w/ calib \textit{imu12}  & 0.544     & 0.178   & 2.175     & 1.493     \\
true w/ calib \textit{imu13}  & 0.544     & 0.176   & 1.914     & 1.472     & perturbed w/ calib \textit{imu13}  & 0.546     & 0.177   & 2.180     & 1.482     \\
true w/ calib \textit{imu14}  & 0.544     & 0.179   & 2.124     & 1.483     & perturbed w/ calib \textit{imu14}  & 0.538     & 0.177   & 2.169     & 1.504     \\ \midrule
true w/ calib \textit{imu21}  & 0.572     & 0.183   & 1.990     & 1.514     & perturbed w/ calib \textit{imu21}  & 0.576     & 0.182   & 2.250     & 1.508     \\
true w/ calib \textit{imu22}  & 0.567     & 0.184   & 2.145     & 1.513     & perturbed w/ calib \textit{imu22}   & 0.590     & 0.187   & 2.194     & 1.561     \\
true w/ calib \textit{imu23}  & 0.571     & 0.183   & 1.962     & 1.514     & perturbed w/ calib \textit{imu23}  & 0.593     & 0.185   & 2.200     & 1.550     \\
true w/ calib \textit{imu24}  & 0.566     & 0.183   & 2.141     & 1.512     & perturbed w/ calib \textit{imu24}  & 0.585     & 0.186   & 2.189     & 1.552     \\  \midrule
true w/ calib \textit{imu31}  & 0.447     & 0.161   & 2.076     & 1.378     & perturbed w/ calib \textit{imu31}  & 0.451     & 0.162   & 2.110     & 1.428     \\
true w/ calib \textit{imu32}  & 0.444     & 0.161   & 1.879     & 1.396     & perturbed w/ calib \textit{imu32}  & 0.447     & 0.162   & 1.968     & 1.430     \\
true w/ calib \textit{imu33}  & 0.529     & 0.175   & 2.096     & 1.378     & perturbed w/ calib \textit{imu33}  & 0.527     & 0.177   & 2.101     & 1.411     \\
true w/ calib \textit{imu34}  & 0.548     & 0.180   & 2.103     & 1.430     & perturbed w/ calib \textit{imu34}  & 0.549     & 0.179   & 2.113     & 1.422     \\ \midrule
true w/ calib \textit{imu6} & 0.572     & 0.183   & 1.734     & 1.517     & perturbed w/ calib \textit{imu6} & 0.567     & 0.179   & 1.910     & 1.491     \\ \midrule
true w/o calib                & 0.433     & 0.159   & 2.069     & 1.332     & perturbed w/o calib                     & -     & -   & -     & -                     \\ \bottomrule
\end{tabular}
\end{adjustbox}
\end{table*}

\begin{figure}
\centering
\begin{subfigure}{.24\textwidth}
\includegraphics[trim=0 0mm 0mm 0,clip,width=\linewidth]{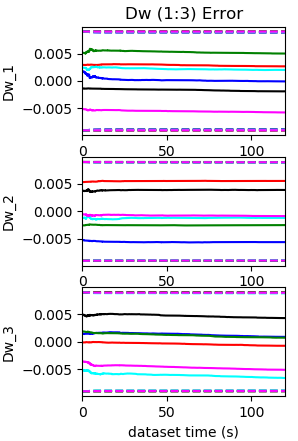}
\end{subfigure}
\begin{subfigure}{.24\textwidth}
\includegraphics[trim=0 0 0mm 0,clip,width=\linewidth]{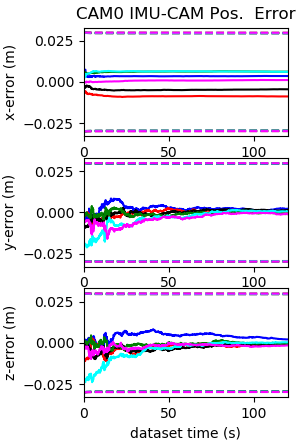}
\end{subfigure}
\caption{
Calibration results for the proposed system evaluated on \textit{tum\_room} with one-axis rotation using \textit{imu22} and \textit{radtan}.
Note that the estimation errors and $3\sigma$ bounds for $d_{w1}$, $d_{w2}$, $d_{w3}$ and the IMU-camera position calibration along the rotation axis can not converge.
}
\label{fig:sim_1axis}
\end{figure}
\begin{figure*}
\centering
\begin{subfigure}{.28\textwidth}
\includegraphics[trim=0 0mm 0mm 0,clip,width=\linewidth]{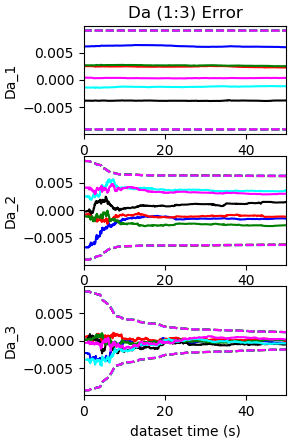}
\end{subfigure}
\begin{subfigure}{.28\textwidth}
\includegraphics[trim=0 0mm 0mm 0,clip,width=\linewidth]{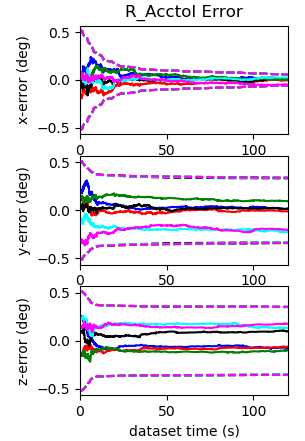}
\end{subfigure}
\begin{subfigure}{.28\textwidth}
\includegraphics[trim=0 0mm 0mm 0,clip,width=\linewidth]{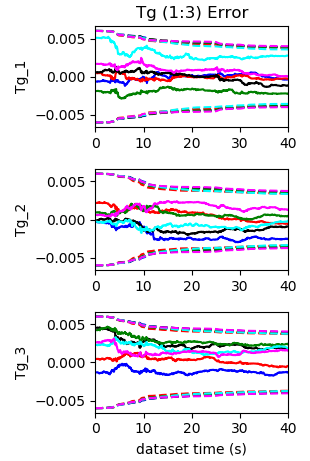}
\end{subfigure}
\caption{
Calibration results for the proposed system evaluated the \textit{sine\_3d} with constant acceleration along x direction using \textit{imu22} and \textit{radtan}. 
The estimation errors and 3$\sigma$ bounds for $d_{a1}$, pitch and yaw of ${}^I_a\mathbf{R}$ cannot converge. 
Note that $t_{g1}$, $t_{g2}$ and $t_{g3}$ are also unobservable. 
}
\label{fig:sim_ax}
\end{figure*}
\begin{figure*}
\centering
\begin{subfigure}{.28\textwidth}
\includegraphics[trim=0 9mm 0mm 0,clip,width=\linewidth]{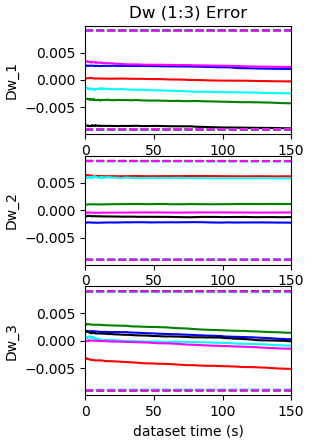}
\end{subfigure}
\begin{subfigure}{.28\textwidth}
\includegraphics[trim=0 9mm 0mm 0,clip,width=\linewidth]{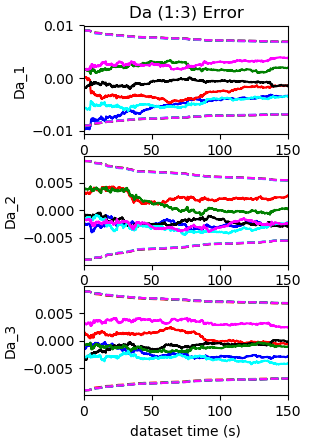}
\end{subfigure}
\begin{subfigure}{.28\textwidth}
\includegraphics[trim=0 9mm 0mm 0,clip,width=\linewidth]{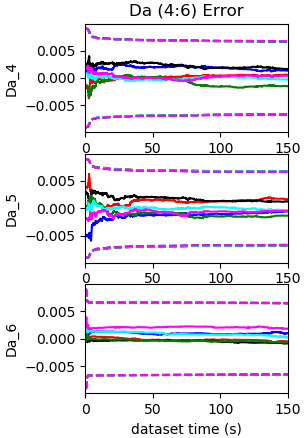}
\end{subfigure}
\begin{subfigure}{.28\textwidth}
\includegraphics[trim=0 0mm 0mm 0,clip,width=\linewidth]{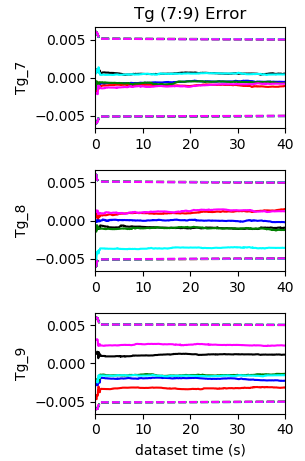}
\end{subfigure}
\begin{subfigure}{.28\textwidth}
\includegraphics[trim=0 0mm 0mm 0,clip,width=\linewidth]{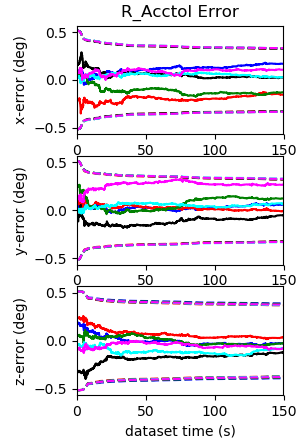}
\end{subfigure}
\begin{subfigure}{.28\textwidth}
\includegraphics[trim=0 0 0mm 0,clip,width=\linewidth]{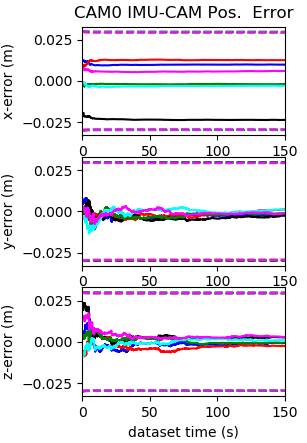}
\end{subfigure}
\caption{
Calibration results of the proposed system evaluated on \textit{udel\_gore} with planar motion using \textit{imu22} and \textit{radtan}. 
With planar motion, the estimation errors and 3 $\sigma$ bounds of $d_{w1}$, $d_{w2}$, $d_{w3}$, $t_{g7}$, $t_{g8}$, $t_{g9}$ and the IMU-camera position cannot converge. 
Due to lack of motion excitation, the parameters of $\mathbf{D}_a$ and ${}^I_a\mathbf{R}$ converge much slower than the other motion cases. 
}
\label{fig:sim_planar}
\vspace*{-6pt}
\end{figure*}
\begin{figure*}
\centering
\begin{subfigure}{.245\textwidth}
\includegraphics[trim=0 0 0 0,clip,width=\linewidth]{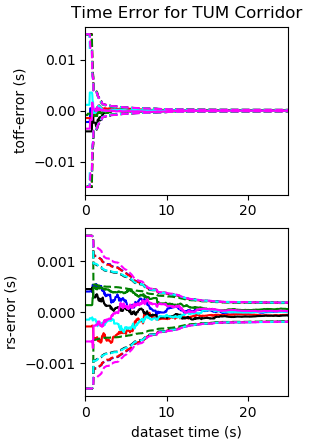}
\end{subfigure}
\begin{subfigure}{.245\textwidth}
\includegraphics[trim=0 0 0 0,clip,width=\linewidth]{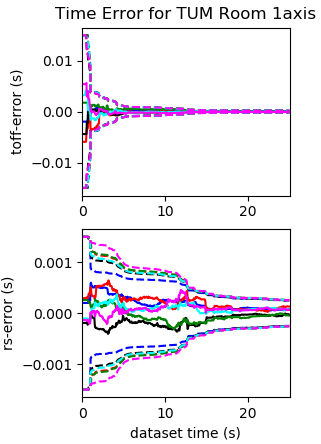}
\end{subfigure}
\begin{subfigure}{.242\textwidth}
\includegraphics[trim=0 0 0 0,clip,width=\linewidth]{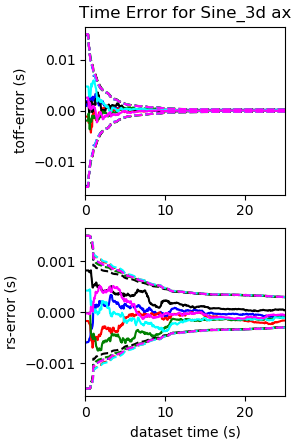}
\end{subfigure}
\begin{subfigure}{.248\textwidth}
\includegraphics[trim=0 0 0 0,clip,width=\linewidth]{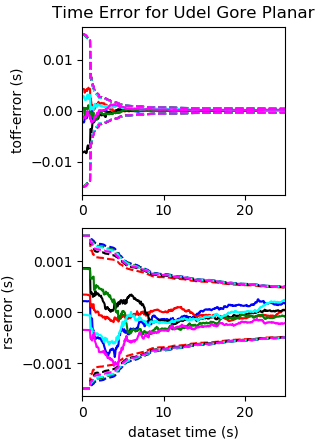}
\end{subfigure}
\caption{
Camera temporal and read out time calibration results of the proposed system  (using \textit{imu22} and \textit{radtan}) for different trajectories. 
The temporal parameters can finally converge in all the 4 motion types.
Note that the readout time converges slower in the planar motion case probably due to the lack of motion in the beginning.
}
\label{fig:sim_time}
\end{figure*}

\begin{figure}
\centering
\begin{subfigure}{.24\textwidth}
\includegraphics[trim=0 0 0 0,clip,width=\linewidth]{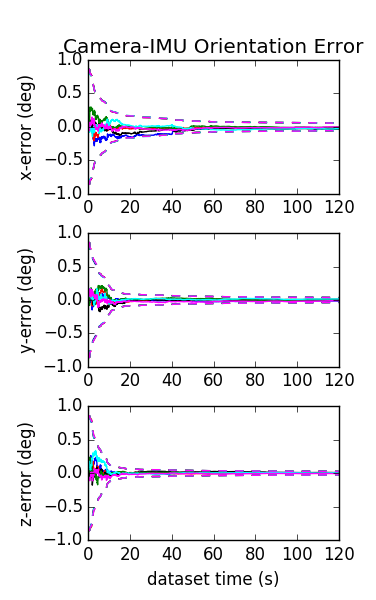}
\end{subfigure}
\begin{subfigure}{.24\textwidth}
\includegraphics[trim=0 0 0 0,clip,width=\linewidth]{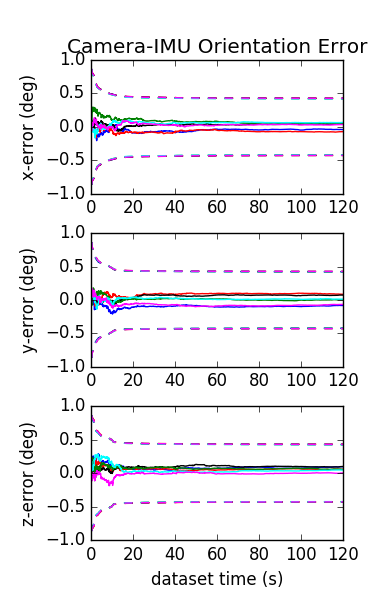}
\end{subfigure}
\caption{
Camera to IMU orientation errors when using IMU \textit{imu2} (left) and the over paramterized \textit{imu5} (right).
Note that only the IMU intrinsics and relative pose between IMU and camera were online calibrated.
}
\label{fig:sim_sine3d_overparam}
\end{figure}

\section{Simulation Validations} \label{sec:exp_sim}

The proposed estimator is implemented within the OpenVINS framework \citep{Geneva2020ICRA},
which contains a real-time modular sliding window EKF-based filter and simulator.
The project can already handle IMU-camera spatial-temporal and camera intrinsic calibrations, excluding RS readout time.
In this work we extend the estimator to address the calibration of all the parameters presented.
The new VINS estimator maintains the original faster than real-time performance in both simulated and real-world datasets.

The core of the simulator has a continuous-time ${SE}(3)$ B-spline trajectory representation which allows for the calculation of pose, velocity, and accelerations at any given timestamp along the trajectory.
The true angular velocity and linear accelerations can be directly found and corrupted using the random walk biases and white noises.
The basic configurations for our simulator are listed in Table \ref{tab:sim_params}.
All simulation convergence figures show the $3\sigma$ bounds (dotted lines) and estimation errors (solid lines) for six different runs (different colors) with different realization of the measurement noises and initial calibration state perturbations.

To simulate RS visual bearing measurements, we follow the logic presented by \cite{Li2013IJRR} and \cite{Eckenhoff2021TRO}.  
Specifically, static environmental features are first generated along the length of the trajectory at random depths and bearings.
Then, for a given imaging time of features in view, we project each into the current image frame using the true camera intrinsic and distortion model to find the corresponding observation row.
Given this projected row and image time, we can find the pose at which this row should have been (i.e., the pose at which that RS row should have been exposed).
We can then re-project this feature into the new pose and iterate until the projected row does not change (which typically requires 2-3 iterations).
We now have a feature measurement which occurred at the correct pose for its given RS row.
This measurement is then corrupted with white noise.
The imaging timestamp corresponding to the starting row is then shifted by the true IMU-camera time offset $t_d$ to simulate cross-sensor delay.

\subsection{Simulation with Fully-Excited Motion} \label{sec:sim fully excited}

We first perform a general trajectory simulation, for which we perform full calibration of the IMU-camera extrinsics, time offset, RS readout time, camera intrinsics with \textit{radtan} model and IMU intrinsics with \textit{imu22}. 
The trajectory, shown in the left of Figure \ref{fig:sim_traj}, is designed based on \textit{tum\_corridor} sequence of TUM visual-inertial dataset with full excitation of all 6 axes and provides a realistic 3D hand-held motion \citep{schubert2018tum}.
From the results shown in Figure \ref{fig:sim_full} and \ref{fig:sim_time}, the estimation errors and $3\sigma$ bounds for all the calibration parameters (including \textit{imu22} and \textit{radtan}) can converge quite nicely, verifying that the analysis for general motions holds true.
We plot results from six different realizations of the initial calibration guesses based on the specified priors, and it is clear that the estimates for all these calibration parameters are able to converge from different initial guesses to near the ground truth.
Each parameter is able to ``gain'' information since their $3\sigma$ bounds shrink.
These results verify our Lemma \ref{lem:obs} that all these online calibration parameters are observable given a fully-excited motion.

\subsection{Sensitivities to Perturbations} \label{sec:sim_perturb}

The next natural question is how robust the system is to the initial perturbations and whether the use of online sensor calibration enables improvements in robustness and accuracy.
Shown in Figure \ref{fig:sim_perturb_errors}, for each of the different calibration parameters we perturb it with different levels of noise on the \textit{tum\_corridor} trajectory (note that we also change the initial prior provided to the filter as the initial prior changes).
We can see that the proposed estimator is relatively invariant to the initial inaccuracies of the parameters and is, in general, able to output a near constant trajectory error.
A filter, which does not perform this online estimation, has its trajectory estimation error quickly increase to non-usable levels.
It is interesting to see that even small levels of perturbations can cause huge trajectory errors which further verifies the motivations to perform online calibration.

\subsection{Comparison of Inertial Model Variants}

We next compare the different proposed inertial model variants.
We estimate all calibration parameters and perturb them based on Table \ref{tab:sim_params}.
Shown in Table \ref{tab:sim_calibcompare}, it is clear that the choice between the variants has little impact on estimation accuracy which indicates they provide almost the same amount of correction to the inertial readings.

The accuracy of the standard VIO system which does not calibrate any parameters online and uses the groundtruth calibration values, denoted \textit{true w/o calib}, has the best accuracy due to the use of the true parameters.
If we do perturb the initial calibration and do not estimate it, denoted \textit{perturbed w/o calib}, the system quickly becomes unstable and diverges unless smaller levels of perturbations are used.
Note that the results presented in Section \ref{sec:sim_perturb} are for each parameter perturbed individually, while here all calibration parameters are perturbed at once thus resulting in divergence.
With full online calibration, the system can still output stable and consistent trajectories with only a small loss in estimation accuracy given perturbed calibration values.

\subsection{Degenerate Motion Verification}

We now verify the identified degenerate motions and present results for three special motions.
In all simulations, we perform full-parameter calibration to fully test our system and present the complete results in Appendix \ref{sec:more_sim_results}. 
The trajectories shown in Figure \ref{fig:sim_traj} are created as follows:
\begin{itemize}
    \item One-axis rotation with a modified \textit{tum\_room} trajectory, see middle left, which has its roll and pitch orientation changes removed to create a yaw and 3D translation only dataset.
    \item Constant local $a_x$ with modified \textit{sine\_3d}, see middle right, for which we have a constant pitch and make the current yaw angle tangent to the trajectory in the x-y plane (gives constant local acceleration along local x-axis).
    \item Planar motion with modified \textit{udel\_gore}, see right, which has its roll and pitch orientation removed and all poses are projected to x-y plane (planar motion in the global x-y plane). 
\end{itemize}

\subsubsection{One-axis rotation motion.}
Shown in Figure \ref{fig:sim_1axis}, the first 3 parameters ($d_{w1}$, $d_{w2}$ and $d_{w3}$) for $\mathbf{D}_w$ do not converge at all (the 3$\sigma$ bounds are almost straight lines), which matches our analysis, see Table \ref{table:degenerate summary}. These parameters should be unobservable in the case of one-axis rotation with ${}^ww_x$ (roll) and ${}^ww_y$ (pitch) are constant.
Additionally, the translation between IMU and camera does not converge either. 
The x-error of the IMU-camera translation even diverges reinforcing the undesirability of degenerate motions and verifies the analysis presented in Table \ref{tab:cam degenerate}.

\subsubsection{Constant local acceleration motion.}
The results shown in Figure \ref{fig:sim_ax}, where we have enforced that the local acceleration along the x-axis, $a_x$, is constant. The $d_{a1}$, and pitch and yaw of ${}^I_a\mathbf{R}$ does not converge, thus validating our analysis shown in Table \ref{table:degenerate summary}.
Note that in the simulation, we have set ${}^I_a\mathbf{R} \simeq \mathbf{I}_3$ and $\mathbf{D}_a \simeq \mathbf{I}_3$.
Hence, ${}^a\hat{\mathbf{a}} \simeq {}^I\mathbf{a}$ and ${}^Ia_x$ is also near constant. 
Therefore, three terms of gravity sensitivity ($t_{g1}$, $t_{g2}$ and $t_{g2}$) are also unobservable and converge much slower than other terms.

\subsubsection{Planar motion.}
Shown in Figure \ref{fig:sim_planar}, with one-axis rotation (yaw axis) for planar motion the $d_{w1}$, $d_{w2}$ and $d_{w3}$ for $\mathbf{D}_w$ and the IMU-camera translation are unobservable and does not converge.
Since the ${}^I\mathbf{a}_z$ is constant, the last three terms of gravity sensitivity ($t_{g7}$, $t_{g8}$ and $t_{g9}$) become unobservable and cannot converge.
Both these results verify our analysis shown in Tables \ref{table:degenerate summary} and \ref{tab:cam degenerate}.
Additionally, this trajectory is quite smooth with small excitation of linear acceleration, hence, the terms of $\mathbf{D}_a$ and ${}^I_a\mathbf{R}$ in general converge much slower than the fully excited motion case.

\subsection{Simulated Over Parametrization} \label{sec:overparameterization}

We now look to investigate the impact of poor choice of calibration parameters which \textit{over parameterizes} the IMU intrinsics.
The \textit{imu5} model, see Table \ref{tab:imu model}, is an over parametrization since we calibrate both 9 parameters for gyroscope and accelerometer, which causes the IMU-camera orientation to be affected since the intermediate inertial frame $\{I\}$ is not constrained.
If we change the relative rotation from $\{I\}$ to $\{C\}$, then this perturbed rotation can be absorbed into the $\{a\}$ to $\{I\}$ and $\{w\}$ to $\{I\}$ terms. Thus, it means we have an extra 3 degrees of freedom (DoF) for rotation not constrained by our measurements.
We compare this \textit{imu5} model to its close equivalent \textit{imu2} model in Figure \ref{fig:sim_sine3d_overparam}.
We can see that even though the trajectory fully excites the sensor platform, the convergence of ${}^C_I\mathbf{R}$ becomes much worse if we calibrate IMU-camera extrinsics and all 18 parameters for the IMU model \textit{imu5} even when the \textit{same} priors and measurements are used.
This further motivates the use of the minimal calibration parameters to ensure fast and robust convergence of all state parameters.

\section{Real-World TUM RS VIO Datasets}
\label{sec:exp_tum_rs}

\begin{figure}
\centering
\includegraphics[width=0.95\linewidth]{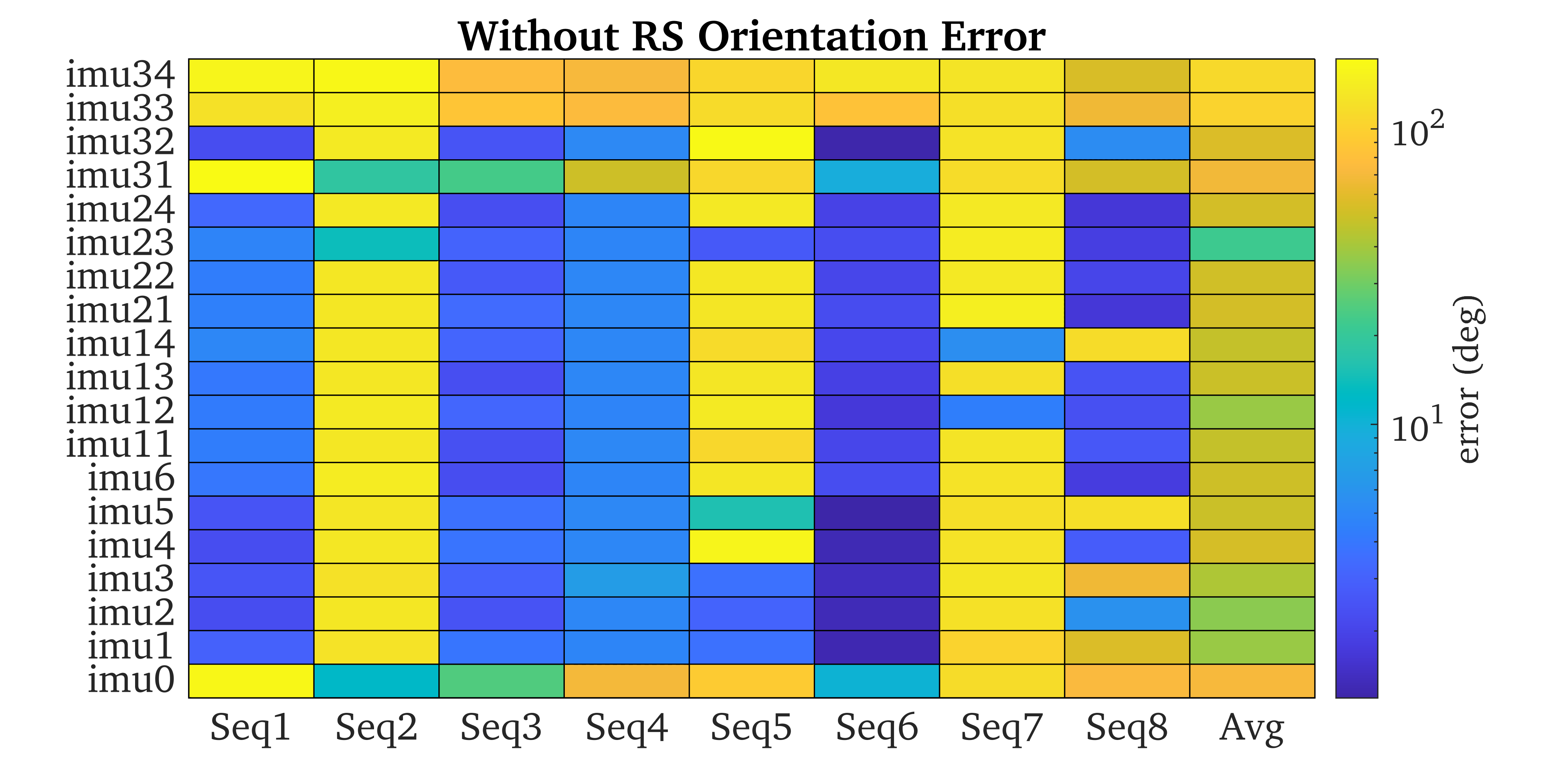}
\includegraphics[width=0.95\linewidth]{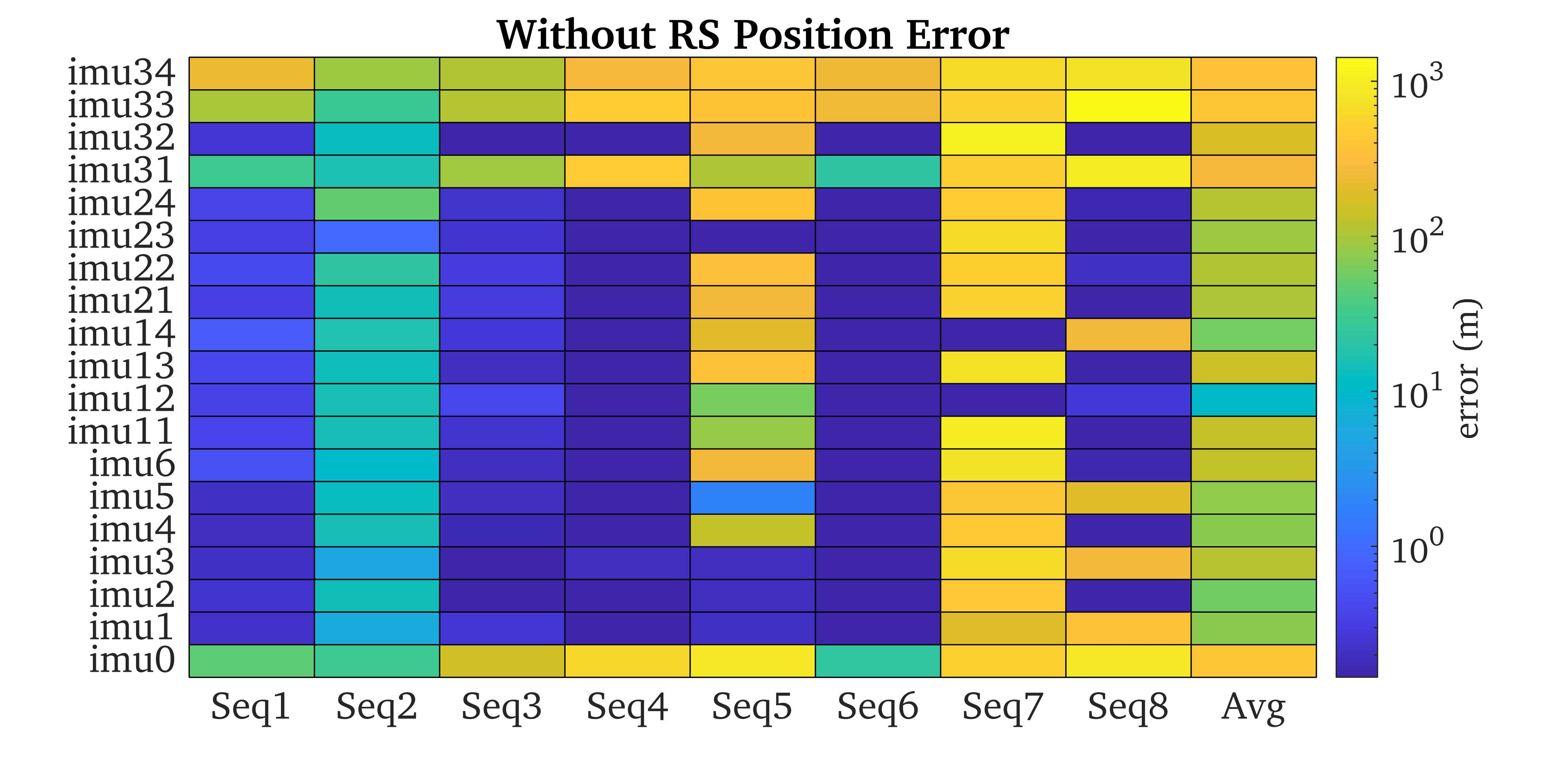}
\caption{
Results on TUM Rolling Shutter VIO Dataset \textit{without} rolling shutter readout time calibration, with different IMU intrinsic models.
The averaged absolute trajectory errors (ATE) of 5 runs in degree (top) and meters (bottom) are provided.
Note that the camera intrinsics, and IMU-camera spatial-temporal calibration. 
}
\label{fig:surf_no_rs}
\end{figure}
\begin{figure}
\centering
\includegraphics[width=0.95\linewidth]{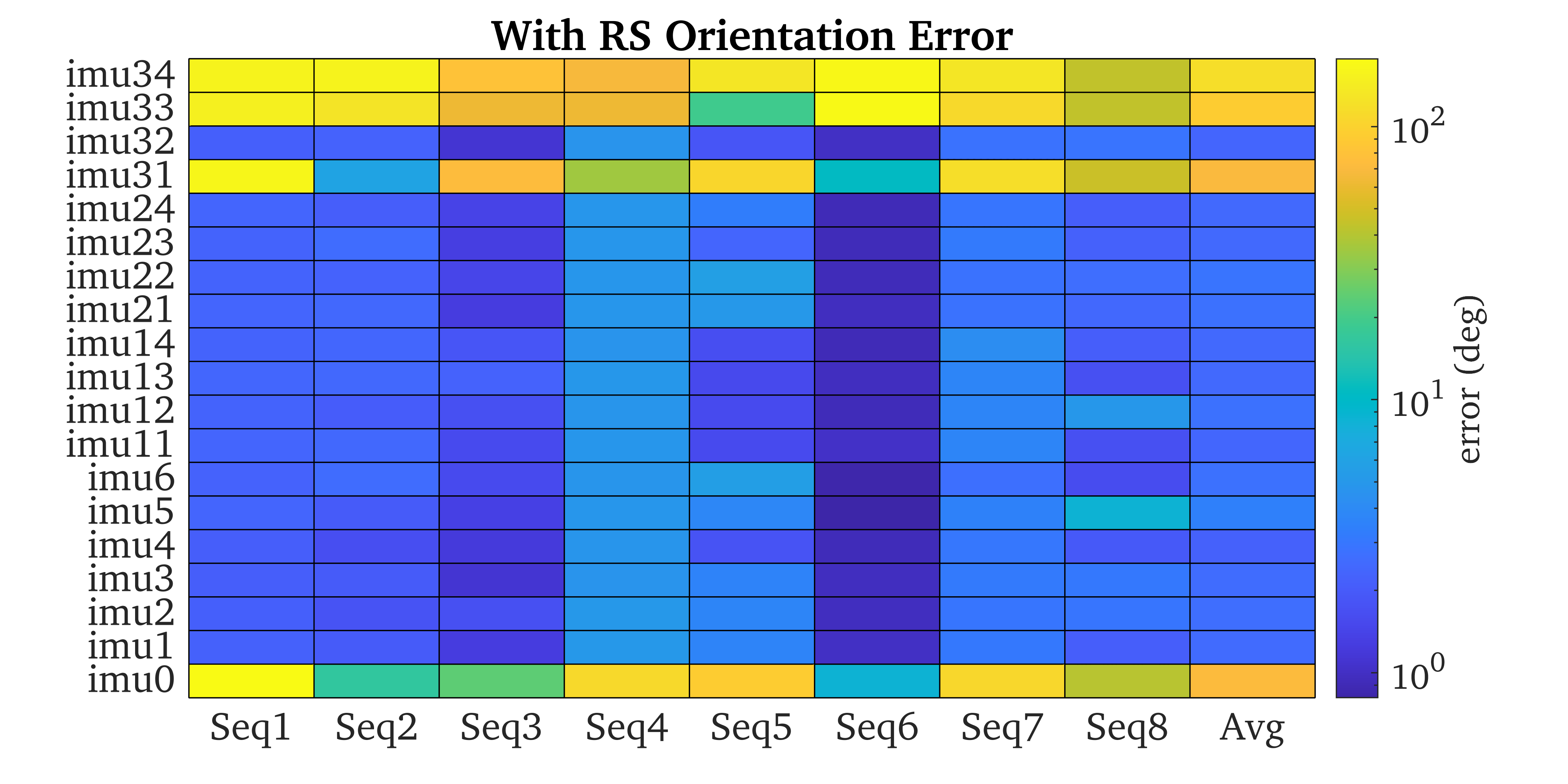}
\includegraphics[width=0.95\linewidth]{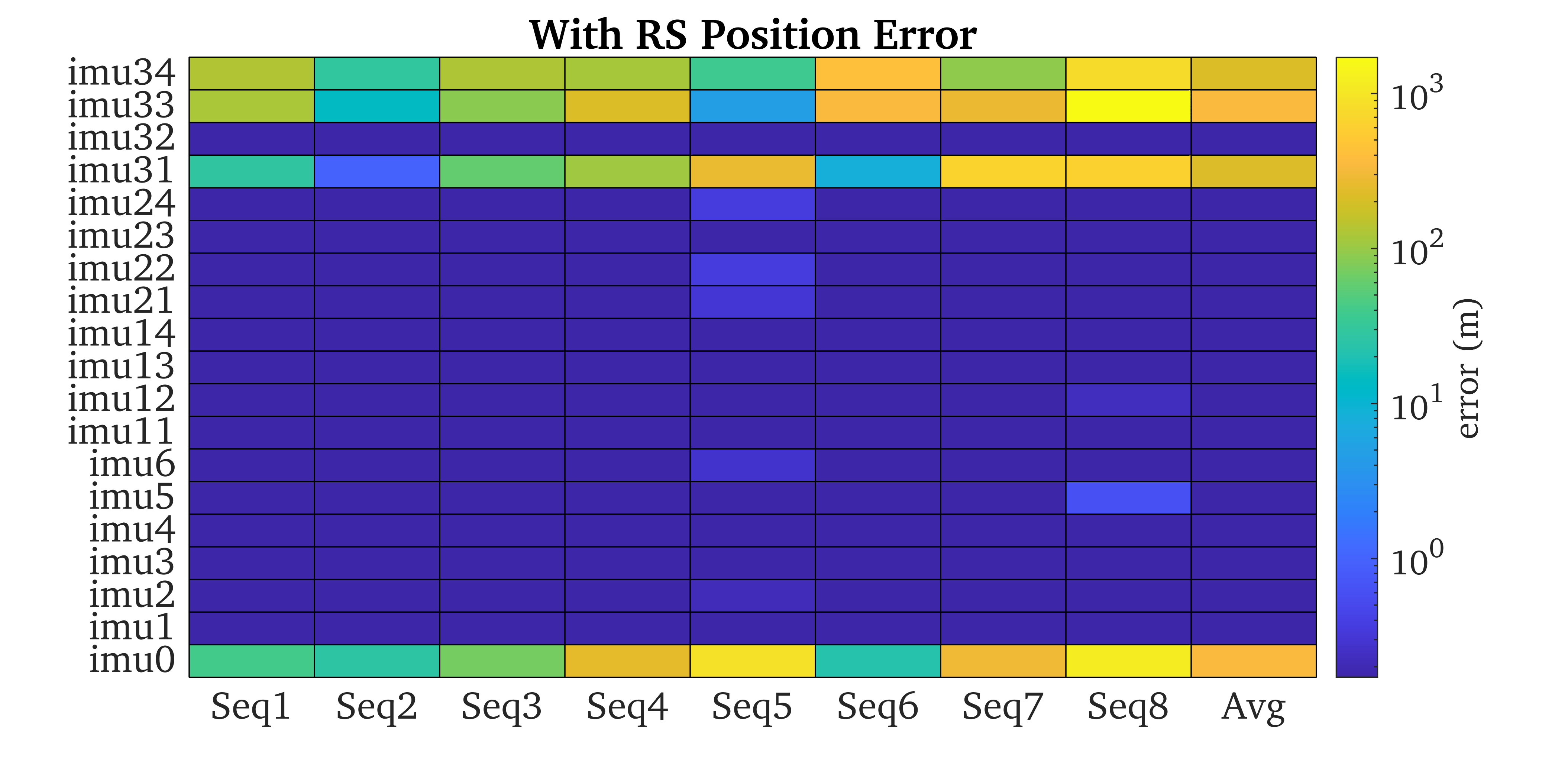}
\caption{
Results on TUM Rolling Shutter VIO Dataset \textit{with} rolling shutter readout time calibration, with different IMU intrinsic models.
The averaged absolute trajectory errors (ATE) of 5 runs in degree (top) and meters (bottom) are provided.
Note that the camera intrinsics, and IMU-camera spatial-temporal calibration. 
}
\label{fig:surf_rs}
\end{figure}

\begin{figure*}
\centering
\begin{subfigure}{.45\textwidth}\centering
\includegraphics[trim=0 9mm 0 0,clip,height=1.3in]{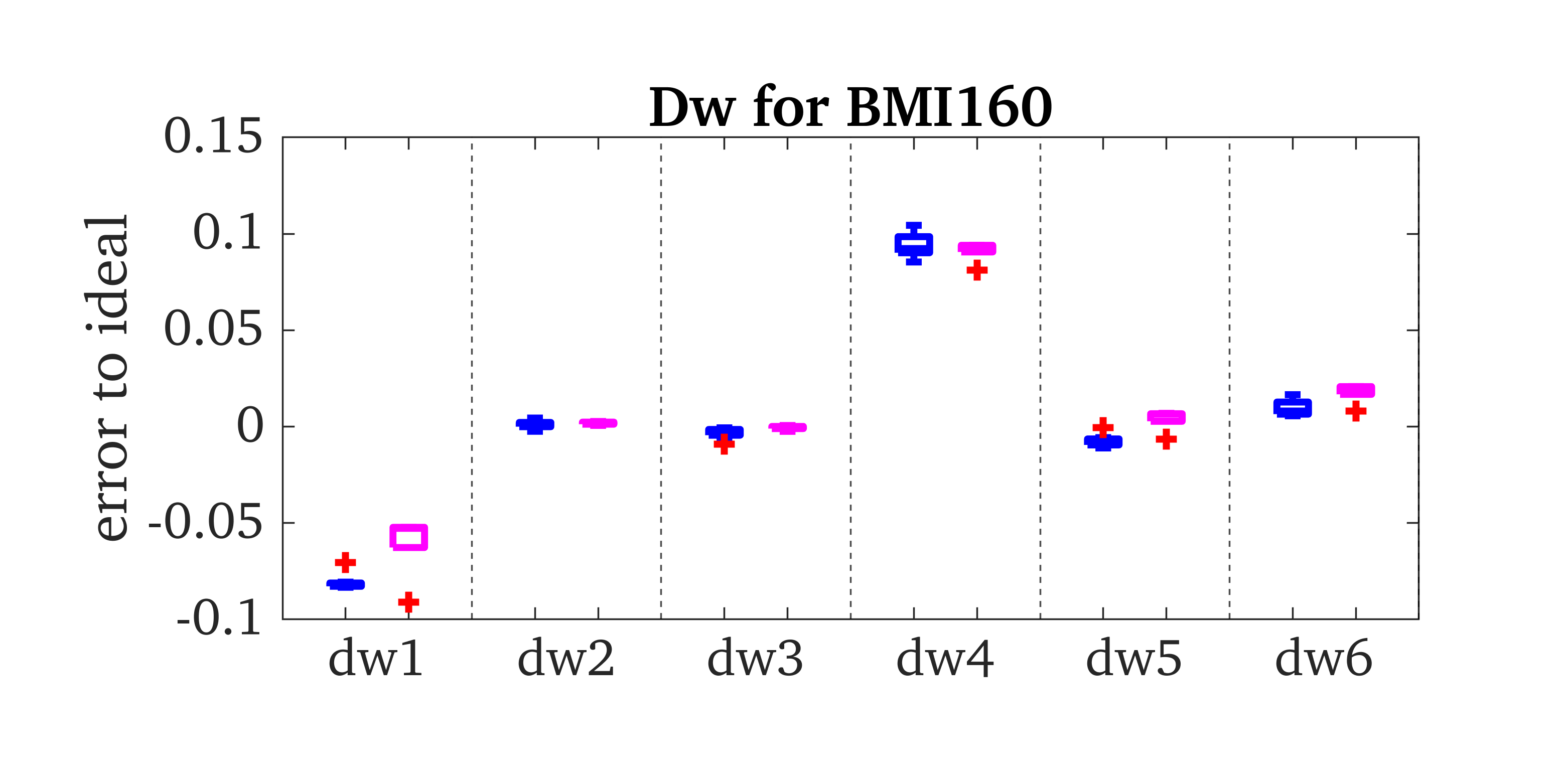}
\end{subfigure}
\begin{subfigure}{.45\textwidth}\centering
\includegraphics[trim=0 9mm 0 0,clip,height=1.3in]{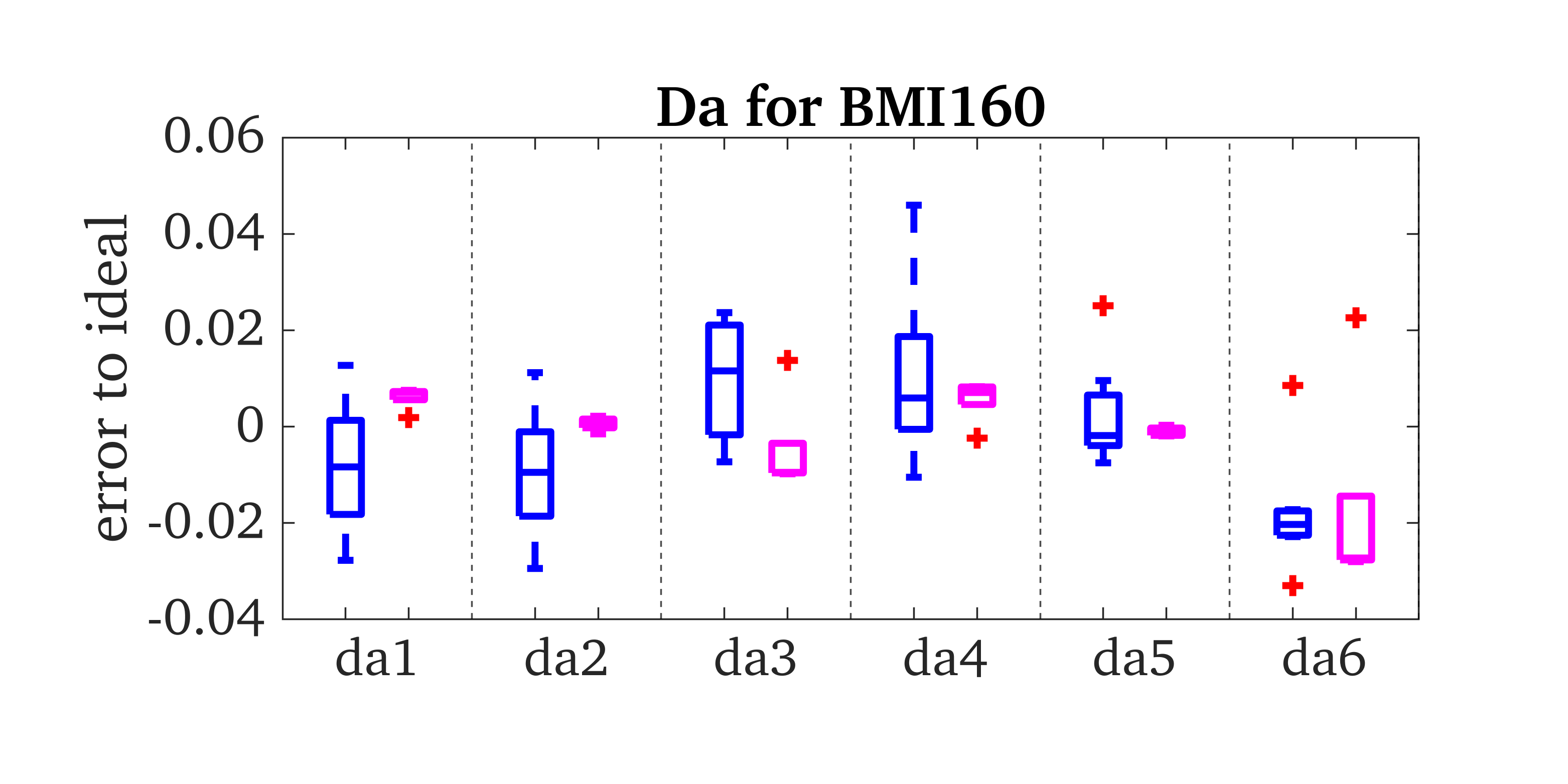}
\end{subfigure}
\begin{subfigure}{.45\textwidth}\centering
\includegraphics[trim=10mm 5mm 0 0,clip,height=1.3in]{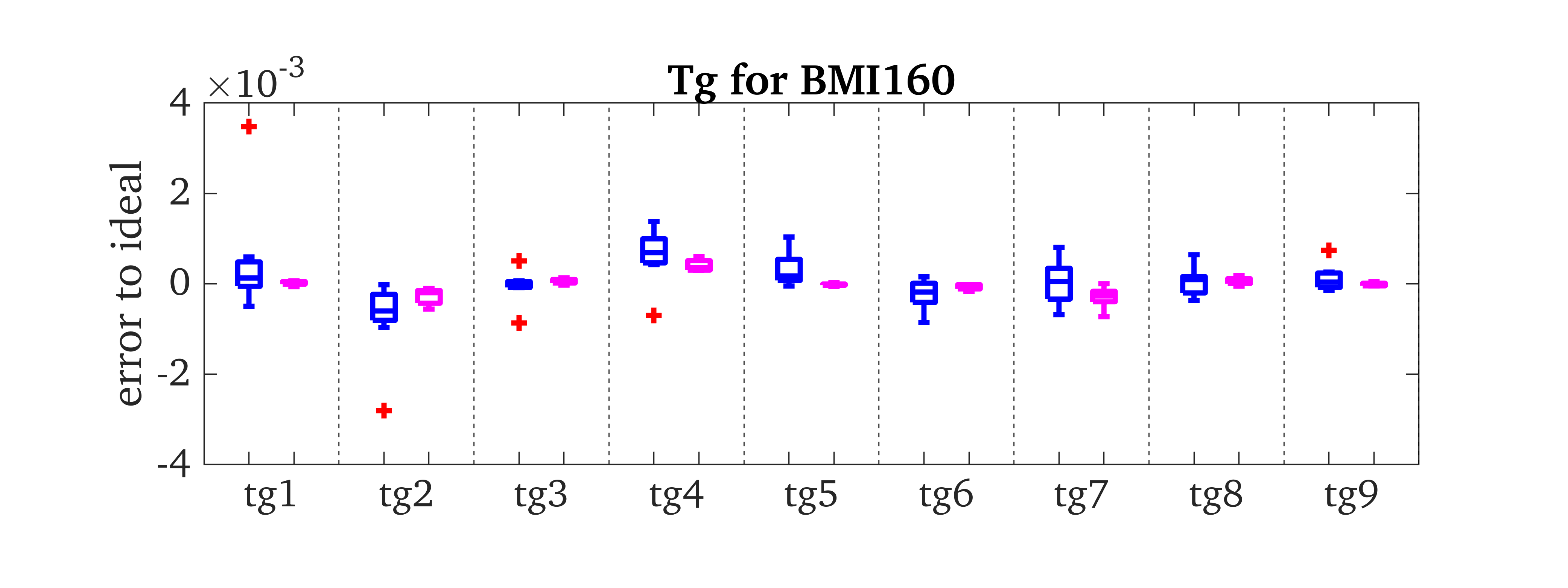}
\end{subfigure}
\begin{subfigure}{.45\textwidth}\centering
\includegraphics[trim=0 5mm 0 0,clip,height=1.3in]{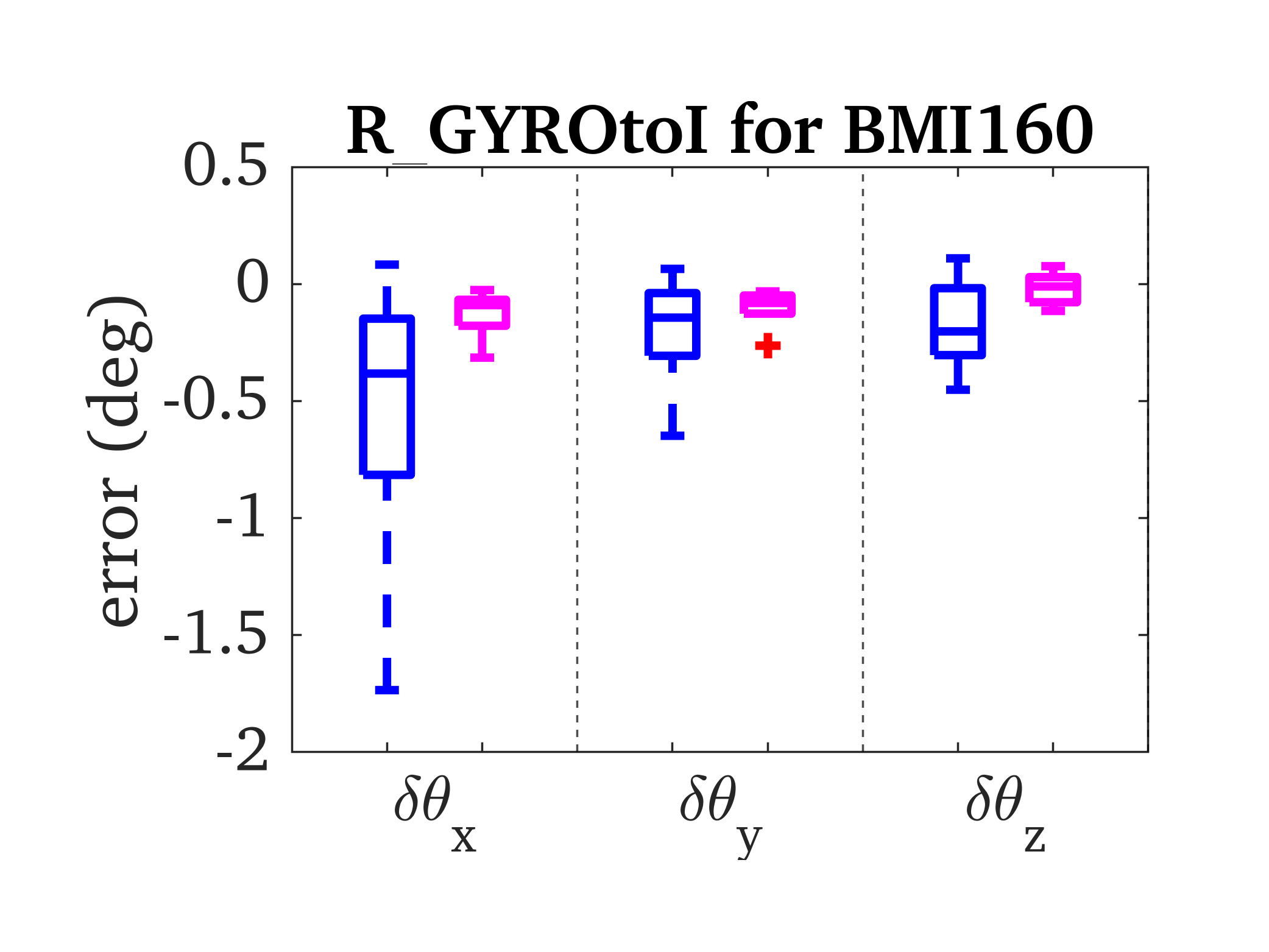}
\end{subfigure}
\caption{
IMU intrinsic evaluation of Bosch BMI160 IMU used in TUM Rolling Shutter VIO datasets using the proposed method with \textit{imu6} and Kalibr relative to the ``ideal'' sensor intrinsics.
The boxplots show the final converged value of both methods.
Kalibr (magenta, right in each group) was run with two global shutter cameras and a Bosch BMI160 IMU available over 5 calibration datasets using April tag board, while the proposed system (blue, left) was run with only one rolling shutter camera and the same IMU on the 8 data sequences without any tags. 
Note that \textit{imu6} is equivalent to the \textit{scale-misalignment} IMU model of Kalibr.
}
\label{fig:exp_kalibr_bmi160}
\end{figure*}

\begin{table}
\renewcommand{\arraystretch}{1.2}
\caption{Averaged absolute trajectory errors (ATE) of 5 runs over all 8 sequences of the TUM Rolling Shutter VIO Dataset with rolling shutter, camera intrinsics, and IMU-CAM spatial-temporal calibration. 
}
\label{tb:ate_tum_rs_1through6}
\centering
\begin{tabular}{ccc} \toprule
\textbf{IMU Model} & \textbf{ATE (deg)} & \textbf{ATE (m)} \\\midrule
\textit{imu0} & 72.994 & 363.610 \\ 
\textit{imu1} & 2.574 & 0.092 \\
\textit{imu2} & 2.679 & 0.094 \\
\textit{imu3} & 2.590 & 0.093 \\
\textit{imu4} & 2.205 & 0.076 \\
\textit{imu5} & 3.418 & 0.149 \\\midrule
\textit{imu32} & 2.368 & 0.079 \\
\bottomrule
\end{tabular}
\end{table}

We first evaluate our proposed algorithms on the TUM RS VIO Dataset which contains a time-synchronized stereo pair of two uEye UI-3241LE-M-GL cameras (left: global-shutter and right: rolling-shutter) and a Bosch BMI160 IMU \citep{Schubert2019IROS}. 
When collecting data, the cameras were operated at 20Hz while the IMU operated at 200Hz and an OptiTrack system captured the ground-truth motion.
The dataset is provided in both ``raw'' and ``calibrated'' formats. The ``calibrate'' dataset has had IMU intrinsics calibrated from Kalibr \citep{Furgale2013IROS} pre-applied to the ``raw'' dataset along with some re-sampling.
We evaluate our proposed system by using the right (RS) camera directly with the raw datasets, which have very noisy measurements with varying sensing rates.
Hence, the raw datasets are more challenging compared to the calibrated datasets. 
We re-calibrated the camera intrinsics and IMU-camera spacial-temporal parameters using the raw calibration datasets as the provided calibration parameters were only for the calibrated datasets.
We directly use Eq. \eqref{eq:rolling_shutter_time} since the dataset timestamps correspond to the first row of the image.
Note that we set the initial values for $\mathbf{D}_a$, $\mathbf{D}_w$, ${}^I_a\mathbf{R}$ and ${}^I_w\mathbf{R}$ as identity and $\mathbf{T}_g$ as zeros, while the initial readout time for the whole RS image is set to 20ms as prior calibration.
All IMU intrinsic models listed in Table \ref{tab:imu model} were run with and without RS calibration. 
The results are presented in the following sections.

\subsection{RS Self-Calibration}

The results are shown in Figure \ref{fig:surf_no_rs} and \ref{fig:surf_rs}, with and without RS readout calibration, respectively.
It is clear that the systems without RS readout time calibration and without IMU intrinsic calibration (\textit{imu0} and \textit{imu31}-\textit{imu34}) are unstable and diverges to large orientation and positional errors. 
With IMU intrinsic calibration, the system still fails for certain datasets (\textit{imu1}-\textit{imu24}) and thus online readout time calibration will greatly improve the system robustness for RS cameras. 
The finally estimated RS readout time for each image is around 30ms, which means given the image resolution of $1280\times 1024$, the row readout time should be around 29us.

\subsection{IMU Intrinsic Self-Calibration}

We focus on the results in Figure \ref{fig:surf_rs} which has RS enabled.
It is clear from the performance of \textit{imu0} that without IMU intrinsic calibration the BMI160 IMU will cause large trajectory errors, with the models which do perform intrinsic calibration being an order of magnitude more accurate.
Table \ref{tb:ate_tum_rs_1through6} shows the average error over all sequences for the first 6 IMU models.
It can be seen that the \textit{imu5} model which over parameterizes the intrinsics has worst accuracy in both orientation and position trajectory estimates, while the accuracy of the other \textit{imu1} - \textit{imu4} models is comparable to each other (similar accuracy level).
We further do an ablation study with models \textit{imu31} - \textit{imu34} to find the individual impact of each of the IMU intrinsic parameters.
We can see that the \textit{imu32} model which estimates $\mathbf{D}_{w9}$ has large accuracy gains over the other four.
This indicates that the readings from gyroscope of BMI160 are very noisy. 
The calibration of $\mathbf{D}_{w9}$ dominates the performance of this VINS system, and just the calibration of it can achieve similar results as full IMU model calibration (see bottom of Table \ref{tb:ate_tum_rs_1through6}).
Through all these we show that online IMU intrinsic calibration can enhance both the system robustness and accuracy.

\subsection{Comparison to Kalibr Calibration}
We run Kalibr's offline calibration with \textit{scale-misalignment} \footnote{https://github.com/ethz-asl/kalibr/wiki/Multi-IMU-and-IMU-intrinsic-calibration} IMU model on 5 calibration datasets provided for the Bosch BMI160 IMU and treat these results as reference values to compare to the proposed \textit{online} calibration results.
The Kalibr calibration datasets were collected with the stereo camera pair both operating with a global shutter mode along with an April tag board \citep{Furgale2013IROS}. For the results evaluation, we directly report $\mathbf{T}'_w = (\mathbf{D}'_{w})^{-1}$, $\mathbf{T}'_a = (\mathbf{D}'_{a})^{-1}$ and ${}^w_I\mathbf{R} = {}^I_w\mathbf{R}^{\top}$.
By contrast, the proposed system is run on one RS camera with only temporal environmental feature tracks on the 8 data sequences using \textit{imu6}, which is equivalent to the \textit{scale-misalignment} IMU model of Kalibr.

As shown in the boxplots in Figure \ref{fig:exp_kalibr_bmi160}, even though we run on more challenging datasets in real-time, our proposed system can still achieve reasonable calibration results for $\mathbf{D}'_w$, $\mathbf{D}'_a$, $\mathbf{T}_g$ and ${}^I_w\mathbf{R}$, which are close to the values from baseline Kalibr. 
Additionally we can see that the values of gravity sensitivity $\mathbf{T}_g$ of the BMI160 IMU are generally one or two orders smaller than the other IMU intrinsics.
This matches the results presented in Figure \ref{fig:surf_rs}, for which the estimation errors of \textit{imu1} - \textit{imu4} (without gravity sensitivity) are similarly to those of \textit{imu11} - \textit{imu14} (with gravity sensitivity of 6 parameters) and \textit{imu21} - \textit{imu24} (with gravity sensitivity of 9 parameters).
This means that the proposed system performance is less sensitive to gravity sensitivity, no matter 0, 6 or 9 parameters are used.

We can also see that the terms in scale-misalignment for gyroscope $\mathbf{D}'_w$ for the BMI160 IMU are much larger than $\mathbf{D}'_{a}$ and $\mathbf{T}_g$.
This indicates that the readings from gyroscope of BMI160 are very noisy and the calibration of $\mathbf{D}'_{w}$ dominates the performance of this VINS system. 
This fact is again confirmed by model \textit{imu32}, which only calibrates the $\mathbf{D}'_w$ and can achieve similar results as full IMU model calibration (see bottom of Table \ref{tb:ate_tum_rs_1through6}).

\section{More Validation and Investigation of Camera and IMU Impacts}
\label{sec:exp_virig}

\begin{figure}
\centering
\includegraphics[width=0.65\linewidth]{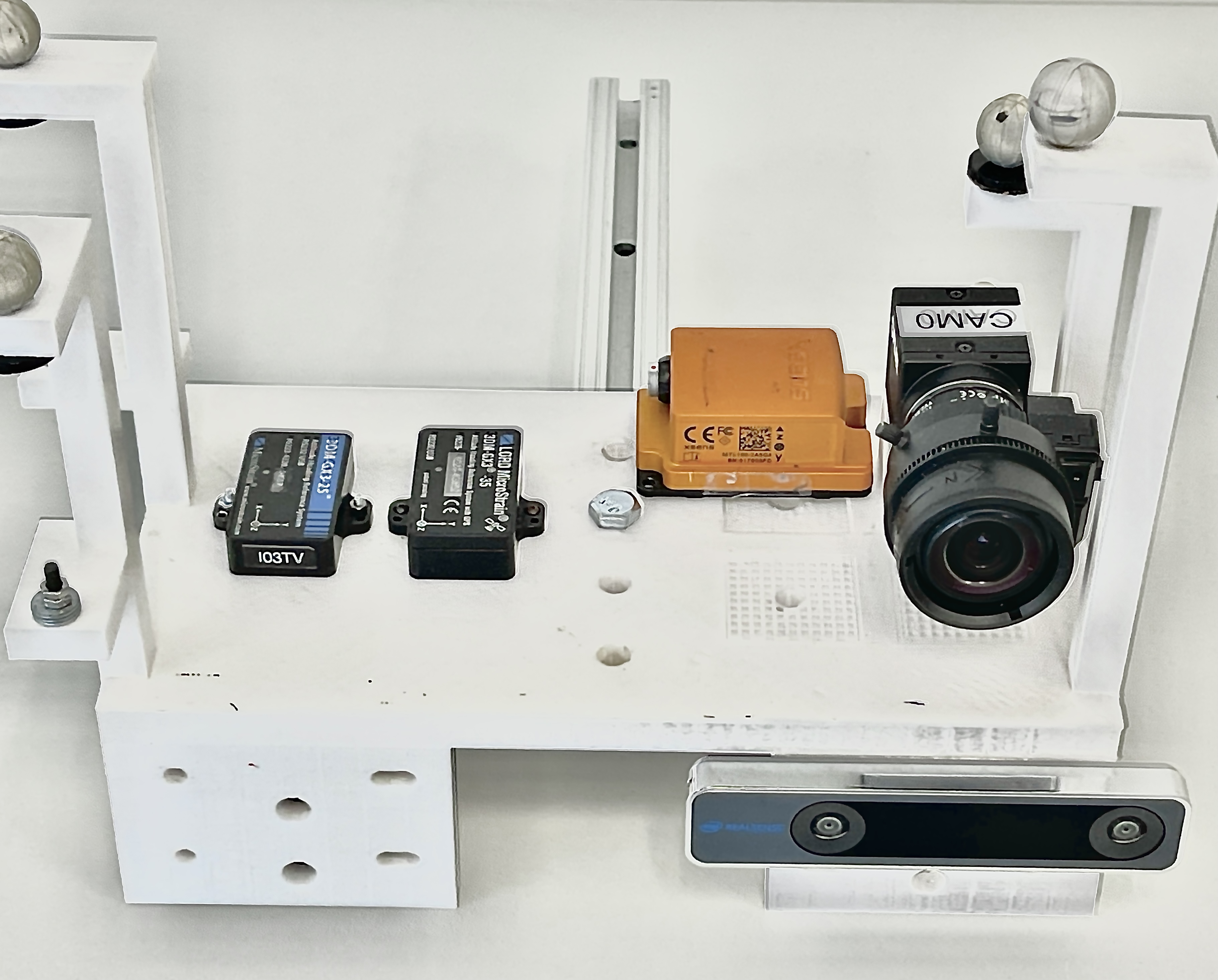}
\caption{Visual-Inertial Sensor Rig contains a MicroStrain GX5-25 IMU, MicroStrain GX5-35, Xsens MTi 100, FLIR Blackfly camera and RealSense T265 tracking camera.
The RealSense T265 tracking camera contains an integrated IMU and a fisheye stereo camera. 
}
\label{fig:vi_rig}
\end{figure}

\begin{table*}[]
\renewcommand{\arraystretch}{1.5}
\caption{Average processing time for each image (including propagation and update) for the proposed system with (w/) and without (w/o) online calibration (unit: second) on the 10 datasets collected with VI-Rig. 
The time increase (0.0036s in average) for online calibration is negligible compared to no calibration. 
}
\label{tab:timing}
\begin{adjustbox}{width=\textwidth,center}
\begin{tabular}{cccccccccccc}
\hline
\textbf{Algorithm}    & \textbf{Data-1} & \textbf{Data-2} & \textbf{Data-3} & \textbf{Data-4} & \textbf{Data-5} & \textbf{Data-6} & \textbf{Data-7} & \textbf{Data-8} & \textbf{Data-9} & \textbf{Data-10} & \textbf{Avg.} \\ \hline
w/ online calib & 0.0227    & 0.0231    & 0.0228    & 0.0223    & 0.0229    & 0.0220    & 0.0214     & 0.0228    & 0.0223    & 0.0215 & 0.0224    \\
w/o online calib     & 0.0187    & 0.0191    & 0.0189    & 0.0190    & 0.0190    & 0.0185    & 0.0186    & 0.0186    & 0.0189    & 0.0190  & 0.0188   \\ \hline
\end{tabular}
\end{adjustbox}
\end{table*}

\begin{figure*}
\centering
\includegraphics[trim=0.5cm 0 1cm 0,clip,height=1.3in]{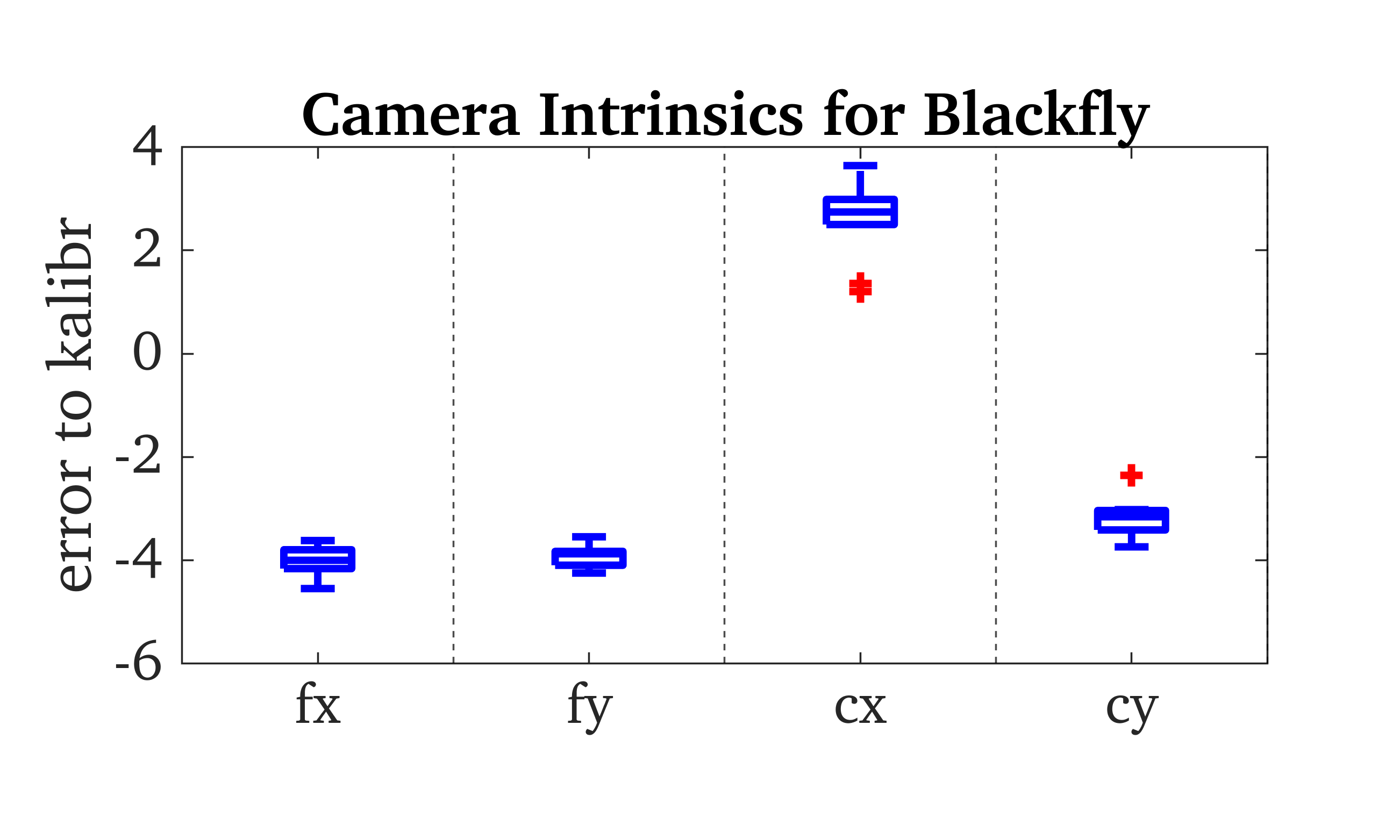}
\includegraphics[trim=0.4cm 0 1cm 0,clip,height=1.3in]{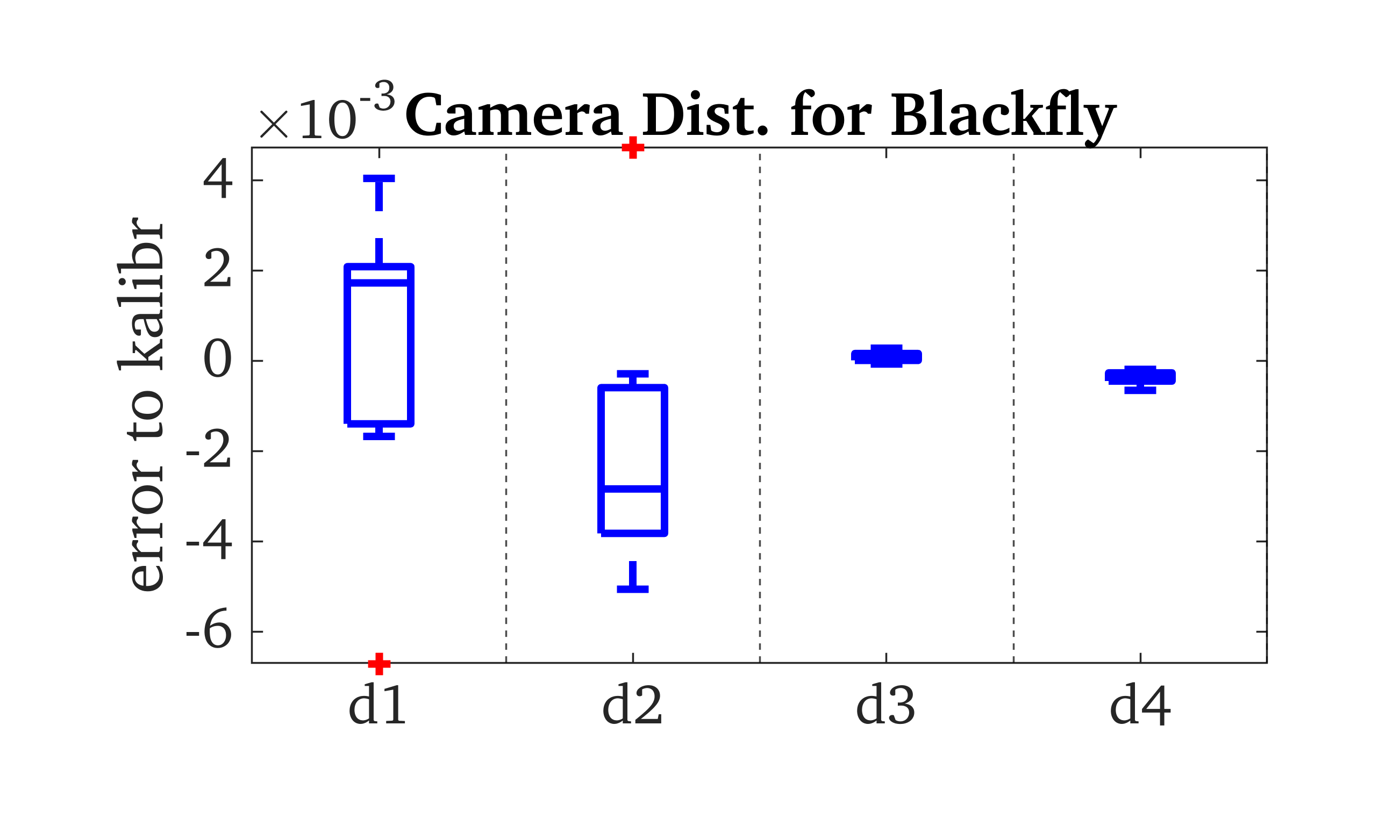}
\includegraphics[trim=0.75cm 0 0.75cm 0,clip,height=1.3in]{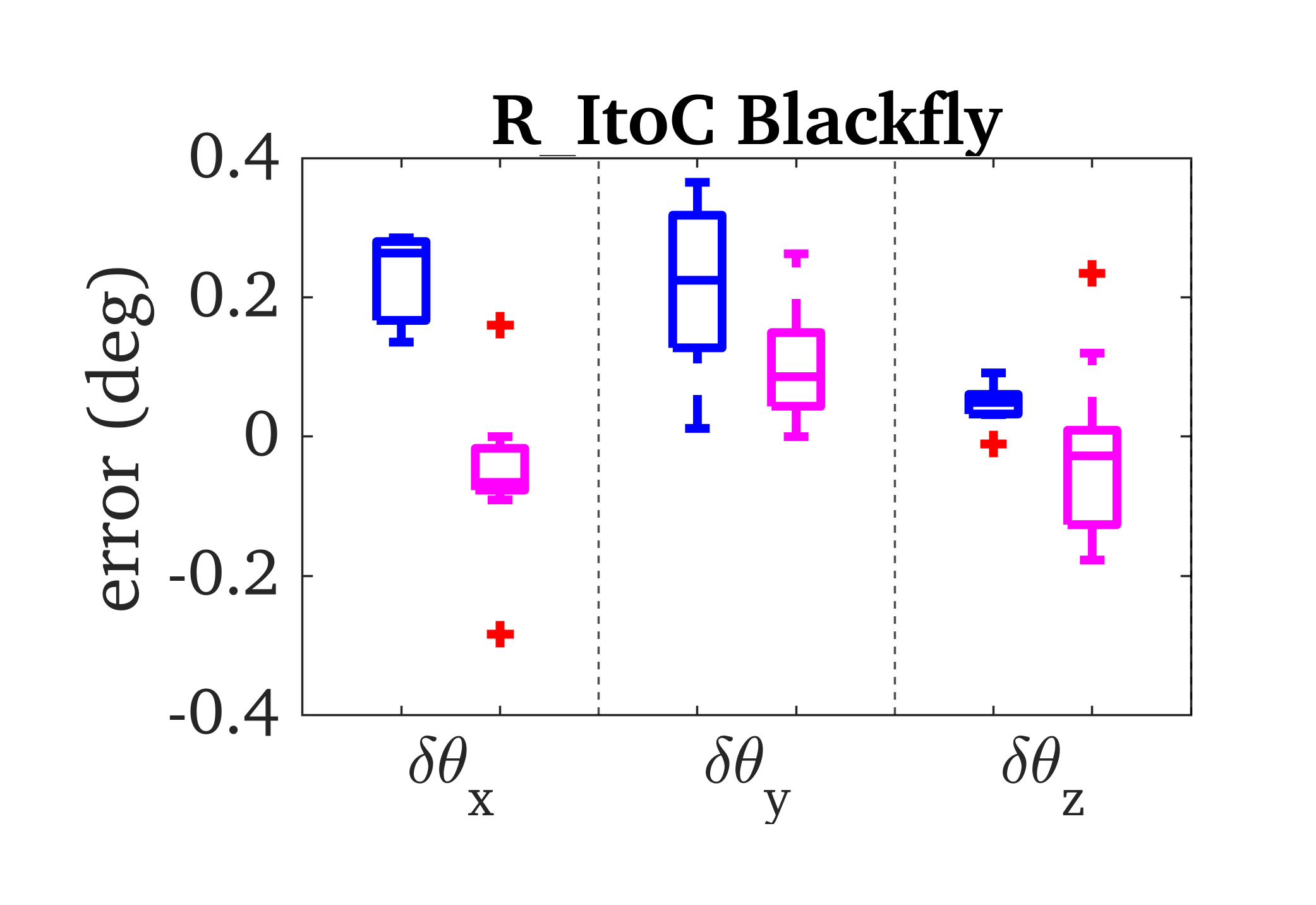}
\includegraphics[trim=0.75cm 0 0.75cm 0,clip,height=1.3in]{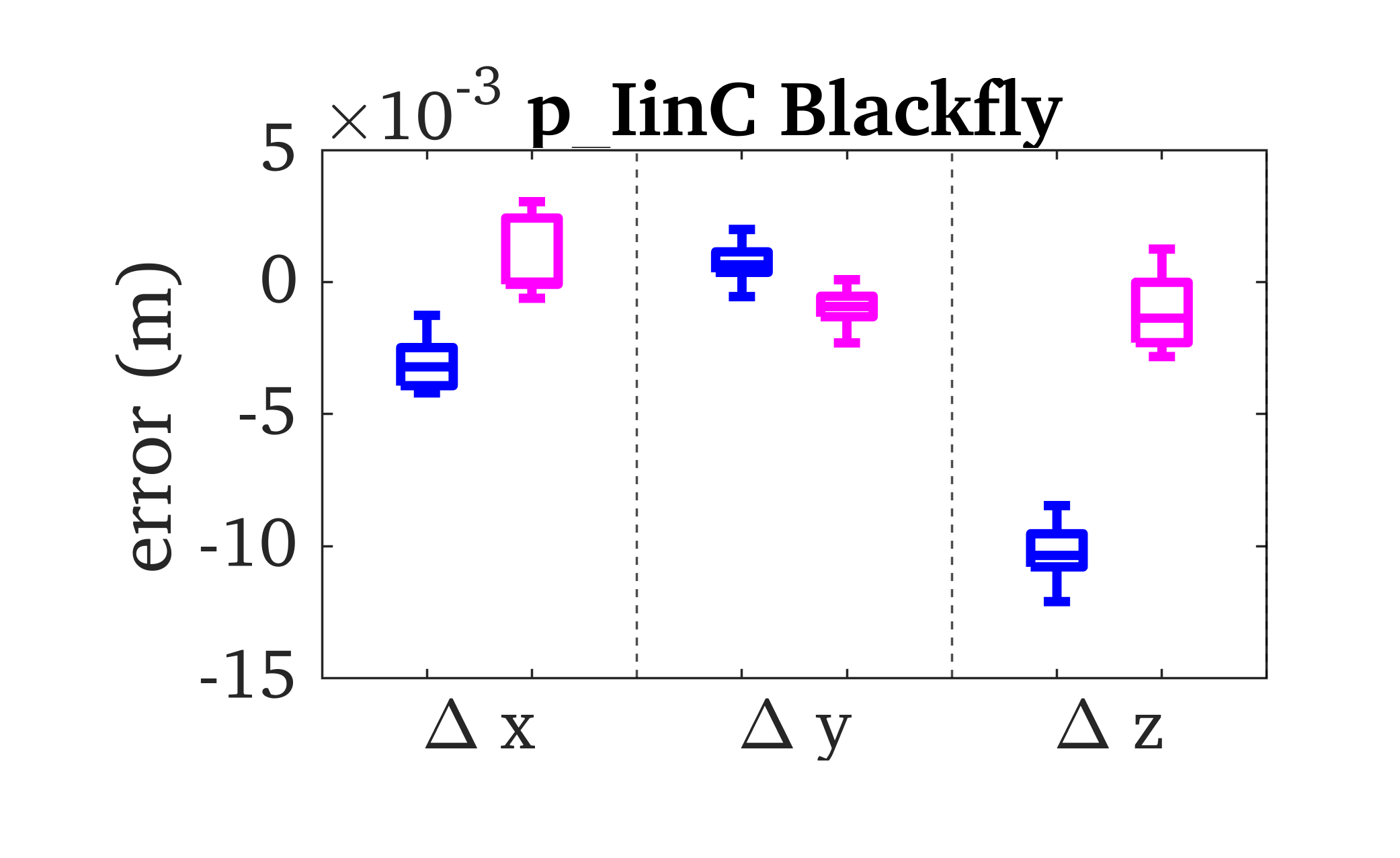}
\includegraphics[trim=0cm 0 0cm 0,clip,height=1.3in]{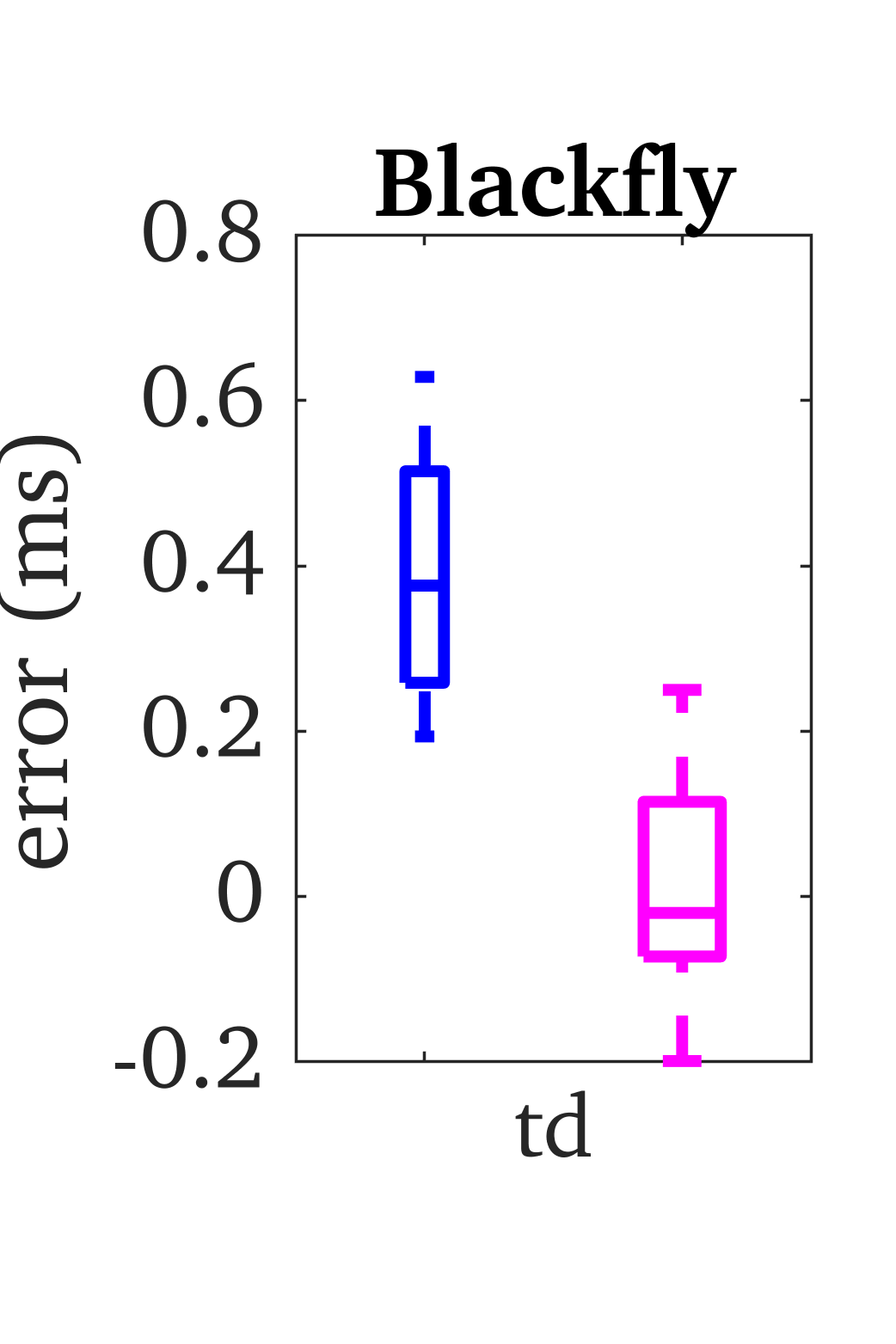}
\includegraphics[trim=0cm 0 0cm 0,clip,height=1.3in]{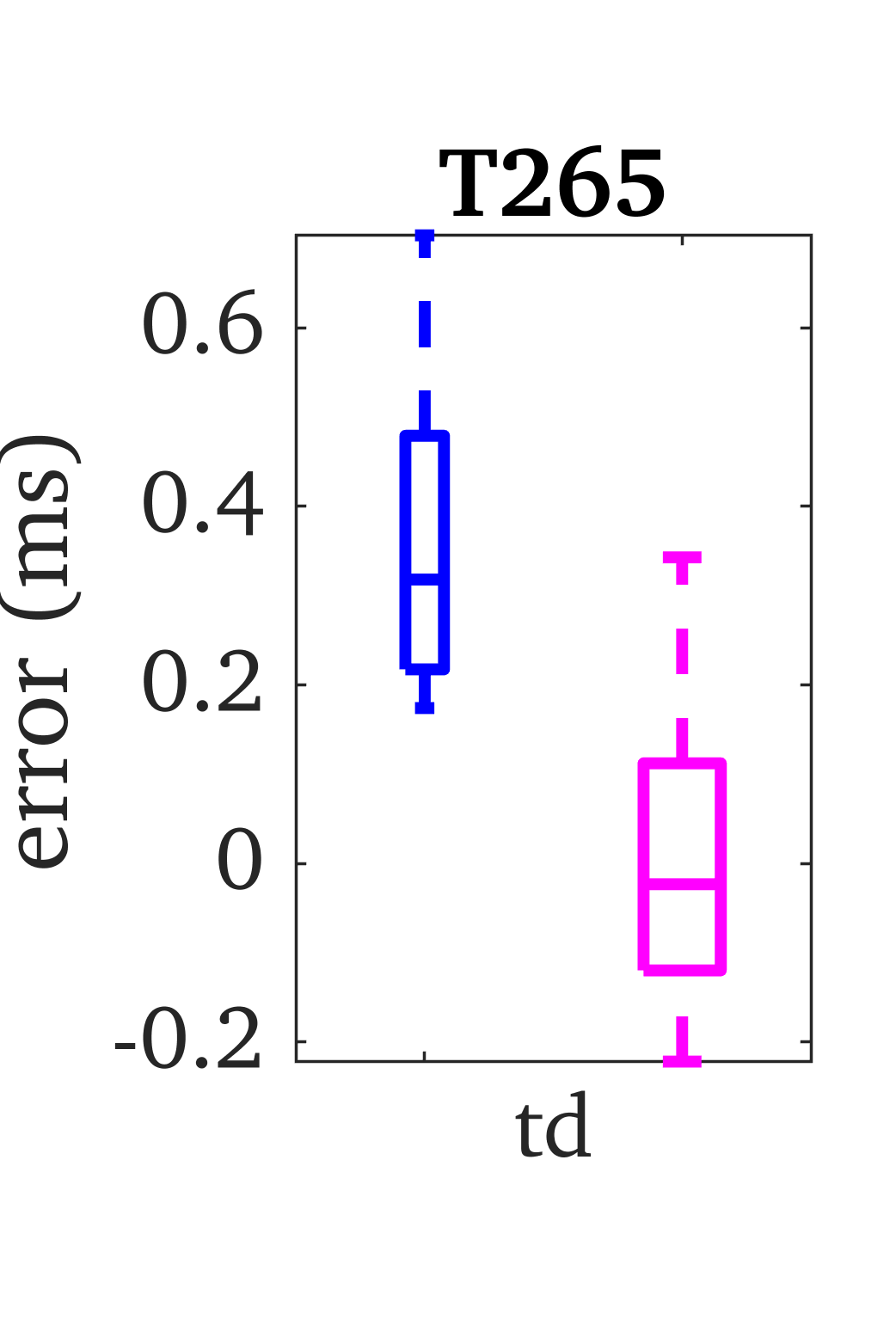}
\includegraphics[trim=0.5cm 0 1cm 0,clip,height=1.3in]{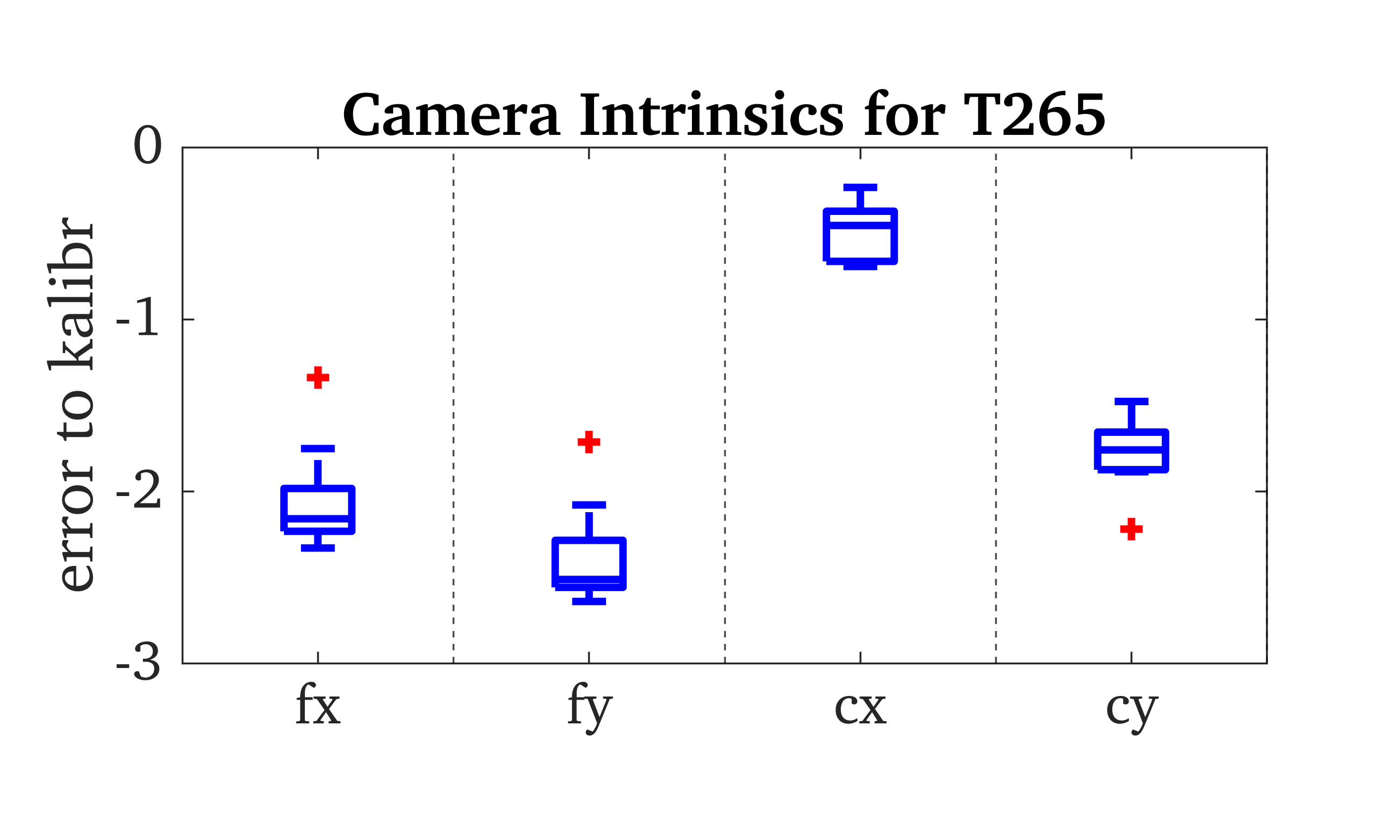}
\includegraphics[trim=0.4cm 0 1cm 0,clip,height=1.3in]{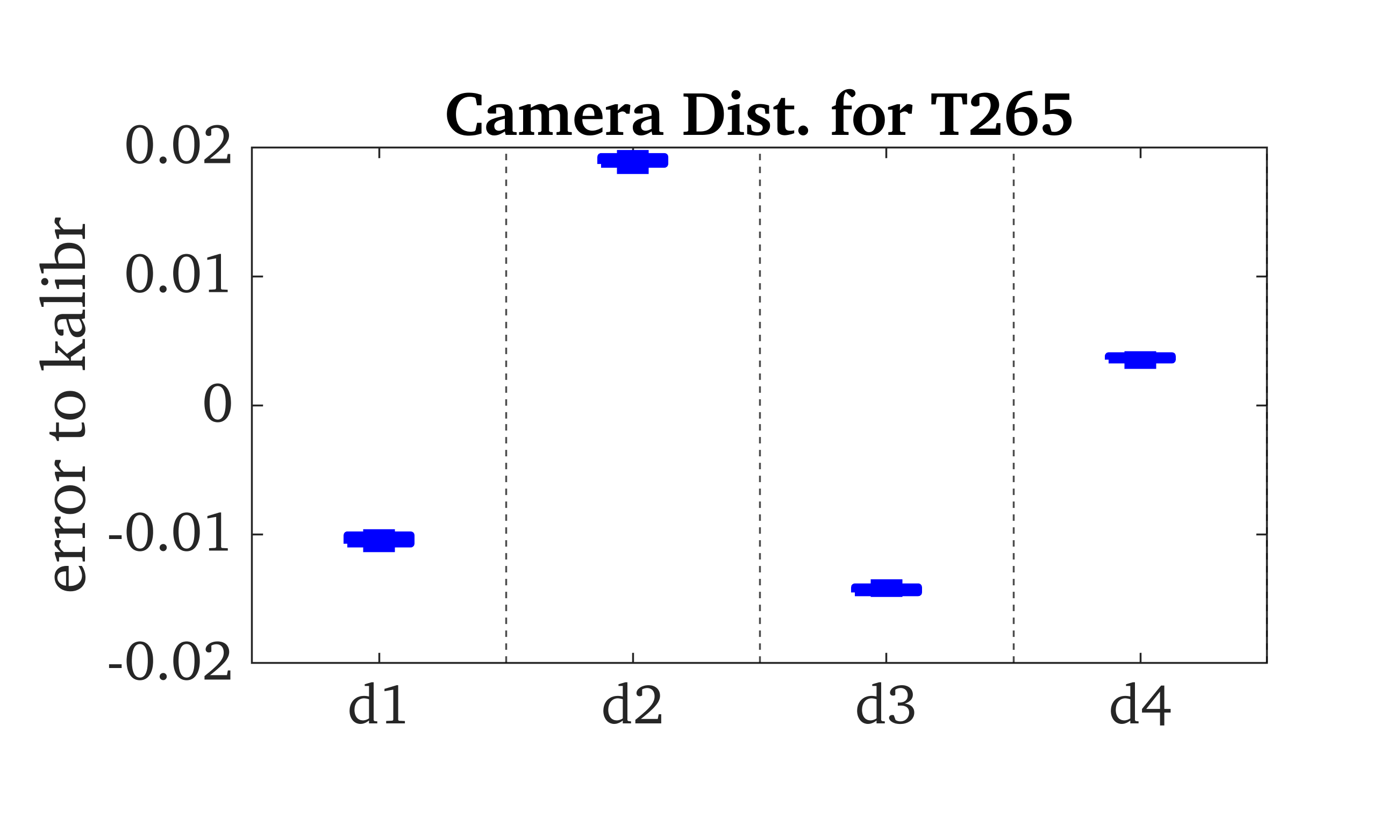}
\includegraphics[trim=0.75cm 0 0.75cm 0,clip,height=1.3in]{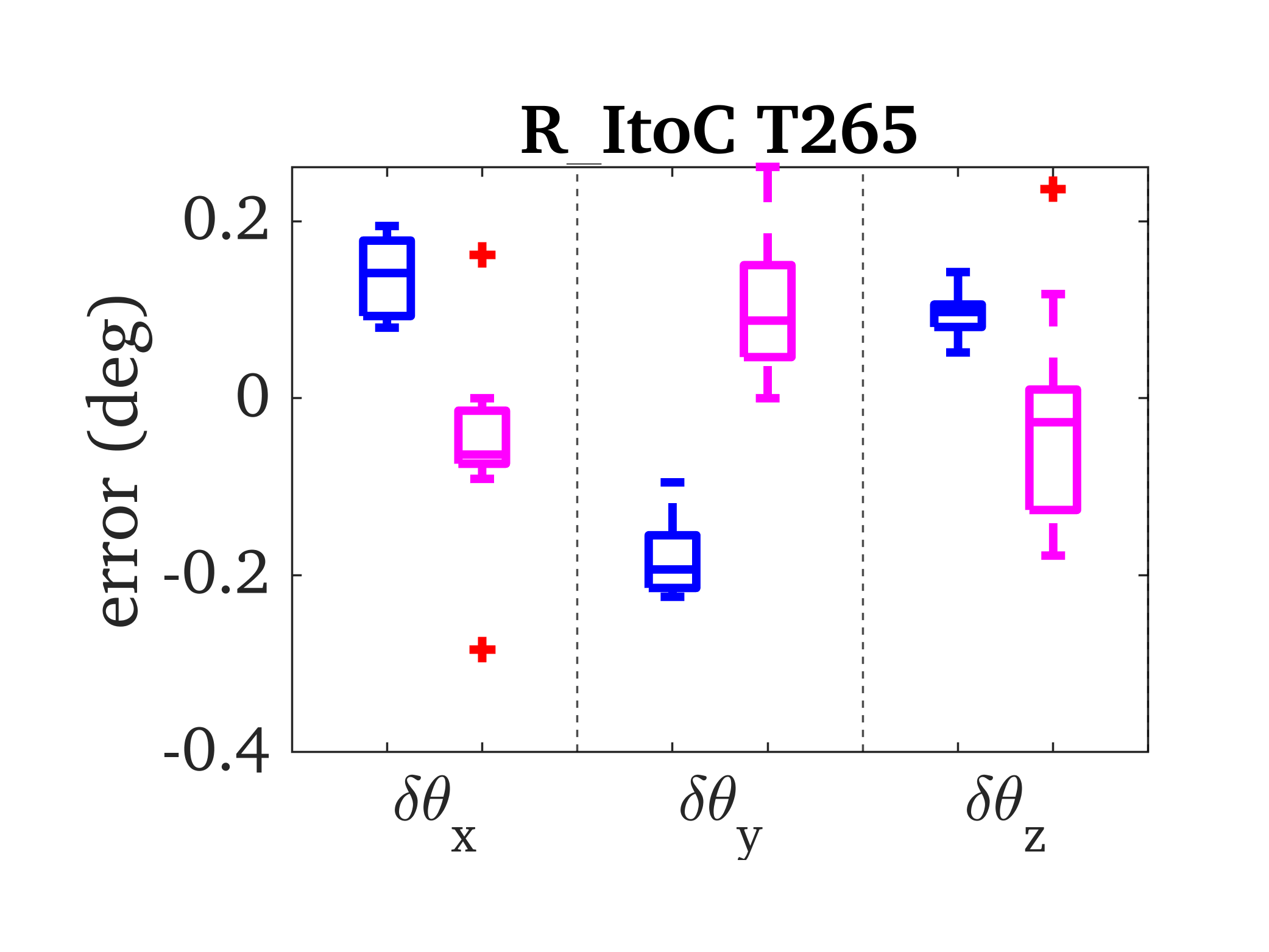}
\includegraphics[trim=0.75cm 0 0.75cm 0,clip,height=1.3in]{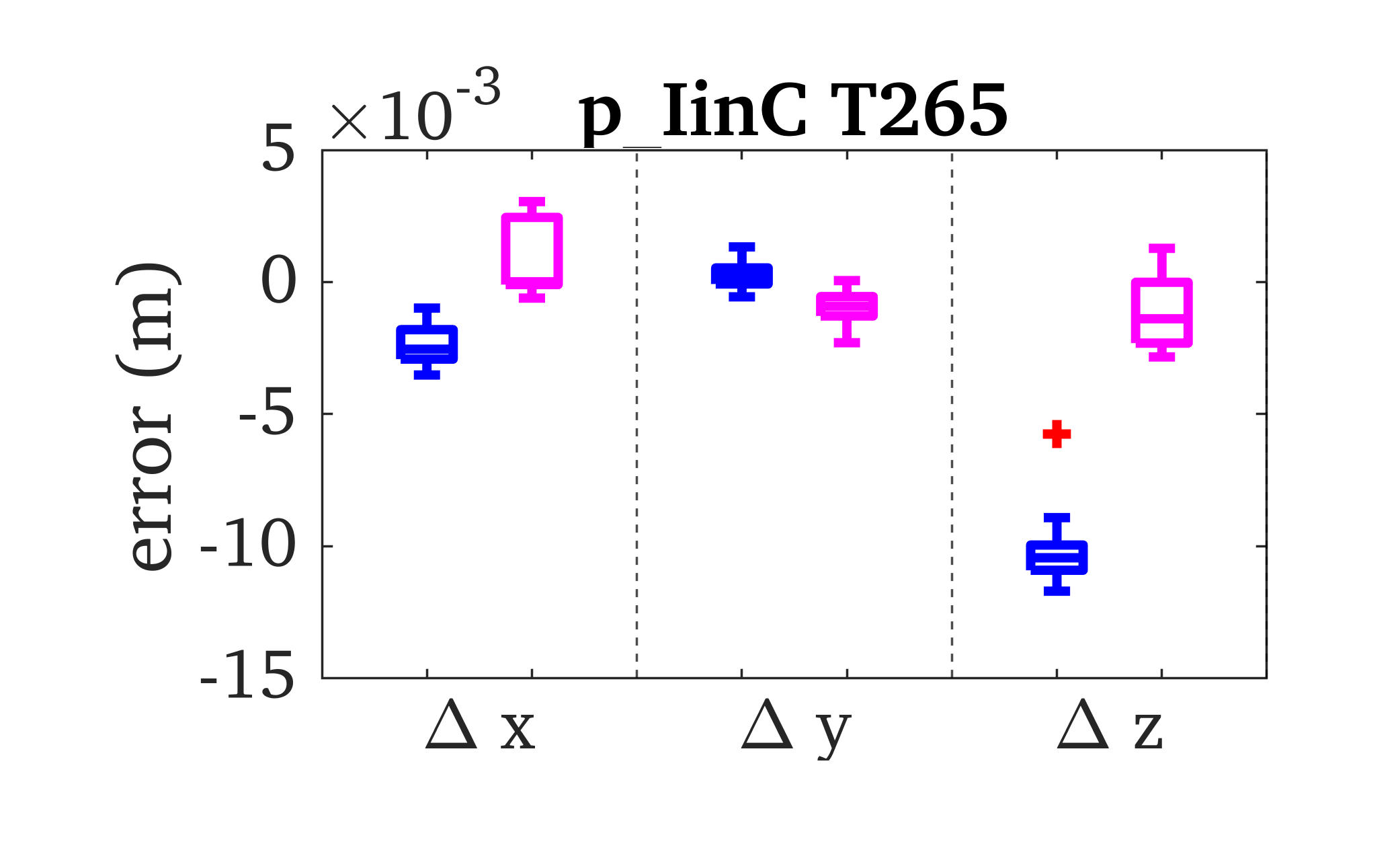}
\caption{
Comparison of the proposed method and Kalibr relative to a baseline Kalibr value.
The boxplots show the final converged value of both methods, while for camera intrinsic only the proposed is reported since Kalibr fixes this during optimization.
Kalibr (magenta, right in each group) was run with all cameras and IMUs available over 10 datasets, while the proposed system (blue, left) was run with either the Blackfly camera or left T265 fisheye and the MicroStrain GX-25 IMU resulting in 10 runs for each.
}
\label{fig:exp_kalibr_compare_cameras}
\end{figure*}

\begin{figure*}
\centering
\begin{subfigure}{.45\textwidth}\centering
\includegraphics[trim=0 9mm 0 0,clip,height=1.3in]{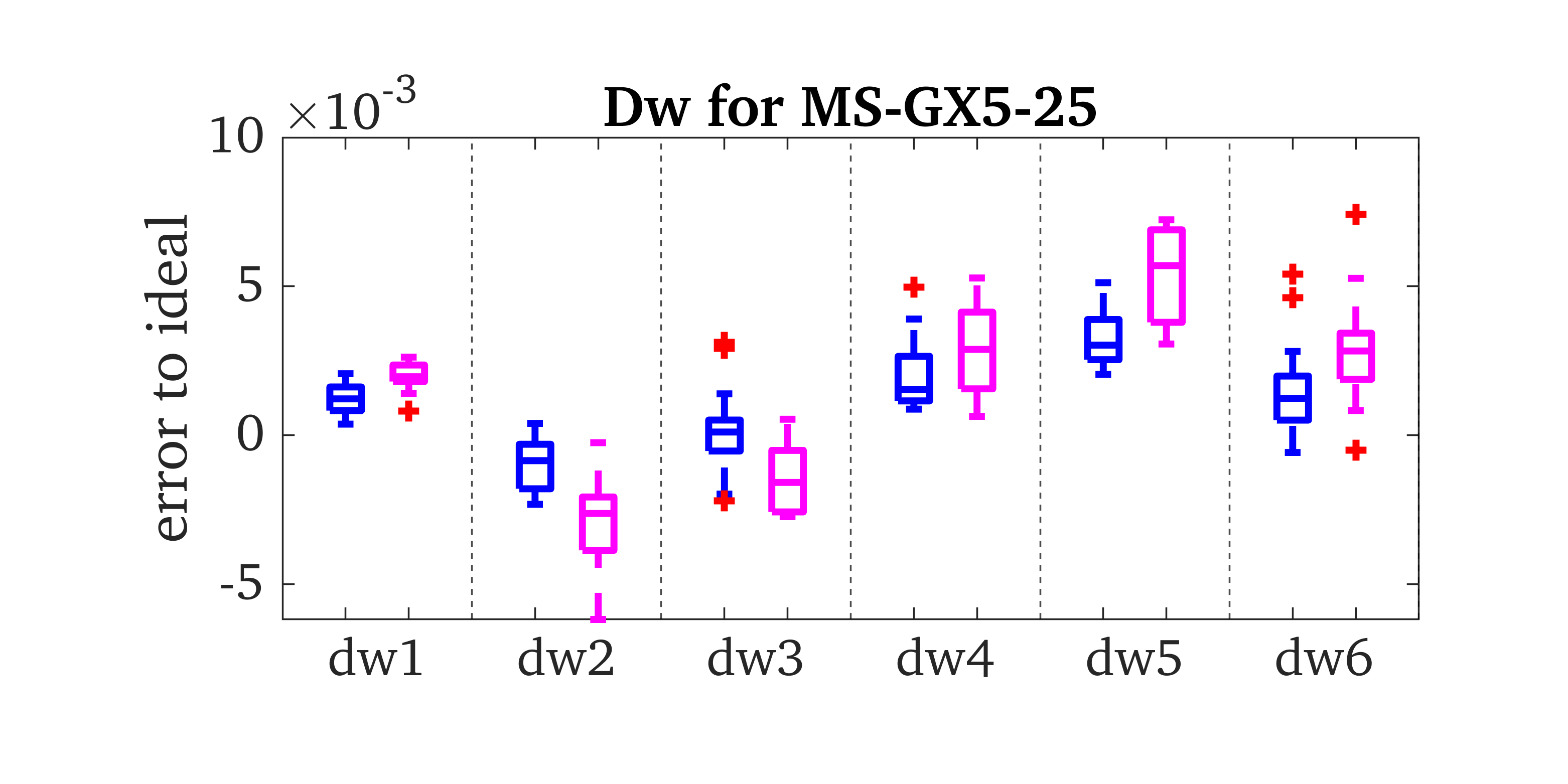}
\end{subfigure}
\begin{subfigure}{.45\textwidth}\centering
\includegraphics[trim=0 9mm 0 0,clip,height=1.3in]{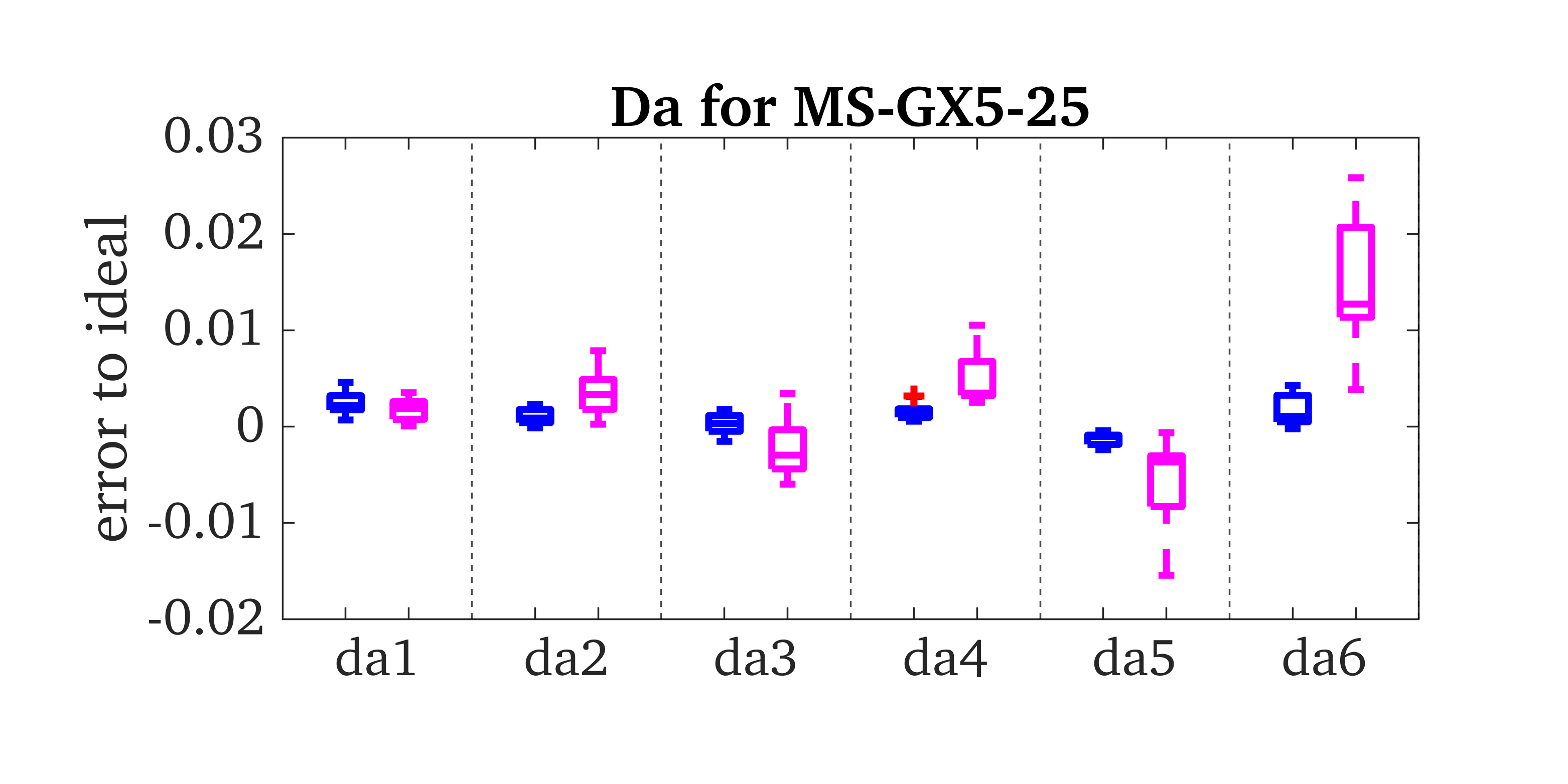}
\end{subfigure}
\begin{subfigure}{.45\textwidth}\centering
\includegraphics[trim=10mm 5mm 0 0,clip,height=1.3in]{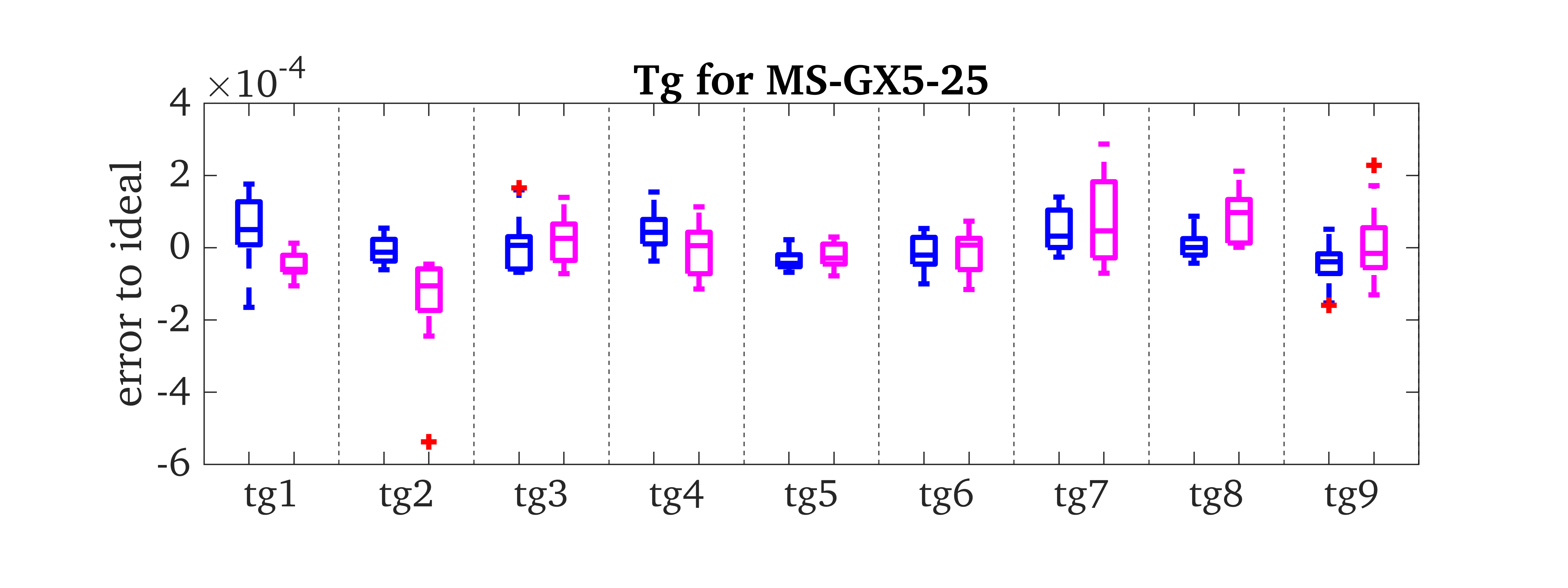}
\end{subfigure}
\begin{subfigure}{.45\textwidth}\centering
\includegraphics[trim=0 5mm 0 0,clip,height=1.3in]{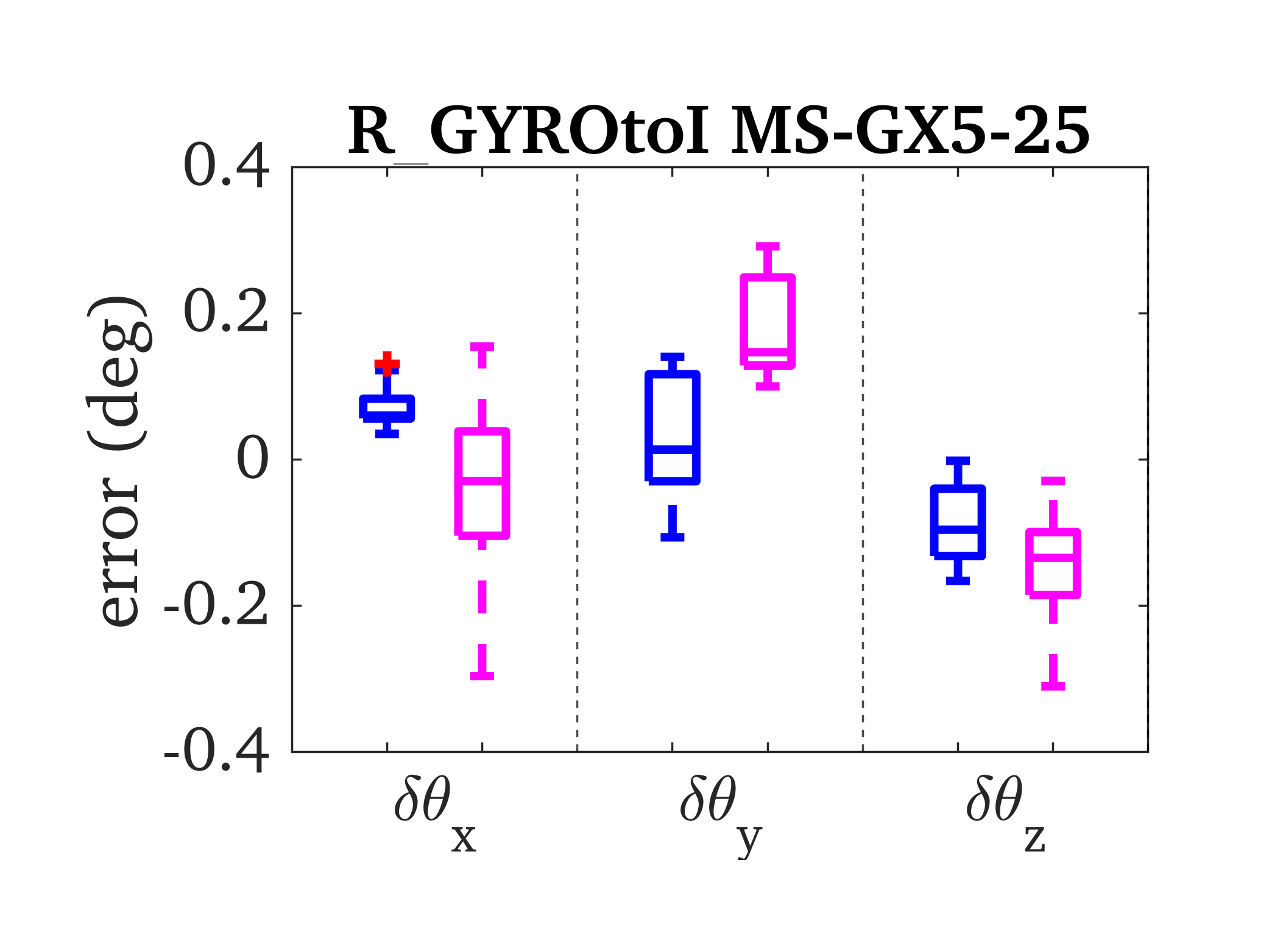}
\end{subfigure}
\caption{
Comparison of the proposed method with \textit{imu6} and Kalibr relative to the ``ideal'' sensor intrinsics.
The boxplots show the final converged value of both methods.
Kalibr (magenta, right in each group) was run with all cameras and IMUs available over 10 datasets, while the proposed system (blue, left) was run with either the Blackfly camera or left T265 fisheye resulting in 20 runs.
}
\label{fig:exp_kalibr_compare_ms25}
\end{figure*}
\begin{figure*}
\centering
\begin{subfigure}{.45\textwidth}\centering
\includegraphics[trim=0 9mm 0 0,clip,height=1.3in]{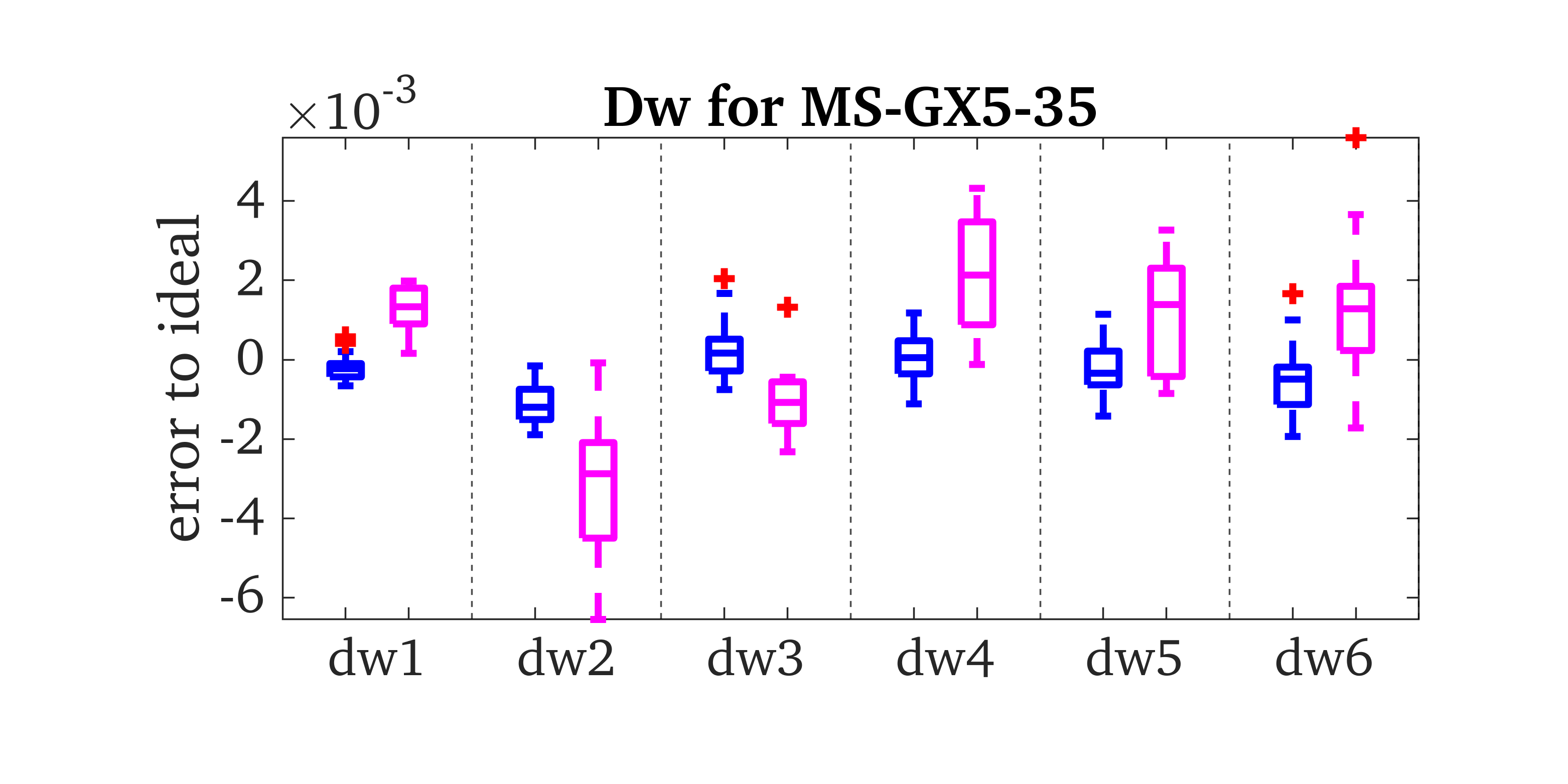}
\end{subfigure}
\begin{subfigure}{.45\textwidth}\centering
\includegraphics[trim=0 9mm 0 0,clip,height=1.3in]{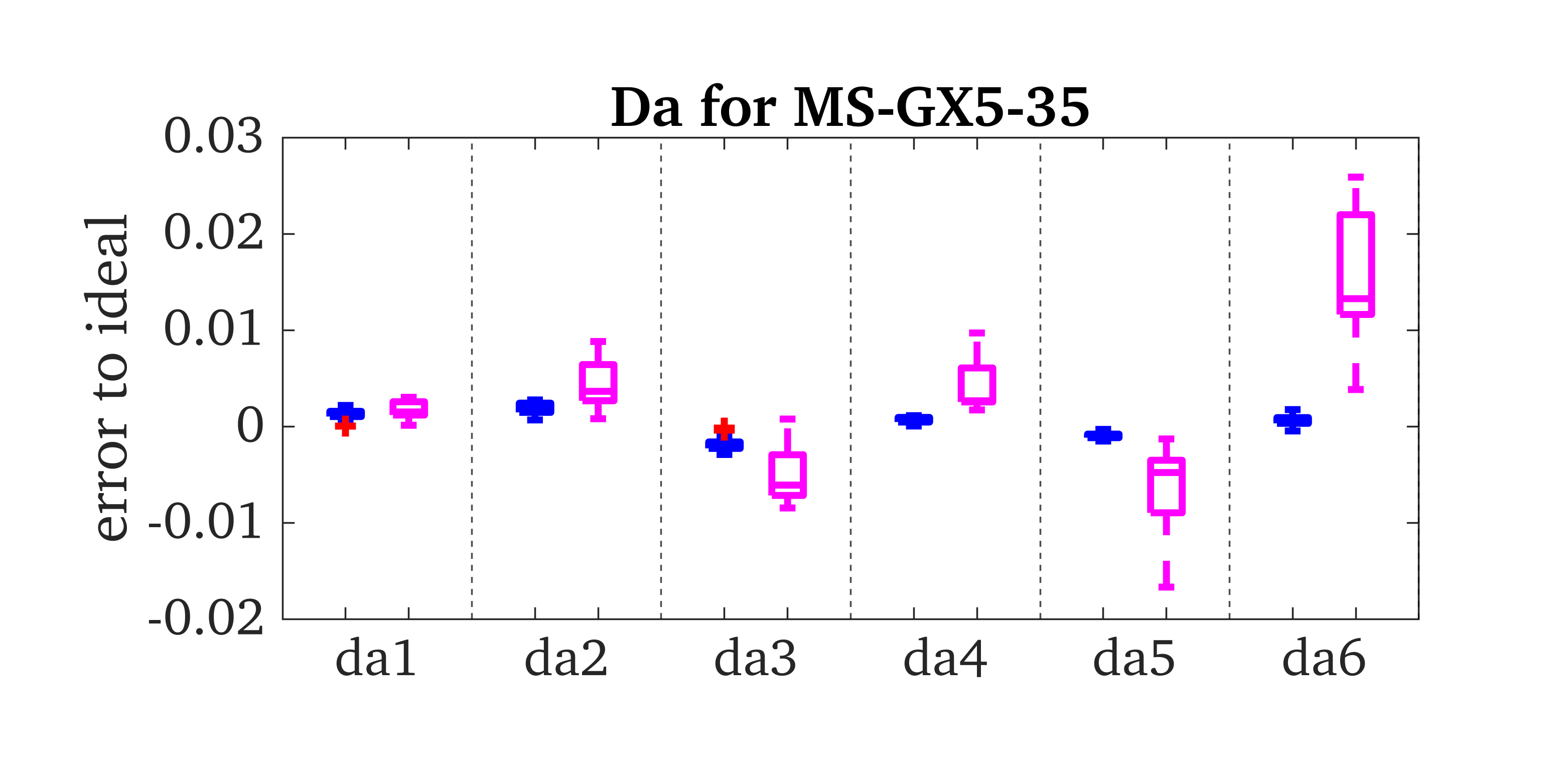}
\end{subfigure}
\begin{subfigure}{.45\textwidth}\centering
\includegraphics[trim=10mm 5mm 0 0,clip,height=1.3in]{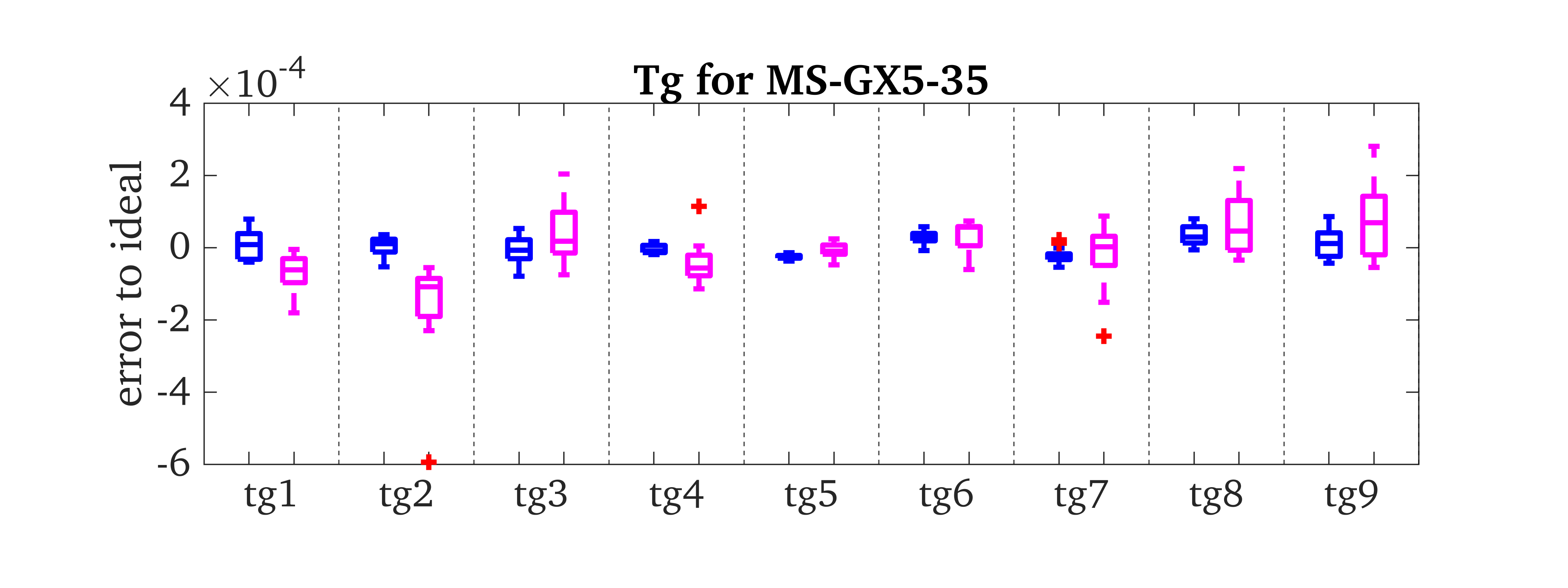}
\end{subfigure}
\begin{subfigure}{.45\textwidth}\centering
\includegraphics[trim=0 5mm 0 0,clip,height=1.3in]{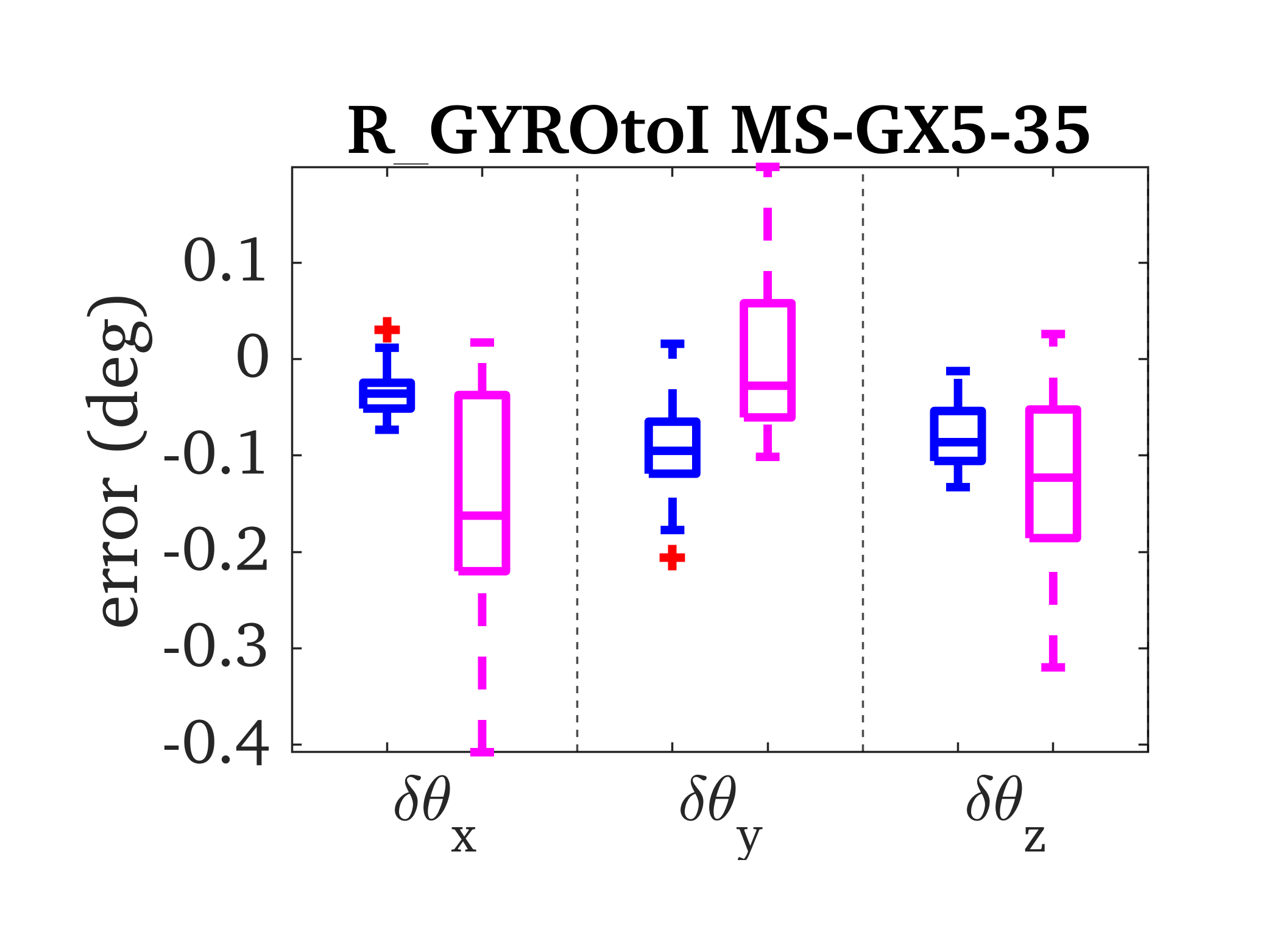}
\end{subfigure}
\caption{
Comparison of the proposed method with \textit{imu6} and Kalibr relative to the ``ideal'' sensor intrinsics.
The boxplots show the final converged value of both methods.
Kalibr (magenta, right in each group) was run with all cameras and IMUs available over 10 datasets, while the proposed system (blue, left) was run with either the Blackfly camera or left T265 fisheye resulting in 20 runs.
}
\label{fig:exp_kalibr_compare_ms35}
\end{figure*}
\begin{figure*}
\centering
\begin{subfigure}{.45\textwidth}\centering
\includegraphics[trim=0 9mm 0 0,clip,height=1.3in]{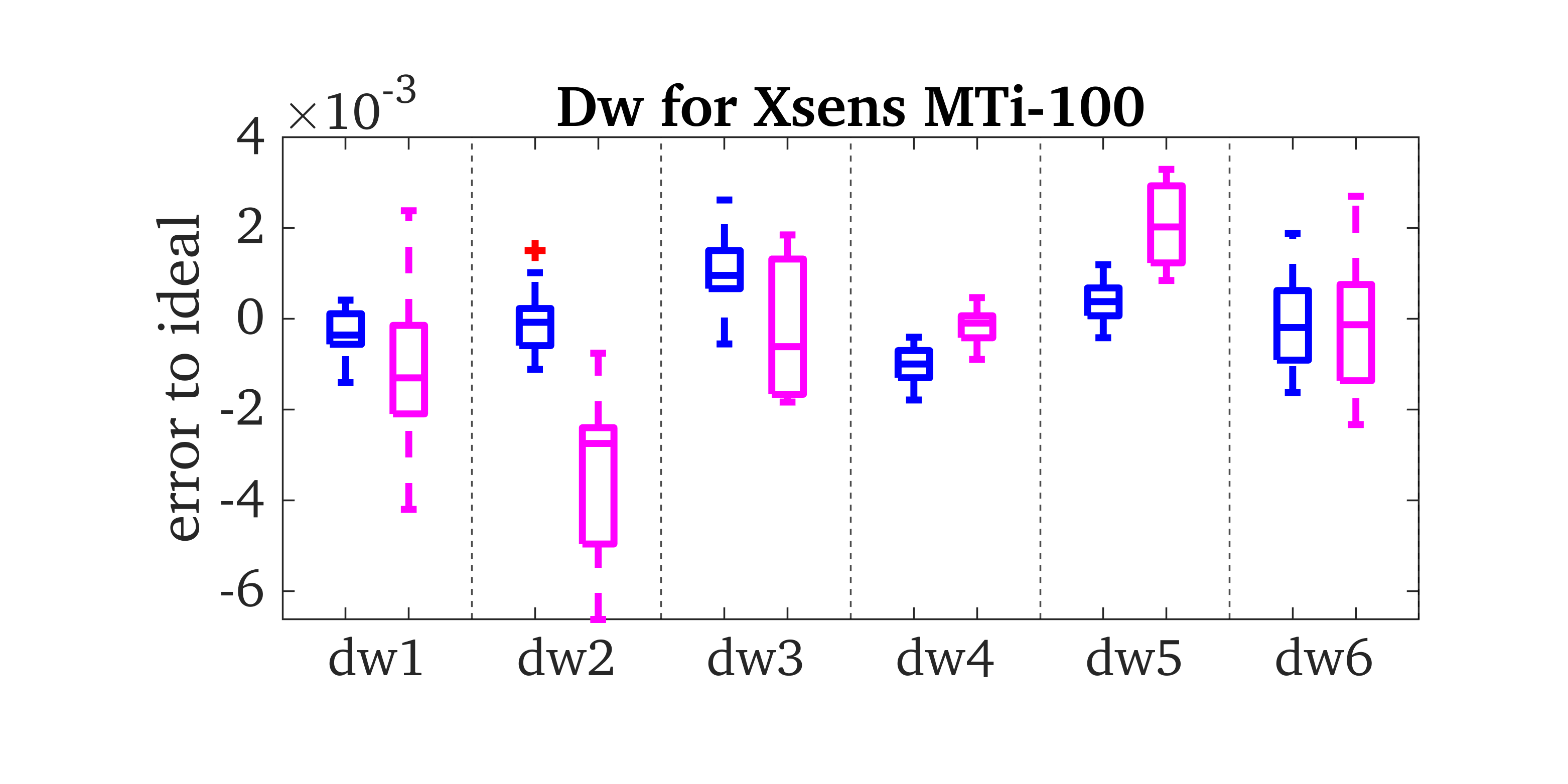}
\end{subfigure}
\begin{subfigure}{.45\textwidth}\centering
\includegraphics[trim=0 9mm 0 0,clip,height=1.3in]{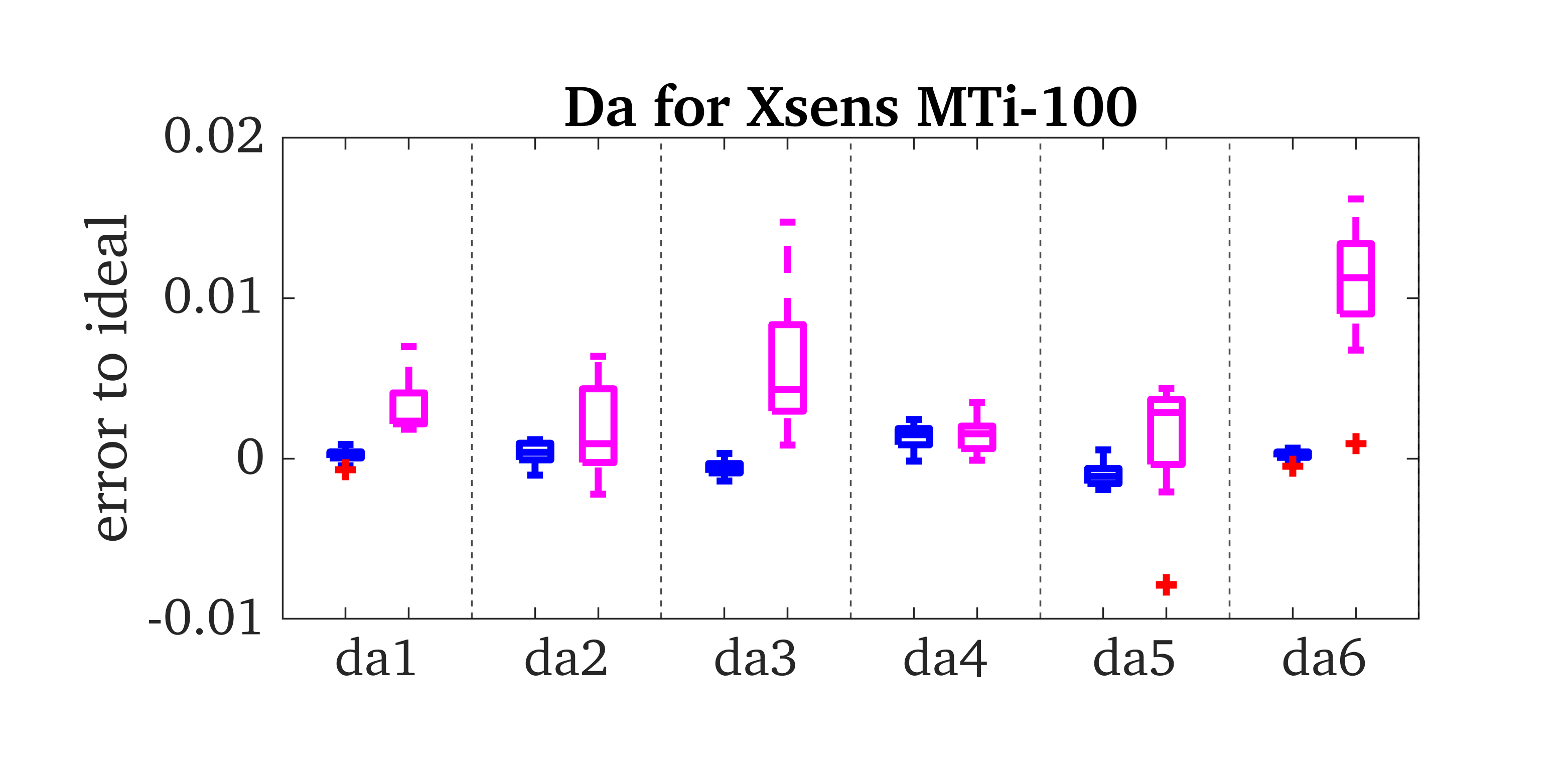}
\end{subfigure}
\begin{subfigure}{.45\textwidth}\centering
\includegraphics[trim=10mm 5mm 0 0,clip,height=1.3in]{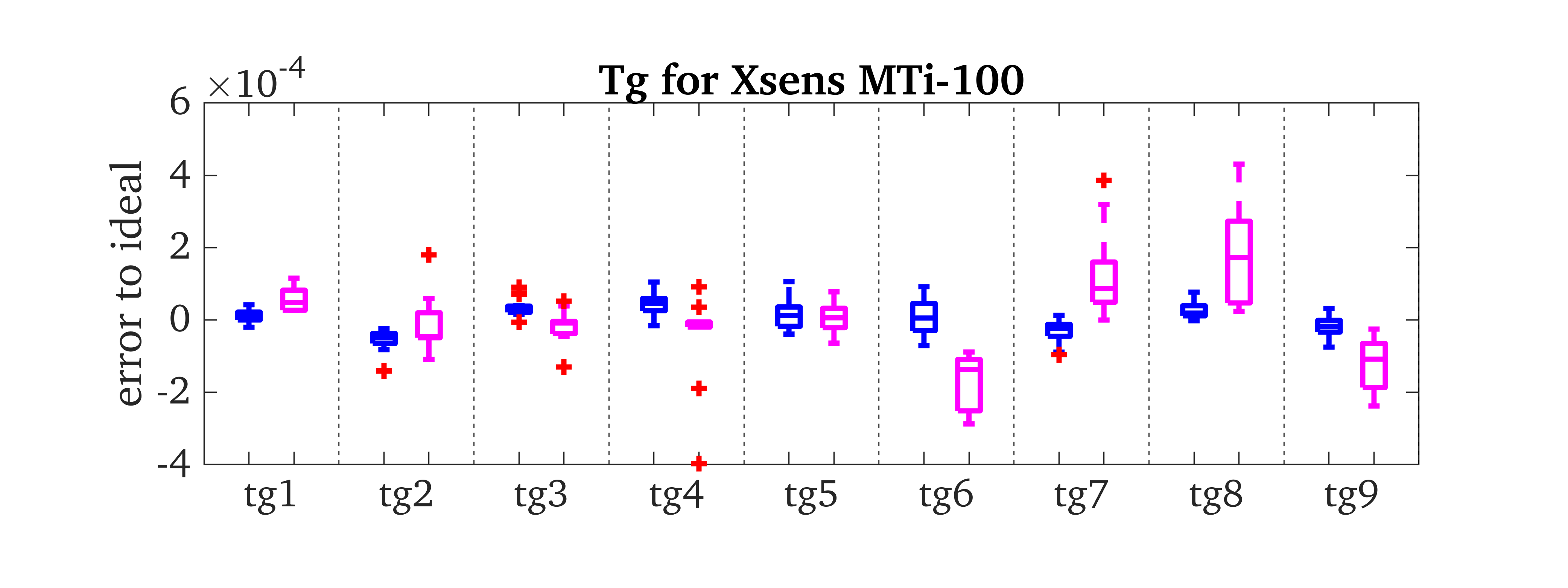}
\end{subfigure}
\begin{subfigure}{.45\textwidth}\centering
\includegraphics[trim=0 5mm 0 0,clip,height=1.3in]{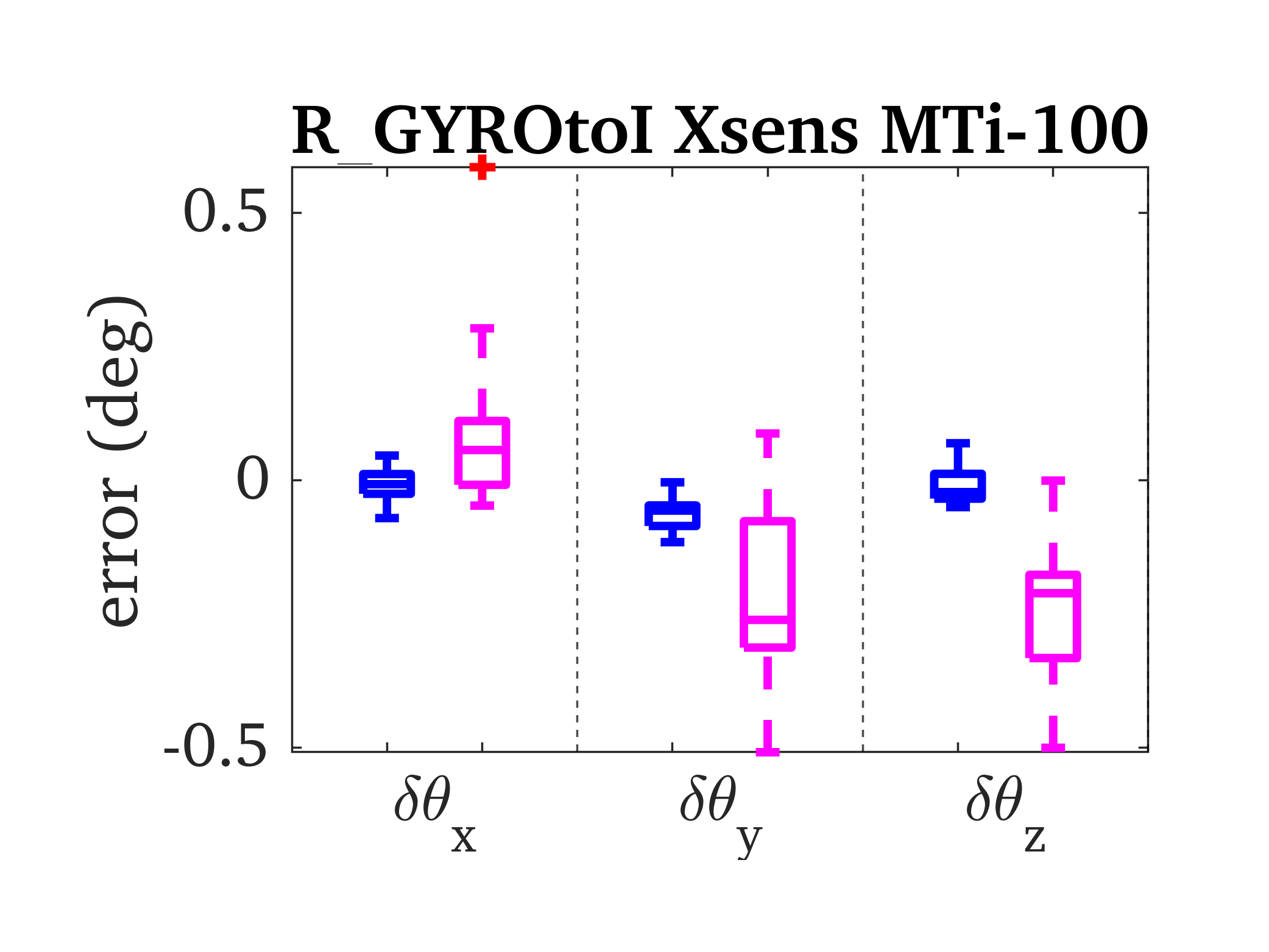}
\end{subfigure}
\caption{
Comparison of the proposed method with \textit{imu6} and Kalibr relative to the ``ideal'' sensor intrinsics.
The boxplots show the final converged value of both methods.
Kalibr (magenta, right in each group) was run with all cameras and IMUs available over 10 datasets, while the proposed system (blue, left) was run with either the Blackfly camera or left T265 fisheye resulting in 20 runs.
}
\label{fig:exp_kalibr_compare_xsens}
\end{figure*}
\begin{figure*}
\centering
\begin{subfigure}{.45\textwidth}\centering
\includegraphics[trim=0 4mm 0 0,clip,height=1.3in]{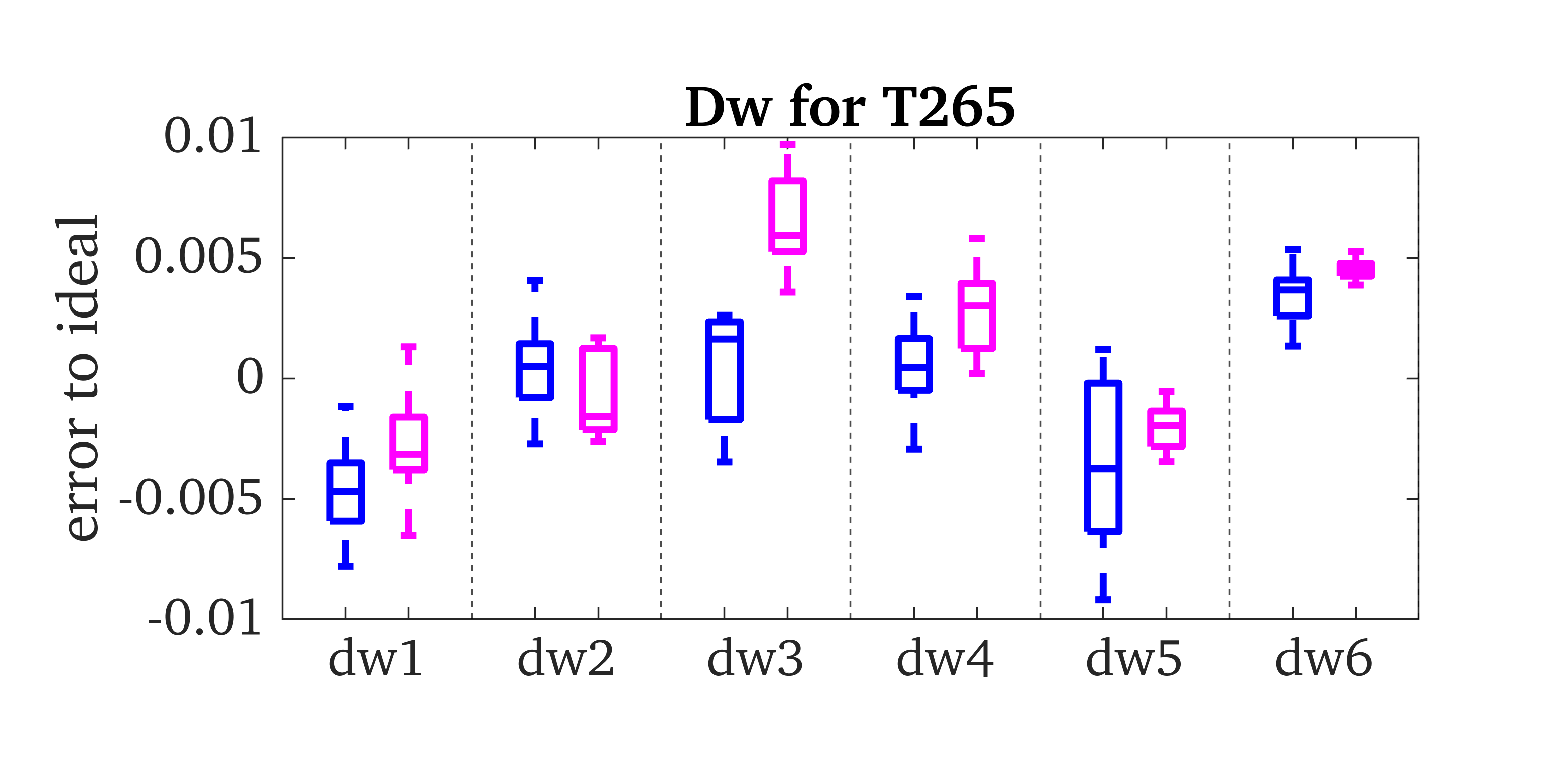}
\end{subfigure}
\begin{subfigure}{.45\textwidth}\centering
\includegraphics[trim=0 9mm 0 0,clip,height=1.3in]{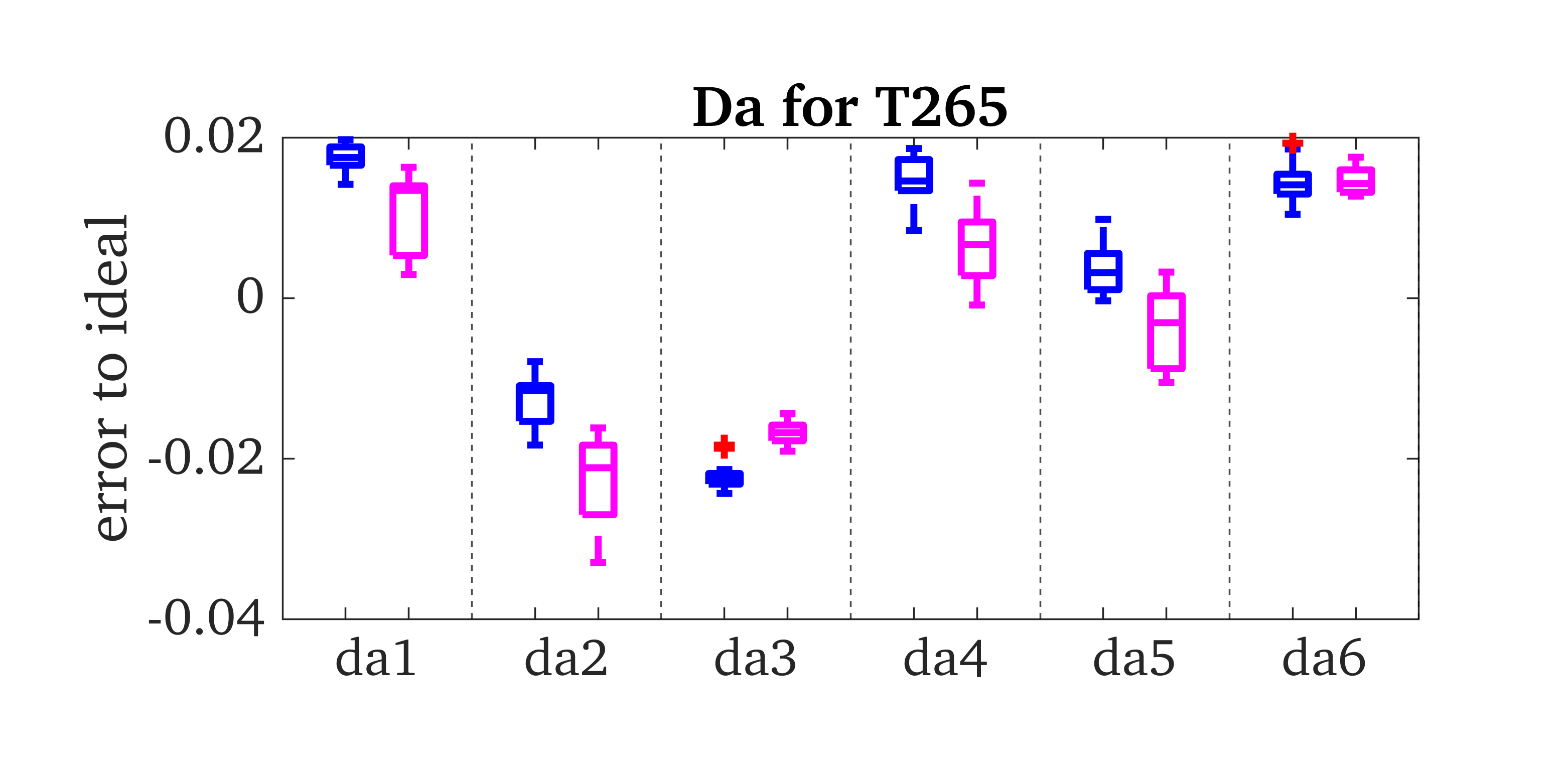}
\end{subfigure}
\begin{subfigure}{.45\textwidth}\centering
\includegraphics[trim=10mm 5mm 0 0,clip,height=1.3in]{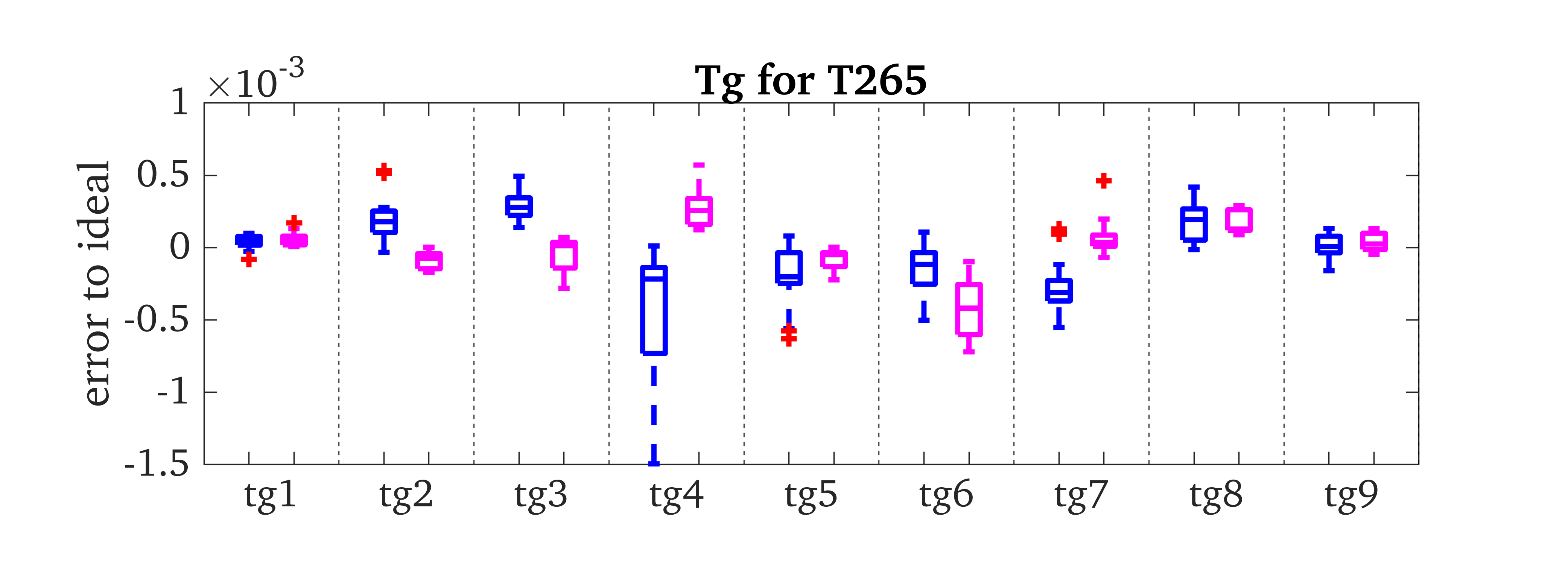}
\end{subfigure}
\begin{subfigure}{.45\textwidth}\centering
\includegraphics[trim=0 5mm 0 0,clip,height=1.3in]{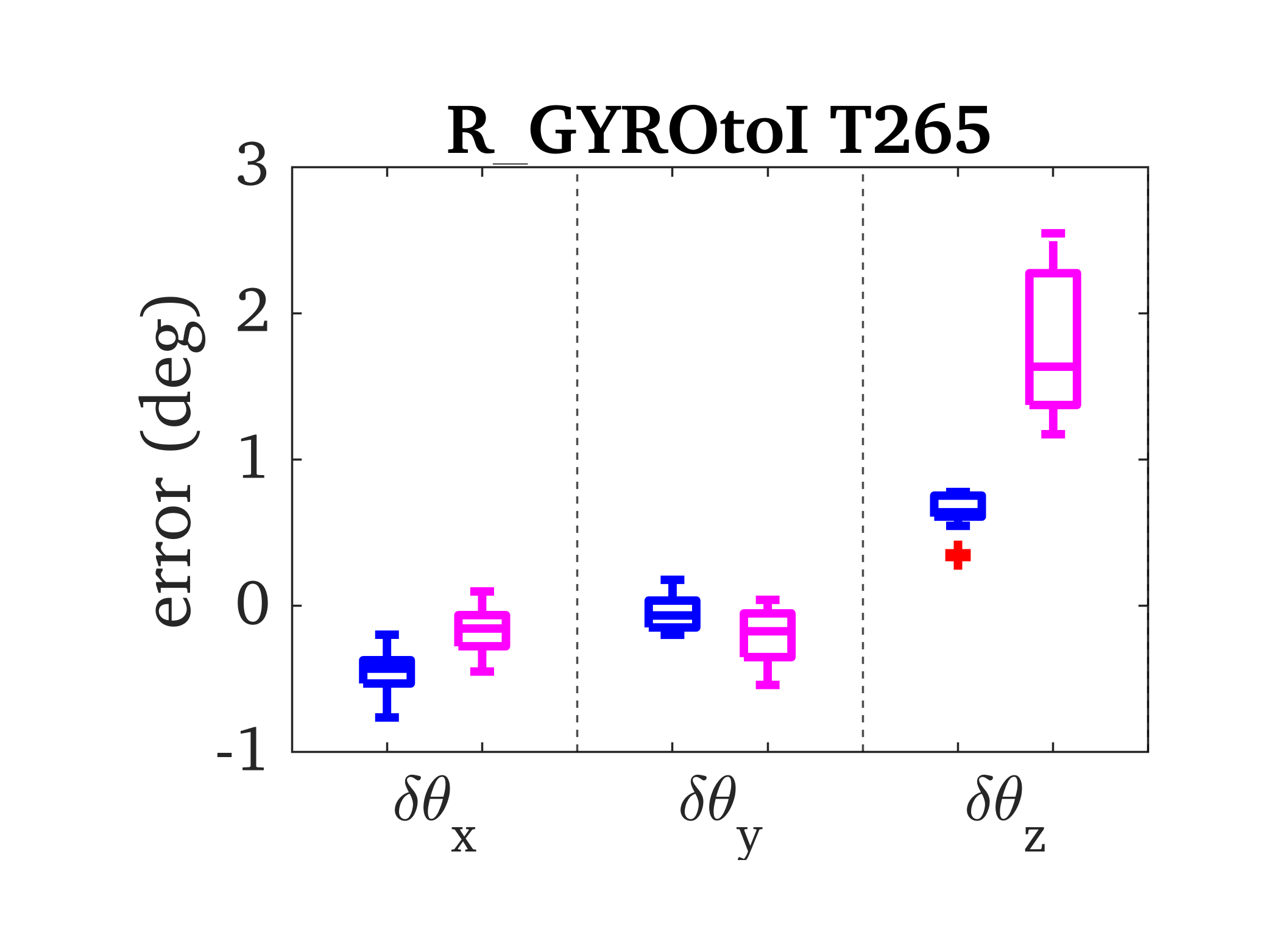}
\end{subfigure}
\caption{
Comparison of the proposed method with \textit{imu6} and Kalibr relative to the ``identity'' sensor intrinsics.
The boxplots show the final converged value of both methods.
Kalibr (magenta, right in each group) was run with all cameras and IMUs available over 10 datasets, while the proposed system (blue, left) was run with either the Blackfly camera or left T265 fisheye resulting in 20 runs.
}
\label{fig:exp_kalibr_compare_t265imu}
\end{figure*}

We further evaluate the proposed self-calibration system on a custom made visual-inertial sensor rig (VI-Rig, shown in Figure \ref{fig:vi_rig}) which contains multiple IMU and camera sensors to facilitate an investigation into how they individually impact overall performance.
Specifically, it contains a MicroStrain GX5-25, MicroStrain GX5-35, Xsens MTi 100, FLIR blackfly camera and RealSense T265 tracking camera which contains an integrated BMI055 IMU along with a fisheye stereo camera. 
Here we note that all cameras used are not rolling shutter to ensure fair comparison against the baseline Kalibr \citep{Furgale2013IROS} which only supports IMU-camera calibration with global shutter cameras.
In total 10 datasets were collected of an April tag board on which both the proposed system and the Kalibr calibration toolbox were run to report repeatbility statistics and expected real-world performance of both systems.
During data collection, all 6-axis motion of VI-Rig were excited to avoid degenerate motions for calibration parameters.

\subsection{Visual Front-End and Initial Conditions}
To provide a fair comparison, we modified the front-end of the proposed system to directly use the same April tag detection as Kalibr and to only use such tags during estimation.
Additionally, while the proposed system was only run with one of the four IMUs and either the Blackfly or left T265 Realsense camera, Kalibr used all the available sensors to ensure the highest and most consistent performance (4 IMUs and 3 cameras). 
The \textit{imu6} model is used during evaluation, which is equivalent to the \textit{scale-misalignment} IMU model of Kalibr.
We define the ``ideal'' IMU sensor intrinsics as $\mathbf{D}'_w = \mathbf{D}'_a = \mathbf{I}_3$, ${}^I_w\mathbf{R}={}^I_a\mathbf{R}=\mathbf{I}_3$ and $\mathbf{T}_g=\mathbf{0}_3$ if factory or offline calibration has been pre-applied.
Generally, these values are what the users expect for a ready-to-use IMU, and are the initial values that the proposed estimator starts from.
The quality of each IMU can be evaluated by how close the converged calibrated values are to these ``ideal'' values.

\subsection{IMU-Camera Spatiotemporal Extrinsics and Intrinsics}

We first investigate the convergence of the IMU-camera extrinsics and temporal parameters along with the camera intrinsics of the proposed system.
The results shown in Figure \ref{fig:exp_kalibr_compare_cameras} demonstrate that the proposed system is able to calibrate the spatial and temporal parameters with both high repeatability and accuracy relative to the offline Kalibr calibration baseline.
Additionally shown is the convergence of camera intrinsics estimated by the proposed algorithm relative to the Kalibr static calibration results which are fixed during their IMU-camera calibration.
Although the camera intrinsic estimates of blackfly and T265 camera have a few deviations compared to the reference values from Kalibr, the proposed system has very good convergence and high repeatability (the groundtruth is not known here).

\subsection{IMU Intrinsic Parameters}
The calibration results are summarized as boxplots shown in Figure \ref{fig:exp_kalibr_compare_ms25} - \ref{fig:exp_kalibr_compare_t265imu} for the MicroStrain MS-GX5-25, MS-GX5-35, Xsens MTi-100 and T265 IMU, respectively. 
As shown, the calibration errors of the proposed system are quite close to the results of Kalibr, and the estimate differences of $\mathbf{D}'_w$, $\mathbf{D}'_a$ and $\mathbf{T}_g$ are around 1e-3, 1e-1 and 1e-4, respectively, for MS-GX5-25, MX-GX5-35 and Xsens MTi-100.
Additionally, our proposed algorithm demonstrates better repeatability as the calibration errors have smaller variances and less outliers. 
In general, we can discuss the following results concerning the IMUs presented throughout the paper (see Figure \ref{fig:exp_kalibr_bmi160} and \ref{fig:exp_kalibr_compare_ms25} - \ref{fig:exp_kalibr_compare_t265imu}):
\begin{itemize}
    \item The MicroStrain MS-GX5-25, MS-GX5-35 and Xsens MTi-100 IMU are more close to ``ideal'' IMU than T265 IMU and BMI160 IMU. This is reasonable since both the MicroStrain and Xsens IMU are more expensive high-end IMUs with likely more sophisticate out-of-factory calibration than T265 IMU and BMI 160 IMU.
    \item For each IMU, the gravity sensitivity terms are, in general, one or two orders smaller than the other terms of the IMU intrinsic model.
    This suggests that the gravity sensitivity should not have significant effects on system performance.
    This is likely due to the handheld motion of the platform and levels of achievable acceleration magnitudes.
    \item The BMI160 IMU (Figure \ref{fig:exp_kalibr_bmi160}), has a much more significant gyroscope calibration, $\mathbf{D}'_w$, compared to its accelerometer calibration and other IMUs.
    Thus the BMI160 can see large accuracy gains from only calibrating $\mathbf{D}'_w$, while for other IMUs, the calibration of $\mathbf{D}'_a$ should be more impactful.
\end{itemize}
Hence, these results validate the accuracy and consistency of the IMU intrinsic calibration of the proposed \textit{online real-time} algorithm, which outputs comparable calibration results to Kalibr's offline calibration procedures.

\subsection{Timing Evaluation}

We also evaluate the running time for the proposed system with and without online sensor calibration shown in Table \ref{tab:timing}.
We use the 10 datasets recorded with the VI-Rig and only use the measurements from MicroStrain GX5-25 IMU and the left camera of T265 for evaluation. 
In order to get more realistic timing evacuation, no April tags are detected and only the natural features tracked from images are used. 
We track 200 features from each image and keep at most 30 SLAM point features in the state vector with a sliding window of 20 clones. The averaged execution time for processing each coming image (including propagation and update) is recorded (shown in Table \ref{tab:timing}). 
The average execution time of the proposed system with online calibration is 0.0224s, which shows negligible increases than 0.0188s, which is the average running time without online calibration.

\section{Real-World Degenerate Motion Demonstration and Analysis}
\label{sec:exp_degenerate}

\begin{table*}[t]
\renewcommand{\arraystretch}{1.5}
\caption{
Absolute Trajectory Error (ATE) on EuRoC MAV Vicon room sequences (with units degrees/meters).
}
\label{tab:ate_euroc}
\begin{adjustbox}{width=\textwidth,center}
\begin{tabular}{cccccccc} \toprule
\textbf{IMU Model} & \textbf{V1\_01\_easy} & \textbf{V1\_02\_medium} & \textbf{V1\_03\_difficult} & \textbf{V2\_01\_easy} & \textbf{V2\_02\_medium} & \textbf{V2\_03\_difficult} & \textbf{Average} \\\midrule
imu0 & 0.657 / 0.043 & 1.805 / 0.060 & 2.437 / 0.069 & 0.869 / 0.109 & 1.373 / 0.080 & 1.277 / 0.180 & 1.403 / 0.090 \\
imu1 & 0.601 / 0.055 & 1.924 / 0.065 & 2.334 / 0.073 & 1.201 / 0.115 & 1.342 / 0.086 & 1.710 / 0.168 & 1.519 / 0.094 \\
imu2 & 0.552 / 0.054 & 1.990 / 0.062 & 2.197 / 0.083 & 0.960 / 0.107 & 1.453 / 0.085 & 1.666 / 0.216 & 1.470 / 0.101 \\
imu3 & 0.606 / 0.055 & 1.905 / 0.065 & 2.359 / 0.073 & 1.180 / 0.114 & 1.335 / 0.088 & 1.640 / 0.167 & 1.504 / 0.094 \\
imu4 & 0.569 / 0.056 & 1.969 / 0.069 & 2.165 / 0.076 & 0.846 / 0.127 & 1.636 / 0.094 & 1.577 / 0.195 & 1.461 / 0.103 \\
\bottomrule
\end{tabular}
\end{adjustbox}
\end{table*}
\begin{figure*}
\centering
\begin{subfigure}{.24\textwidth}
\includegraphics[trim=0mm 0 0mm 0,clip,width=\linewidth]{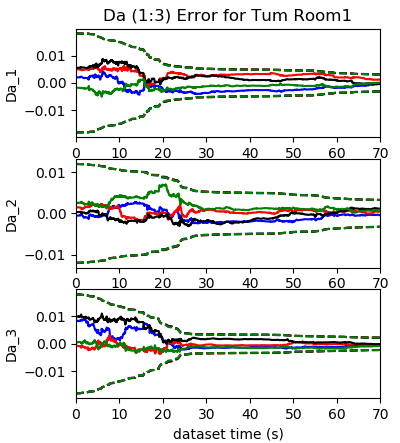}
\end{subfigure}
\begin{subfigure}{.24\textwidth}
\includegraphics[trim=0 0 0 0,clip,width=\linewidth]{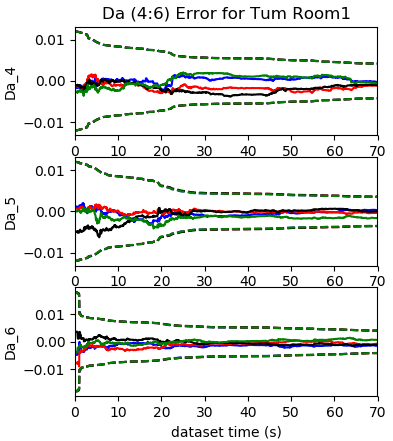}
\end{subfigure}
\begin{subfigure}{.24\textwidth}
\includegraphics[trim=0 0 0 0,clip,width=\linewidth]{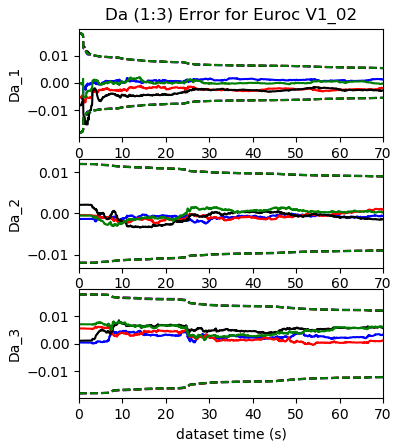}
\end{subfigure}
\begin{subfigure}{.24\textwidth}
\includegraphics[trim=0 0 0 0,clip,width=\linewidth]{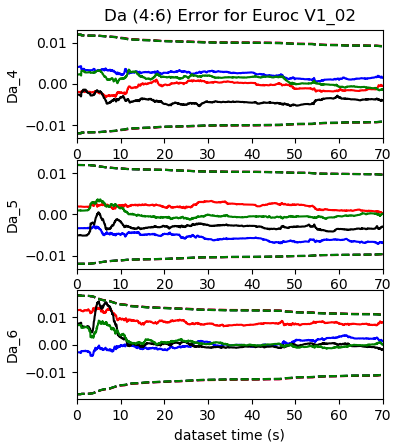}
\end{subfigure}
\caption{
Simulation results for $\mathbf{D}_a$ of the proposed system evaluated with groundtruth \textit{tum\_room1} (left) and \textit{EuRoc V1\_02} (right) trajectories using \textit{imu2} and \textit{radtan}. 
3 sigma bounds (dotted lines) and estimation errors (solid lines) for four different runs (different colors) with different realization of the measurement noise and initial perturbations are drawn. 
The \textit{tum\_room1} estimation errors and $3\sigma$ bounds converge nicely, while due to lack of motion excitation, the convergence of $\mathbf{D}_a$, especially $d_{a4}$, $d_{a5}$ and $d_{a6}$, for the \textit{EuRoc V1\_02} is poor. 
}
\label{fig:sim_tum_eth_da}
\end{figure*}
\begin{figure}
\centering
\begin{subfigure}{.24\textwidth}
    \includegraphics[trim=0mm 0 0mm 0,clip,width=\linewidth]{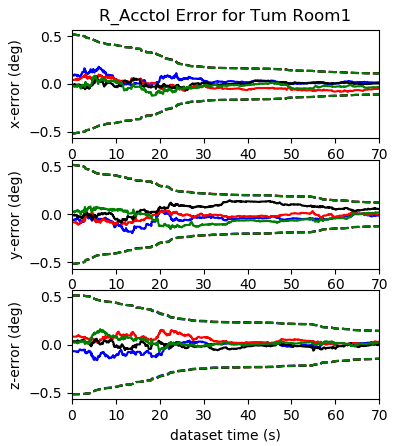}
\end{subfigure}
\begin{subfigure}{.24\textwidth}
\includegraphics[trim=0 0 0 0,clip,width=\linewidth]{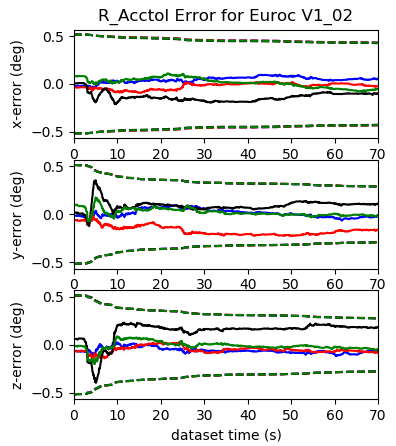}
\end{subfigure}
\caption{
Simulation results for ${}^I_a\mathbf{R}$ of the proposed system evaluated with groundtruth \textit{tum\_room1} and \textit{EuRoc V1\_02} trajectories using \textit{imu2} and \textit{radtan}. 
With \textit{tum\_room1} trajectory, the estimation errors and $3\sigma$ bounds of ${}^I_a\mathbf{R}$ converge nicely, while the convergence is poor on the \textit{EuRoc V1\_02}.
}
\label{fig:sim_tum_eth_atoI}
\end{figure}
\begin{figure}
\centering
\includegraphics[trim=25 0 50 0,clip,width=0.85\linewidth]{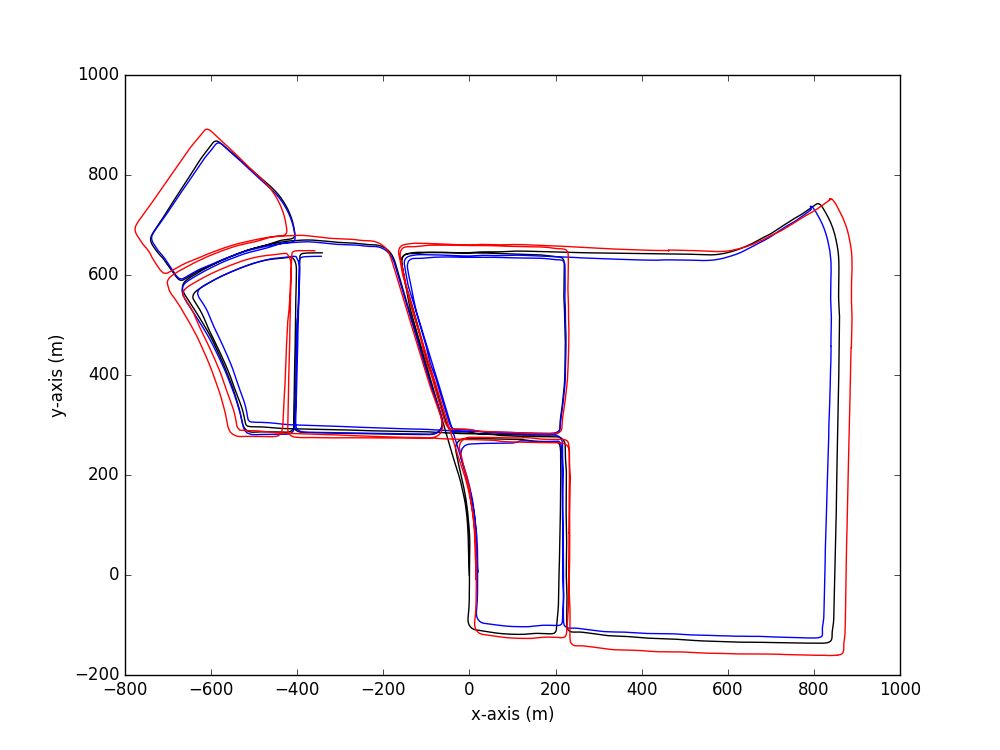}
\includegraphics[trim=25 0 50 0,clip,width=0.85\linewidth]{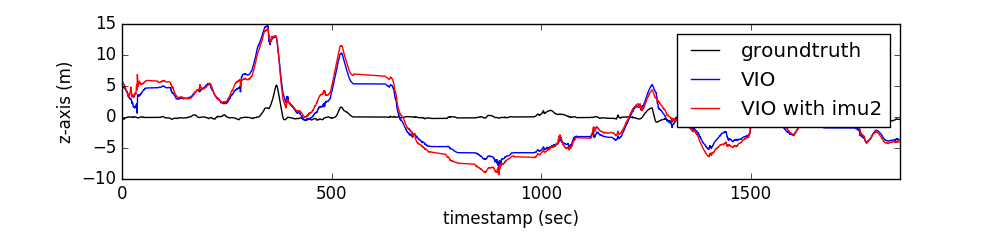}
\caption{
Trajectory plots for the KAIST Urban 39 dataset 10km in total length.
Figure is best seen in color.
}
\label{fig:exp_kaist}
\end{figure}

Next, we evaluate the proposed system on a collection of real-world datasets which exhibit varying degrees of degenerate motions.
Specifically, we evaluate on the EuRoc MAV dataset \citep{Burri2016IJRR} which has under-actuated MAV with weakly excited acceleration and approximately 1-axis rotation. We also evaluate on the KAIST complex urban dataset \citep{Jeong2019IJRR} which is a planar autonomous vehicle dataset with relatively constant velocity throughout.
We note that we do not estimate the gravity sensitivity since it has been demonstrated in the preceding sections that it is not significant for VINS performance, and test only the \textit{imu1} - \textit{imu4} models.

\subsection{EuRoC MAV: Under-Actuated Motion}

The EuRoC MAV dataset \citep{Burri2016IJRR} contains a series of trajectories from a MAV and provides 20 Hz grayscale stereo images, 200 Hz inertial readings, and an external groundtruth pose from a motion capture system.
The proposed estimator is run with just the left camera on each of the Vicon room datasets and report the results in Table \ref{tab:ate_euroc}.
Here we can see that not performing IMU intrinsic calibration, \textit{imu0} model, outperforms the methods which additionally estimate the IMU intrinsics.
This makes sense since the IMU intrinsics suffer from a large number of degenerate motions caused with constant local angular velocity and linear acceleration which can be expected for the MAV platform.
Additionally, we believe that this is specifically caused by the MAV being unable to fully excite its 6DoF motion for a given small time interval and thus undergoes (nearly) degenerate motions locally throughout the whole trajectory, hurting the sliding-window filter.
We can see that for more dynamic datasets, such as V2\_03\_difficult, there are still some partial improvements in accuracy possibly due to the more dynamic motion exhibited.

In order to verify our reasoning that the accuracy loss is caused by degenerate motions, we use the groundtruth trajectories of \textit{tum\_room1} with full 6DoF motion and \textit{EuRoc V1\_02} to simulate synthetic inertial and feature bearing measurements (see Section \ref{sec:exp_sim} on how we perform simulation) and evaluate our system with these simulated data.
Figure \ref{fig:sim_tum_eth_da} and \ref{fig:sim_tum_eth_atoI} shows four different run with estimation errors and $3\sigma$ bounds for $\mathbf{D}_a$ and ${}^I_a\mathbf{R}$.
It is clear that the motion of sensor on the \textit{EuRoc V1\_02} trajectory (right) is mildly excited and the acceleration readings are varying very slowly within the local window, causing poor convergence of the $\mathbf{D}_a$ and ${}^I_a\mathbf{R}$ with relatively flat $3\sigma$ bounds as compared to the \textit{tum\_room1} (left).
This verifies that the online IMU intrinsic calibration will benefit VINS with fully-excited motion (e.g., the \textit{tum\_room1} trajectory) and might not be a good option for under-actuated motions such as the \textit{EuRoc V1\_02} trajectory.

\begin{figure*}
\centering
\begin{subfigure}{.24\textwidth}
    \includegraphics[trim=0mm 0 0mm 0,clip,width=\linewidth]{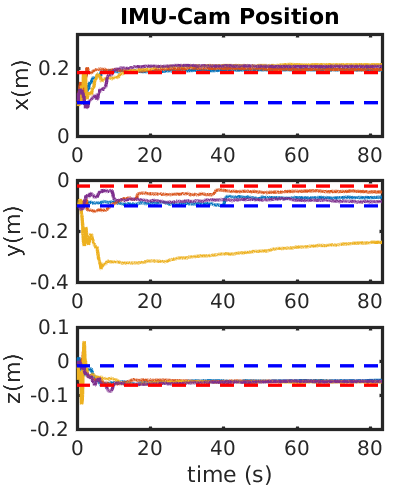}
\end{subfigure}
\begin{subfigure}{.24\textwidth}
\includegraphics[trim=0 0 0 0,clip,width=\linewidth]{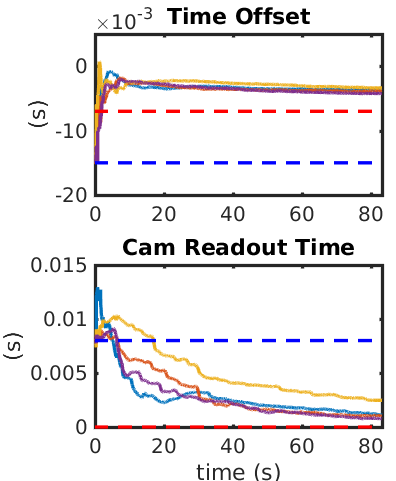}
\end{subfigure}
\centering
\begin{subfigure}{.255\textwidth}
    \includegraphics[trim=0mm 0 0mm 0,clip,width=\linewidth]{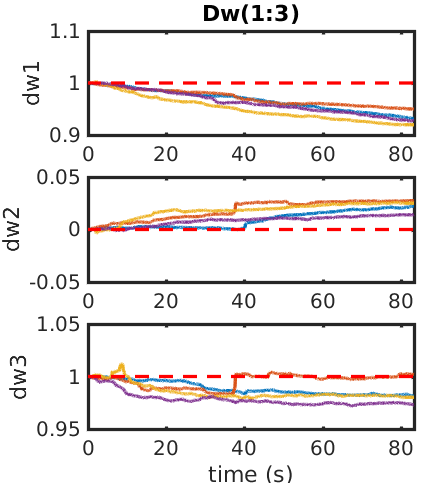}
\end{subfigure}
\begin{subfigure}{.230\textwidth}
\includegraphics[trim=0 0 0 0,clip,width=\linewidth]{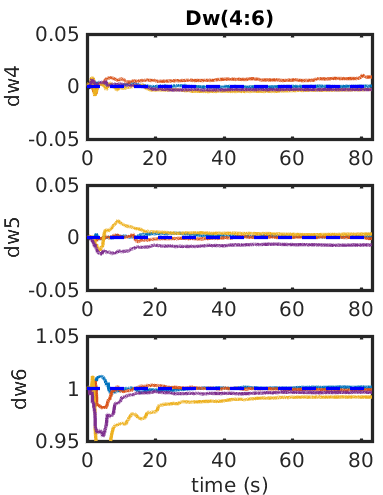}
\end{subfigure}
\caption{
Calibration results (colored solid lines) for ${}^C\mathbf{p}_I$, $t_d$, $t_r$ and $\mathbf{D}_w$ of the proposed system evaluated with four VI-Rig planar motion datasets using \textit{imu2} and \textit{equi-dist}. 
Red and blue dotted lines denote the reference value from Kalibr and initial (perturbed) values, respectively. 
Colored solid line represents the estimated calibration parameters during online calibration for each dataset.
All the temporal calibration, $dw_4$, $dw_5$ and $dw_6$ can converge well to the reference values, while the y component of ${}^C\mathbf{p}_I$, $dw_1$, $dw_2$ and $dw_3$ diverges.  
}
\label{fig:sim_virig_planar}
\end{figure*}
\begin{figure*}
\centering
\begin{subfigure}{.24\textwidth}
    \includegraphics[trim=0mm 0 0mm 0,clip,width=\linewidth]{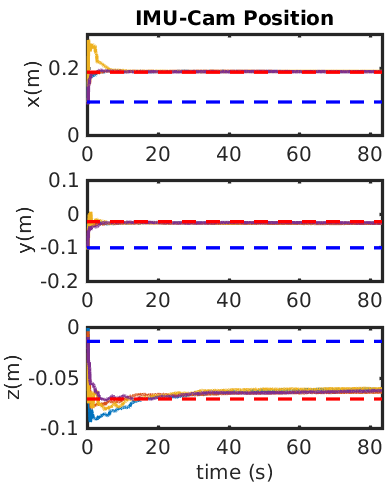}
\end{subfigure}
\begin{subfigure}{.265\textwidth}
\includegraphics[trim=0 0 0 0,clip,width=\linewidth]{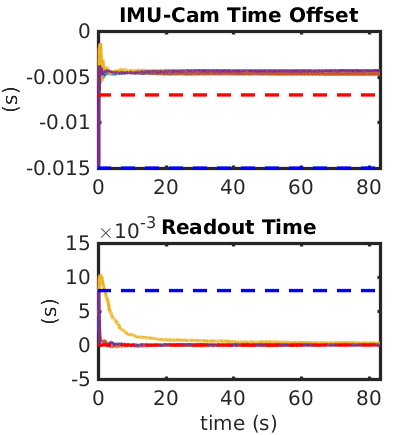}
\end{subfigure}
\centering
\begin{subfigure}{.24\textwidth}
    \includegraphics[trim=0mm 0 0mm 0,clip,width=\linewidth]{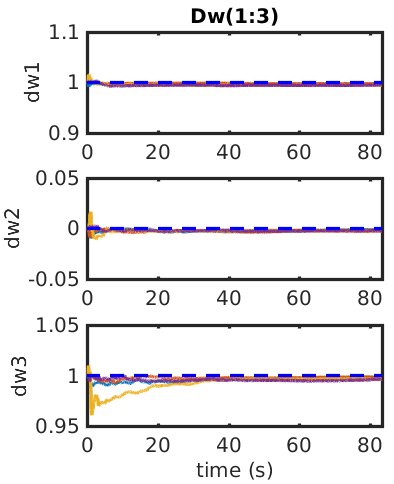}
\end{subfigure}
\begin{subfigure}{.24\textwidth}
\includegraphics[trim=0 0 0 0,clip,width=\linewidth]{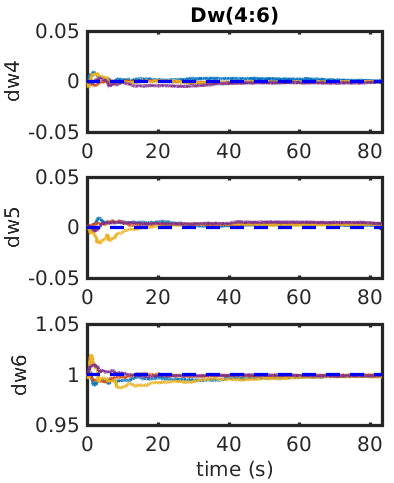}
\end{subfigure}
\caption{
Calibration results (colored solid lines) for ${}^C\mathbf{p}_I$, $t_d$, $t_r$ and $\mathbf{D}_w$ of the proposed system evaluated with four 3D-motion datasets from Section \ref{sec:exp_virig} using \textit{imu2} and \textit{equidist}. 
Red and blue dotted lines denote the reference value from Kalibr and initial (perturbed) values, respectively. 
Colored solid line represents the estimated calibration parameters during online calibration for each dataset.
All the temporal calibration, ${}^C\mathbf{p}_I$, and $\mathbf{D_w}$ can converge well to the reference values.
}
\label{fig:sim_virig_3d}
\end{figure*}

\subsection{KAIST Complex Urban: Planar Motion}
Next we evaluate on the KAIST Complex Urban dataset \citep{Jeong2019IJRR}, which provides stereo grayscale images at 10 Hz, 200 Hz inertial readings, and a pseudo-groundtruth from an offline optimization method with the RTK GPS.
As in the Table \ref{table:degenerate summary} and \ref{tab:cam degenerate}, this dataset contains planar motion and, in the case of \textit{imu2}, 6 parameters are unobservable.
We use the stereo camera pair along with the IMU to remove the scale ambiguity for monocular VINS caused by constant local acceleration \citep{Wu2017ICRA,Yang2019TRO}.

We plot Urban 39 trajectory estimates of the proposed system with and without IMU intrinsics in Figure \ref{fig:exp_kaist}. 
It can be seen that the system with IMU intrinsic calibration has larger drift compared to the one without. 
When looking at the absolute trajectory error (ATE) in respect to the dataset's groundtruth, the error of the standard VIO, is 1.58 degrees with 13.03 meters (0.12\%), while with \textit{imu2} online IMU intrinsic calibration it is 1.41 degrees with 23.13 meters (0.22\%).
We propose that this larger error is due to the introduced unobservable directions in online IMU intrinsic calibration when experiencing planar motion.

\subsection{VI-Rig Planar Motion Datasets}

Lastly, we evaluate on 4 datasets collected with VI-Rig (shown in Figure \ref{fig:vi_rig}) under planar motion.
In this evaluation, Microstrain GX5-25 and the left camera of T265 are used. 
When collecting data, we put the VI-Vig on a chair and only move the chairs in the ground plane to make sure VI-Vig is performing planar motion with global yaw as rotation axis, which is also the y-axis (pointing downward) of the camera. 
We calibrate all calibration parameters using \textit{imu2} and \textit{equidist} when running the system. 
Since T265 is a global shutter camera, the readout time is 0.00s. 

The calibration results for the translation parameters ${}^C\mathbf{p}_I$, time offset $t_d$, readout time $t_r$ and the $\mathbf{D}w$ are shown in Figure \ref{fig:sim_virig_planar}.
All the temporal calibration can converge well to the reference values based on offline calibration results of Kalibr.
Note that $t_r$ converges to almost 0s as expected and $t_d$ converges from 0.015s to 0.005s with reference values as 0.007s. The final estimation errors are around 0.002s, which is pretty small. 
While the x and z components of ${}^C\mathbf{p}_I$ can also converge well to the reference values with small standard deviations (smaller than 0.4cm), the y component diverges with estimation errors more than 5cm and the standard deviation reach 3cm since it is along the rotation axis of the camera and hence, unobservable. 
Since the system has only yaw rotation for the IMU sensor, the $dw_1$, $dw_2$ and $dw_3$ are also unobservable (see Table \ref{table:degenerate summary}), and their calibration results diverge a lot compared to those of $dw_4$, $dw_5$ and $dw_6$. 
This result verifies our degenerate motion analysis for IMU-camera and IMU intrinsic calibration.   
As comparison, we also plot the online calibration results of the proposed system running on another four datasets from Section \ref{sec:exp_virig} with fully excited motions in Figure \ref{fig:sim_virig_3d}.
We use the same scale to plot the results for both Figure \ref{fig:sim_virig_planar} and \ref{fig:sim_virig_3d}. It is clear that all these calibration parameters ($t_r$, $t_d$, ${}^C\mathbf{p}_I$ and $\mathbf{D}_w$) can converge better in fully excited motions than planar motion.

\section{Discussion: Online Self-Calibration?} \label{sec:discussion}

As learnt from the preceding extensive Monte-Carlo simulations and real world experiments, 
we highly recommend online self-calibration for VINS  especially in the following scenarios: 
\begin{itemize}
    \item Poor calibration priors are provided.
    \item Low-end IMUs or cameras are used. 
    \item RS cameras are used. 
    \item The system undergoes fully-excited motions.
\end{itemize}

Specifically, as shown in Table \ref{tab:sim_calibcompare} (simulation) and Figures \ref{fig:surf_no_rs} and \ref{fig:surf_rs} (TUM RS VIO Datasets), if the system starts with imperfect calibration, the system without online self-calibration is highly likely to fail, clearly demonstrated by Figure \ref{fig:sim_perturb_errors}.
But online calibration can greatly improve the system robustness and accuracy.
From Figure \ref{fig:surf_rs}, we see that online calibration for the low-end IMU (Bosch BMI160) and RS readout time is necessary, which can improve the system performance greatly. 
Based on the results from the EuRoC MAV dataset (Table \ref{tab:ate_euroc}) and KAIST datasets (Figure \ref{fig:exp_kaist}), online calibration, especially IMU intrinsic calibration, can hurt the system performance when the system undergoes underactuated motions.

Based on our analysis on the degenerate motions for these calibration parameters, we can recommend what calibration parameters should be calibrated with different motion types: 
\begin{itemize}
    \item Fully excited 3D motion: all calibration parameters.
    \item Mildly excited 3D motion: IMU-camera spatial-temporal calibration, readout time, and camera intrinsics.
    \item Under-actuated motions: readout time and camera intrinsics. 
\end{itemize}

As shown in our degenerate motion analysis, there are a large number of motion types that prohibit accurate calibration of the IMU intrinsics and IMU-camera spatial calibration, while the camera intrinsics and IMU-camera temporal calibration are more robust to different motions.
More importantly, in the  most commonly-seen motion cases of aerial and ground vehicles,
there is usually at least one unobservable direction due to these robots traveling with either underactuated 3D or planar motion.
The impact on performance was shown with the EuRoC MAV and KAIST Urban datasets (Table \ref{tab:ate_euroc} and Figure \ref{fig:exp_kaist}) where the use of online IMU calibration may hurt the estimator.

Due to the high likelihood of experiencing degenerate motions for some periods of time, 
solely based on our analysis and results,
we do {\em not} recommend performing online IMU intrinsic and IMU-camera spatial calibration during real-time operations for most underactuated motions (e.g., planar motion and one-axis rotation for most ground vehicles).
The exception to this is the handheld cases (e.g., mobile AR/VR), which often exhibit full 6DoF motions and thus is recommended to perform online calibration to improve estimation accuracy, especially when low-end IMUs or RS cameras are used.
For both of these applications, we do recommend using an offline batch optimization to obtain an accurate initial calibration guess for the filter and/or treat the calibration parameters (especially intrinsics) as ``true'' if one knows they are going to experience degenerate motions.
For online IMU intrinsic calibration, it is not necessary to calibrate the full IMU model and instead
one may  calibrate only the dominating parameters in the inertial models 
(e.g., $\mathbf{D}_w$ for  BMI160 IMU or  $\mathbf{D}_a$ for MicroStrain, Xsens and T265 IMUs).

\section{Conclusions and Future Work} \label{sec:conclusion}

In this paper, we address the problem of online full-parameter self-calibration for visual-inertial navigation to achieve accurate and robust performance. 
We first investigate different IMU intrinsic model variants which contain scale correction, axis misalignment and gravity sensitivity. These model variants can cover most used inertial models in practice. 
We also introduce the full visual measurement model which contains IMU-camera spatial-temporal parameters including rolling shutter readout time.  
After computing the state transition matrix and measurements Jacobian regarding the state containing full calibration parameters, we perform observability analysis based on the linearized VINS system and show that with full-parameter calibration it still has only 4 unobservable directions, which relate to global yaw and global translation. 
All the calibration parameters of VINS are observable given fully excited motions. 

Based on the observability analysis, we, for the first time, have identified basic degenerate motion patterns for IMU/camera intrinsics, and any combination of these degenerate motions will still be unobservable directions.
Extensive validation on simulated and real-world datasets are performed to verify both the observability and degenerate motion analysis.
We also show that online self-calibration can improve the robustness and accuracy of VINS. 
As shown through our experiments, online IMU intrinsic calibration is risky due to its dependence on the motion profile to ensure observability.
In the case of autonomous (ground) vehicles, most trajectories have degenerate motions, thus resulting in   {\em not} recommending online calibration of IMU intrinsics for robots with underactuated motions.
In the case of handheld motion, however we found that the estimation of calibration parameters improved performance as expected.

In the future, 
we will investigate a complete degenerate motion analysis for multi-visual-inertial system. 
In addition, robust algorithms to perform online calibration under degenerate motions will also be studied.

\bibliographystyle{SageH}
\bibliography{xins.bib}

\section{Appendix A: IMU Intrinsic Jacobians}
\label{adp:imu jacobians}

In the following derivations, we will compute the Jacobians for all the variables that might appear in the IMU models, including scale/axis correction for gyroscope $\mathbf{D}_w$ (6 parameters), scale/axis correction for accelerometer $\mathbf{D}_a$ (6 parameters), rotation from gyroscope to IMU frame ${}^I_w\mathbf{R}$, rotation from accelerometer to IMU frame ${}^I_a\mathbf{R}$ and gravitiy sensitivity $\mathbf{T}_g$ (9 parameters). 
We repeat the corrected IMU readings for easier derivation:
\begin{align}
    {}^I\boldsymbol{\omega} & = 
    {}^I_{w}\mathbf{R} \mathbf{D}_w 
    \left(
    {}^w\boldsymbol{\omega}_m - \mathbf{T}_g {}^I\mathbf{a} - \mathbf{b}_g - \mathbf{n}_g
    \right)
    \\
    {}^I\mathbf{a} & = {}^I_a\mathbf{R} \mathbf{D}_a 
    \left(
    {}^a\mathbf{a}_m - \mathbf{b}_a - \mathbf{n}_a    \right)
\end{align}
To simplify the derivations, we define ${}^I\hat{\mathbf{a}}$ and ${}^I\tilde{\mathbf{a}}$ as: 
\begin{align}
    {}^I{\mathbf{a}} &   \simeq 
    {}^I_a\hat{\mathbf{R}} \hat{\mathbf{D}}_a 
    \left(
    {}^a\mathbf{a}_m - \hat{\mathbf{b}}_a
    \right)
     + 
    {}^I_a\hat{\mathbf{R}} \mathbf{H}_{Da} \tilde{\mathbf{x}}_{Da} 
    \notag
    \\
    &
    ~~~
    + 
    \lfloor {}^I_a\hat{\mathbf{R}} \hat{\mathbf{D}}_a 
    \left(
    {}^a\mathbf{a}_m - \hat{\mathbf{b}}_a
    \right)  \rfloor 
    \delta \boldsymbol{\theta}_{Ia}
    - {}^I_a\hat{\mathbf{R}} \hat{\mathbf{D}}_a \tilde{\mathbf{b}}_a
    - {}^I_a\hat{\mathbf{R}} \hat{\mathbf{D}}_a {\mathbf{n}}_a
    \notag
    \\
    {}^I\hat{\mathbf{a}} & = {}^I_a\hat{\mathbf{R}} \hat{\mathbf{D}}_a 
    \left(
    {}^a\mathbf{a}_m - \hat{\mathbf{b}}_a
    \right) 
    \notag
    \\
    {}^I\tilde{\mathbf{a}} & = 
    {}^I_a\hat{\mathbf{R}} \mathbf{H}_{Da} \tilde{\mathbf{x}}_{Da} 
    + 
    \lfloor {}^I\hat{\mathbf{a}} \rfloor 
    \delta \boldsymbol{\theta}_{Ia}
    - {}^I_a\hat{\mathbf{R}} \hat{\mathbf{D}}_a \tilde{\mathbf{b}}_a
    - {}^I_a\hat{\mathbf{R}} \hat{\mathbf{D}}_a {\mathbf{n}}_a
    \notag
\end{align}
We define ${}^I\hat{\boldsymbol{\omega}}$ and ${}^I\tilde{\boldsymbol{\omega}}$ as:
\begin{align}
    {}^{I}\boldsymbol{\omega} & = 
    {}^I_w\hat{\mathbf{R}} \hat{\mathbf{D}}_w
    \left(
    {}^w\boldsymbol{\omega}_m - \hat{\mathbf{T}}_g {}^I\hat{\mathbf{a}}-\hat{\mathbf{b}}_g
    \right) 
    \notag
    \\
    & 
    + \lfloor 
    {}^I_w\hat{\mathbf{R}} \hat{\mathbf{D}}_w
    \left(
    {}^w\boldsymbol{\omega}_m - \hat{\mathbf{T}}_g {}^I\hat{\mathbf{a}}-\hat{\mathbf{b}}_g
    \right) 
    \rfloor \delta \boldsymbol{\theta}_{Iw} 
    \notag
    \\
    &
    + {}^I_w\hat{\mathbf{R}} \mathbf{H}_{Dw} \tilde{\mathbf{x}}_{Dw}
    - {}^I_w\hat{\mathbf{R}}\hat{\mathbf{D}}_w 
    \left(
    \hat{\mathbf{T}}_g {}^I\tilde{\mathbf{a}}
    + \mathbf{H}_{Tg}\tilde{\mathbf{x}}_{Tg} 
    \right)
    \notag
    \\
    & 
    - {}^I_w\hat{\mathbf{R}}\hat{\mathbf{D}}_w 
    \left(
    \tilde{\mathbf{b}}_g 
    +\mathbf{n}_g
    \right)
    \notag
    \\
    {}^{I}\hat{\boldsymbol{\omega}} & = 
    {}^I_w\hat{\mathbf{R}} \hat{\mathbf{D}}_w
    \left(
    {}^w\boldsymbol{\omega}_m - \hat{\mathbf{T}}_g {}^I\hat{\mathbf{a}}-\hat{\mathbf{b}}_g
    \right) 
    \notag
    \\
    {}^{I}\tilde{\boldsymbol{\omega}} & = 
    - {}^I_w\hat{\mathbf{R}}\hat{\mathbf{D}}_w \tilde{\mathbf{b}}_g
    + {}^I_w\hat{\mathbf{R}}\hat{\mathbf{D}}_w \hat{\mathbf{T}}_g {}^I_a\hat{\mathbf{R}} \hat{\mathbf{D}}_a \tilde{\mathbf{b}}_a
    \notag
    \\
    & 
    + {}^I_w\hat{\mathbf{R}}\mathbf{H}_{Dw}\tilde{\mathbf{x}}_{Dw}
    - {}^I_w\hat{\mathbf{R}}\hat{\mathbf{D}}_w \hat{\mathbf{T}}_g  {}^I_a\hat{\mathbf{R}}\mathbf{H}_{Da}\tilde{\mathbf{x}}_{Da}
    \notag
    \\
    &
    + \lfloor {}^I\hat{\boldsymbol{\omega}} \rfloor \delta \boldsymbol{\theta}_{Iw} 
    - {}^I_w\hat{\mathbf{R}}\hat{\mathbf{D}}_w \hat{\mathbf{T}}_g
    \lfloor {}^I\hat{\mathbf{a}} \rfloor \delta \boldsymbol{\theta}_{Ia} 
    \notag
    \\ 
    &
    - {}^I_w\hat{\mathbf{R}}\hat{\mathbf{D}}_w \mathbf{H}_{Tg}
    \tilde{\mathbf{x}}_{Tg} - {}^I_w\hat{\mathbf{R}} \hat{\mathbf{D}}_w \mathbf{n}_g 
    \notag
    \\
    &
    + {}^I_w\hat{\mathbf{R}}\hat{\mathbf{D}}_w \hat{\mathbf{T}}_g {}^I_a\hat{\mathbf{R}}\hat{\mathbf{D}}_a \mathbf{n}_a 
    \notag
\end{align}
where we have:
\begin{align}
    \mathbf{H}_{Dw} & = 
    \begin{bmatrix}
    {}^w\hat{w}_1\mathbf{e}_1 & {}^w\hat{w}_2\mathbf{e}_1 & {}^w\hat{w}_2\mathbf{e}_2 & {}^w\hat{w}_3 \mathbf{I}_3
    \end{bmatrix}
    \\
    \mathbf{H}_{Da} & = 
    \begin{bmatrix}
    {}^a\hat{a}_1\mathbf{e}_1 & {}^a\hat{a}_2\mathbf{e}_1 & {}^a\hat{a}_2\mathbf{e}_2 & {}^a\hat{a}_3 \mathbf{I}_3
    \end{bmatrix}
    \\
    \mathbf{H}_{Tg} & = 
    \begin{bmatrix}
    {}^I\hat{a}_1 \mathbf{I}_3 & {}^I\hat{a}_2 \mathbf{I}_3 & {}^I\hat{a}_3 \mathbf{I}_3 
    \end{bmatrix}
\end{align}
By summarizing the above equations, we have: 
\begin{align}
    \label{eq:wa}
    \begin{bmatrix}
    {}^{I_k}\tilde{\boldsymbol{\omega}} \\
    {}^{I_k}\tilde{\mathbf{a}}
    \end{bmatrix}
    & = 
    \begin{bmatrix}
    \mathbf{H}_b & \mathbf{H}_{in}
    \end{bmatrix}
    \begin{bmatrix}
    \tilde{\mathbf{x}}_{b} \\
    \tilde{\mathbf{x}}_{in}
    \end{bmatrix}
    +
    \mathbf{H}_n 
    \begin{bmatrix}
    \mathbf{n}_{g} \\
    \mathbf{n}_{a}
    \end{bmatrix}
\end{align}
where we have defined: 
\begin{align}
    \mathbf{H}_b & = \mathbf{H}_n = 
    \begin{bmatrix}
    -{}^I_w\hat{\mathbf{R}} \hat{\mathbf{D}}_w & 
    {}^I_w\hat{\mathbf{R}} \hat{\mathbf{D}}_w
    \hat{\mathbf{T}}_g {}^I_a\hat{\mathbf{R}} \hat{\mathbf{D}}_a \\
    \mathbf{0}_3 & -{}^I_a\hat{\mathbf{R}}\hat{\mathbf{D}}_a
    \end{bmatrix}
    \\
    \mathbf{H}_{in} & = 
    \begin{bmatrix}
    \mathbf{H}_w & \mathbf{H}_a & \mathbf{H}_{Iw} & \mathbf{H}_{Ia} & \mathbf{H}_{g}
    \end{bmatrix}
\end{align}
with: 
\begin{align}
    \mathbf{H}_w & = 
    \begin{bmatrix}
    {}^I_w\hat{\mathbf{R}} \mathbf{H}_{Dw} \\
    \mathbf{0}_3
    \end{bmatrix}
    \\
    \mathbf{H}_a & = 
    \begin{bmatrix}
    - {}^I_w\hat{\mathbf{R}}\hat{\mathbf{D}}_w \hat{\mathbf{T}}_g {}^I_a\hat{\mathbf{R}}\mathbf{H}_{Da} \\
    {}^I_a\hat{\mathbf{R}}\mathbf{H}_{Da}
    \end{bmatrix}
    \\
    \mathbf{H}_{Iw} & = 
    \begin{bmatrix}
    \lfloor {}^{I}\hat{\boldsymbol{\omega}} \rfloor \\
    \mathbf{0}_3
    \end{bmatrix}
    \\
    \mathbf{H}_{Ia} & = 
    \begin{bmatrix}
    -{}^I_a\hat{\mathbf{R}} \hat{\mathbf{D}}_w \hat{\mathbf{T}}
    \lfloor {}^{I}\hat{\mathbf{a}} \rfloor \\
    \lfloor {}^I\hat{\mathbf{a}} \rfloor
    \end{bmatrix}
    \\
    \mathbf{H}_{g} & = 
    \begin{bmatrix}
    -{}^I_w\hat{\mathbf{R}} \hat{\mathbf{D}}_w \mathbf{H}_{Tg} \\
    \mathbf{0}_3
    \end{bmatrix}
\end{align}
Hence, $\Delta \mathbf{R}_k$ from Eq. \eqref{eq:integration_components_1}, can be written as:
\begin{align}
    \Delta \mathbf{R}_k & 
    \simeq  \exp \left( {}^{I_{k}}\boldsymbol{\omega} \delta t_k \right) 
    = \exp \left( ({}^{I_{k}}\hat{\boldsymbol{\omega}} + {}^{I_k}\tilde{\boldsymbol{\omega}}) \delta t_k \right) 
    \\
    &
    = 
    \exp \left(\Delta \hat{\boldsymbol{\theta}}_k\right)
    \exp \left(
    \mathbf{J}_r(\Delta \hat{\boldsymbol{\theta}}_k)  {}^{I_k}\tilde{\boldsymbol{\omega}} \delta t_k
    \right)
\end{align}
where $\Delta \hat{\boldsymbol{\theta}}_k = 
{}^{I_k}\hat{\boldsymbol{\omega}}\delta t_k$. 

$\Delta \mathbf{p}_k$ from Eq. \eqref{eq:integration_components_2} can be written as:
\begin{align}
    \Delta \mathbf{p}_{k} & = 
    \scalemath{0.9}{\int^{t_{k+1}}_{t_{k}} \int^{s}_{t_{k}} {}^{I_k}_{I_\tau}\mathbf{R} {}^{I_{\tau}} \mathbf{a}  d \tau d s}
    \notag
    \\
    & \simeq
    \int^{t_{k+1}}_{t_{k}} \int^{s}_{t_{k}} 
    \exp \left( 
    {}^{I_k}\boldsymbol{\omega} \delta \tau
    \right)
    {}^{I_{k}} \mathbf{a}  
    d \tau d s
    \notag
    \\
    & =
    \int^{t_{k+1}}_{t_{k}} \int^{s}_{t_{k}} 
    \exp \left( 
    \left(
    {}^{I_k}\hat{\boldsymbol{\omega}} + 
    {}^{I_k}\tilde{\boldsymbol{\omega}} 
    \right)\delta \tau 
    \right)
    \left(
    {}^{I_{k}} \hat{\mathbf{a}} + {}^{I_{k}} \tilde{\mathbf{a}} 
    \right) 
    d \tau d s
    \notag
    \\
    & \simeq 
    \underbrace{\int^{t_{k+1}}_{t_{k}} \int^{s}_{t_{k}} 
    \exp \left( 
    {}^{I_k}\hat{\boldsymbol{\omega}}\delta \tau 
    \right)
    {}^{I_{k}} \hat{\mathbf{a}}
    d \tau d s}_{\Delta \hat{\mathbf{p}}_k}
    \notag
    \\
    & -
    \underbrace{\int^{t_{k+1}}_{t_{k}} \int^{s}_{t_{k}} 
    \exp \left( 
    {}^{I_k}\hat{\boldsymbol{\omega}}\delta \tau 
    \right)
    \lfloor {}^{I_{k}} \hat{\mathbf{a}} \rfloor
    \mathbf{J}_r({}^{I_k}\boldsymbol{\omega} \delta \tau) 
    \delta \tau
    d \tau d s}_{\boldsymbol{\Xi}_4}
    {}^{I_k}\tilde{\boldsymbol{\omega}}
    \notag
    \\
    & + 
    \underbrace{\int^{t_{k+1}}_{t_{k}} \int^{s}_{t_{k}} 
    \exp \left( 
    {}^{I_k}\hat{\boldsymbol{\omega}}\delta \tau 
    \right) d \tau d s}_{\boldsymbol{\Xi}_2}
    {}^{I_k}\tilde{\mathbf{a}}
    \notag
\end{align}
$\Delta \mathbf{v}_k$ from Eq. \eqref{eq:integration_components_3} can be written as:
\begin{align}
    \Delta \mathbf{v}_{k} & =  
    \scalemath{0.9}{
    \int^{t_{k+1}}_{t_{k}}
    {}^{I_k}_{I_\tau}\mathbf{R} {}^{I_{\tau}} \mathbf{a}  d \tau
    }
    \notag
    \\
    & \simeq
    \int^{t_{k+1}}_{t_{k}} 
    \exp \left( 
    {}^{I_k}\boldsymbol{\omega} \delta \tau
    \right)
    {}^{I_{k}} \mathbf{a}  
    d \tau 
    \notag
    \\
    & =
    \int^{t_{k+1}}_{t_{k}} 
    \exp \left( 
    \left(
    {}^{I_k}\hat{\boldsymbol{\omega}} + 
    {}^{I_k}\tilde{\boldsymbol{\omega}} 
    \right)\delta \tau 
    \right)
    \left(
    {}^{I_{k}} \hat{\mathbf{a}} + {}^{I_{k}} \tilde{\mathbf{a}} 
    \right) 
    d \tau 
    \notag
    \\
    & \simeq 
    \underbrace{\int^{t_{k+1}}_{t_{k}} 
    \exp \left( 
    {}^{I_k}\hat{\boldsymbol{\omega}}\delta \tau 
    \right)
    {}^{I_{k}} \hat{\mathbf{a}}
    d \tau }_{\Delta\hat{\mathbf{v}}_k}
    \notag
    \\
    & -
    \underbrace{\int^{t_{k+1}}_{t_{k}} 
    \exp \left( 
    {}^{I_k}\hat{\boldsymbol{\omega}}\delta \tau 
    \right)
    \lfloor {}^{I_{k}} \hat{\mathbf{a}} \rfloor
    \mathbf{J}_r({}^{I_k}\boldsymbol{\omega} \delta \tau) 
    \delta \tau
    d \tau }_{\boldsymbol{\Xi}_3}
    {}^{I_k}\tilde{\boldsymbol{\omega}}
    \notag
    \\
    & + 
    \underbrace{\int^{t_{k+1}}_{t_{k}} 
    \exp \left( 
    {}^{I_k}\hat{\boldsymbol{\omega}}\delta \tau 
    \right) d \tau}_{\boldsymbol{\Xi}_1}
    {}^{I_k}\tilde{\mathbf{a}}
    \notag
\end{align}
By summarizing the above derivations, we have: 
\begin{align}
    \Delta \mathbf{R}_k & = 
    \Delta \hat{\mathbf{R}}_k
    \exp \left(
    \mathbf{J}_r (\Delta \hat{\boldsymbol{\theta}}_k)
    {}^{I_k}\tilde{\boldsymbol{\omega}}\delta t_k
    \right)
    \\
    \Delta \mathbf{p}_k & = 
    \Delta \hat{\mathbf{p}}_k -\boldsymbol{\Xi}_4 {}^{I_k}\tilde{\boldsymbol{\omega}} + 
    \boldsymbol{\Xi}_2 {}^{I_k}\tilde{\mathbf{a}} 
    \\
    \Delta \mathbf{v}_k & = 
    \Delta \hat{\mathbf{v}}_k -\boldsymbol{\Xi}_3 {}^{I_k}\tilde{\boldsymbol{\omega}} + 
    \boldsymbol{\Xi}_1 {}^{I_k}\tilde{\mathbf{a}} 
\end{align}
The linearized model for IMU dynamics can be written as:
\begin{align}
    \label{eq:xn_dynamics}
    \tilde{\mathbf{x}}_{n_{k+1}} & = 
    \boldsymbol{\Phi}_{nn} \tilde{\mathbf{x}}_{n_{k}} + 
    \boldsymbol{\Phi}_{wa} 
    \begin{bmatrix}
    {}^{I_k}\tilde{\boldsymbol{\omega}} \\
    {}^{I_k}\tilde{\mathbf{a}}
    \end{bmatrix}
\end{align}
where: 
\begin{align}
    \boldsymbol{\Phi}_{nn} & =
    \begin{bmatrix}
    \Delta \hat{\mathbf{R}}^{\top}_k & \mathbf{0}_3 & \mathbf{0}_3 \\
    -{}^G_{I_k}\hat{\mathbf{R}}\lfloor \Delta \hat{\mathbf{p}}_k \rfloor & \mathbf{I}_3 & \mathbf{I}_3 \delta t_k \\
    -{}^G_{I_k}\hat{\mathbf{R}}\lfloor \Delta \hat{\mathbf{v}}_k \rfloor & \mathbf{0}_3 & \mathbf{I}_3 
    \end{bmatrix}
    \\
    \boldsymbol{\Phi}_{wa} & = 
    \begin{bmatrix}
    \mathbf{J}_r(\Delta \hat{\boldsymbol{\theta}}_k) \delta t_k & \mathbf{0}_3 \\
    -{}^G_{I_k}\hat{\mathbf{R}}\boldsymbol{\Xi}_4 & {}^G_{I_k}\hat{\mathbf{R}}\boldsymbol{\Xi}_2 \\
    -{}^G_{I_k}\hat{\mathbf{R}}\boldsymbol{\Xi}_3 & {}^G_{I_k}\hat{\mathbf{R}}\boldsymbol{\Xi}_1 
    \end{bmatrix}
\end{align}
By plugging Eq. \eqref{eq:wa} into Eq. \eqref{eq:xn_dynamics} and adding biases, the overall linearized system for IMU state can be written as: 
\begin{align}
    \tilde{\mathbf{x}}_{I_{k+1}} & = 
    \boldsymbol{\Phi}_{I(k+1, k)} 
    \tilde{\mathbf{x}}_{I_{k}} + 
    \mathbf{G}_k \mathbf{n}_{dI}
\end{align}
where $\mathbf{n}_{dI}= 
\begin{bmatrix}
\mathbf{n}^{\top}_{dg} \ \mathbf{n}^{\top}_{da}\ \mathbf{n}^{\top}_{dwg}\ \mathbf{n}^{\top}_{dwa}
\end{bmatrix}^{\top}$ denotes the discretized IMU noises; $\boldsymbol{\Phi}_{I(k+1,k)}$ and $\mathbf{G}_k$ are computed as:

\begin{align}
    \boldsymbol{\Phi}_{I(k+1,k)} & = 
    \begin{bmatrix}
    \boldsymbol{\Phi}_{nn}  & 
    \boldsymbol{\Phi}_{wa} \mathbf{H}_b & 
    \boldsymbol{\Phi}_{wa} \mathbf{H}_{in} \\
    \mathbf{0}_{6\times 9} & 
    \mathbf{I}_6 & 
    \mathbf{0}_{6\times m} \\
    \mathbf{0}_{m\times 9} & \mathbf{0}_{m\times 6} & \mathbf{I}_{m}
    \end{bmatrix}
    \\
    \mathbf{G}_{k} & = 
    \begin{bmatrix}
    \boldsymbol{\Phi}_{wa} \mathbf{H}_n & \mathbf{0}_{9\times 6} \\
    \mathbf{0}_{6} & \mathbf{I}_6 \delta t_k \\
    \mathbf{0}_{m\times 6} & \mathbf{0}_{m\times 6}
    \end{bmatrix}
\end{align}
Note that $\mathbf{n}_{d*}\sim \mathcal{N}(\mathbf{0}, \frac{\sigma^2_{*}\mathbf{I}_3}{\delta t_k})$ and hence the covariance for $\mathbf{n}_{dI}$ can be written as: 
\begin{align}
    \mathbf{Q}_{dI} & =
    \begin{bmatrix}
    \frac{\sigma^2_g}{\delta t_k} \mathbf{I}_3 & \mathbf{0}_3 & \mathbf{0}_3 & \mathbf{0}_3 \\
    \mathbf{0}_3 & \frac{\sigma^2_a}{\delta t_k} \mathbf{I}_3 & \mathbf{0}_3 & \mathbf{0}_3 \\
    \mathbf{0}_3 & \mathbf{0}_3 & \frac{\sigma^2_{wg}}{\delta t_k} \mathbf{I}_3 & \mathbf{0}_3 \\
    \mathbf{0}_3 & \mathbf{0}_3 & \mathbf{0}_3 & \frac{\sigma^2_{wa}}{\delta t_k} \mathbf{I}_3
    \end{bmatrix}
\end{align}

\section{Appendix B: Camera Jacobians} \label{sec:apx_cam_jacob}
We will show the detailed derivations for the Jacobians shown in Eq. \eqref{eq:Hc}. 
The camera intrinsic Jacobians $\mathbf{H}_{Cin}$ can be written as:
\begin{align}
    & \mathbf{H}_{Cin} = 
    \begin{bmatrix}
    \frac{\partial \tilde{\mathbf{z}}_C}
    {\partial {\begin{bmatrix}
    \tilde{f_u} ~\tilde{f_v} ~ \tilde{c_u} ~\tilde{c_v}
    \end{bmatrix}}^{\top}} & 
    \frac{\partial \tilde{\mathbf{z}}_C}
    {\partial {\begin{bmatrix}
    \tilde{k_1} ~ \tilde{k_2} ~ \tilde{p_1} ~ \tilde{p_2}
    \end{bmatrix}}^{\top}} 
    \end{bmatrix}
    \\
    & \frac{\partial \tilde{\mathbf{z}}_C}
    {\partial {\begin{bmatrix}
    \tilde{f_u} ~\tilde{f_v} ~ \tilde{c_u} ~\tilde{c_v}
    \end{bmatrix}}^{\top}}  = 
    \begin{bmatrix}
    u_d & 0 & 1 & 0  \\
    0 & v_d & 0 & 1 
    \end{bmatrix}
    \\
    & \frac{\partial \tilde{\mathbf{z}}_C}
    {\partial {\begin{bmatrix}
    \tilde{k_1} ~ \tilde{k_2} 
    \end{bmatrix}}^{\top}} = 
    \begin{bmatrix}
     f_uu_nr^2 & f_uu_nr^4  \\
    f_vv_nr^2 & f_vv_nr^4 
    \end{bmatrix}
    \\
    & \frac{\partial \tilde{\mathbf{z}}_C}
    {\partial {\begin{bmatrix}
    \tilde{p_1} ~ \tilde{p_2}
    \end{bmatrix}}^{\top}} = 
    \begin{bmatrix}
      2f_uu_nv_n & f_u(r^2+2u^2_n) \\
     f_v(r^2+2v^2_n) & 2f_vu_nv_n
    \end{bmatrix}
\end{align}
We continue to compute $\frac{\partial \tilde{\mathbf{z}}_C}{\partial \tilde{\mathbf{z}}_n} $ and $\frac{\partial \tilde{\mathbf{z}}_n}{\partial {}^{C}\tilde{\mathbf{p}}_f}$ for $\mathbf{H}_{\mathbf{p}_f}$ within Eq. \eqref{eq:Hc} as: 
\begin{align}
    \frac{\partial \tilde{\mathbf{z}}_C}{\partial \tilde{\mathbf{z}}_{n}} & = 
    \begin{bmatrix}
    h_{11} & h_{12} \\
    h_{21} & h_{22}
    \end{bmatrix}
    \\
    \frac{\partial \tilde{\mathbf{z}}_C}{\partial {}^C\tilde{\mathbf{p}}_{f}} & = 
    \frac{1}{{}^Cz^2_f}
    \begin{bmatrix}
    {}^Cz_f & 0 & -{}^Cx_f \\
    0 & {}^Cz_f & -{}^Cy_f
    \end{bmatrix}
\end{align}
with $\mathbf{h}_{11}$,  $\mathbf{h}_{12}$,  $\mathbf{h}_{21}$ and  $\mathbf{h}_{22}$ defined as:
\begin{align}
    \mathbf{h}_{11} & = f_u (d+2k_1u^2_n+4k_2u^2_nr^2+2p_1v_n+6p_2u_n) 
    \notag
    \\
    \mathbf{h}_{12} & = f_u (2k_1u_nv_n+4k_2u_nv_nr^2+2p_1u_n+2p_2v_n) 
    \notag
    \\
    \mathbf{h}_{21} & = f_v (2k_1u_nv_n+4k_2u_nv_nr^2+2p_1u_n+2p_2v_n)  
    \notag
    \\
    \mathbf{h}_{22} & = f_v (d+2k_1v^2_n+4k_2v^2_nr^2+6p_1v_n+2p_2u_n) 
    \notag
\end{align}
The Jacobians of ${}^C\mathbf{p}_f$ regarding to the IMU state $\mathbf{x}_I$ are written as:
\begin{align}
    \frac{\partial {}^{C}\tilde{\mathbf{p}}_f}{\partial \tilde{\mathbf{x}}_I} & = 
    \begin{bmatrix}
    \frac{\partial {}^{C}\tilde{\mathbf{p}}_f}{\partial \tilde{\mathbf{x}}_n} & 
    \frac{\partial {}^{C}\tilde{\mathbf{p}}_f}{\partial \tilde{\mathbf{x}}_b} & 
    \frac{\partial {}^{C}\tilde{\mathbf{p}}_f}{\partial \tilde{\mathbf{x}}_{in}}
    \end{bmatrix}
    \\
    \frac{\partial {}^{C}\tilde{\mathbf{p}}_f}{\partial \tilde{\mathbf{x}}_n}
    & = 
    {}^C_I\hat{\mathbf{R}}{}^I_G\hat{\mathbf{R}}
    \begin{bmatrix}
    \lfloor {}^G\hat{\mathbf{p}}_{f} - {}^G\hat{\mathbf{p}}_{I} \rfloor {}^G_I\hat{\mathbf{R}} &
    -\mathbf{I}_3 & 
    \mathbf{0}_{3} 
    \end{bmatrix}
    \\
    \frac{\partial {}^{C}\tilde{\mathbf{p}}_f}{\partial \tilde{\mathbf{x}}_b} & = \mathbf{0}_{3\times 6}, ~~
    \frac{\partial {}^{C}\tilde{\mathbf{p}}_f}{\partial \tilde{\mathbf{x}}_{in}} = \mathbf{0}_{3\times 24}
\end{align}
The Jacobians of ${}^C\mathbf{p}_f$ regarding to the IMU-camera spatial-temporal calibration state $\mathbf{x}_{IC}$ are written as:
\begin{align}
    \frac{\partial {}^{C}\tilde{\mathbf{p}}_f}{\partial \tilde{\mathbf{x}}_{IC}} & = 
    \begin{bmatrix}
    \frac{\partial {}^{C}\tilde{\mathbf{p}}_f}{\partial \delta \boldsymbol{\theta}_{IC}} & 
    \frac{\partial {}^{C}\tilde{\mathbf{p}}_f}{\partial {}^C\tilde{\mathbf{p}}_I} & 
    \frac{\partial {}^{C}\tilde{\mathbf{p}}_f}{\partial \tilde{t}_{d}} & 
    \frac{\partial {}^{C}\tilde{\mathbf{p}}_f}{\partial \tilde{t}_{r}} 
    \end{bmatrix}
    \\
    \frac{\partial {}^{C}\tilde{\mathbf{p}}_f}{\partial \delta \boldsymbol{\theta}_{IC}} & = 
    \lfloor
    {}^C_I\hat{\mathbf{R}}
    {}^I_G\hat{\mathbf{R}}
    \left(
    {}^G\hat{\mathbf{p}}_f - {}^G\hat{\mathbf{p}}_I
    \right)
    \rfloor
    \\
    \frac{\partial {}^{C}\tilde{\mathbf{p}}_f}{\partial {}^C\tilde{\mathbf{p}}_I} & = 
    \mathbf{I}_3
    \\
    \frac{\partial {}^{C}\tilde{\mathbf{p}}_f}{\partial \tilde{t}_d} & = 
    \scalemath{0.9}{
    {}^C_I\hat{\mathbf{R}} {}^I_G\hat{\mathbf{R}}
    \left(
    \lfloor
    \left(
    {}^G\hat{\mathbf{p}}_f - {}^G\hat{\mathbf{p}}_I
    \right)
    \rfloor
    {}^G_I\hat{\mathbf{R}}
    {}^{I}\hat{\boldsymbol{\omega}}
    -
    {}^G\hat{\mathbf{v}}_I
    \right)
    }
    \\
    \label{eq:job tr}
    \frac{\partial {}^{C}\tilde{\mathbf{p}}_f}{\partial \tilde{t}_r} & = 
    \frac{m}{M} \frac{\partial {}^{C}\tilde{\mathbf{p}}_f}{\partial \tilde{t}_d} 
\end{align}
Note that when computing the Jacobians for $t_d$ and $t_r$, we are using the following linearization: 
\begin{align}
    {}^G_{I{(t)}}\mathbf{R} & \simeq 
    {}^G_{I{(\hat{t})}}\hat{\mathbf{R}}\exp(\delta \boldsymbol{\theta}_{I})
    \exp({}^I\hat{\boldsymbol{\omega}}\tilde{t}_d + \frac{m}{M}{}^I\hat{\boldsymbol{\omega}}\tilde{t}_r)
    \\
    {}^G\mathbf{p}_{I(t)} & \simeq {}^G\hat{\mathbf{p}}_{I(\hat{t})} + {}^G\tilde{\mathbf{p}}_{I} + {}^G\hat{\mathbf{v}}_I \tilde{t}_d + \frac{m}{M}{}^G\hat{\mathbf{v}}_I \tilde{t}_r
\end{align}
The Jacobians of ${}^C\mathbf{p}_f$ regarding to the feature state $\mathbf{x}_{f}$ is written as:
\begin{align}
    \frac{\partial {}^{C}\tilde{\mathbf{p}}_f}{\partial \tilde{\mathbf{x}}_{f}} & = 
    \frac{\partial {}^{C}\tilde{\mathbf{p}}_f}{\partial \delta {}^G\tilde{\mathbf{p}}_f} 
    =
    {}^C_I\hat{\mathbf{R}}{}^I_G\hat{\mathbf{R}}
\end{align}

\section{Appendix C: Observability Matrix} \label{apx:obs_matrix_M}

We show the detailed derivations for $\mathbf{M}_n$, $\mathbf{M}_b$, $\mathbf{M}_{in}$, $\mathbf{M}_{IC}$, $\mathbf{M}_{Cin}$ and $\mathbf{M}_f$.
The $\mathbf{M}_n$ is computed as:
\begin{align}
    \mathbf{M}_n  
    = &
    \mathbf{H}_{\mathbf{p}_f}
    {}^C_I\hat{\mathbf{R}}{}^{I_k}_G\hat{\mathbf{R}}
    \begin{bmatrix}
    \boldsymbol{\Gamma}_{1} & 
    \boldsymbol{\Gamma}_{2} & 
    \boldsymbol{\Gamma}_{3} 
    \end{bmatrix}
\end{align}
with: 
\begin{align}
\boldsymbol{\Gamma}_{1} & = 
\lfloor 
{}^G\hat{\mathbf{p}}_{f}-{}^G\hat{\mathbf{p}}_{I_1} - {}^G\hat{\mathbf{v}}_{I_1}\delta t_k + \frac{1}{2}{}^G\mathbf{g}\delta t^2_k
\rfloor {}^G_{I_1}\hat{\mathbf{R}} 
\notag
\\
\boldsymbol{\Gamma}_{2} & = - \mathbf{I}_3
\notag
\\
\boldsymbol{\Gamma}_{3} & = - \mathbf{I}_3 \delta t_k 
\notag
\end{align}
The $\mathbf{M}_{b}$ is computed as:
\begin{align}
    \mathbf{M}_b  
    = &
    \mathbf{H}_{\mathbf{p}_f}
    {}^C_I\hat{\mathbf{R}}{}^{I_k}_G\hat{\mathbf{R}}
    \begin{bmatrix}
    \boldsymbol{\Gamma}_{4} & 
    \boldsymbol{\Gamma}_{5} 
    \end{bmatrix}
\end{align}
with:
\begin{align}
\boldsymbol{\Gamma}_4 & = 
- \bigg(
\lfloor {}^G\hat{\mathbf{p}}_{f}-{}^G\hat{\mathbf{p}}_{I_k} \rfloor {}^G_{I_k}\hat{\mathbf{R}} 
\mathbf{J}_r\left( \Delta \hat{\boldsymbol{\theta}}_k \right)\delta t_k 
\notag
\\
& ~~~~
+ {}^G_{I_k}\hat{\mathbf{R}}\boldsymbol{\Xi}_4
\bigg)
{}^I_w\hat{\mathbf{R}}\hat{\mathbf{D}}_w
\notag
\\
\boldsymbol{\Gamma}_5  &= 
\bigg(
\lfloor {}^G\hat{\mathbf{p}}_{f}-{}^G\hat{\mathbf{p}}_{I_k} \rfloor {}^G_{I_k}\hat{\mathbf{R}} 
\mathbf{J}_r\left( \Delta \hat{\boldsymbol{\theta}}_k\right)
{}^I_w\hat{\mathbf{R}}\hat{\mathbf{D}}_w\hat{\mathbf{T}}_g \delta t_k  
\notag
\\
& ~~~~
+ 
{}^G_{I_k}\hat{\mathbf{R}}
\left(
\boldsymbol{\Xi}_4 {}^I_w\hat{\mathbf{R}}\hat{\mathbf{D}}_w\hat{\mathbf{T}}_g +
\boldsymbol{\Xi}_2
\right)
\bigg)
 {}^I_a\hat{\mathbf{R}}\hat{\mathbf{D}}_a 
\notag
\end{align}
The $\mathbf{M}_{in}$ can be computed as:
\begin{align}
\label{eq:M_in}
    \mathbf{M}_{in}  
    = &
    \mathbf{H}_{\mathbf{p}_f}
    {}^C_I\hat{\mathbf{R}}{}^{I_k}_G\hat{\mathbf{R}}
    \begin{bmatrix}
    \boldsymbol{\Gamma}_{6} & 
    \boldsymbol{\Gamma}_{7} & 
    \boldsymbol{\Gamma}_{8} & 
    \boldsymbol{\Gamma}_{9}
    \end{bmatrix}
\end{align}
with: 
\begin{align}
\boldsymbol{\Gamma}_6 & = 
\bigg(\lfloor {}^G\hat{\mathbf{p}}_{f} - {}^G\hat{\mathbf{p}}_{I_k} \rfloor
{}^G_{I_k}\hat{\mathbf{R}} 
\mathbf{J}_r\left( \Delta \hat{\boldsymbol{\theta}}_k\right)\delta t_k 
+ {}^G_{I_k}\hat{\mathbf{R}}\boldsymbol{\Xi}_4
\bigg) \mathbf{H}_{Dw}
\notag
\\
\boldsymbol{\Gamma}_7  &= 
- 
\bigg(
\lfloor {}^G\hat{\mathbf{p}}_{f}-{}^G\hat{\mathbf{p}}_{I_k} \rfloor {}^G_{I_k}\hat{\mathbf{R}} 
\mathbf{J}_r\left( \Delta \hat{\boldsymbol{\theta}}_k\right)
{}^I_w\hat{\mathbf{R}}\hat{\mathbf{D}}_w\hat{\mathbf{T}}_g \delta t_k  
\notag
\\
& ~~~~~
+ 
{}^G_{I_k}\hat{\mathbf{R}}
\left(
\boldsymbol{\Xi}_4 {}^I_w\hat{\mathbf{R}}\hat{\mathbf{D}}_w\hat{\mathbf{T}}_g +
\boldsymbol{\Xi}_2
\right)
\bigg)
{}^{I}_a\hat{\mathbf{R}}\mathbf{H}_{Da}
\notag
\\
\boldsymbol{\Gamma}_8  & = 
- 
\bigg(
\lfloor {}^G\hat{\mathbf{p}}_{f}-{}^G\hat{\mathbf{p}}_{I_k} \rfloor {}^G_{I_k}\hat{\mathbf{R}} 
\mathbf{J}_r\left( \Delta \hat{\boldsymbol{\theta}}_k \right)
{}^I_w\hat{\mathbf{R}}\hat{\mathbf{D}}_w\hat{\mathbf{T}}_g \delta t_k  
\notag
\\
&~~~~~
+ 
{}^G_{I_k}\hat{\mathbf{R}}
\big(
\boldsymbol{\Xi}_4 {}^I_w\hat{\mathbf{R}}\hat{\mathbf{D}}_w\hat{\mathbf{T}}_g
+
\boldsymbol{\Xi}_2
\big)
\bigg) 
\lfloor {}^{I_k}\hat{\mathbf{a}} \rfloor 
\notag
\\
\boldsymbol{\Gamma}_9  & = 
- \bigg(\lfloor {}^G\hat{\mathbf{p}}_{f}-{}^G\hat{\mathbf{p}}_{I_k} \rfloor {}^G_{I_k}\hat{\mathbf{R}} 
\mathbf{J}_r\left( \Delta \hat{\boldsymbol{\theta}}_k\right)\delta t_k 
+ {}^G_{I_k}\hat{\mathbf{R}}\boldsymbol{\Xi}_4
\bigg) \times 
\notag
\\
&~~~~~
{}^I_w\hat{\mathbf{R}}\hat{\mathbf{D}}_w \mathbf{H}_{Tg}
\notag
\end{align}
The $\mathbf{M}_{IC}$ can be computed as:
\begin{align}
\label{eq:M_IC}
    \mathbf{M}_{IC}  & = 
    \mathbf{H}_{\mathbf{p}_f}
    {}^C_I\hat{\mathbf{R}}{}^{I_k}_G\hat{\mathbf{R}}
    \begin{bmatrix}
    \boldsymbol{\Gamma}_{10} & 
    \boldsymbol{\Gamma}_{11} & 
    \boldsymbol{\Gamma}_{12} & 
    \boldsymbol{\Gamma}_{13} 
    \end{bmatrix}
    \\
    \boldsymbol{\Gamma}_{10} & = 
    \lfloor
    \left(
    {}^G\hat{\mathbf{p}}_f - {}^G\hat{\mathbf{p}}_{I_k}
    \right)
    \rfloor
    {}^G_{I_k}\hat{\mathbf{R}} {}^I_C\hat{\mathbf{R}}
    \\
    \boldsymbol{\Gamma}_{11} & = 
    {}^G_{I_k}\hat{\mathbf{R}} {}^I_C\hat{\mathbf{R}}
    \\
    \boldsymbol{\Gamma}_{12} & = 
    \lfloor
    \left(
    {}^G\hat{\mathbf{p}}_f - {}^G\hat{\mathbf{p}}_{I_k}
    \right)
    \rfloor
    {}^G_{I_k}\hat{\mathbf{R}}
    {}^{I_k}\hat{\boldsymbol{\omega}}
    -
    {}^G\hat{\mathbf{v}}_{I_k}
    \\
    \boldsymbol{\Gamma}_{13} & = \frac{m}{M} \boldsymbol{\Gamma}_{12}
\end{align}
The $\mathbf{M}_{Cin}$ and $\mathbf{M}_{f}$ can be written as:
\begin{align}
    & \mathbf{M}_{Cin}  = \mathbf{H}_{Cin}
    \\
    & \mathbf{M}_{f}  = \mathbf{H}_{\mathbf{p}_f}{}^{C}_I\hat{\mathbf{R}}{}^{I_{k}}_G\hat{\mathbf{R}}
\end{align}

\section{Appendix D: Proof of Lemma \ref{lem:obs}}
\label{sec:proof_of_lemma1}

For Eq. \eqref{eq:N}, we first verify $\mathcal{O}\mathbf{N} = \mathbf{0}$ as:
\begin{align}
    \Leftrightarrow & \mathcal{O}_k \mathbf{N}  = \mathbf{0}
    \notag
    \\
    \Leftrightarrow &
    \left(\boldsymbol{\Gamma}_1 {}^{I_1}_G\hat{\mathbf{R}}
    - \boldsymbol{\Gamma}_2 
    \lfloor 
    {}^G\hat{\mathbf{p}}_{I_1}
    \rfloor
    - \boldsymbol{\Gamma}_3 
    \lfloor 
    {}^G\hat{\mathbf{v}}_{I_1}
    \rfloor
    - \lfloor 
    {}^G\hat{\mathbf{p}}_f
    \rfloor
    \right)
    {}^G\mathbf{g} = 0
    \notag
    \\
    \Leftrightarrow
    &
    \left(
    \lfloor
    {}^G\hat{\mathbf{p}}_f 
    - {}^G\hat{\mathbf{p}}_{I_1}
    - {}^G\hat{\mathbf{v}}_{I_1} \delta t_k
    +
    \frac{1}{2}{}^G\mathbf{g}\delta t^2_k
    \rfloor \right) {}^G\mathbf{g}
    \notag
    \\
    &
    + \left(\lfloor 
    {}^G\hat{\mathbf{p}}_{I_1}
    \rfloor
    +
    \lfloor 
    {}^G\hat{\mathbf{v}}_{I_1}
    \rfloor \delta t_k
    -\lfloor 
    {}^G\hat{\mathbf{p}}_f
    \rfloor
    \right) {}^G\mathbf{g}
     = 
    \mathbf{0}
    \notag
\end{align}

Hence, we can conclude that the observability matrix $\mathcal{O}$ has at least 4 unobservable directions. 

In the following, we will try to show that there are only 4 unobservable directions under general situations. 
With abusing of notion, we can rewrite the observability matrix as: 
\begin{align}
    \mathcal{O} & =
    \begin{bmatrix}
    \mathcal{O}^{\top}_{1} &
    \dots &
    \mathcal{O}^{\top}_{k}
    \end{bmatrix}^{\top}
    \notag
    \\
    &
    =
    \begin{bmatrix}
    \mathbf{M}_{n,1} & \mathbf{M}_{b,1} & \mathbf{M}_{in,1} & 
    \mathbf{M}_{IC,1} & \mathbf{M}_{Cin,1} & \mathbf{M}_{f,1} \\
    \vdots & \vdots & \vdots & \vdots & \vdots & \vdots \\
    \mathbf{M}_{n,k} & \mathbf{M}_{b,k} & \mathbf{M}_{in,k} & 
    \mathbf{M}_{IC,k} & \mathbf{M}_{Cin,k} & \mathbf{M}_{f,k}
    \end{bmatrix}
    \notag
\end{align}
By adjusting the column of $\mathbf{M}_{f}$ in $\mathcal{O}$, we can get $\mathcal{O}'$: 
\begin{align}
    \mathcal{O}' & \triangleq 
    \left[\begin{array}{c|c|c|c } 
    \mathcal{O}_{I}  & \mathcal{O}_{in}  & \mathcal{O}_{IC}  & \mathcal{O}_{f}
    \end{array}\right]
    \notag
    \\
    & \triangleq  
    \left[\begin{array}{ccc|c|c|c } 
	\mathbf{M}_{n,1} & \mathbf{M}_{b,1}  & \mathbf{M}_{f,1} & \mathbf{M}_{in,1} & 
    \mathbf{M}_{IC,1} & \mathbf{M}_{Cin,1} \\
    \vdots & \vdots & \vdots & \vdots & \vdots & \vdots \\
    \mathbf{M}_{n,k} & \mathbf{M}_{b,k}  & \mathbf{M}_{f,k} & \mathbf{M}_{in,k} & 
    \mathbf{M}_{IC,k} & \mathbf{M}_{Cin,k}  
\end{array}\right] 
    \notag
\end{align}
It is clear that the column rank of $\mathcal{O}'$ is the same as $\mathcal{O}$. 

$\mathcal{O}_I$ corresponds to the IMU navigation state, IMU bias state and feature state. $\mathcal{O}_I$ is equivalent to the standard VINS observability matrix in \citep{Hesch2013TRO} and it has null space of 4DoF. 

$\mathcal{O}_{in}$ corresponds to the IMU intrinsic parameters. By checking the Eq. \eqref{eq:M_in}, it is clearly that $\mathcal{O}_{in}$ will be affected by time-varying ${}^w\boldsymbol{\omega}(t)$ (in $\mathbf{H}_{Dw}$), ${}^a\mathbf{a}(t)$ (in $\mathbf{H}_{Da}$) and ${}^I\mathbf{a}(t)$ (in $\lfloor {}^I\mathbf{a} \rfloor$ and $\mathbf{H}_{Tg}$). Under generate motions, $\mathcal{O}_{in}$ can be of full column rank. 

$\mathcal{O}_{IC}$ corresponds to the IMU-camera spatial and temporal calibration parameters. By checking the Eq. \eqref{eq:M_IC}, we can see that the $\mathcal{O}_{IC}$ is affected by the time-varying IMU pose $\{{}^{I}_G\mathbf{R}(t), {}^G\mathbf{p}_{I}(t)\}$ and the IMU kinematics $\{{}^{I}\boldsymbol{\omega}(t), {}^I\mathbf{v}(t)\}$. In addition, $\Gamma_{13}$ in $\mathbf{M}_{IC}$ are also affected by the point feature measurement through $\frac{m}{M}$, of which $m$ will change under general measurement assumptions. Hence, $\mathbf{O}_{IC}$ can be of full column rank with random motions. 

$\mathcal{O}_{Cin}$ corresponds to the camera intrinsic parameters. It is clear that $\mathcal{O}_{Cin}$ is only affected the environmental structure and is of full column rank as long as $\{u_n,v_n\}$ varies in different image tracks. 

Since $\mathbf{O}_{in}$, $\mathbf{O}_{IC}$ and $\mathbf{O}_{Cin}$ are affected by different system parameters, and under general motion conditions, $\left[\mathbf{O}_{in} ~~ \mathbf{O}_{IC} ~~ \mathbf{O}_{Cin}\right]$ is also of full column rank. Therefore, the column rank of $\mathcal{O}'$ is determined by $\mathbf{O}_I$. Since $\mathbf{O}_{I}$ has 4 DoF null space, the $\mathcal{O}'$ also has 4 DoF. Hence, we can conclude that $\mathcal{O}$ only has 4 DoF null space. 
We also verify this conclusion through simulation results shown in Fig~\ref{fig:sim_full}.

\section{Appendix E: Null Space Proofs}
\label{sec:proof}

In this section, we provide the verification for the null spaces listed in this paper.

\subsection{Verification of Lemma \ref{lem:dw}}

For $\mathbf{N}_{w1}$, we have:
\begin{align}
    \Leftrightarrow & \mathcal{O}_k \mathbf{N}_{w1} = \mathbf{0}
    \notag
    \\
    \Leftrightarrow & 
    \boldsymbol{\Gamma}_4 \hat{\mathbf{D}}^{-1}_{w} 
    {}^I_w\hat{\mathbf{R}}^{\top} \mathbf{e}_1 {}^w\omega_1
    +
    \boldsymbol{\Gamma}_6 \times
    \begin{bmatrix}
    1~ 0 ~0 ~ 0 ~ 0 ~ 0
    \end{bmatrix}^{\top} = \mathbf{0}
    \notag
    \\
    \Leftrightarrow &
    (\lfloor {}^G\hat{\mathbf{p}}_{f} - {}^G\hat{\mathbf{p}}_{I_k} \rfloor
{}^G_{I_k}\hat{\mathbf{R}} 
\mathbf{J}_r\left( \Delta \boldsymbol{\theta}_k\right)\delta t_k 
+ {}^G_{I_k}\hat{\mathbf{R}}\boldsymbol{\Xi}_4)\times
\notag
\\ &
\left(
-\mathbf{e}_1{}^w\omega_1 + \mathbf{e}_1{}^w\omega_1
\right) = \mathbf{0}
\notag
\end{align}
For $\mathbf{N}_{w2}$, we have: 
\begin{align}
    \Leftrightarrow & \mathcal{O}_k \mathbf{N}_{w2} = \mathbf{0}
    \notag
    \\
    \Leftrightarrow & 
    \boldsymbol{\Gamma}_4 \hat{\mathbf{D}}^{-1}_{w} 
    {}^I_w\hat{\mathbf{R}}^{\top} 
    \left[ \mathbf{e}_1 ~ \mathbf{e}_2 \right]{}^w\omega_2
    +
    \boldsymbol{\Gamma}_6 \times
    \begin{bmatrix}
    0~ 1 ~0 ~ 0 ~ 0 ~ 0 \\
    0~ 0 ~1 ~ 0 ~ 0 ~ 0
    \end{bmatrix}^{\top} = \mathbf{0}
    \notag
    \\
    \Leftrightarrow &
    (\lfloor {}^G\hat{\mathbf{p}}_{f} - {}^G\hat{\mathbf{p}}_{I_k} \rfloor
{}^G_{I_k}\hat{\mathbf{R}} 
\mathbf{J}_r\left( \Delta \boldsymbol{\theta}_k\right)\delta t_k 
+ {}^G_{I_k}\hat{\mathbf{R}}\boldsymbol{\Xi}_4)\times
\notag
\\ &
\left(
-\left[ \mathbf{e}_1 ~ \mathbf{e}_2 \right]{}^w\omega_2 + \left[ \mathbf{e}_1 ~ \mathbf{e}_2 \right]{}^w\omega_2
\right) = \mathbf{0}
\notag
\end{align}
For $\mathbf{N}_{w3}$, we have: 
\begin{align}
    \Leftrightarrow & \mathcal{O}_k \mathbf{N}_{w3} = \mathbf{0}
    \notag
    \\
    \Leftrightarrow & 
    \boldsymbol{\Gamma}_4 \hat{\mathbf{D}}^{-1}_{w} 
    {}^I_w\hat{\mathbf{R}}^{\top} 
    \mathbf{I}_3{}^w\omega_3
    +
    \boldsymbol{\Gamma}_6 \times
    \begin{bmatrix}
    \mathbf{0}_3 ~~ \mathbf{I}_3
    \end{bmatrix}^{\top} = \mathbf{0}
    \notag
    \\
    \Leftrightarrow &
    (\lfloor {}^G\hat{\mathbf{p}}_{f} - {}^G\hat{\mathbf{p}}_{I_k} \rfloor
{}^G_{I_k}\hat{\mathbf{R}} 
\mathbf{J}_r\left( \Delta \boldsymbol{\theta}_k\right)\delta t_k 
+ {}^G_{I_k}\hat{\mathbf{R}}\boldsymbol{\Xi}_4)\times
\notag
\\ &
\left(
-\mathbf{I}_3{}^w\omega_3 + \mathbf{I}_3{}^w\omega_3
\right) = \mathbf{0}
\notag
\end{align}

\subsection{Verification of Lemma \ref{lem:da}}

We first verify the first column of $\mathbf{N}_{a1}$: 
\begin{align}
    \Leftrightarrow & 
    \mathcal{O}_k \mathbf{N}_{a1}\mathbf{e}_1 = \mathbf{0}
    \notag
    \\
    \Leftrightarrow & 
    \boldsymbol{\Gamma}_5 \hat{\mathbf{D}}^{-1}_a \mathbf{e}_1 {}^aa_1 + 
    \boldsymbol{\Gamma}_7 \times 
    [1~0~0~0~0~0]^{\top} = \mathbf{0}
    \notag
    \\
    \Leftrightarrow & 
    \bigg(
\lfloor {}^G\hat{\mathbf{p}}_{f}-{}^G\hat{\mathbf{p}}_{I_k} \rfloor {}^G_{I_k}\hat{\mathbf{R}} 
\mathbf{J}_r\left( \Delta \boldsymbol{\theta}_k\right)
{}^I_w\hat{\mathbf{R}}\hat{\mathbf{D}}_w\hat{\mathbf{T}}_g \delta t_k  
\notag
\\
& ~~~~
+ 
{}^G_{I_k}\hat{\mathbf{R}}
\left(
\boldsymbol{\Xi}_4 {}^I_w\hat{\mathbf{R}}\hat{\mathbf{D}}_w\hat{\mathbf{T}}_g +
\boldsymbol{\Xi}_2
\right)
\bigg)
 {}^I_a\hat{\mathbf{R}} \times
 \notag
 \\
 & 
 \left(
 \mathbf{e}_1 {}^aa_1 - \mathbf{e}_1 {}^aa_1
 \right) = \mathbf{0}
 \notag
\end{align}
We then verify the second column of $\mathbf{N}_{a1}$: 
\begin{align}
\Leftrightarrow &
    \mathcal{O}_{k}\mathbf{N}_{a1}\mathbf{e}_2 = \mathbf{0}
    \notag
    \\
\Leftrightarrow &    
    \boldsymbol{\Gamma}_5 \hat{\mathbf{D}}^{-1}_a \mathbf{e}_2 d_{a1}{}^aa_1 
    + \boldsymbol{\Gamma}_7 \times 
    [0~d_{a3}~-d_{a2}~d_{a5}~-d_{a4}~0]^{\top} 
    \notag
    \\
    & - \boldsymbol{\Gamma}_8 {}^I_a\hat{\mathbf{R}}\mathbf{e}_3 =\mathbf{0}
    \notag
    \\
    \Leftrightarrow & 
    \begin{bmatrix}
    0 \\
    d_{a1}{}^aa_1 \\
    0
    \end{bmatrix} -
    \begin{bmatrix}
    d_{a3}{}^aa_2+d_{a5}{}^aa_3 \\
    -d_{a2}{}^aa_2 - d_{a4}{}^aa_3 \\
    0
    \end{bmatrix}
    \notag
    \\
    &
    + 
    \begin{bmatrix}
    d_{a3}{}^aa_2+d_{a5}{}^aa_3 \\
    -(
    d_{a1}{}^aa_1 + d_{a2}{}^aa_2 + d_{a4}{}^aa_3 
    )\\
    0
    \end{bmatrix} = \mathbf{0}
    \notag
\end{align}
The third column of $\mathbf{N}_{a1}$ can be verified as: 
\begin{align}
    \Leftrightarrow & 
    \mathcal{O}_k \mathbf{N}_{a1}\mathbf{e}_3 = \mathbf{0} 
    \notag
    \\
    \Leftrightarrow & 
    \boldsymbol{\Gamma}_5 \hat{\mathbf{D}}^{-1}_a \mathbf{e}_3 d_{a1}d_{a3}{}^aa_1 + \boldsymbol{\Gamma}_8 {}^I_a\mathbf{R}(\mathbf{e}_1d_{a2}+\mathbf{e}_2d_{a3}) +
    \notag
    \\
    &\boldsymbol{\Gamma}_7 \times
    [0~0~0~d_{a6}d_{a3}~~-d_{a2}d_{a6}~~d_{a2}d_{a5}-d_{a4}d_{a3}] = \mathbf{0}
    \notag
    \\
    \Leftrightarrow & 
    \begin{bmatrix}
    0 \\
    0 \\
    d_{a1}d_{a3}{}^aa_1
    \end{bmatrix}
    - 
    \begin{bmatrix}
    d_{a6}d_{a3}{}^aa_3 \\
    - d_{a2}d_{a6}{}^aa_3\\
    d_{a2}d_{a5}{}^aa_3 - d_{a4}d_{a3}{}^aa_3
    \end{bmatrix}
    \notag
    \\
    & 
    -
    \begin{bmatrix}
    -d_{a6}d_{a3}{}^aa_3 \\
    d_{a2}d_{a6}{}^aa_3 \\
    d_{a1}d_{a3}{}^aa_1 - d_{a5}d_{a2}{}^aa_3 + d_{a4}d_{a3}{}^aa_3
    \end{bmatrix}
    =\mathbf{0} \notag
\end{align}
The first two columns of $\mathbf{N}_{a2}$ can be verified as:
\begin{align}
    \Leftrightarrow & 
    \mathcal{O}_k \mathbf{N}_{a2}[\mathbf{e}_1 ~ \mathbf{e}_2 ]= \mathbf{0} 
    \notag
    \\
    \Leftrightarrow & 
    \boldsymbol{\Gamma}_5 \hat{\mathbf{D}}^{-1}_a [\mathbf{e}_1 ~ \mathbf{e}_2 ]{}^aa_2 + \boldsymbol{\Gamma}_7 \times
    \begin{bmatrix}
    0~ 1 ~0 ~ 0 ~ 0 ~ 0 \\
    0~ 0 ~1 ~ 0 ~ 0 ~ 0
    \end{bmatrix}^{\top} 
    = \mathbf{0}
    \notag
    \\
    \Leftrightarrow & 
    \bigg(
\lfloor {}^G\hat{\mathbf{p}}_{f}-{}^G\hat{\mathbf{p}}_{I_k} \rfloor {}^G_{I_k}\hat{\mathbf{R}} 
\mathbf{J}_r\left( \Delta \boldsymbol{\theta}_k\right)
{}^I_w\hat{\mathbf{R}}\hat{\mathbf{D}}_w\hat{\mathbf{T}}_g \delta t_k  
\notag
\\
& ~~~~
+ 
{}^G_{I_k}\hat{\mathbf{R}}
\left(
\boldsymbol{\Xi}_4 {}^I_w\hat{\mathbf{R}}\hat{\mathbf{D}}_w\hat{\mathbf{T}}_g +
\boldsymbol{\Xi}_2
\right)
\bigg)
 {}^I_a\hat{\mathbf{R}} \times
 \notag
 \\
 & 
 \left(
 [\mathbf{e}_1 ~ \mathbf{e}_2 ] {}^aa_2 - [\mathbf{e}_1 ~ \mathbf{e}_2 ] {}^aa_2
 \right) = \mathbf{0}
 \notag
\end{align}
The third column of $\mathbf{N}_{a3}$ can be verified as:
\begin{align}
    \Leftrightarrow & 
    \mathcal{O}_k \mathbf{N}_{a2}\mathbf{e}_1 d_{a3}{}^aa_2= \mathbf{0} 
    \notag
    \\
    \Leftrightarrow & 
    \boldsymbol{\Gamma}_5 \hat{\mathbf{D}}^{-1}_a \mathbf{e}_1 d_{a3}{}^aa_2 + \boldsymbol{\Gamma}_7 \times
    [0~0~0~0~d_{a6}~-d_{a5}]^{\top} 
    \notag
    \\
    & -\boldsymbol{\Gamma}_8 {}^I_a\mathbf{R}\mathbf{e}_1
    = \mathbf{0}
    \notag
    \\
    \Leftrightarrow & 
    \scalemath{0.85}{
    \begin{bmatrix}
    0 \\
    0\\
    d_{a3}{}^aa_2
    \end{bmatrix}
    -
    \begin{bmatrix}
    0 \\
    d_{a6}{}^aa_3 \\
    -d_{a5}{}^aa_3
    \end{bmatrix}
    +
    \begin{bmatrix}
    0 \\
    d_{a6}{}^aa_3 \\
    -(
    d_{a3}{}^aa_2+ d_{a5}{}^aa_3
    )
    \end{bmatrix}} = \mathbf{0}
    \notag
\end{align}
The verification of $\mathbf{N}_{a3}$ can be described as:
\begin{align}
    \Leftrightarrow & 
    \mathcal{O}_k \mathbf{N}_{a3}= \mathbf{0} 
    \notag
    \\
    \Leftrightarrow & 
    \boldsymbol{\Gamma}_5 \hat{\mathbf{D}}^{-1}_a \mathbf{I}_3{}^aa_3 + \boldsymbol{\Gamma}_7 \times
    [\mathbf{0}_3 ~~ \mathbf{I}_3]^{\top} = \mathbf{0}
    \notag
    \\
    \Leftrightarrow & 
    \bigg(
\lfloor {}^G\hat{\mathbf{p}}_{f}-{}^G\hat{\mathbf{p}}_{I_k} \rfloor {}^G_{I_k}\hat{\mathbf{R}} 
\mathbf{J}_r\left( \Delta \boldsymbol{\theta}_k\right)
{}^I_w\hat{\mathbf{R}}\hat{\mathbf{D}}_w\hat{\mathbf{T}}_g \delta t_k  
\notag
\\
& ~~~~
+ 
{}^G_{I_k}\hat{\mathbf{R}}
\left(
\boldsymbol{\Xi}_4 {}^I_w\hat{\mathbf{R}}\hat{\mathbf{D}}_w\hat{\mathbf{T}}_g +
\boldsymbol{\Xi}_2
\right)
\bigg)
 {}^I_a\hat{\mathbf{R}} \times
 \notag
 \\
 & 
 \left(
 \mathbf{I}_3 {}^aa_3 - \mathbf{I}_3 {}^aa_3
 \right) = \mathbf{0}
 \notag
\end{align}
\subsection{Verification of Lemma \ref{lem:tg}}

The $\mathbf{N}_{g1}$ can be verified as: 
\begin{align}
    \Leftrightarrow & 
    \mathcal{O}_k \mathbf{N}_{g1}  = \mathbf{0} 
    \notag
    \\
    \Leftrightarrow & 
    \boldsymbol{\Gamma}_4 \mathbf{I}_3 {}^Ia_1 - 
    \boldsymbol{\Gamma}_9 \times
    [\mathbf{I}_3 ~\mathbf{0}_3 ~ \mathbf{0}_3]^{\top}
    = \mathbf{0}
    \notag
    \\
    \Leftrightarrow & 
    - (\lfloor {}^G\hat{\mathbf{p}}_{f}-{}^G\hat{\mathbf{p}}_{I_k} \rfloor {}^G_{I_k}\hat{\mathbf{R}} 
\mathbf{J}_r\left( \Delta \boldsymbol{\theta}_k\right)\delta t_k 
+ {}^G_{I_k}\hat{\mathbf{R}}\boldsymbol{\Xi}_4) \times 
\notag
\\
&~~~~~
{}^I_w\hat{\mathbf{R}}\hat{\mathbf{D}}_w 
\left(
\mathbf{I}_3 {}^Ia_1 - \mathbf{I}_3 {}^Ia_1
\right)
=\mathbf{0}
\notag
\end{align}
The $\mathbf{N}_{g2}$ can be verified as: 
\begin{align}
    \Leftrightarrow & 
    \mathcal{O}_k \mathbf{N}_{g2}  = \mathbf{0} 
    \notag
    \\
    \Leftrightarrow & 
    \boldsymbol{\Gamma}_4 \mathbf{I}_3 {}^Ia_2 - 
    \boldsymbol{\Gamma}_9 \times
    [\mathbf{0}_3 ~\mathbf{I}_3 ~ \mathbf{0}_3]^{\top}
    = \mathbf{0}
    \notag
    \\
    \Leftrightarrow & 
    - (\lfloor {}^G\hat{\mathbf{p}}_{f}-{}^G\hat{\mathbf{p}}_{I_k} \rfloor {}^G_{I_k}\hat{\mathbf{R}} 
\mathbf{J}_r\left( \Delta \boldsymbol{\theta}_k\right)\delta t_k 
+ {}^G_{I_k}\hat{\mathbf{R}}\boldsymbol{\Xi}_4) \times 
\notag
\\
&~~~~~
{}^I_w\hat{\mathbf{R}}\hat{\mathbf{D}}_w 
\left(
\mathbf{I}_3 {}^Ia_2 - \mathbf{I}_3 {}^Ia_2
\right)
=\mathbf{0}
\notag
\end{align}
The $\mathbf{N}_{g3}$ can be verified as: 
\begin{align}
    \Leftrightarrow & 
    \mathcal{O}_k \mathbf{N}_{g3}  = \mathbf{0} 
    \notag
    \\
    \Leftrightarrow & 
    \boldsymbol{\Gamma}_4 \mathbf{I}_3 {}^Ia_3 - 
    \boldsymbol{\Gamma}_9 \times
    [\mathbf{0}_3 ~\mathbf{0}_3 ~ \mathbf{I}_3]^{\top}
    = \mathbf{0}
    \notag
    \\
    \Leftrightarrow & 
    - (\lfloor {}^G\hat{\mathbf{p}}_{f}-{}^G\hat{\mathbf{p}}_{I_k} \rfloor {}^G_{I_k}\hat{\mathbf{R}} 
\mathbf{J}_r\left( \Delta \boldsymbol{\theta}_k\right)\delta t_k 
+ {}^G_{I_k}\hat{\mathbf{R}}\boldsymbol{\Xi}_4) \times 
\notag
\\
&~~~~~
{}^I_w\hat{\mathbf{R}}\hat{\mathbf{D}}_w 
\left(
\mathbf{I}_3 {}^Ia_3 - \mathbf{I}_3 {}^Ia_3
\right)
=\mathbf{0}
\notag
\end{align}

\subsection{Verification of Lemma \ref{lem:cam-intrinsics}}

Note that: 
\begin{align}
    & \mathbf{M}_{Cin} \times [f_u~f_v~0~0~2k_1~4k_2~p_1~p_2]^{\top}  
    \notag
    \\
    &
    =
    \scalemath{0.85}{
    \begin{bmatrix}
    f_u\left(u_d+2k_1u_nr^2 + 4k_2u_nr^4+2p_1u_nv_n + p_2(r^2+2u^2_n)\right) \\
    f_v\left(v_d+2k_1v_nr^2 + 4k_2v_nr^4+ p_1(r^2+2u^2_n)+2p_2u_nv_n \right)
    \end{bmatrix}
    }
    \notag
    \\
    & =
    \scalemath{0.85}{
    \begin{bmatrix}
    f_u\left(u_n+3k_1u_nr^2 + 5k_2u_nr^4+4p_1u_nv_n + 2p_2(r^2+2u^2_n)\right) \\
    f_v\left(v_n+3k_1v_nr^2 + 5k_2v_nr^4+ 2p_1(r^2+2u^2_n)+4p_2u_nv_n \right)
    \end{bmatrix}
    }
    \notag
\end{align}
At the same time, with one-axis rotation assumption, we have: 
\begin{align}
    \mathbf{M}_f {}^G\mathbf{k} & =
    \mathbf{H}_{\mathbf{p}_f}{}^{C}_I\hat{\mathbf{R}}{}^{I_{k}}_G\hat{\mathbf{R}} \cdot {}^G_{I_1}\hat{\mathbf{R}} {}^I_C\hat{\mathbf{R}} {}^C\mathbf{k} {}^C{z}_f
    \notag
    \\
    & = \mathbf{H}_{\mathbf{p}_f}\mathbf{e}_3 {}^Cz_f
    \notag
    \\
    & = 
    \frac{\partial \tilde{\mathbf{z}}_C}{\partial \tilde{\mathbf{z}}_n} 
    \begin{bmatrix}
    1 & 0 & -u_n \\
    0 & 1 & -v_n
    \end{bmatrix}
    \mathbf{e}_3
    \notag
    \\
    & =
    -\frac{\partial \tilde{\mathbf{z}}_C}{\partial \tilde{\mathbf{z}}_n}
    \begin{bmatrix}
    u_n \\
    v_n
    \end{bmatrix}
    \notag
\end{align}
Note that ${}^{I_{k}}_G\hat{\mathbf{R}}={}^{I_{k}}_{I_1}\hat{\mathbf{R}}{}^{I_{1}}_G\hat{\mathbf{R}}$ and ${}^{I_{k}}_{I_1}\hat{\mathbf{R}}{}^I\mathbf{k}={}^I\mathbf{k}$ (due to one-axis rotation). 
We can easily verify that: 
\begin{align}
       & \frac{\partial \tilde{\mathbf{z}}_C}{\partial \tilde{\mathbf{z}}_n}
    \begin{bmatrix}
    u_n \\
    v_n
    \end{bmatrix} =
    \notag
    \\
    & 
    \scalemath{0.85}{
    \begin{bmatrix}
    f_u\left(u_n+3k_1u_nr^2 + 5k_2u_nr^4+4p_1u_nv_n + 2p_2(r^2+2u^2_n)\right) \\
    f_v\left(v_n+3k_1v_nr^2 + 5k_2v_nr^4+ 2p_1(r^2+2u^2_n)+4p_2u_nv_n \right)
    \end{bmatrix}
    }
    \notag
\end{align}
Therefore, the verification of $\mathbf{N}_{Cin}$ can be written as:
\begin{align}
    \Leftrightarrow & 
    \mathcal{O}_k \mathbf{N}_{Cin} = \mathbf{0} 
    \notag
    \\
    \Leftrightarrow & 
    \mathbf{M}_{Cin} \times [f_u~f_v~0~0~2k_1~4k_2~p_1~p_2]^{\top} 
     + \mathbf{M}_{f}{}^G\mathbf{k} = \mathbf{0}
     \notag
\end{align}

\section{Appendix F: Interpolation Jacobians}
\label{apd:inter jacob}

We perturb ${}^G_{I_{ci-1}}\mathbf{R}$ and ${}^G_{I_{ci}}\mathbf{R}$ as:
\begin{align}
    {}^G_{I_{ci-1}}\mathbf{R} &= {}^G_{I_{ci-1}}\hat{\mathbf{R}} 
    \exp 
    \left(
    \delta \boldsymbol{\theta}_{I_{ci-1}}    
    \right)
    \\
    {}^G_{I_{ci}}\mathbf{R} &= {}^G_{I_{ci}}\hat{\mathbf{R}} 
    \exp 
    \left(
    \delta \boldsymbol{\theta}_{I_{ci}}    
    \right)
\end{align}
By representing $\boldsymbol{\theta}_{i-1,i} = \log 
\left(
{}^G_{I_{ci-1}}\mathbf{R}^{\top}{}^G_{I_{ci}}\mathbf{R}
\right)$, we have the linearization for the interpolation as:
\begin{align}
    \delta \boldsymbol{\theta}_{I(t)} & \simeq
    \scalemath{0.95}{
    \left(
    \exp (-\hat{\lambda}\hat{\boldsymbol{\theta}}_{i-1,i}) - 
    \mathbf{J}_r(\hat{\lambda}\hat{\boldsymbol{\theta}}_{i-1,i})
    \mathbf{J}^{-1}_l(\hat{\boldsymbol{\theta}}_{i-1,i})
    \right)\delta \boldsymbol{\theta}_{I_{ci-1}}  
    }
    \notag
    \\
    &
    + \mathbf{J}_r(\hat{\lambda}\hat{\boldsymbol{\theta}}_{i-1,i})
    \mathbf{J}^{-1}_r(\hat{\boldsymbol{\theta}}_{i-1,i})
    \delta \boldsymbol{\theta}_{I_{ci}}  
    +
    \frac{m}{M}
    \boldsymbol{\omega}_{i-1,i} \tilde{t}_r 
    \notag
    \\
    {}^G\tilde{\mathbf{p}}_{I(t)} & \simeq 
    (1-\hat{\lambda}){}^G\tilde{\mathbf{p}}_{I_{ci-1}} + 
    \hat{\lambda}{}^G\tilde{\mathbf{p}}_{I_{ci}}
    + 
    \scalemath{0.95}{
    \frac{m}{M}
    \mathbf{v}_{i-1,i}
    \tilde{t}_r 
    }
    \notag
\end{align}
where: 
\begin{align}
    \hat{\boldsymbol{\omega}}_{i-1,i} & = 
    \frac{\hat{\boldsymbol{\theta}}_{i-1,i}}{t_{ci}-t_{ci-1}}
    \\
    \hat{\mathbf{v}}_{i-1,i} & =
    \frac{{}^G\hat{\mathbf{p}}_{I_{ci}} - {}^G\hat{\mathbf{p}}_{I_{ci-1}}}{t_{ci}-t_{ci-1}}
\end{align}

\section{Appendix G: Degenerate Motion Simulation Results}
\label{sec:more_sim_results}

The complete calibration plots for IMU/camera intrinsic and IMU-camera spatial calibration under one-axis rotation, constant $a_x$ acceleration and planar motion are shown in Figure \ref{fig:sim_1axis_apd}, \ref{fig:sim_ax_apd} and \ref{fig:sim_planar_apd}. 
All the temporal calibration results for IMU-camera time offset and rolling shutter readout time on the four simulated trajectories are shown in Figure \ref{fig:sim_time}. 

\begin{figure*}
\centering
\begin{subfigure}{.245\textwidth}
\includegraphics[trim=0 9mm 0mm 0,clip,width=\linewidth]{figures/sim_new/dw_1_3_1axis}
\end{subfigure}
\begin{subfigure}{.245\textwidth}
\includegraphics[trim=0 9mm 0mm 0,clip,width=\linewidth]{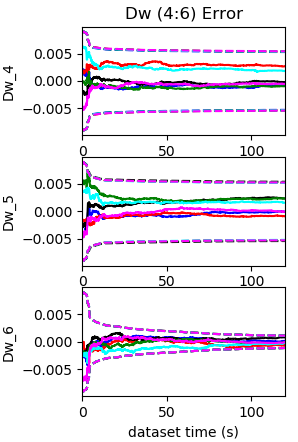}
\end{subfigure}
\begin{subfigure}{.245\textwidth}
\includegraphics[trim=0 9mm 0mm 0,clip,width=\linewidth]{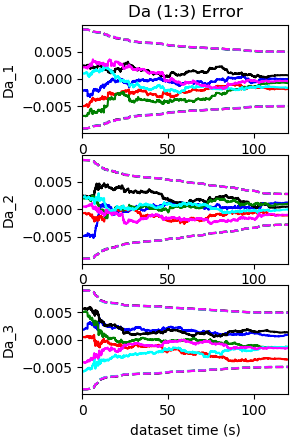}
\end{subfigure}
\begin{subfigure}{.245\textwidth}
\includegraphics[trim=0 9mm 0mm 0,clip,width=\linewidth]{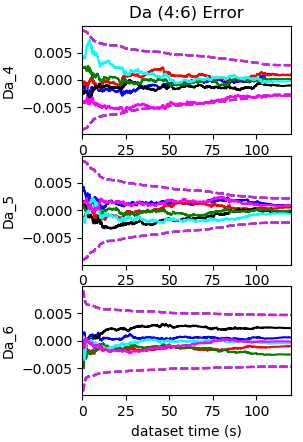}
\end{subfigure}
\begin{subfigure}{.245\textwidth}
\includegraphics[trim=0 9mm 0mm 0,clip,width=\linewidth]{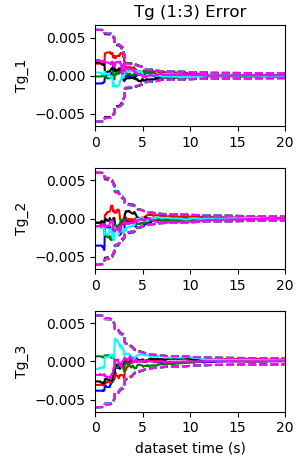}
\end{subfigure}
\begin{subfigure}{.245\textwidth}
\includegraphics[trim=0 9mm 0mm 0,clip,width=\linewidth]{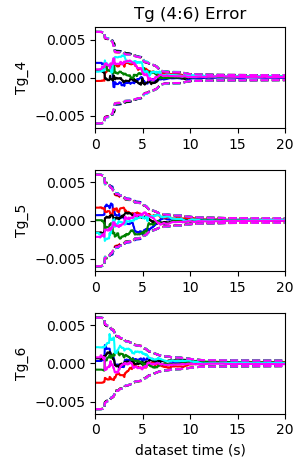}
\end{subfigure}
\begin{subfigure}{.245\textwidth}
\includegraphics[trim=0 9mm 0mm 0,clip,width=\linewidth]{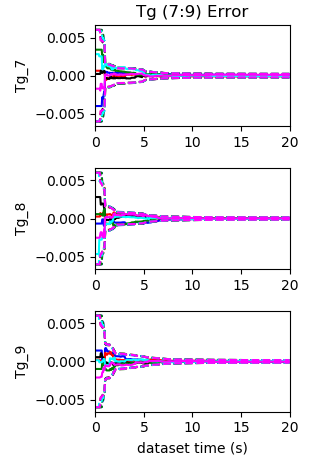}
\end{subfigure}
\begin{subfigure}{.245\textwidth}
\includegraphics[trim=0 9mm 0mm 0,clip,width=\linewidth]{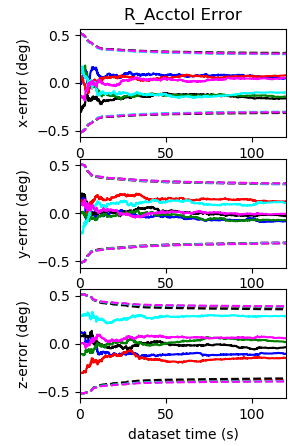}
\end{subfigure}
\begin{subfigure}{.245\textwidth}
\includegraphics[trim=0 0 0mm 0,clip,width=\linewidth]{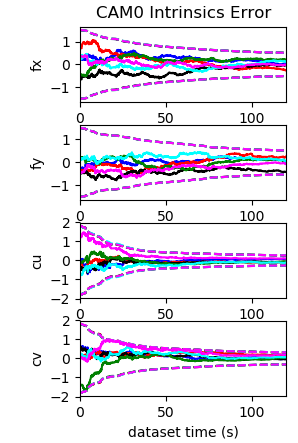}
\end{subfigure}
\begin{subfigure}{.245\textwidth}
\includegraphics[trim=0 0 0mm 0,clip,width=\linewidth]{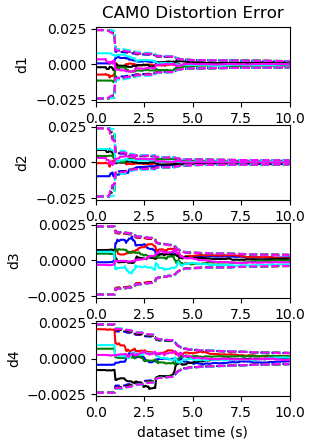}
\end{subfigure}
\begin{subfigure}{.245\textwidth}
\includegraphics[trim=0 0 0mm 0,clip,width=\linewidth]{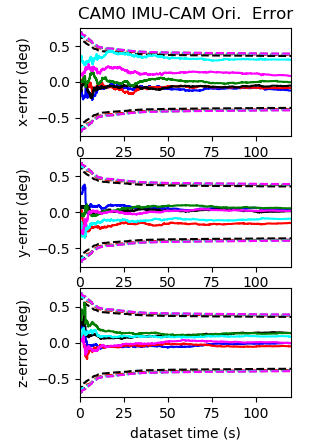}
\end{subfigure}
\begin{subfigure}{.245\textwidth}
\includegraphics[trim=0 0 0mm 0,clip,width=\linewidth]{figures/sim_new/pos_I_C_1axis}
\end{subfigure}
\caption{
Calibration results for the proposed system evaluated on \textit{tum\_room} with one-axis rotation using \textit{imu22} and \textit{radtan}. 
3 sigma bounds (dotted lines) and estimation errors (solid lines) for six different runs (different colors) with different realization of the measurement noise and initial perturbations.
Note that the estimation errors and 3 $\sigma$ bounds for $d_{w1}$, $d_{w2}$, $d_{w3}$ and the IMU-CAM position calibration along the rotation axis can not converge. 
}
\label{fig:sim_1axis_apd}
\vspace*{-6pt}
\end{figure*}

\begin{figure*}
\centering
\begin{subfigure}{.245\textwidth}
\includegraphics[trim=0 8mm 0mm 0,clip,width=\linewidth]{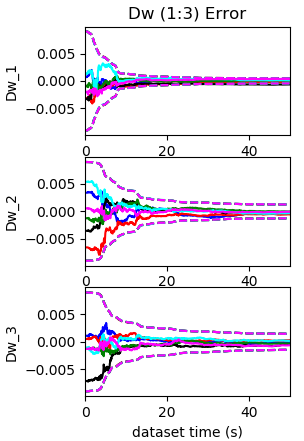}
\end{subfigure}
\begin{subfigure}{.245\textwidth}
\includegraphics[trim=0 8mm 0mm 0,clip,width=\linewidth]{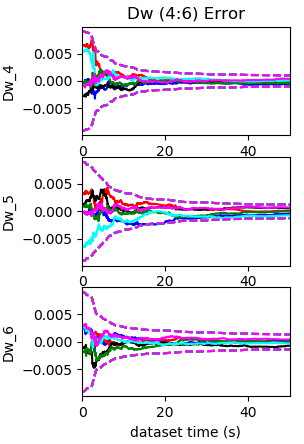}
\end{subfigure}
\begin{subfigure}{.245\textwidth}
\includegraphics[trim=0 8mm 0mm 0,clip,width=\linewidth]{figures/sim_new/da_1_3_ax}
\end{subfigure}
\begin{subfigure}{.245\textwidth}
\includegraphics[trim=0 8mm 0mm 0,clip,width=\linewidth]{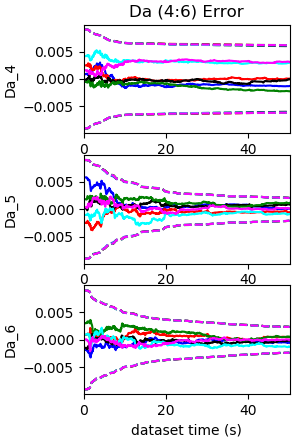}
\end{subfigure}
\begin{subfigure}{.245\textwidth}
\includegraphics[trim=0 8mm 0mm 0,clip,width=\linewidth]{figures/sim_new/tg_1_3_ax}
\end{subfigure}
\begin{subfigure}{.245\textwidth}
\includegraphics[trim=0 8mm 0mm 0,clip,width=\linewidth]{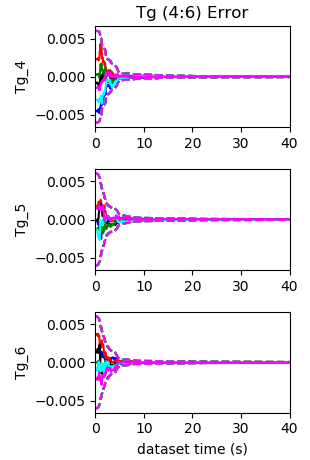}
\end{subfigure}
\begin{subfigure}{.245\textwidth}
\includegraphics[trim=0 8mm 0mm 0,clip,width=\linewidth]{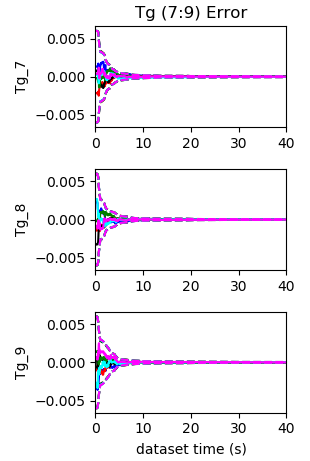}
\end{subfigure}
\begin{subfigure}{.245\textwidth}
\includegraphics[trim=0 8mm 0mm 0,clip,width=\linewidth]{figures/sim_new/rot_atoI_ax}
\end{subfigure}
\begin{subfigure}{.245\textwidth}
\includegraphics[trim=0 0 0mm 0,clip,width=\linewidth]{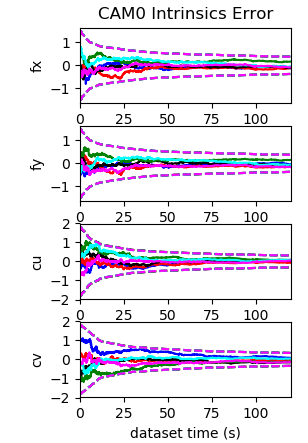}
\end{subfigure}
\begin{subfigure}{.245\textwidth}
\includegraphics[trim=0 0 0mm 0,clip,width=\linewidth]{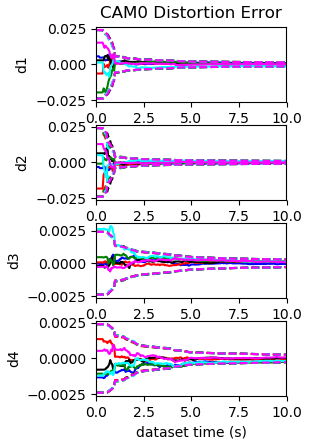}
\end{subfigure}
\begin{subfigure}{.245\textwidth}
\includegraphics[trim=0 0 0mm 0,clip,width=\linewidth]{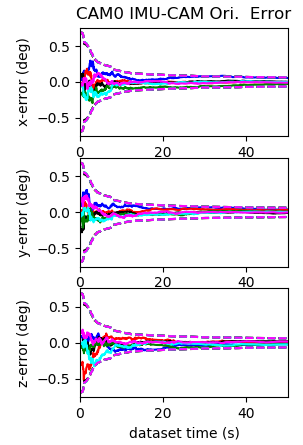}
\end{subfigure}
\begin{subfigure}{.245\textwidth}
\includegraphics[trim=0 0 0mm 0,clip,width=\linewidth]{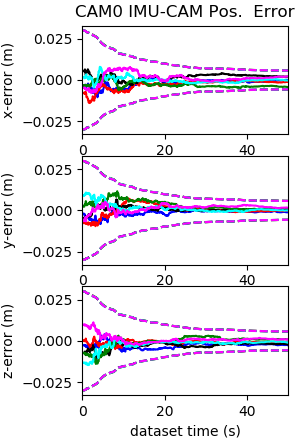}
\end{subfigure}
\caption{
Calibration results for the proposed system evaluated the \textit{sine\_3d} with constant acceleration along x direction using \textit{imu22} and \textit{radtan}. 
3 sigma bounds (dotted lines) and estimation errors (solid lines) for six different runs (different colors) with different realization of the measurement noise and initial perturbations. 
The estimation errors and 3$\sigma$ bounds for $d_{a1}$, pitch and yaw of ${}^I_a\mathbf{R}$ cannot converge. 
Note that $t_{g1}$, $t_{g2}$ and $t_{g3}$ are also unobservable. 
}
\label{fig:sim_ax_apd}
\vspace*{-6pt}
\end{figure*}

\begin{figure*}
\centering
\begin{subfigure}{.245\textwidth}
\includegraphics[trim=0 9mm 0mm 0,clip,width=\linewidth]{figures/sim_new/dw_1_3_planar}
\end{subfigure}
\begin{subfigure}{.245\textwidth}
\includegraphics[trim=0 9mm 0mm 0,clip,width=\linewidth]{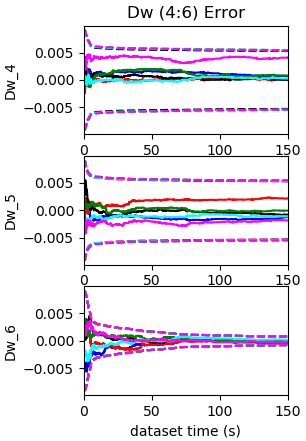}
\end{subfigure}
\begin{subfigure}{.245\textwidth}
\includegraphics[trim=0 9mm 0mm 0,clip,width=\linewidth]{figures/sim_new/da_1_3_planar}
\end{subfigure}
\begin{subfigure}{.245\textwidth}
\includegraphics[trim=0 9mm 0mm 0,clip,width=\linewidth]{figures/sim_new/da_4_6_planar}
\end{subfigure}
\begin{subfigure}{.245\textwidth}
\includegraphics[trim=0 9mm 0mm 0,clip,width=\linewidth]{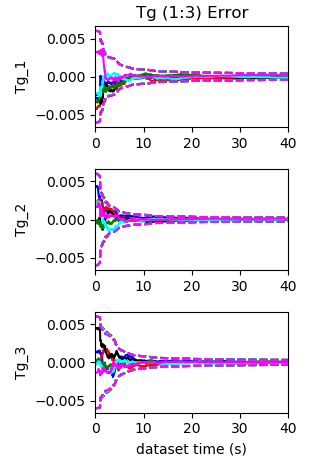}
\end{subfigure}
\begin{subfigure}{.245\textwidth}
\includegraphics[trim=0 9mm 0mm 0,clip,width=\linewidth]{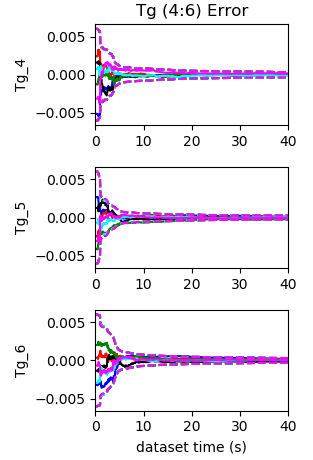}
\end{subfigure}
\begin{subfigure}{.245\textwidth}
\includegraphics[trim=0 9mm 0mm 0,clip,width=\linewidth]{figures/sim_new/tg_7_9_planar}
\end{subfigure}
\begin{subfigure}{.245\textwidth}
\includegraphics[trim=0 9mm 0mm 0,clip,width=\linewidth]{figures/sim_new/rot_atoI_planar}
\end{subfigure}
\begin{subfigure}{.245\textwidth}
\includegraphics[trim=0 0 0mm 0,clip,width=\linewidth]{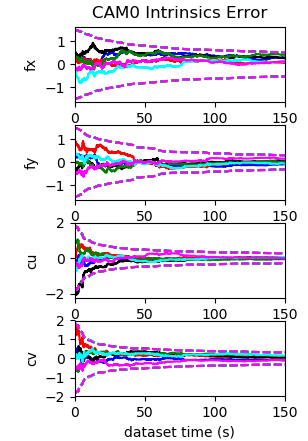}
\end{subfigure}
\begin{subfigure}{.245\textwidth}
\includegraphics[trim=0 0 0mm 0,clip,width=\linewidth]{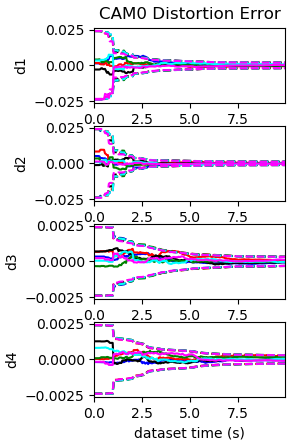}
\end{subfigure}
\begin{subfigure}{.245\textwidth}
\includegraphics[trim=0 0 0mm 0,clip,width=\linewidth]{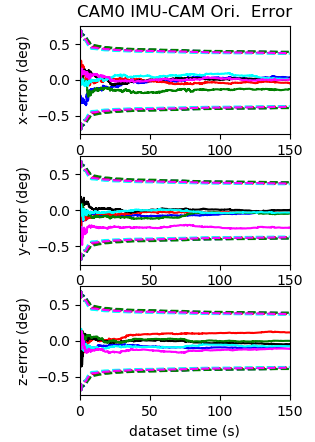}
\end{subfigure}
\begin{subfigure}{.245\textwidth}
\includegraphics[trim=0 0 0mm 0,clip,width=\linewidth]{figures/sim_new/pos_I_C_planar}
\end{subfigure}
\caption{
Calibration results of the proposed system evaluated on \textit{udel\_gore} with planar motion using \textit{imu22} and \textit{radtan}. 
3 sigma bounds (dotted lines) and estimation errors (solid lines) for six different runs (different colors) with different realization of the measurement noise and initial perturbations. 
With planar motion, the estimation errors and 3 $\sigma$ bounds of $d_{w1}$, $d_{w2}$, $d_{w3}$, $t_{g7}$, $t_{g8}$, $t_{g9}$ and the IMU-CAM position cannot converge. 
Due to lack of motion excitation, the parameters of $\mathbf{D}_a$ and ${}^I_a\mathbf{R}$ converge much slower than the other motion cases. 
}
\label{fig:sim_planar_apd}
\vspace*{-6pt}
\end{figure*}

\begin{acks}
This work was partially supported by the University of Delaware (UD) College of Engineering, 
the NSF (IIS-1924897), 
and Google ARCore. 
Geneva was partially supported by the University Doctoral Fellowship.
Thanks Dr. Kun Fu for 3D printing of the VI-Rig. 
Thanks Dr. Kevin Eckenhoff for the rolling shutter implementation from MIMC-VINS project \citep{Eckenhoff2021TRO}.  
Thanks David Schubert and Nikolaus Demmel from TUM for the helpful discussion on rolling-shutter visual-inertial datasets \citep{Schubert2019IROS}. 
\end{acks}

\end{document}